\documentclass[10pt]{article}

\usepackage{graphicx}
\usepackage[top=1.25in,bottom=1.25in,left=1in,right=1in]{geometry}
\usepackage{amsmath,amsfonts,amscd,amssymb,bm,epsfig,epsf,bbm,amssymb, prettyref} 
\usepackage{amsthm} 
\usepackage{algorithm, algorithmic}
\usepackage[usenames]{color}  
\usepackage{mathtools}     

\usepackage{tikz}     
\usepackage{pgfplots}  
\usepackage{dsfont}    
\usepackage[toc,page]{appendix}  
\usepackage{makecell}
\usepackage{caption}      
\usepackage{subcaption}  

\usepackage{natbib} 

\DeclareFontEncoding{FMS}{}{}
\DeclareFontSubstitution{FMS}{futm}{m}{n}
\DeclareFontEncoding{FMX}{}{}
\DeclareFontSubstitution{FMX}{futm}{m}{n}
\DeclareSymbolFont{fouriersymbols}{FMS}{futm}{m}{n} 
\DeclareSymbolFont{fourierlargesymbols}{FMX}{futm}{m}{n}
\DeclareMathDelimiter{\VERT}{\mathord}{fouriersymbols}{152}{fourierlargesymbols}{147}

\definecolor{darkred}{RGB}{100,0,0}
\definecolor{darkgreen}{RGB}{0,100,0}
\definecolor{darkblue}{RGB}{0,0,150}
\definecolor{red}{RGB}{255,0,0}

\usepackage{hyperref}
\hypersetup{colorlinks=true, linkcolor=darkred, citecolor=darkgreen, urlcolor=darkblue}
\usepackage{url}

\newtheorem{theorem}{Theorem}[section]
\newtheorem{lemma}[theorem]{Lemma}

\newtheorem{claim}[theorem]{Claim}

\newcommand{\sgn}{\textrm{sgn}}

\newcommand{\vct}[1]{\bm{#1}}
\newcommand{\mtx}[1]{\bm{#1}}

\newcommand{\colspan}{\operatorname{colspan}}

\newcommand{\rank}{\operatorname{rank}}

\numberwithin{equation}{section}

\numberwithin{equation}{section}
\DeclareMathOperator*{\argmin}{\arg\!\min}

\newcommand*\circled[1]{\tikz[baseline=(char.base)]{\node[shape=circle,draw,inner sep=2pt] (char) {#1};}}

\begin{document}
\title{Nonconvex Rectangular Matrix Completion via Gradient Descent without $\ell_{2,\infty}$ Regularization}
\author{
Ji Chen\thanks{Department of Mathematics, UC Davis, Davis, CA, 95616, USA; Email: ljichen@ucdavis.edu} \and Dekai Liu\thanks{Department of Mathematics, Zhejiang University, Hangzhou, Zhejiang, 310027, China; Email: dekailiu@zju.edu.cn}  \and Xiaodong Li\thanks{Department of Statistics, UC Davis, Davis, CA, 95616, USA; Email: xdgli@ucdavis.edu} 
}
\date{}     
\maketitle  

\begin{abstract}
	The analysis of nonconvex matrix completion has recently attracted much attention in the community of machine learning thanks to its computational convenience. Existing analysis on this problem, however, usually relies on $\ell_{2,\infty}$ projection or regularization that involves unknown model parameters, although they are observed to be unnecessary in numerical simulations, see, e.g., \cite{zheng2016convergence}. In this paper, we extend the analysis of the vanilla gradient descent for positive semidefinite matrix completion proposed in \cite{ma2017implicit} to the rectangular case, and more significantly, improve the required sampling rate from $O(\operatorname{poly}(\kappa)\mu^3 r^3 \log^3 n/n )$ to $O(\mu^2 r^2 \kappa^{14} \log n/n )$. Our technical ideas and contributions are potentially useful in improving the leave-one-out analysis in other related problems.
\end{abstract}

\section{Introduction}
%
%
%
%
Matrix completion techniques have found applications in a variety of modern machine learning problems thanks to the common incompleteness in big datasets. Examples include collaborative filtering, which predicts unobserved user-item scores based on a highly incomplete matrix of user-item ratings, and pairwise ranking, in which a key step is to complete the matrix of item-item aggregated comparison scores \citep{Gleich2011rank}. Sometimes a high-dimensional matrix may be intentionally generated as a highly incomplete one due to memory and computational issues. Examples include fast kernel matrix approximation via matrix completion \citep{Graepel2002kernel, Paisley2010kernel} and memory-efficient kernel PCA only with partial entries \citep{chen2017memory}.

The problem can be simply put as follows: Given an $n_1 \times n_2$ data matrix $\mtx{M}$ that is known to be of low rank, suppose we only observe a small portion of its entries on the index set $\Omega \subset [n_1] \times [n_2]$, can we recover $\mtx{M}$ accurately or even exactly from the available entries $M_{i,j}$ for $(i, j) \in \Omega$? Which algorithms are able to achieve the accurate recovery? Under what conditions on the low-rank matrix $\mtx{M}$ and the sampling index set $\Omega$ is the exact recovery guaranteed in theory?

Theoretical analysis of convex optimization methods for matrix completion has been well-investigated. For example, it was shown in \cite{candes2009exact} that linearly constrained nuclear norm minimization is guaranteed to complete low-rank matrices exactly as long as the sample complexity is large enough in comparison with the rank, dimensions and incoherence parameter of $\mtx{M}$. Their result in the required sampling complexity was later improved in the literature, e.g. \cite{candes2010power, gross2011recovering, recht2011simpler}.

In spite of the theoretical advantages of convex optimization, nonconvex optimization methods \citep{rennie2005fast} based on low-rank factorization can reduce memory and computation costs and avoid iterative singular value decompositions, thereby much more scalable to large datasets than convex optimization. The successes of nonconvex optimization in matrix completion suggest that inconsistent local minima can be bypassed or even just do not exist, but it was unclear under what conditions on the sampling complexity and the low-rank matrix the global minimum is attainable by a vanilla gradient descent method with theoretical guarantees. In \cite{keshavan2010matrix, keshavan2010matrixnoise}, a nonconvex optimization has been proposed, in which the constraint is the Cartesian product of two Grassmann manifolds. Under certain requirements on the sampling complexity in comparison with the rank, incoherence and condition number of the matrix to complete, a method of alternating gradient descent with initialization is proven to converge to the global minimum and recover the low rank matrix accurately. Alternating minimization via the low-rank factorization $\mtx{M}\approx \mtx{X}\mtx{Y}^\top$ was analyzed in \cite{jain2013low} provided independent samples are used to update $\mtx{X}$ and $\mtx{Y}$ in each step of the iteration. Their theoretical results were later improved and extended in \cite{Hardt14, pmlr-v35-hardt14a, Zhao2015nonconvex}. 
 
Matrix completion algorithms with brand new samples in each iteration may be impractical given the observed entries are usually highly limited. Instead, gradient descent for an $\ell_{2,\infty}$-norm regularized nonconvex optimization was shown in \cite{sun2016guaranteed} to converge to the global minimum and thereby recover the low-rank matrix, provided there hold some assumptions on the sampling complexity and the low-rank matrix. The $\ell_{2,\infty}$-norm regularization or projection has become a standard assumption for nonconvex matrix completion ever since, given they can explicitly control the $\ell_{2, \infty}$ norms of $\mtx{X}$ and $\mtx{Y}$, which is crucial in the theoretical analysis. However it has also been observed that $\ell_{2,\infty}$-norm regularization is numerically inactive in general, and vanilla methods without such regularization has almost the same effects.

Let's consider \cite{zheng2016convergence} as an example. By assuming that $\rank(\mtx{M})=r$ is known and that $\Omega$ satisfies an i.i.d. Bernoulli model with parameter $p$ (i.e., all entries are independently sampled with probability $p$), the nonconvex optimization 
\begin{equation}
\label{eq:mc_noncvx}
\begin{split}
 \min_{\mtx{X} \in \mathbb{R}^{n_1 \times r}, \mtx{Y} \in \mathbb{R}^{n_2 \times r}}	f(\boldsymbol{X},\boldsymbol{Y})  \coloneqq & \frac{1}{2p} \left\| \mathcal{P}_{\Omega}\left( \boldsymbol{X}\boldsymbol{Y}^\top - \boldsymbol{M} \right) \right\|_F^2 + \frac{1}{8}\left\| \boldsymbol{X}^\top\boldsymbol{X}-\boldsymbol{Y}^\top\boldsymbol{Y} \right\|_F^2
\end{split}
\end{equation}
was proposed there to recover $\mtx{M}$ through $\widehat{\mtx{X}}\widehat{\mtx{Y}}^\top$. Here $\mathcal{P}_{\Omega}:\mathbb{R}^{n_1\times n_2}\rightarrow \mathbb{R}^{n_1\times n_2}$ is a projector such that 
\begin{equation}
\label{eq:POmega}
(\mathcal{P}_{\Omega}(\boldsymbol{M}))_{i,j} = 
\begin{cases} M_{i,j} \quad \text{~if~} (i,j) \in \Omega \\   0 \quad \quad \quad \text{~otherwise}.\end{cases}
\end{equation}
The sampling rate $p$ is usually unknown but is almost identical to its empirical version $|\Omega|/(n_1n_2)$. In order to show that \prettyref{eq:mc_noncvx} is able to recover $\mtx{M}$ exactly, a projected gradient descent algorithm was proposed in \cite{zheng2016convergence} where the projection depending on unknown parameters is intended to control the $\ell_{2, \infty}$ norms of the updates of $\mtx{X}$ and $\mtx{Y}$. It was shown that with spectral initialization, projected gradient decent is guaranteed to converge to the global minimum and recover $\mtx{M}$ exactly, provided the sampling rate satisfies $p \geqslant C_0 \mu r^2 \kappa^2 \max(\mu, \log (n_1 \vee n_2))/(n_1 \wedge n_2)$. Here $\mu$ is the incoherence parameter introduced in \cite{candes2009exact}, $\kappa$ is the condition number of the rank-$r$ matrix $\mtx{M}$, i.e., the ratio between the largest and smallest nonzero singular values of $\mtx{M}$, and $C_0$ is an absolute constant. On the other hand, it has also been pointed out in \cite{zheng2016convergence} that the vanilla gradient descent without $\ell_{2, \infty}$-norm projection is observed to recover $\mtx{M}$ exactly in simulations.

Similar $\ell_{2, \infty}$-norm regularizations have also been used in other related works, see, e.g., \cite{chen2015fast, yi2016fast, pmlr-v54-wang17b}, and a crucial question is how to control the $\ell_{2, \infty}$-norms of the updates of $\mtx{X}$ and $\mtx{Y}$ without explicit regularization that involves extra tuning parameters. This issue has been initiatively addressed in \cite{ma2017implicit}, in which the matrix to complete is assumed to be symmetric and positive semidefinite, and the nonconvex optimization \prettyref{eq:mc_noncvx} is thereby reduced to 
\begin{equation}
\label{eq:mc_noncvx_psd}
\min_{\mtx{X} \in \mathbb{R}^{n \times r}} \frac{1}{2p} \left\| \mathcal{P}_{\Omega}\left( \boldsymbol{X}\boldsymbol{X}^\top - \boldsymbol{M} \right) \right\|_F^2.
\end{equation}
The work is focused on analyzing the convergence of vanilla gradient descent for \eqref{eq:mc_noncvx_psd}. In particular, the leave-one-out technique well known in the regression analysis \citep{El_Karoui2013leaveoneout} is employed in order to control the $\ell_{2, \infty}$-norms of the updates of $\mtx{X}$ in each step of iteration without explicit regularization or projection. \cite{ma2017implicit} shows that vanilla gradient descent is guaranteed to recover $\mtx{M}$, provided the sampling rate satisfies $p\geqslant C\operatorname{poly}(\kappa) \mu^3 r^3\log^3 n/n$, which is somehow inferior to that in \cite{zheng2016convergence}. This naturally raises several questions: Can we improve the required sampling rate from $O(\operatorname{poly}(\mu,\kappa,\log n)r^3/n)$ to $O(\operatorname{poly}(\mu,\kappa,\log n)r^2/ n)$ for vanilla gradient descent without $\ell_{2, \infty}$-norm regularization? Or is explicit $\ell_{2, \infty}$-norm regularization/projection avoidable for achieving the $O(\operatorname{poly}(\mu,\kappa,\log n)r^2 /n)$ sampling rate? Also, can we extend the nonconvex analysis in \cite{ma2017implicit} to the rectangular case discussed in \cite{zheng2016convergence}? This work is intended to answer these questions.

\subsection{Our contributions}
\label{sec:contributions}
As aforementioned, this paper aims to establish the assumptions on the sampling complexity and the low-rank matrix $\mtx{M}$, under which $\mtx{M}$ can be recovered by the nonconvex optimization \prettyref{eq:mc_noncvx} via vanilla gradient descent. Roughly speaking, our main result says that as long as $ p \geqslant C_S {\mu^2 r^2 \kappa^{14}\log (n_1\vee n_2)}/({n_1\wedge n_2})$ with some absolute constant $C_S$, vanilla gradient descent for \prettyref{eq:mc_noncvx} with spectral initialization is guaranteed to recover $\mtx{M}$ accurately. Compared to \cite{ma2017implicit} we have made several technical contributions including the following:

\begin{itemize}
\item By assuming the incoherence parameter $\mu = O(1)$ and the condition number $\kappa = O(1)$, regardless of the logarithms, the sampling rate $\widetilde{O}( r^3/n)$ in \cite{ma2017implicit} is improved to $\widetilde{O}(r^2/(n_1\wedge n_2))$, which is consistent with the result in \cite{zheng2016convergence} where $\ell_{2,\infty}$-norm projected gradient descent is employed;
\item The leave-one-out analysis for positive semidefinite matrix completion in \cite{ma2017implicit} is extended to the rectangular case in our paper; 
\item In the case $\mu=O(1)$, $\kappa = O(1)$ and $r=O(1)$, the sampling rate $O(\log^3 n/n)$ in \cite{ma2017implicit} is improved to $O(\log (n_1 \vee n_2)/(n_1\wedge n_2))$ in our work, which is consistent with the result in \cite{zheng2016convergence} where $\ell_{2,\infty}$-norm projected gradient descent is used.
\end{itemize}

To achieve these theoretical improvements and extensions, we need to make a series of modifications for the proof framework in \cite{ma2017implicit}. The following technical novelties are worth highlighting, and the details are deferred to the remaining sections in this paper:

\begin{itemize}
\item In order to reduce the sampling rate $\widetilde{O}(r^3/n)$ in \cite{ma2017implicit} to $\widetilde{O}(r^2/n)$ (assuming $\mu=O(1)$, $\kappa = O(1)$), a series of technical novelties are required. First, in the analysis of the spectral initialization for the gradient descent sequences and those for the leave-one-out sequences, $\|\frac{1}{p}\mathcal{P}_{\Omega}(\boldsymbol{M}) - \boldsymbol{M}\|$ is bounded in \cite{ma2017implicit} basically based on Lemma 39 therein. Instead, we give tighter bounds by applying \citet[Lemma 2]{chen2015incoherence} (Lemma \ref{chen_lemma2} in this paper), and the difference is a factor of $\sqrt{r}$. 
Second, two pillar lemmas, Lemma 37 in \cite{ma2017implicit} (restated as Lemma \ref{ma_lemma37} in our paper) and a result in \cite{mathias1993perturbation} (restated as Lemma \ref{ma_lemma36} in our paper), are repeatedly used in the leave-one-out analysis of \cite{ma2017implicit}. We find that applying a concentration result introduced in \cite{bhojanapalli2014universal} and \cite{li2016recovery} (restated as Lemma \ref{lemma_spectral_gap} in our paper) to verify the conditions in these lemmas could lead to sharper error bounds for the leave-one-out sequences. 
Third, also in the leave-one-out analysis, we need to modify the application of matrix Bernstein inequality in \cite{ma2017implicit} in order to achieve sharper error bounds.

\item In order to improve the orders of logarithms, we must improve the Hessian analysis in \cite{ma2017implicit}, i.e., Lemma 7 therein, and it turns out that Lemma 4.4 from \cite{chen2017memory} (restated as Lemma \ref{chen_lemma4.4} in this paper) and Lemma 9 from \cite{zheng2016convergence} (Lemma \ref{zheng_lemma9} in this paper) are effective to achieve this goal. These two lemmas are also effective in simplifying the proof in the Hessian analysis.
\end{itemize}
 
\subsection{Other related work}
We have already introduced a series of related works in the previous sections, and this section is intended to introduce other related works on nonconvex matrix completion, particularly on the theoretical side.

Besides algorithmic analysis for nonconvex matrix completion, \cite{ge2016matrix} and following works \cite{ge2017no, chen2017memory} have been dedicated to the geometric analysis: deriving the sampling rate conditions under which certain regularized nonconvex objective functions have no spurious local minima. That is, any local minimum is the global minimum, and thereby recovers the underlying low-rank matrix. 

It is also noteworthy that besides matrix completion, algorithmic and geometric nonconvex analyses have also been conducted for other low-rank recovery problems, such as phase retrieval \citep{candes2015phase, sun2018geometric, cai2016, chen2018gradient},  matrix sensing \citep{NIPS2015_5830, Tu2016procrustes, li2018algorithmic}, blind deconvolution \citep{LI2018}, etc.

Leave-one-out analysis has been employed in \cite{El_Karoui2013leaveoneout} to establish the asymptotic sampling distribution for robust estimators in high/moderate dimensional regression. This technique has also been utilized in \cite{abbe2017entrywise} to control $\ell_\infty$ estimation errors for eigenvectors in stochastic spectral problems, with applications in exact spectral clustering in community detection without cleaning or regularization. 
As aforementioned, in \cite{ma2017implicit}, the authors have employed the leave-one-out technique to control $\ell_{2, \infty}$ estimation errors for the updates of low-rank factors in each step of gradient descent that solves \prettyref{eq:mc_noncvx_psd}. Besides matrix completion, they also show that similar techniques can be utilized to show the convergence of vanilla gradient descent in other low-rank recovery problems such as phase retrieval and blind deconvolution. Leave-one-out analysis has also been successfully employed in the study of Singular Value Projection (SVP) for matrix completion \citep{ding2018leave} and gradient descent with random initialization for phase retrieval \citep{chen2018gradient}.

Implicit regularization for gradient descent has also been studied in matrix sensing with over-parameterization. When the sampling matrices satisfying certain commutative assumptions, it has been shown in \cite{gunasekar2017implicit} that gradient descent algorithm with near-origin starting point is guaranteed to recover the underlying low-rank matrix even under over-parameterized factorization. The result was later extended to the case in which the sensing operators satisfy certain RIP properties \citep{li2018algorithmic}.

\subsection{Notations}

Throughout the paper, matrices and vectors are denoted as bold uppercase and lowercase letters,  and all the vectors without the symbol of transpose are column vectors. Fixed absolute constants are defined as $C_0,C_1,C_2,\cdots,C_I,\cdots$ (their values are fixed and thereby not allowed to be changed from line to line). For two real numbers $x$ and $y$, we denote $x\wedge y \coloneqq \min\{x,y\}$ and $x\vee y \coloneqq \max\{x,y\}$. We denote the matrix with all $1$'s as $\boldsymbol{J}$, whose dimensions depend on the context. Some other notations used throughout the paper are listed in Table \ref{tab:notations} with matrices $\boldsymbol{A}$ and $\boldsymbol{B}$.
\begin{table}[ht]
\caption{Notations Used Throughout the Paper}
	\center
	\begin{tabular}{ll} 
	$\sigma_{i}(\boldsymbol{A})$ &    the $i$-th largest singular value of $\boldsymbol{A}$ \\
	$\lambda_i(\boldsymbol{A})$ & the $i$-th largest eigenvalue of the symmetric matrix $\boldsymbol{A}$\\
	$A_{i,j}$ &  the $(i,j)$-th entry of $\boldsymbol{A}$  \\
	$\boldsymbol{A}_{i,\cdot}$ & the $i$-th row vector of $\boldsymbol{A}$, taken as a column vector   \\
	$\boldsymbol{A}_{\cdot,j}$ &  the $j$-th column vector of $\boldsymbol{A}$, taken as a column vector  \\
	$\|\boldsymbol{A}\|$ & the spectral norm of $\boldsymbol{A}$\\
	$\|\boldsymbol{A}\|_F$ & the Frobenius norm of $\boldsymbol{A}$\\
	$\|\boldsymbol{A}\|_{2,\infty}$ &  the $\ell_{2,\infty}$ norm of $\boldsymbol{A}$, i.e., $\|\boldsymbol{A}\|_{2,\infty}\coloneqq\max_{i}\|\boldsymbol{A}_{i,\cdot}\|_2$\\
	$\|\boldsymbol{A}\|_{\ell_{\infty}}$ & the largest absolute value of entries of $\boldsymbol{A}$, i.e., $\|\boldsymbol{A}\|_{\ell_{\infty}}\coloneqq \max_{i,j}|A_{i,j}|$\\
	$\langle \boldsymbol{A}, \boldsymbol{B} \rangle$ & the inner product of matrices $\boldsymbol{A}$ and $\boldsymbol{B}$ with the same dimensions,\\
	& i.e., $\langle \boldsymbol{A}, \boldsymbol{B} \rangle \coloneqq \sum_{i,j} A_{i,j}B_{i,j}$\\
	$\operatorname{sgn}(\boldsymbol{A})$ & the sign matrix of $\boldsymbol{A}$, i.e., if $\boldsymbol{A}$ has an singular value decomposition $\boldsymbol{U}\boldsymbol{\Lambda}\boldsymbol{V}^\top$,\\
	& then $\operatorname{sgn}(\boldsymbol{A})\coloneqq \boldsymbol{U}\boldsymbol{V}^\top$\\
	$\nabla f(\boldsymbol{A},\boldsymbol{B})$ & the gradient of $f(\boldsymbol{A},\boldsymbol{B})$\\
	$\nabla^2 f(\boldsymbol{A},\boldsymbol{B})$ & the Hessian of $f(\boldsymbol{A}, \mtx{B})$\\
	\end{tabular}
	\label{tab:notations}
\end{table}

\section{Algorithm and Main Results}

Recall that our setup for the nonconvex optimization \eqref{eq:mc_noncvx} is the same as that in \cite{zheng2016convergence}: the matrix $\mtx{M}$ is of rank-$r$; the sampling scheme $\Omega$ satisfies the i.i.d. Bernoulli model with parameter $p$, i.e., each entry is observed independently with probability $p$; the operator $\mathcal{P}_{\Omega}$ is defined as in \eqref{eq:POmega}. In Section \ref{sub_sec:vanilla_gd}, we give the formula of the gradient descent. And in Section \ref{sub_sec:main}, we present the main result.


\subsection{Gradient descent and spectral initialization} \label{sub_sec:vanilla_gd}
We consider the initialization through a simple singular value decomposition: Let 
\begin{equation}
\label{eq:M0}
\mtx{M}^0  := \frac{1}{p}\mathcal{P}_{\Omega}(\boldsymbol{M}) \approx  \widetilde{\boldsymbol{X}}^0 \boldsymbol{\Sigma}^0 (\widetilde{\boldsymbol{Y}}^0)^\top
\end{equation}
be the top-$r$ partial singular value decomposition of $\mtx{M}^0$. In other words, the columns of $\widetilde{\boldsymbol{X}}^0 \in \mathbb{R}^{n_1 \times r}$ consist of the leading $r$ left singular vectors of $\mtx{M}^0 $; the diagonal entries of the diagonal matrix $\boldsymbol{\Sigma}^0\in\mathbb{R}^{r\times r}$ consist of the corresponding leading $r$ singular values; and the columns of $\widetilde{\boldsymbol{Y}}^0 \in \mathbb{R}^{n_2 \times r}$ consist of the corresponding leading $r$ right singular vectors. Let
\begin{equation}
\label{eq:initialization}
\boldsymbol{X}^0 = \widetilde{\boldsymbol{X}}^0 (\boldsymbol{\Sigma}^0)^{1/2}, \quad \boldsymbol{Y}^0 = \widetilde{\boldsymbol{Y}}^0 (\boldsymbol{\Sigma}^0)^{1/2}.
\end{equation}
We choose $(\mtx{X}^0, \mtx{Y}^0)$ as the initialization for the gradient descent.

The nonconvex optimization \prettyref{eq:mc_noncvx} yields the following formula for gradients:
\begin{equation*}
\begin{split}
		 \nabla_{X}f(\mtx{X}, \mtx{Y}) =& \frac{1}{p}\mathcal{P}_{\Omega}\left( \boldsymbol{X} \boldsymbol{Y}^\top - \boldsymbol{M} \right)\boldsymbol{Y} + \frac{1}{2}\boldsymbol{X} \left(\boldsymbol{X}^\top\boldsymbol{X} - \boldsymbol{Y}^\top\boldsymbol{Y} \right),\\
		 ~\\
		\nabla_{Y}f(\mtx{X}, \mtx{Y})  =& \frac{1}{p}\left[\mathcal{P}_{\Omega}\left( \boldsymbol{X} \boldsymbol{Y}^\top - \boldsymbol{M} \right)\right]^\top \boldsymbol{X} + \frac{1}{2}\boldsymbol{Y} \left(\boldsymbol{Y}^\top\boldsymbol{Y} - \boldsymbol{X}^\top\boldsymbol{X} \right).
\end{split}
\end{equation*}
Then the gradient descent algorithm solving \prettyref{eq:mc_noncvx} with some fixed step size $\eta$ can be explicitly stated as follows:
\begin{equation}
\label{eq:grad_desc}
\begin{split}
		\boldsymbol{X}^{t+1} =& \boldsymbol{X}^t - \frac{\eta}{p}\mathcal{P}_{\Omega}\left( \boldsymbol{X}^t\left( \boldsymbol{Y}^t \right)^\top - \boldsymbol{M} \right)\boldsymbol{Y}^t  -\frac{\eta}{2}\boldsymbol{X}^t\left( \left( \boldsymbol{X}^t \right)^\top\boldsymbol{X}^t - \left( \boldsymbol{Y}^t \right)^\top\boldsymbol{Y}^t \right),\\
		~\\
		\boldsymbol{Y}^{t+1} =& \boldsymbol{Y}^t - \frac{\eta}{p}\left[\mathcal{P}_{\Omega}\left( \boldsymbol{X}^t \left( \boldsymbol{Y}^t \right)^\top - \boldsymbol{M} \right)\right]^\top \boldsymbol{X}^t - \frac{\eta}{2}\boldsymbol{Y}^t \left( \left( \boldsymbol{Y}^t \right)^\top\boldsymbol{Y}^t - \left(\boldsymbol{X}^t\right)^\top\boldsymbol{X}^t \right).
\end{split}
\end{equation}
For any $m$, we obtain an estimate of $\mtx{M}$ after $m$ iterations as $\widehat{\mtx{M}}^m = \mtx{X}^m (\mtx{Y}^m)^\top$. We aim to study how close the estimate $\widehat{\mtx{M}}^m$ is from the ground truth $\mtx{M}$ under certain assumptions of the sampling complexity.

\subsection{Main results}\label{sub_sec:main}
In this section, we specify the conditions for $\mtx{M}$ and $\Omega$ to guarantee the convergence of the vanilla gradient descent \eqref{eq:grad_desc} with the spectral initialization \eqref{eq:initialization}. To begin with, we list some necessary assumptions and notations as follows:
First, $\rank(\mtx{M})=r$ is assumed to be known and thereby used in the nonconvex optimization \prettyref{eq:mc_noncvx}. The singular value decomposition of $\mtx{M}$ is $\boldsymbol{M} =  \widetilde{\boldsymbol{U}}  \boldsymbol{\Sigma}  \widetilde{\boldsymbol{V}}^\top  = \boldsymbol{U}\boldsymbol{V}^\top$
where 
\[
\boldsymbol{U} = \widetilde{\boldsymbol{U}} (\boldsymbol{\Sigma})^{1/2} \in \mathbb{R}^{n_1 \times r}, \quad \text{and} \quad \boldsymbol{V} = \widetilde{\boldsymbol{V}} (\boldsymbol{\Sigma})^{1/2} \in \mathbb{R}^{n_2 \times r}.
\]
 Second, denote by $\mu$ the subspace incoherence parameter of the rank-$r$ matrix $\mtx{M}$ as in \cite{candes2009exact}, i.e., 
 \[
 \mu\coloneqq\max(\mu(\colspan(\mtx{U})), \mu(\colspan(\mtx{V}))). 
 \]
 Here for any $r$-dimensional subspace $\mathcal{U}$ of $\mathbb{R}^n$, its incoherence parameter is defined as $\mu(\mathcal{U}) \coloneqq \frac{n}{r}\max\limits_{1\leqslant i\leqslant n}\|\mathcal{P}_{\mathcal{U}}\boldsymbol{e}_i\|_2^2$ with $\vct{e}_1, \ldots, \vct{e}_n$ being the standard orthogonal basis. Third, denote the condition number of $\mtx{M}$ as $\kappa = {\sigma_1(\boldsymbol{M})}/{\sigma_r(\boldsymbol{M})}$, where $\sigma_1(\boldsymbol{M})$ and $\sigma_r(\boldsymbol{M})$ are the first and the $r$-th singular value of $\mtx{M}$. Finally, assume that there is some absolute constant $C_0>1$ such that $1/C_0 < n_1/n_2 <C_0$. With these assumptions and notations, our main result is stated as follows:

\begin{theorem}
\label{thm:main}
Let $\Omega$ be sampled according to the i.i.d. Bernoulli model with the parameter $p$. If $p \geqslant C_S \frac{\mu^2 r^2 \kappa^{14}\log (n_1\vee n_2)}{n_1\wedge n_2}$ for some absolute constant $C_S$, then, as long as the gradient descent step size $\eta$ in \eqref{eq:grad_desc} satisfies $\eta \leqslant \frac{\sigma_r(\boldsymbol{M})}{200\sigma_1^2(\boldsymbol{M})}$, in an event $E$ with probability $\mathbb{P}[E] \geqslant 1- (n_1+n_2)^{-3}$, the gradient descent iteration \eqref{eq:grad_desc} starting from the spectral initialization \eqref{eq:initialization} converges linearly for at least the first $(n_1+n_2)^3$ steps:
\[ 
	\min_{\boldsymbol{R}\in\mathsf{O}(r)}\left\|\left[ \begin{array}{c} \boldsymbol{X}^t\\\boldsymbol{Y}^t\end{array} \right] \boldsymbol{R} - \left[\begin{array}{c} \boldsymbol{U}\\\boldsymbol{V} \end{array}\right]\right\|_F \leqslant \rho^t \sqrt{\sigma_r(\boldsymbol{M})},  
\]
$0\leqslant t\leqslant (n_1+n_2)^3$. Here $\mathsf{O}(r)$ denotes the set of $r\times r$ orthogonal matrices, and $\rho \coloneqq 1-0.05\eta\sigma_r(\boldsymbol{M})$ satisfies $0<\rho<1$. If additionally assume $\eta\geqslant \frac{\sigma_r(\boldsymbol{M})}{1000\sigma_1^2(\boldsymbol{M})}$, the above inequality implies
\[
	\min_{\boldsymbol{R}\in\mathsf{O}(r)}\left\|\left[ \begin{array}{c} \boldsymbol{X}^T \\\boldsymbol{Y}^T \end{array} \right] \boldsymbol{R} - \left[\begin{array}{c} \boldsymbol{U}\\\boldsymbol{V} \end{array}\right]\right\|_F \leqslant e^{-(n_1+n_2)^3/C_R} \sqrt{\sigma_r(\boldsymbol{M})}
\]
for $T \coloneqq (n_1+n_2)^{3}$ and an absolute constant $C_R>0$. 
\end{theorem}
The comparison between our result and that in \cite{ma2017implicit} (Theorem 2 therein) has already been summarized in Section \ref{sec:contributions}, so we don't repeat the details here.

\section{The Leave-one-out Sequences and the Roadmap of Proof}
The proof framework of Theorem \ref{thm:main} relies crucially on extending the leave-one-out sequences in \cite{ma2017implicit} from positive definite matrix completion \prettyref{eq:mc_noncvx_psd} to rectangular matrix completion \prettyref{eq:mc_noncvx}. Roughly speaking, the proof consists of three major parts: some local properties for the Hessian of the nonconvex objective function $f(\mtx{X}, \mtx{Y})$ defined in \eqref{eq:mc_noncvx}, error bounds for the initialization $(\boldsymbol{X}^0,\boldsymbol{Y}^0)$ and those of the leave-one-out sequences $(\boldsymbol{X}^{0,(l)},\boldsymbol{Y}^{0,(l)})$, error bounds for the gradient sequence $(\mtx{X}^t, \mtx{Y}^t)$ and the leave-one-out sequences $(\mtx{X}^{t,(l)}, \mtx{Y}^{t,(l)})$. We first give the definition of the leave-one-out sequences rigorously. 

\subsection{Leave-one-out sequences}  
Let's start with the following notations:
\begin{itemize}
\item Denote by $\Omega_{-i,\cdot}\coloneqq \{(k, l) \in \Omega: k \neq i\}$ the subset of $\Omega$ where entries in the $i$-th row are removed; 

\item Denote by $\Omega_{\cdot,-j}\coloneqq\{(k, l) \in \Omega: l \neq j\}$ the subset of $\Omega$ where entries in the $j$-th column are removed; 

\item Denote by $\Omega_{i,\cdot}\coloneqq \{(i,k)\in \Omega\}$ the subset of $\Omega$ where only entries in the $i$-th row are kept;

\item Denote by $\Omega_{\cdot,j}\coloneqq \{(k,j)\in \Omega\}$ the subset of $\Omega$ where only entries in the $j$-th column are kept;

\item The definitions of the projectors $\mathcal{P}_{\Omega_{-i,\cdot}}$, $\mathcal{P}_{\Omega_{\cdot,-j}}$, $\mathcal{P}_{\Omega_{i,\cdot}}$ and $\mathcal{P}_{\Omega_{\cdot,j}}$ are similar to that of $\mathcal{P}_\Omega$ as in \prettyref{eq:POmega};

\item Denote by $\mathcal{P}_{i,\cdot}(\cdot)$/$\mathcal{P}_{\cdot,j}(\cdot): \mathbb{R}^{n_1 \times n_2} \rightarrow \mathbb{R}^{n_1 \times n_2}$ the orthogonal projector that transforms a matrix by keeping its $i$-th row/$j$-th column and setting all other entries into zeros: 
\begin{equation*} 
\begin{split}
&(\mathcal{P}_{i,\cdot}(\boldsymbol{M}))_{k,l} =  \begin{cases} M_{k, l} \quad \text{~if~} k=i \\   0 \quad \quad  \text{~otherwise},\end{cases} \\
&(\mathcal{P}_{\cdot,j}(\boldsymbol{M}))_{k,l} =  \begin{cases} M_{k,l} \quad \text{~if~} l=j \\   0 \quad \quad  \text{~otherwise}.\end{cases}
\end{split}
\end{equation*}
\end{itemize}

These notations facilitate the leave-one-out analysis in rectangular matrix completion, in which each row/column is associated with a separate ``leave-one-out" sequence. The initialization for the ``leave-one-out" sequences are defined similarly to the initialization $(\mtx{X}^0, \mtx{Y}^0)$ for the gradient descent flow. To be concrete, for the $i$-th row, define
\[
\boldsymbol{M}^{0,(i)}\coloneqq
\frac{1}{p}\mathcal{P}_{\Omega_{-i,\cdot}}(\boldsymbol{M})+\mathcal{P}_{i,\cdot}(\boldsymbol{M}),
\]
i.e., the $i$-th row of $\frac{1}{p}\mathcal{P}_{\Omega}(\boldsymbol{M})$ is replaced with the complete $i$-th row of $\mtx{M}$. Similarly, for the $j$-th column, define
\[
\boldsymbol{M}^{0,(n_1+j)}\coloneqq
\frac{1}{p}\mathcal{P}_{\Omega_{\cdot,-j}}(\boldsymbol{M})+\mathcal{P}_{\cdot, j}(\boldsymbol{M}),
\]
i.e., the $j$-th column of $\frac{1}{p}\mathcal{P}_{\Omega}(\boldsymbol{M})$ is replaced with the complete $j$-th column of $\mtx{M}$. In short, we write
\begin{equation}
\label{eq:M0l}
\begin{split}
	& \boldsymbol{M}^{0,(l)}\coloneqq \left\{\begin{array}{ll}
\left(\frac{1}{p}\mathcal{P}_{\Omega_{-l,\cdot}}+\mathcal{P}_{l,\cdot}\right)(\boldsymbol{M}) & 1\leqslant l\leqslant n_1
\\
 \left(\frac{1}{p}\mathcal{P}_{\Omega_{\cdot,-(l-n_1)}} +\mathcal{P}_{\cdot,l-n_1}\right)(\boldsymbol{M}) & n_1+1\leqslant l\leqslant n_1+n_2.
\end{array}\right.
\end{split}
\end{equation}
For $1 \leqslant l \leqslant n_1+n_2$, as with the spectral initialization for gradient descent, let $\widetilde{\boldsymbol{X}}^{0, (l)} \boldsymbol{\Sigma}^{0, (l)}\left(\widetilde{\boldsymbol{Y}}^{0, (l)}\right)^\top$ be top-$r$ partial singular value decomposition of $\boldsymbol{M}^{0,(l)}$. Further, as with the definition of $(\mtx{X}^0, \mtx{Y}^0)$ in \eqref{eq:initialization}, we define the initialization for the $l$-th leave-one-out sequence as
\begin{equation}
\label{eq:init_leave_one_out}
\begin{split}
\boldsymbol{X}^{0, (l)} = \widetilde{\boldsymbol{X}}^{0, (l)} \left(\boldsymbol{\Sigma}^{0, (l)}\right)^{1/2}, 
\\
\boldsymbol{Y}^{0, (l)} = \widetilde{\boldsymbol{Y}}^{0, (l)} \left(\boldsymbol{\Sigma}^{0, (l)}\right)^{1/2}.
\end{split}
\end{equation}
It is clear that if $1\leqslant l \leqslant n_1$, $(\boldsymbol{X}^{0, (l)}, \boldsymbol{Y}^{0, (l)})$ is the initialization for the leave-one-out sequence associated with the $l$-th row, while if $n_1 + 1 \leqslant l \leqslant n_1 + n_2$, $(\boldsymbol{X}^{0, (l)}, \boldsymbol{Y}^{0, (l)})$ is associated with the $(l-n_1)$-th column.\\

Starting with $(\boldsymbol{X}^{0,(l)}, \boldsymbol{Y}^{0,(l)})$, we define the $l$-th leave-one-out sequence by considering the corresponding modification of the nonconvex optimization \eqref{eq:mc_noncvx}. For $1\leqslant l\leqslant n_1$, the nonconvex optimization \eqref{eq:mc_noncvx} is modified as 
\begin{equation*}
\begin{split}
 \min_{\substack{\mtx{X} \in \mathbb{R}^{n_1 \times r} \\ \mtx{Y} \in \mathbb{R}^{n_2 \times r}}}	 f(\boldsymbol{X},\boldsymbol{Y})  \coloneqq & \frac{1}{2p} \left\| \left(\mathcal{P}_{\Omega_{-l,\cdot}}+ p\mathcal{P}_{l,\cdot}\right)\left( \boldsymbol{X}\boldsymbol{Y}^\top - \boldsymbol{M} \right) \right\|_F^2  + \frac{1}{8}\left\| \boldsymbol{X}^\top\boldsymbol{X}-\boldsymbol{Y}^\top\boldsymbol{Y} \right\|_F^2.
\end{split}
\end{equation*}
The leave-one-out sequence associated with the $l$-th row is defined as the corresponding gradient descent sequence with the same step size $\eta$:
\begin{equation}
\label{eq:leaveoneout1}
	\begin{split}
		 \boldsymbol{X}^{t+1,(l)}  =& \boldsymbol{X}^{t,(l)} - \frac{\eta}{p}\mathcal{P}_{\Omega_{-l,\cdot}}\left( \boldsymbol{X}^{t,(l)} (\boldsymbol{Y}^{t,(l)})^\top- \mtx{M} \right) \boldsymbol{Y}^{t,(l)} - \eta \mathcal{P}_{l,\cdot}\left( \boldsymbol{X}^{t,(l)} (\boldsymbol{Y}^{t,(l)})^\top- \mtx{M} \right)\boldsymbol{Y}^{t,(l)} \\
		&-\frac{\eta}{2}\boldsymbol{X}^{t,(l)}\left( (\boldsymbol{X}^{t,(l)})^\top\boldsymbol{X}^{t,(l)} - (\boldsymbol{Y}^{t,(l)})^\top\boldsymbol{Y}^{t,(l)} \right)
	\end{split}
\end{equation}
and
\begin{equation}
\label{eq:leaveoneout2}
\begin{split}
		 \boldsymbol{Y}^{t+1,(l)} =&  \boldsymbol{Y}^{t,(l)} -\frac{\eta}{p} \left[ \mathcal{P}_{\Omega_{-l,\cdot}}\left( \boldsymbol{X}^{t,(l)} (\boldsymbol{Y}^{t,(l)})^\top- \mtx{M} \right) \right]^\top \boldsymbol{X}^{t,(l)}  - \eta \left[ \mathcal{P}_{l,\cdot}\left( \boldsymbol{X}^{t,(l)} (\boldsymbol{Y}^{t,(l)})^\top- \mtx{M} \right) \right]^\top \boldsymbol{X}^{t,(l)}  
		\\
		& -\frac{\eta}{2}\boldsymbol{Y}^{t,(l)}\left( (\boldsymbol{Y}^{t,(l)})^\top\boldsymbol{Y}^{t,(l)} - (\boldsymbol{X}^{t,(l)})^\top\boldsymbol{X}^{t,(l)} \right)
\end{split}
\end{equation}
Similarly, for $n_1+1\leqslant l \leqslant n_1+n_2$, consider the nonconvex optimization
\begin{equation*}
\begin{split}
\min_{\substack{\mtx{X} \in \mathbb{R}^{n_1 \times r} \\ \mtx{Y} \in \mathbb{R}^{n_2 \times r}}}	 f(\boldsymbol{X},\boldsymbol{Y}) \coloneqq &\frac{1}{2p} \left\| \left(\mathcal{P}_{\Omega_{\cdot,-(l-n_1)}} + p \mathcal{P}_{\cdot,l-n_1}\right)\left( \boldsymbol{X}\boldsymbol{Y}^\top - \boldsymbol{M} \right) \right\|_F^2  + \frac{1}{8}\left\| \boldsymbol{X}^\top\boldsymbol{X}-\boldsymbol{Y}^\top\boldsymbol{Y} \right\|_F^2.
\end{split}
\end{equation*}
Subsequently, the leave-one-out sequence associated with the $(l-n_1)$-th column is defined as the sequence:
\begin{equation}
\label{eq:leaveoneout3}
\begin{split}
		 \boldsymbol{X}^{t+1,(l)} =& \boldsymbol{X}^{t,(l)} - \frac{\eta}{p}\mathcal{P}_{\Omega_{\cdot,-(l-n_1)}}\left( \boldsymbol{X}^{t,(l)} (\boldsymbol{Y}^{t,(l)})^\top- \mtx{M} \right) \boldsymbol{Y}^{t,(l)} - \eta \mathcal{P}_{\cdot,l-n_1}\left( \boldsymbol{X}^{t,(l)} (\boldsymbol{Y}^{t,(l)})^\top- \mtx{M} \right)\boldsymbol{Y}^{t,(l)} 
		\\
		& -\frac{\eta}{2}\boldsymbol{X}^{t,(l)}\left( (\boldsymbol{X}^{t,(l)})^\top\boldsymbol{X}^{t,(l)} - (\boldsymbol{Y}^{t,(l)})^\top\boldsymbol{Y}^{t,(l)} \right)
\end{split}
\end{equation}
and
\begin{equation}
\label{eq:leaveoneout4}
\begin{split}
		 \boldsymbol{Y}^{t+1,(l)}  =&  \boldsymbol{Y}^{t,(l)} -\frac{\eta}{p} \left[ \mathcal{P}_{\Omega_{\cdot,-(l-n_1)}}\left( \boldsymbol{X}^{t,(l)} (\boldsymbol{Y}^{t,(l)})^\top- \mtx{M} \right) \right]^\top \boldsymbol{X}^{t,(l)} \\
		&- \eta \left[ \mathcal{P}_{\cdot,l-n_1}\left( \boldsymbol{X}^{t,(l)} (\boldsymbol{Y}^{t,(l)})^\top- \mtx{M} \right) \right]^\top \boldsymbol{X}^{t,(l)} 
		\\
		& -\frac{\eta}{2}\boldsymbol{Y}^{t,(l)}\left( (\boldsymbol{Y}^{t,(l)})^\top\boldsymbol{Y}^{t,(l)} - (\boldsymbol{X}^{t,(l)})^\top\boldsymbol{X}^{t,(l)} \right).
\end{split}
\end{equation}
These $n_1 + n_2$ leave-one-out sequences will be employed to prove the convergence of vanilla gradient descent \eqref{eq:grad_desc} as with \cite{ma2017implicit} as will be detailed in next few sections.

\subsection{Local properties of the Hessian} 
As with \citet[Lemma 7]{ma2017implicit}, we characterize some local properties of the Hessian of the objective function $f(\mtx{X}, \mtx{Y})$:
\begin{lemma}
\label{lemma_hessian}
If the sampling rate satisfies 
\[
p \geqslant C_{S1} \frac{\mu r \kappa \log(n_1\vee n_2)}{n_1\wedge n_2} 
\]
for some absolute constant $C_{S1}$, then on an event $E_H$ with probability $\mathbb{P}[E_H]\geqslant 1-3(n_1+n_2)^{-11}$, we have
\begin{equation}\label{eq_hessian1}
\begin{split}
 \operatorname{vec}\left( \left[ \begin{array}{c}\boldsymbol{D}_{\boldsymbol{X}}\\ \boldsymbol{D}_{\boldsymbol{Y}} \end{array}\right] \right)^\top \nabla^2f(\boldsymbol{X},\boldsymbol{Y})\operatorname{vec}\left( \left[ \begin{array}{c}\boldsymbol{D}_{\boldsymbol{X}}\\ \boldsymbol{D}_{\boldsymbol{Y}} \end{array}\right] \right) \geqslant & \frac{1}{5}\sigma_r(\boldsymbol{M})\left\| \left[ \begin{array}{c}\boldsymbol{D}_{\boldsymbol{X}}\\ \boldsymbol{D}_{\boldsymbol{Y}} \end{array}\right]  \right\|_F^2
\end{split}
\end{equation}
and
\begin{equation}\label{eq_hessian2}
		\|\nabla^2f(\boldsymbol{X},\boldsymbol{Y})\|\leqslant 5\sigma_1(\boldsymbol{M}),
\end{equation}
uniformly for all $\boldsymbol{X} \in \mathbb{R}^{n_1\times r},\boldsymbol{Y} \in\mathbb{R}^{n_2\times r}$ satisfying
	\begin{equation}\label{lemma_hessian1}
		\left\| \left[ \begin{array}{c}
		\boldsymbol{X}-\boldsymbol{U}\\
		\boldsymbol{Y}-\boldsymbol{V}
		 \end{array}\right]\right\|_{2,\infty} \leqslant \frac{1}{500 \kappa\sqrt{n_1+n_2}}\sqrt{\sigma_1(\boldsymbol{M})}
	\end{equation}
	and all $\boldsymbol{D}_{\boldsymbol{X}} \in \mathbb{R}^{n_1\times r},~\boldsymbol{D}_{\boldsymbol{Y}} \in \mathbb{R}^{n_2\times r}$ such that $\left[ \begin{array}{c} \boldsymbol{D}_{\boldsymbol{X}} \\ \boldsymbol{D}_{\boldsymbol{Y}} \end{array} \right]$ is in the set 
\begin{equation}
\label{lemma_hessian2}
\begin{split}
& \left\{\left[ \begin{array}{c} \boldsymbol{X}_1 \\\boldsymbol{Y}_1 \end{array}\right]\widehat{\boldsymbol{R}} - \left[ \begin{array}{c} \boldsymbol{X}_2\\\boldsymbol{Y}_2 \end{array} \right]: \left\| \left[ \begin{array}{c} \boldsymbol{X}_2-\boldsymbol{U}\\ \boldsymbol{Y}_2-\boldsymbol{V} \end{array}\right] \right\| \leqslant  \frac{\sqrt{\sigma_1(\boldsymbol{M})}}{500 \kappa},  \widehat{\boldsymbol{R}}\coloneqq \argmin_{\boldsymbol{R}\in\mathsf{O}(r)} \left\| \left[ \begin{array}{c} \boldsymbol{X}_1 \\\boldsymbol{Y}_1 \end{array}\right]\boldsymbol{R} - \left[ \begin{array}{c} \boldsymbol{X}_2\\\boldsymbol{Y}_2 \end{array} \right] \right\|_F \right\}.
\end{split}
\end{equation} 
\end{lemma}
The proof is similar to \citet[Lemma 7]{ma2017implicit}, but as mentioned in Section \ref{sec:contributions}, we apply Lemma 4.4 from \cite{chen2017memory} and Lemma 9 from \cite{zheng2016convergence} to improve the order of logarithms. The details are relegated to \prettyref{sec:hess_proof} in the appendix.

\subsection{Analysis of the initializations for the Leave-one-out sequences}
As with \citet[Lemma 13]{ma2017implicit}, we now specify how close the spectral initialization $(\mtx{X}^0, \mtx{Y}^0)$ in \eqref{eq:initialization} and its leave-one-out counterparts $(\mtx{X}^{0, (l)}, \mtx{Y}^{0, (l)})$ in \eqref{eq:init_leave_one_out} are from the ground truth $(\mtx{U}, \mtx{V})$ (recall that $\mtx{M} = \mtx{U}\mtx{V}^\top$). To begin with, we list some convenient notations for several orthogonal matrices that relate $(\mtx{X}^0, \mtx{Y}^0)$, $(\mtx{X}^{0, (l)}, \mtx{Y}^{0, (l)})$ and $(\mtx{U}, \mtx{V})$:
\[
\boldsymbol{R}^0 \coloneqq \argmin_{\boldsymbol{R}\in\mathsf{O}(r)}\left\| \left[\begin{array}{c}\boldsymbol{X}^0\\ \boldsymbol{Y}^0\end{array}\right]\boldsymbol{R} - \left[\begin{array}{c}\boldsymbol{U}\\ \boldsymbol{V}\end{array}\right] \right\|_F,
\]

\[
 \boldsymbol{R}^{0,(l)} \coloneqq \argmin_{\boldsymbol{R}\in\mathsf{O}(r)}\left\|  \left[\begin{array}{c}\boldsymbol{X}^{0,(l)}\\ \boldsymbol{Y}^{0,(l)} \end{array}\right] \boldsymbol{R} - \left[\begin{array}{c}\boldsymbol{U}\\ \boldsymbol{V}\end{array}\right] \right\|_F
\]  
 and 
\begin{equation}\label{eq:T0l}
\boldsymbol{T}^{0,(l)} \coloneqq \argmin_{\boldsymbol{R}\in \mathsf{O}(r)} \left\| \left[\begin{array}{c}\boldsymbol{X}^0\\ \boldsymbol{Y}^0\end{array}\right] \boldsymbol{R}^0-\left[\begin{array}{c}\boldsymbol{X}^{0,(l)}\\ \boldsymbol{Y}^{0,(l)} \end{array}\right] \boldsymbol{R} \right\|_F.
\end{equation}

\begin{lemma}
\label{lemma_initialization}
If 
\[
p \geqslant C_{S2} \frac{\mu^2 r^2\kappa^6\log (n_1\vee n_2)}{n_1\wedge n_2}, 
\]
then on an event $E_{init}\subset E_H$ (defined in Lemma \ref{lemma_hessian}) with probability $\mathbb{P}[E_{init}]\geqslant 1-(n_1+n_2)^{-10}$, there hold the following inequalities
\begin{align}
	 \left\| \left[\begin{array}{c}\boldsymbol{X}^0\\ \boldsymbol{Y}^0\end{array}\right]\boldsymbol{R}^0 - \left[\begin{array}{c}\boldsymbol{U}\\ \boldsymbol{V}\end{array}\right] \right\| \leqslant& C_I\sqrt{\frac{\mu r \kappa^6\log(n_1\vee n_2)}{(n_1\wedge n_2)p}}\sqrt{\sigma_1(\boldsymbol{M})}, \label{eq:ini1}\\
		  \left\| \left( \left[\begin{array}{c}\boldsymbol{X}^{0,(l)}\\ \boldsymbol{Y}^{0,(l)} \end{array}\right] \boldsymbol{R}^{0,(l)} - \left[\begin{array}{c}\boldsymbol{U}\\ \boldsymbol{V}\end{array}\right]\right)_{l,\cdot} \right\|_2 \leqslant& 100 C_I \sqrt{ \frac{\mu^2 r^2 \kappa^7 \log (n_1\vee n_2)}{(n_1\wedge n_2)^2p}} \sqrt{\sigma_1(\boldsymbol{M})},\label{eq:ini2}\\
		 \left\| \left[ \begin{array}{c}\boldsymbol{X}^0\\\boldsymbol{Y}^0\end{array}\right]\boldsymbol{R}^0 -\left[ \begin{array}{c}\boldsymbol{X}^{0,(l)}\\\boldsymbol{Y}^{0,(l)}\end{array}\right]\boldsymbol{T}^{0,(l)}  \right\|_F  \leqslant&  C_I \sqrt{\frac{\mu^2 r^2 \kappa^{10} \log (n_1\vee n_2)}{(n_1\wedge n_2)^2p}}\sqrt{\sigma_1(\boldsymbol{M})}, \label{eq:ini3} 
	\end{align}
	for all $1\leqslant l\leqslant n_1+n_2$. Here $C_I$ and $C_{S2}$ are two fixed absolute constants.
\end{lemma}

The detailed proof of Lemma \ref{lemma_initialization} is deferred to Appendix \ref{sec:initialization_proof}, while we here highlight some key ideas in the proof. 
First, in order to transform the problem of rectangular matrix completion into symmetric matrix completion, the trick of ``symmetric dilation" introduced in \cite{paulsen2002completely,abbe2017entrywise} is employed.  Moreover, a major technical novelty in our proof is to replace \citet[Lemma 39]{ma2017implicit} with \citet[Lemma 2]{chen2015incoherence} to obtain sharper error bounds as shown in \eqref{eq:ini1}, \eqref{eq:ini2} and \eqref{eq:ini3}. We restate that lemma here:

\begin{lemma}[Modification of {\citet[Lemma 2]{chen2015incoherence}}]\label{chen_lemma2}
	Let $\boldsymbol{A}$ be any fixed $n_1\times n_2$ matrix, and let the index set $\Omega \in [n_1] \times [n_2]$ satisfy the i.i.d. Bernoulli model with parameter $p$. Denote
	\[
		\overline{\boldsymbol{A}} \coloneqq \left[ \begin{array}{cc}
			\boldsymbol{0} & \boldsymbol{A}\\
			\boldsymbol{A}^\top & \boldsymbol{0}
		\end{array}\right],
	\]
	\[
	\begin{split} 
		 \overline{\Omega} \coloneqq \{(i,j)\mid& 1\leqslant i,j\leqslant n_1+n_2, (i,j-n_1)\in \Omega\;\textrm{or}\; (j,i-n_1)\in\Omega\}.
	\end{split}	
	\]
	There is an absolute constant $C_4$ and an event $E_{Ch}$ with probability $\mathbb{P}[E_{Ch}]\geqslant 1-(n_1+n_2)^{-11}$, such that for all $1\leqslant l\leqslant n_1+n_2$, there holds
		\begin{equation}\label{eq_chen01}
		\begin{split}
			& \left\|\frac{1}{p}\mathcal{P}_{\overline{\Omega}_{-l}}(\overline{\boldsymbol{A}}) + \mathcal{P}_{l}(\overline{\boldsymbol{A}}) - \overline{\boldsymbol{A}}\right\| \\
			\leqslant & \left\|\frac{1}{p}\mathcal{P}_{\overline{\Omega}}(\overline{\boldsymbol{A}}) - \overline{\boldsymbol{A}}\right\|\\
			\leqslant& C_4\left( \frac{\log (n_1\vee n_2)}{p}\|\overline{\boldsymbol{A}}\|_{\ell_{\infty}} +\sqrt{\frac{\log (n_1\vee n_2)}{p}}\|\overline{\boldsymbol{A}}\|_{2,\infty} \right).
			\end{split}
		\end{equation}
Here
\[
			\mathcal{P}_{\overline{\Omega}_{-l}}(\overline{\boldsymbol{A}}) := \sum_{(i,j)\in\overline{\Omega},i\neq l,j\neq l}  \overline{A}_{i,j}\boldsymbol{e}_i\boldsymbol{e}_j^\top,\]
			\[\mathcal{P}_l(\overline{\boldsymbol{A}}) := \sum_{(i,j)\in[n_1+n_2]\times [n_1+n_2],i= l\;\textrm{or}\;j = l}  \overline{A}_{i,j}\boldsymbol{e}_i\boldsymbol{e}_j^\top,
\]
and $\vct{e}_1, \ldots \vct{e}_{n_1 + n_2}$ are the standard basis of $\mathbb{R}^{n_1 + n_2}$.
\end{lemma}

The second inequality in \eqref{eq_chen01} is directly implied by \cite{chen2015incoherence}. In fact, \cite{chen2015incoherence} yields the bound for $\left\|\frac{1}{p}\mathcal{P}_{\Omega}(\boldsymbol{A}) - \boldsymbol{A}\right\|$. On the other hand, the equalities
	\[
	\begin{split}
		& \left\|\frac{1}{p}\mathcal{P}_{\overline{\Omega}}(\overline{\boldsymbol{A}}) - \overline{\boldsymbol{A}}\right\| \\
		=& \left\|\left[\begin{array}{cc}
			\boldsymbol{0} & \frac{1}{p}\mathcal{P}_{\Omega}(\boldsymbol{A}) - \boldsymbol{A}\\
			\left(\frac{1}{p}\mathcal{P}_{\Omega}(\boldsymbol{A}) - \boldsymbol{A}\right)^\top & \boldsymbol{0}
		\end{array}\right]\right\| \\
		=& \left\| \frac{1}{p}\mathcal{P}_{\Omega}(\boldsymbol{A}) - \boldsymbol{A}\right\|
	\end{split}
	\]	
as well as	$\|\overline{\boldsymbol{A}}\|_{\ell_{\infty}} = \|\boldsymbol{A}\|_{\ell_{\infty}} $ and $\|\overline{\boldsymbol{A}}\|_{2,\infty} = \max\{ \|\boldsymbol{A}\|_{2,\infty},\|\boldsymbol{A}^\top\|_{2,\infty}\}$ translate the bound in \cite{chen2015incoherence} to our result. As to the first inequality in \eqref{eq_chen01}, it holds due simply to the fact that $\frac{1}{p}\mathcal{P}_{\overline{\Omega}_{-l}}(\overline{\boldsymbol{A}}) + \mathcal{P}_{l}(\overline{\boldsymbol{A}}) - \overline{\boldsymbol{A}}$ is essentially a submatrix of $\frac{1}{p}\mathcal{P}_{\overline{\Omega}}(\overline{\boldsymbol{A}}) - \overline{\boldsymbol{A}}$ (the $l$-th column and $l$-th row are changed to zeros.)

\subsection{Analysis for the leave-one-out sequences}
In this section we are about to introduce the lemma that guarantees the convergence of the gradient descent for the nonconvex optimization \eqref{eq:mc_noncvx} with the leave-one-out technique. To be concrete, we are going to control certain distances between the gradient descent sequence $(\mtx{X}^t, \mtx{Y}^t)$ in \eqref{eq:grad_desc}, the leave-one-out sequences $(\mtx{X}^{t, (l)}, \mtx{Y}^{t, (l)})$ in \eqref{eq:leaveoneout1}, \eqref{eq:leaveoneout2}, \eqref{eq:leaveoneout3} and \eqref{eq:leaveoneout4}, and the low-rank factors $(\mtx{U}, \mtx{V})$. Again, we denote some orthogonal matrices that relate $(\mtx{X}^t, \mtx{Y}^t)$, $(\mtx{X}^{t, (l)}, \mtx{Y}^{t, (l)})$ and $(\mtx{U}, \mtx{V})$ for $1\leqslant l\leqslant n_1+n_2$:
\begin{equation}\label{eq:Rt}
	\begin{split}
		\boldsymbol{R}^t \coloneqq & \argmin_{\boldsymbol{R}\in\mathsf{O}(r)} \left\|\left[ \begin{array}{c} \boldsymbol{X}^t \\ \boldsymbol{Y}^t \end{array} \right] \boldsymbol{R}-\left[ \begin{array}{c} \boldsymbol{U}  \\ \boldsymbol{V}  \end{array} \right]\right\|_F,\\
		\boldsymbol{R}^{t,(l)} \coloneqq & \argmin_{\boldsymbol{R}\in\mathsf{O}(r)} \left\|\left[ \begin{array}{c} \boldsymbol{X}^{t,(l)} \\ \boldsymbol{Y}^{t,(l)} \end{array} \right] \boldsymbol{R}-\left[ \begin{array}{c} \boldsymbol{U}  \\ \boldsymbol{V}  \end{array} \right]\right\|_F,\\
		\boldsymbol{T}^{t,(l)} \coloneqq & \argmin_{\boldsymbol{R}\in\mathsf{O}(r)} \left\|\left[ \begin{array}{c} \boldsymbol{X}^t \\ \boldsymbol{Y}^t \end{array} \right] \boldsymbol{R}^t- \left[ \begin{array}{c} \boldsymbol{X}^{t,(l)} \\ \boldsymbol{Y}^{t,(l)} \end{array} \right] \boldsymbol{R}  \right\|_F.
	\end{split}
\end{equation}

\begin{lemma}
\label{lemma_induction}
Suppose that the the step size satisfies 
\[
\eta \leqslant \frac{ \sigma_r(\boldsymbol{M})}{200\sigma_1^2(\boldsymbol{M})},
\]
and that the sampling rate satisfies 
\[
p \geqslant C_{S3}\frac{\mu^2 r^2 \kappa^{14}\log (n_1\vee n_2)}{n_1\wedge n_2} 
\]
for some absolute constant $C_{S3}$.

For any fixed $t\geqslant 0$, if on an event $E_{gd}^t\subset E_H$ (defined in Lemma \ref{lemma_hessian}) there hold
\begin{align}
		  \left\| \left[ \begin{array}{c} \boldsymbol{X}^t \\\boldsymbol{Y}^t \end{array}\right]\boldsymbol{R}^t - \left[ \begin{array}{c} \boldsymbol{U}\\ \boldsymbol{V}\end{array}\right] \right\|  \leqslant & C_I \rho^t \sqrt{\frac{\mu r \kappa^6 \log (n_1\vee n_2)}{(n_1\wedge n_2)p}}\sqrt{\sigma_1(\boldsymbol{M})},\label{eq_ind2}  \\  
		 \left\| \left( \left[ \begin{array}{c} \boldsymbol{X}^{t,(l)}\\\boldsymbol{Y}^{t,(l)} \end{array}\right]\boldsymbol{R}^{t,(l)} - \left[ \begin{array}{c} \boldsymbol{U}\\\boldsymbol{V} \end{array} \right] \right)_{l,\cdot} \right\|_2  \leqslant & 100C_I \rho^t \sqrt{\frac{\mu^2 r^2 \kappa^{10} \log (n_1\vee n_2)}{(n_1\wedge n_2)^2 p}}\sqrt{\sigma_1(\boldsymbol{M})},\label{eq_ind4}\\ 
		 \left\| \left[ \begin{array}{c} \boldsymbol{X}^t \\\boldsymbol{Y}^t \end{array}\right]\boldsymbol{R}^t - \left[\begin{array}{c} \boldsymbol{X}^{t,(l)} \\ \boldsymbol{Y}^{t,(l)} \end{array}\right] \boldsymbol{T}^{t,(l)} \right\|_F \leqslant & C_I \rho^t \sqrt{\frac{\mu^2 r^2 \kappa^{10}\log (n_1\vee n_2)}{(n_1\wedge n_2)^2 p}}\sqrt{\sigma_1(\boldsymbol{M})}, \label{eq_ind3}\\
		 \left\| \left[ \begin{array}{c} \boldsymbol{X}^t \\\boldsymbol{Y}^t \end{array}\right]\boldsymbol{R}^t - \left[ \begin{array}{c} \boldsymbol{U}\\ \boldsymbol{V}\end{array}\right] \right\|_{2,\infty} \leqslant & 110 C_I \rho^t\sqrt{\frac{\mu^2 r^2 \kappa^{12} \log (n_1\vee n_2)}{(n_1\wedge n_2)^2p}}\sqrt{\sigma_1(\boldsymbol{M})},\label{eq_ind5}
\end{align}

for all $1\leqslant l\leqslant n_1+n_2$, where $C_I$ is the absolute constant defined in Lemma \ref{lemma_initialization} and  $\rho := 1-0.05 \eta \sigma_r(\boldsymbol{M})$, then on an event $E_{gd}^{t+1}\subset E_{gd}^t$ satisfying $\mathbb{P}[E_{gd}^t \backslash E_{gd}^{t+1}]\leqslant (n_1+n_2)^{-10}$, the above inequalities \eqref{eq_ind2}, \eqref{eq_ind4}, \eqref{eq_ind3} and \eqref{eq_ind5} also hold for $t+1$.
\end{lemma}

 If we translate the inequalities (70) in \cite{ma2017implicit} in terms of $\sqrt{\sigma_1(\boldsymbol{M})}$, a straightforward comparison shows that our bounds are $O(\sqrt{r})$ tighter. Our key technical novelty for this improvement has been summarized in Section \ref{sec:contributions} and is thereby omitted here. The detailed proof is deferred to Section \ref{proof:induction}.

\subsection{Proof of the main theorem}
We are now ready to give a proof for the main theorem based upon the above lemmas:
\begin{proof}[Proof of Theorem \ref{thm:main}]
We choose $C_S = C_{S2}+C_{S3}+2C_I^2$ where $C_{S2}$, $C_{S3}$ and $C_I$ are defined in Lemma \ref{lemma_initialization} and \ref{lemma_induction}. Then the requirements on the sampling rate $p$ in both Lemma \ref{lemma_initialization} and \ref{lemma_induction} are satisfied. By Lemma \ref{lemma_initialization}, the inequalities \eqref{eq:ini1}, \eqref{eq:ini2} and \eqref{eq:ini3} hold on the event $E_{init}$ defined there, which implies that the inequalities \eqref{eq_ind2}, \eqref{eq_ind4} and \eqref{eq_ind3} hold for $t = 0$ on $E_{init}$. Moreover, \eqref{eq_ind5} can be straightforwardly implied by \eqref{eq_ind2}, \eqref{eq_ind4} and \eqref{eq_ind3} (the proof is deferred to Section \ref{sec_proof_eq_ind5}), and thereby also holds for $t=0$. Let $E_{gd}^0 = E_{init}$. By applying Lemma \ref{lemma_induction} iteratively for $t=1, 2, \ldots, (n_1+n_2)^3$, we know on an event $E:=E_{gd}^{(n_1+n_2)^3}\subset \cdots\subset E_{gd}^0 = E_{init}$ there holds
\[
\begin{split}
 \left\|\left[ \begin{array}{c} \boldsymbol{X}^t\\\boldsymbol{Y}^t\end{array} \right] \boldsymbol{R}^t - \left[\begin{array}{c} \boldsymbol{U}\\\boldsymbol{V} \end{array}\right]\right\| \leqslant & C_I \rho^t \sqrt{\frac{\mu r \kappa^6 \log (n_1\vee n_2)}{(n_1\wedge n_2)p}}\sqrt{\sigma_1(\boldsymbol{M})}
\end{split}
\]
for all $t$ satisfying $0\leqslant t\leqslant (n_1+n_2)^3$ and $\rho = 1-0.05 \eta \sigma_r(\boldsymbol{M})$. This further implies that
\begin{equation}
\label{eq_050}
\begin{split}
 \left\|\left[ \begin{array}{c} \boldsymbol{X}^t\\\boldsymbol{Y}^t\end{array} \right] \boldsymbol{R}^t - \left[\begin{array}{c} \boldsymbol{U}\\\boldsymbol{V} \end{array}\right]\right\|_F \leqslant & \sqrt{2r} \left\|\left[ \begin{array}{c} \boldsymbol{X}^t\\\boldsymbol{Y}^t\end{array} \right] \boldsymbol{R}^t - \left[\begin{array}{c} \boldsymbol{U}\\\boldsymbol{V} \end{array}\right]\right\| \\
 \leqslant&   \sqrt{2r}  C_I \rho^t \sqrt{\frac{\mu r \kappa^6 \log (n_1\vee n_2)}{(n_1\wedge n_2)p}}\sqrt{\sigma_1(\boldsymbol{M})}\\
  \leqslant &  \rho^t \sqrt{\sigma_r(\boldsymbol{M})},
\end{split} 
\end{equation}
where the last inequality is due to our assumption 
\[
p\geqslant 2C_I^2 \frac{\mu r^2\kappa^7 \log(n_1\vee n_2)}{n_1\wedge n_2}.
\]
Lemma \ref{lemma_induction} also implies that 
\[
\begin{split}
\mathbb{P}[E_{gd}^{(n_1+n_2)^3}]\geqslant& 1-\left(1+(n_1+n_2)^3\right)(n_1+n_2)^{-10}\\
\geqslant & 1-(n_1+n_2)^{-3},
\end{split}
\]
which gives the proof of the first part of Theorem \ref{thm:main}. If we assume additionally that $\eta\geqslant \frac{\sigma_r(\boldsymbol{M})}{1000\sigma_1^2(\boldsymbol{M})}$, which directly gives $0<\rho\leqslant 1-5\times 10^{-5}$. This implies that 
\[
\begin{split}
\rho^{(n_1+n_2)^3}\leqslant& \exp(\log(1-5\times 10^{-5})(n_1+n_2)^3)\\
\leqslant &\exp(-(n_1+n_2)^3/C_R)	.
\end{split}
\]
for some absolute constant $C_R$.	
\end{proof}

\section{Proof of Lemma \ref{lemma_induction}}
\label{proof:induction}
In this section, we give the proof of Lemma \ref{lemma_induction}. Within the proof, we will mainly follow the proof structure introduced in \cite{ma2017implicit}, and useful lemmas from \cite{ma2017implicit} such as Lemma \ref{ma_lemma37} and Lemma \ref{ma_lemma36} are intensively used. Moreover, we use Lemma \ref{lemma_spectral_gap} throughout this section to simplify the proof, and we also conduct a more meticulous application of the matrix Bernstein inequality. These efforts result an $O(\sqrt{r})$ tighter on our error bounds.

\subsection{Key Lemmas}
 
In this subsection, we list some useful lemmas which will be used to prove Lemma \ref{lemma_induction}.

First, we need a lemma from \cite{ma2017implicit}:
\begin{lemma}[{\citealt[Lemma 37]{ma2017implicit}}]\label{ma_lemma37}
	Suppose $\boldsymbol{X}_0,\boldsymbol{X}_1,\boldsymbol{X}_2\in\mathbb{R}^{n\times r}$ are matrices such that 
	\begin{equation}\label{eq_011}
			\|\boldsymbol{X}_1-\boldsymbol{X}_0\|\|\boldsymbol{X}_0\| \leqslant  \frac{\sigma_r^2(\boldsymbol{X}_0)}{2},\quad  \|\boldsymbol{X}_1-\boldsymbol{X}_2\|\|\boldsymbol{X}_0\| \leqslant  \frac{\sigma_r^2(\boldsymbol{X}_0)}{4}.
	\end{equation}
	Denote
	\[
		\boldsymbol{R}_1 \coloneqq \argmin_{\boldsymbol{R}\in\mathsf{O}(r)} \|\boldsymbol{X}_1\boldsymbol{R} -  \boldsymbol{X}_0\|_F,
		\]
	\[
	 \boldsymbol{R}_2 \coloneqq \argmin_{\boldsymbol{R}\in\mathsf{O}(r)} \|\boldsymbol{X}_2\boldsymbol{R} -  \boldsymbol{X}_0\|_F.
	\]
	Then the following two inequalities hold true:
	\[
	\begin{split}
		\|\boldsymbol{X}_1\boldsymbol{R}_1 - \boldsymbol{X}_2\boldsymbol{R}_2\| \leqslant & 5\frac{\sigma_1^2(\boldsymbol{X}_0)}{\sigma_r^2(\boldsymbol{X}_0)} \|\boldsymbol{X}_1-\boldsymbol{X}_2\|, \\
		\|\boldsymbol{X}_1\boldsymbol{R}_1 - \boldsymbol{X}_2\boldsymbol{R}_2\|_F \leqslant & 5\frac{\sigma_1^2(\boldsymbol{X}_0)}{\sigma_r^2(\boldsymbol{X}_0)} \|\boldsymbol{X}_1-\boldsymbol{X}_2\|_F.
	\end{split}
	\]
\end{lemma}

In order to control $\|\mathcal{P}_{\Omega}(\boldsymbol{A}\boldsymbol{B}^\top) - p \boldsymbol{A}\boldsymbol{B}^\top\|$, \cite{bhojanapalli2014universal} and \cite{li2016recovery} introduced the following spectral lemma: 
\begin{lemma}[{\citealt{bhojanapalli2014universal,li2016recovery}}]
\label{lemma_spectral_gap}
Let $\Omega\subset[n_1]\times [n_2]$ be set of indices of revealed entries, and $\boldsymbol{\Omega}$ be the matrix such that $\Omega_{i,j} = 1$ if $(i,j)\in\Omega$, $\Omega_{i,j} = 0$ otherwise. For any matrix $\boldsymbol{A},\boldsymbol{B}$ with suitable shape, we have
\[
\|\mathcal{P}_{\Omega}(\boldsymbol{A}\boldsymbol{B}^\top) - p\boldsymbol{A}\boldsymbol{B}^\top\|\leqslant \|\boldsymbol{\Omega}-p\boldsymbol{J}\| \|\boldsymbol{A}\|_{2,\infty}\|\boldsymbol{B}\|_{2,\infty}.
\]
\end{lemma}

In order to proceed, we also need a control of $\|\boldsymbol{\Omega}-p\boldsymbol{J}\|$, which has been discussed in the literature; see, e.g., \cite{bandeira2016sharp} and \cite{vu2018simple}:
\begin{lemma}
\label{vu_lemma2.2}
There is a constant $C_3>0$ such that if $p\geqslant C_3\frac{\log (n_1\vee n_2)}{n_1\wedge n_2}$, then on an event $E_S$ with probability $\mathbb{P}[E_S]\geqslant 1- (n_1+n_2)^{-11}$, we have 
\[
  \|\boldsymbol{\Omega}-p\boldsymbol{J}\|\leqslant C_3\sqrt{(n_1\wedge n_2)p}.
\]
\end{lemma}
Here we use the assumption that $1/C_0<n_1/n_2<C_0$ and $C_3$ is dependent on $C_0$. 

~\\
Finally, we need a lemma to control the norm of $\sgn(\boldsymbol{C}+\boldsymbol{E}) - \sgn(\boldsymbol{C})$ by the norm of $\boldsymbol{E}$:
\begin{lemma}[{\citealt{mathias1993perturbation,ma2017implicit}}]
\label{ma_lemma36}
Let $\boldsymbol{C}\in\mathbb{R}^{r\times r}$ be a nonsingular matrix. Then for any matrix $\boldsymbol{E}\in\mathbb{R}^{r\times r}$ with $\|\boldsymbol{E}\|\leqslant \sigma_r(\boldsymbol{C})$ and any unitarily invariant norm $\VERT\cdot\VERT$, one have
\[
\VERT \sgn(\boldsymbol{C}+\boldsymbol{E}) - \sgn(\boldsymbol{C}) \VERT \leqslant \frac{2}{\sigma_{r-1}(\boldsymbol{C})+\sigma_r(\boldsymbol{C})}\VERT \boldsymbol{E}\VERT.
\] 
\end{lemma}

\subsection{Proof of \eqref{eq_ind2}}
For the spectral norm, first consider the auxiliary iterates defined as following:
		\begin{equation}\label{eq_066}
			\begin{split}
				\widetilde{\boldsymbol{X}}^{t+1} \coloneqq& \boldsymbol{X}^t\boldsymbol{R}^t - \frac{\eta}{p}\mathcal{P}_{\Omega}\left( \boldsymbol{X}^t\left( \boldsymbol{Y}^t \right)^\top - \boldsymbol{U}\boldsymbol{V}^\top \right)\boldsymbol{V}  -\frac{\eta}{2}\boldsymbol{U}(\boldsymbol{R}^t)^\top\left( \left( \boldsymbol{X}^t \right)^\top\boldsymbol{X}^t - \left( \boldsymbol{Y}^t \right)^\top\boldsymbol{Y}^t \right)\boldsymbol{R}^t,\\
				~\\
		\widetilde{\boldsymbol{Y}}^{t+1} \coloneqq & \boldsymbol{Y}^t\boldsymbol{R}^t - \frac{\eta}{p}\left[\mathcal{P}_{\Omega}\left( \boldsymbol{X}^t \left( \boldsymbol{Y}^t \right)^\top - \boldsymbol{U}\boldsymbol{V}^\top \right)\right]^\top \boldsymbol{U} - \frac{\eta}{2}\boldsymbol{V}(\boldsymbol{R}^t)^\top\left( \left( \boldsymbol{Y}^t \right)^\top\boldsymbol{Y}^t - \left(\boldsymbol{X}^t\right)^\top\boldsymbol{X}^t \right)\boldsymbol{R}^t.
			\end{split}
		\end{equation}
		Denote
		\[
		\begin{split}
			\widetilde{\mathbb{E}}\widetilde{\boldsymbol{X}}^{t+1} \coloneqq& \boldsymbol{X}^t\boldsymbol{R}^t - \eta \left( \boldsymbol{X}^t\left( \boldsymbol{Y}^t \right)^\top - \boldsymbol{U}\boldsymbol{V}^\top \right)\boldsymbol{V}  -\frac{\eta}{2}\boldsymbol{U}(\boldsymbol{R}^t)^\top\left( \left( \boldsymbol{X}^t \right)^\top\boldsymbol{X}^t - \left( \boldsymbol{Y}^t \right)^\top\boldsymbol{Y}^t \right)\boldsymbol{R}^t
		\end{split}
		\]
		and 
		\[
		\begin{split}
			\widetilde{\mathbb{E}}\widetilde{\boldsymbol{Y}}^{t+1} \coloneqq & \boldsymbol{Y}^t\boldsymbol{R}^t - \eta  \left( \boldsymbol{X}^t \left( \boldsymbol{Y}^t \right)^\top - \boldsymbol{U}\boldsymbol{V}^\top \right) ^\top \boldsymbol{U} - \frac{\eta}{2}\boldsymbol{V}(\boldsymbol{R}^t)^\top\left( \left( \boldsymbol{Y}^t \right)^\top\boldsymbol{Y}^t - \left(\boldsymbol{X}^t\right)^\top\boldsymbol{X}^t \right)\boldsymbol{R}^t.
		\end{split}
		\]
		
		Then by triangle inequality, we have the following decomposition:
		\begin{equation}\label{eq_046}
			\begin{split}
				& \left\|  \left[ \begin{array}{c} \boldsymbol{X}^{t+1}\\\boldsymbol{Y}^{t+1}\end{array} \right] \boldsymbol{R}^{t+1} - \left[ \begin{array}{c} \boldsymbol{U}\\\boldsymbol{V} \end{array}\right] \right\| \\
				\leqslant &  \left\|  \left[ \begin{array}{c} \widetilde{\boldsymbol{X}}^{t+1}\\ \widetilde{\boldsymbol{Y}}^{t+1}\end{array} \right]  - \left[ \begin{array}{c} \boldsymbol{U}\\\boldsymbol{V} \end{array}\right] \right\| +\left\|  \left[ \begin{array}{c} \boldsymbol{X}^{t+1}\\\boldsymbol{Y}^{t+1}\end{array} \right] \boldsymbol{R}^{t+1} -  \left[ \begin{array}{c} \widetilde{\boldsymbol{X}}^{t+1}\\ \widetilde{\boldsymbol{Y}}^{t+1}\end{array} \right]  \right\|\\
				\leqslant &  \underbrace{\left\|  \left[ \begin{array}{c} \widetilde{\mathbb{E}}\widetilde{\boldsymbol{X}}^{t+1}\\ \widetilde{\mathbb{E}}\widetilde{\boldsymbol{Y}}^{t+1}\end{array} \right] -  \left[ \begin{array}{c} \widetilde{\boldsymbol{X}}^{t+1}\\ \widetilde{\boldsymbol{Y}}^{t+1}\end{array} \right]\right\|}_{\alpha_1} + \underbrace{\left\|  \left[ \begin{array}{c} \widetilde{\mathbb{E}}\widetilde{\boldsymbol{X}}^{t+1}\\ \widetilde{\mathbb{E}}\widetilde{\boldsymbol{Y}}^{t+1}\end{array} \right]  -\left[ \begin{array}{c} \boldsymbol{U}\\\boldsymbol{V} \end{array}\right]  \right\|}_{\alpha_2}  +\underbrace{\left\| \left[ \begin{array}{c} \boldsymbol{X}^{t+1}\\\boldsymbol{Y}^{t+1}\end{array} \right] \boldsymbol{R}^{t+1} -  \left[ \begin{array}{c} \widetilde{\boldsymbol{X}}^{t+1}\\ \widetilde{\boldsymbol{Y}}^{t+1}\end{array} \right] \right\|}_{\alpha_3} .
			\end{split}
		\end{equation}
		
		\subsubsection{Analysis of $\alpha_1$}
		
		First for $\alpha_1$, since
		 \[ 
		\begin{split}
			&  \left[ \begin{array}{c} \widetilde{\mathbb{E}}\widetilde{\boldsymbol{X}}^{t+1}\\ \widetilde{\mathbb{E}}\widetilde{\boldsymbol{Y}}^{t+1}\end{array} \right] -  \left[ \begin{array}{c} \widetilde{\boldsymbol{X}}^{t+1}\\ \widetilde{\boldsymbol{Y}}^{t+1}\end{array} \right]  \\
			=& \eta \left[ \begin{array}{c} \frac{1}{p}\mathcal{P}_{\Omega}\left( \boldsymbol{X}^t (\boldsymbol{Y}^t)^\top - \boldsymbol{U}\boldsymbol{V}^\top \right)\boldsymbol{V}  \\
			\frac{1}{p}\left[\mathcal{P}_{\Omega}\left( \boldsymbol{X}^t (\boldsymbol{Y}^t)^\top - \boldsymbol{U}\boldsymbol{V}^\top \right)\right]^\top\boldsymbol{U}  
			  \end{array} \right] -\eta \left[ \begin{array}{c}  \left( \boldsymbol{X}^t (\boldsymbol{Y}^t)^\top - \boldsymbol{U}\boldsymbol{V}^\top \right)\boldsymbol{V}\\
			\left( \boldsymbol{X}^t (\boldsymbol{Y}^t)^\top - \boldsymbol{U}\boldsymbol{V}^\top \right)^\top \boldsymbol{U}
			  \end{array} \right],
		\end{split} 
		\]
		and using the facts $\left\| \left[ \begin{array}{c} \boldsymbol{A}\\\boldsymbol{B} \end{array}\right] \right\| \leqslant \|\boldsymbol{A}\|+\|\boldsymbol{B}\|$ and $\|\boldsymbol{U}\| = \|\boldsymbol{V}\|$, we have
		\[  
			\begin{split}
				&\alpha_1 \\
				=& \eta \left\| \left[ \begin{array}{c}
					\left(\frac{1}{p}\mathcal{P}_{\Omega}\left( \boldsymbol{X}^t (\boldsymbol{Y}^t)^\top - \boldsymbol{U}\boldsymbol{V}^\top \right) - \left( \boldsymbol{X}^t (\boldsymbol{Y}^t)^\top - \boldsymbol{U}\boldsymbol{V}^\top \right)\right) \boldsymbol{V} \\
					\left(\frac{1}{p}\mathcal{P}_{\Omega}\left( \boldsymbol{X}^t (\boldsymbol{Y}^t)^\top - \boldsymbol{U}\boldsymbol{V}^\top \right) - \left( \boldsymbol{X}^t (\boldsymbol{Y}^t)^\top - \boldsymbol{U}\boldsymbol{V}^\top \right)\right)^\top \boldsymbol{U}
				\end{array} \right] \right\| \\
				\leqslant & 2\eta \|\boldsymbol{U}\| \left\|  \frac{1}{p}\mathcal{P}_{\Omega}\left( \boldsymbol{X}^t (\boldsymbol{Y}^t)^\top - \boldsymbol{U}\boldsymbol{V}^\top \right)  - \left( \boldsymbol{X}^t (\boldsymbol{Y}^t)^\top - \boldsymbol{U}\boldsymbol{V}^\top \right) \right\|\\
				\leqslant & 2\eta  \|\boldsymbol{U}\| \left( \left\|\frac{1}{p} \mathcal{P}_{\Omega}(\boldsymbol{\Delta}_{\boldsymbol{X}}^t\boldsymbol{V}^\top) - \boldsymbol{\Delta}_{\boldsymbol{X}}^t\boldsymbol{V}^\top\right\|  + \left\| \frac{1}{p}\mathcal{P}_{\Omega}\left( \boldsymbol{U}(\boldsymbol{\Delta}_{\boldsymbol{Y}}^t)^\top \right) - \boldsymbol{U}(\boldsymbol{\Delta}_{\boldsymbol{Y}}^t)^\top \right\| \right)\\
				& + 2\eta  \|\boldsymbol{U}\|\left\| \frac{1}{p}\mathcal{P}_{\Omega}\left( \boldsymbol{\Delta}_{\boldsymbol{X}}^t (\boldsymbol{\Delta}_{\boldsymbol{Y}}^t)^\top\right) -\boldsymbol{\Delta}_{\boldsymbol{X}}^t (\boldsymbol{\Delta}_{\boldsymbol{Y}}^t)^\top \right\| .
			\end{split} 
		\]
		Here we denote $\boldsymbol{\Delta}_{\boldsymbol{X}}^t \coloneqq \boldsymbol{X}^t\boldsymbol{R}^t - \boldsymbol{U},\boldsymbol{\Delta}_{\boldsymbol{Y}}^t \coloneqq \boldsymbol{Y}^t\boldsymbol{R}^t - \boldsymbol{V}$, and $\boldsymbol{\Delta}^t \coloneqq \left[ \begin{array}{c}  \boldsymbol{\Delta}_{\boldsymbol{X}}^t\\\boldsymbol{\Delta}_{\boldsymbol{Y}}^t\end{array} \right]$. The last inequality uses the fact that
		\begin{equation}\label{eq_a001}
			\begin{split}
				 \boldsymbol{X}^t(\boldsymbol{Y}^t)^\top - \boldsymbol{U}\boldsymbol{V}^\top  =&  \boldsymbol{X}^t\boldsymbol{R}^t (\boldsymbol{R}^t)^\top (\boldsymbol{Y}^t)^\top - \boldsymbol{U}\boldsymbol{V}^\top\\
				=& (\boldsymbol{\Delta}_{\boldsymbol{X}}^t + \boldsymbol{U})(\boldsymbol{\Delta}_{\boldsymbol{Y}}^t + \boldsymbol{V})^\top - \boldsymbol{U}\boldsymbol{V}^\top\\
				=& \boldsymbol{\Delta}_{\boldsymbol{X}}^t \boldsymbol{V}^\top + \boldsymbol{U}(\boldsymbol{\Delta}_{\boldsymbol{Y}}^t )^\top +  \boldsymbol{\Delta}_{\boldsymbol{X}}^t(\boldsymbol{\Delta}_{\boldsymbol{Y}}^t )^\top.
			\end{split}
		\end{equation} 
		 Using Lemma \ref{lemma_spectral_gap}, we can show that 
		\[
		 \begin{split}
			\alpha_1\leqslant \frac{2\eta}{p}\|\boldsymbol{U}\|\|\boldsymbol{\Omega}-p\boldsymbol{J}\| ( &\|\boldsymbol{\Delta}_{\boldsymbol{X}}^t\|_{2,\infty}\|\boldsymbol{V}\|_{2,\infty} + \|\boldsymbol{U}\|_{2,\infty}\|\boldsymbol{\Delta}_{\boldsymbol{Y}}^t\|_{2,\infty} +\|\boldsymbol{\Delta}_{\boldsymbol{X}}^t\|_{2,\infty}\|\boldsymbol{\Delta}_{\boldsymbol{Y}}^t\|_{2,\infty}).
		\end{split}
		\]

		From \eqref{eq_ind5}, if 
		\[
			p\geqslant 110^2 C_I^2\frac{\mu r \kappa^{11}\log (n_1\vee n_2)}{n_1\wedge n_2},	
		\] 
		then 
		\[
			\|\boldsymbol{\Delta}^t\|_{2,\infty}\leqslant \sqrt{\frac{\mu r \kappa}{n_1\wedge n_2}}\sqrt{\sigma_1(\boldsymbol{M})}.	
		\]
		Here we also use the fact that $\rho<1$. Recall that $\|\boldsymbol{U}\|_{2,\infty}, \|\boldsymbol{V}\|_{2,\infty} \leqslant \sqrt{\frac{\mu r \kappa}{n_1\wedge n_2}}\sqrt{\sigma_1(\boldsymbol{M})}$, and by \eqref{eq_ind5},
		\[
			\begin{split}
				\alpha_1 \leqslant & \frac{2\eta}{p}\sqrt{\sigma_1(\boldsymbol{M})} \|\boldsymbol{\Omega}-p\boldsymbol{J}\| \times 3\left( 110C_I \rho^t\sqrt{\frac{\mu^2 r^2 \kappa^{12}\log(n_1\vee n_2)}{(n_1\wedge n_2)^2p} }\sqrt{\sigma_1(\boldsymbol{M})}  \sqrt{\frac{\mu r \kappa}{n_1\wedge n_2}}\sqrt{\sigma_1(\boldsymbol{M})}\right).
			\end{split}
		\] 
		Moreover, using Lemma \ref{vu_lemma2.2}, if in addtion 
		\[
			p\geqslant (C_3+16\times 660^2)\frac{\mu^2 r^2\kappa^{9}\log (n_1\vee n_2)}{n_1\wedge n_2},
		\]
		then on the event $E_{gd}^t\subset E_H \subset E_S$, we have
		\begin{equation}\label{eq_010}
			\begin{split}
				\alpha_1 \leqslant & \frac{2\eta}{p}\sqrt{\sigma_1(\boldsymbol{M})}\sqrt{(n_1\wedge n_2)p} \times 3\left( 110C_I \rho^t\sqrt{\frac{\mu^2 r^2 \kappa^{12}\log(n_1\vee n_2)}{(n_1\wedge n_2)^2p} }\sqrt{\sigma_1(\boldsymbol{M})}   \sqrt{\frac{\mu r \kappa}{n_1\wedge n_2}}\sqrt{\sigma_1(\boldsymbol{M})}\right)\\
				= & 660\eta C_I \rho^t \sqrt{\frac{\mu^3 r^3\kappa^{13}\log(n_1\vee n_2)}{(n_1\wedge n_2)^2p^2}}\sqrt{\sigma_1(\boldsymbol{M})}^3 \\
				\leqslant & 0.25\eta  \sigma_r(\boldsymbol{M}) C_I\rho^t \sqrt{\frac{\mu r \kappa^6 \log (n_1\vee n_2)}{(n_1\wedge n_2)p}}\sqrt{\sigma_1(\boldsymbol{M})}.
			\end{split}
		\end{equation}
		
		\subsubsection{Analysis of $\alpha_2$}
		
Since $\boldsymbol{\Delta}_{\boldsymbol{X}}^t = \boldsymbol{X}^t\boldsymbol{R}^t - \boldsymbol{U}$ and $\boldsymbol{\Delta}_{\boldsymbol{Y}}^t =\boldsymbol{Y}^t\boldsymbol{R}^t - \boldsymbol{V}$, we have  
		\begin{equation}\label{eq_a002}
			\begin{split}
				& (\boldsymbol{R}^t)^\top \left[\left( \boldsymbol{X}^t \right)^\top\boldsymbol{X}^t - \left( \boldsymbol{Y}^t \right)^\top\boldsymbol{Y}^t\right] \boldsymbol{R}^t\\
				 =&  \left(\boldsymbol{\Delta}_{\boldsymbol{X}}^t + \boldsymbol{U}\right)^\top \left(\boldsymbol{\Delta}_{\boldsymbol{X}}^t + \boldsymbol{U}\right) -\left(\boldsymbol{\Delta}_{\boldsymbol{Y}}^t +  \boldsymbol{V}\right)^\top\left(\boldsymbol{\Delta}_{\boldsymbol{Y}}^t +  \boldsymbol{V}\right) \\
				=& \left(\boldsymbol{\Delta}_{\boldsymbol{X}}^t  \right)^\top  \boldsymbol{\Delta}_{\boldsymbol{X}}^t + \left(\boldsymbol{\Delta}_{\boldsymbol{X}}^t  \right)^\top\boldsymbol{U} + \boldsymbol{U}^\top \boldsymbol{\Delta}_{\boldsymbol{X}}^t +\boldsymbol{U}^\top \boldsymbol{U}  -\left[ \left(\boldsymbol{\Delta}_{\boldsymbol{Y}}^t  \right)^\top \boldsymbol{\Delta}_{\boldsymbol{Y}}^t +   \left(\boldsymbol{\Delta}_{\boldsymbol{Y}}^t  \right)^\top \boldsymbol{V}+ \boldsymbol{V}^\top   \left(\boldsymbol{\Delta}_{\boldsymbol{Y}}^t  \right)+\boldsymbol{V}^\top \boldsymbol{V}\right] \\
				=& \left(\boldsymbol{\Delta}_{\boldsymbol{X}}^t  \right)^\top  \boldsymbol{\Delta}_{\boldsymbol{X}}^t + \left(\boldsymbol{\Delta}_{\boldsymbol{X}}^t  \right)^\top\boldsymbol{U} + \boldsymbol{U}^\top \boldsymbol{\Delta}_{\boldsymbol{X}}^t - \left(\boldsymbol{\Delta}_{\boldsymbol{Y}}^t  \right)^\top \boldsymbol{\Delta}_{\boldsymbol{Y}}^t -  \left(\boldsymbol{\Delta}_{\boldsymbol{Y}}^t  \right)^\top \boldsymbol{V}- \boldsymbol{V}^\top   \boldsymbol{\Delta}_{\boldsymbol{Y}}^t   . 
			\end{split} 
		\end{equation}
		Therefore, for $\alpha_2$, 
		\begin{equation}\label{eq_076}  
			\begin{split}
				& \left[ \begin{array}{c} \widetilde{\mathbb{E}}\widetilde{\boldsymbol{X}}^{t+1}\\ \widetilde{\mathbb{E}}\widetilde{\boldsymbol{Y}}^{t+1}\end{array} \right]  -\left[ \begin{array}{c} \boldsymbol{U}\\\boldsymbol{V} \end{array}\right] \\
				=& \left[      \begin{array}{c}
				 \boldsymbol{X}^t\boldsymbol{R}^t - \eta\left(\boldsymbol{X}^t(\boldsymbol{Y}^t)^\top - \boldsymbol{U}\boldsymbol{V}^\top \right)\boldsymbol{V}  - \frac{\eta}{2}\boldsymbol{U}(\boldsymbol{R}^t)^\top\left( \left( \boldsymbol{X}^t \right)^\top\boldsymbol{X}^t - \left( \boldsymbol{Y}^t \right)^\top\boldsymbol{Y}^t \right)\boldsymbol{R}^t -\boldsymbol{U}\\
				\boldsymbol{Y}^t\boldsymbol{R}^t - \eta\left(\boldsymbol{X}^t(\boldsymbol{Y}^t)^\top - \boldsymbol{U}\boldsymbol{V}^\top \right)^\top \boldsymbol{U} - \frac{\eta}{2}\boldsymbol{V}(\boldsymbol{R}^t)^\top\left(\left( \boldsymbol{Y}^t \right)^\top\boldsymbol{Y}^t -  \left( \boldsymbol{X}^t \right)^\top\boldsymbol{X}^t   \right)\boldsymbol{R}^t -\boldsymbol{V}  \end{array}\right]\\
				=& \left[ \begin{array}{c} \boldsymbol{\Delta}_{\boldsymbol{X}}^t - \eta \boldsymbol{\Delta}_{\boldsymbol{X}}^t \boldsymbol{V}^\top \boldsymbol{V} - \eta\boldsymbol{U}(\boldsymbol{\Delta}_{\boldsymbol{Y}}^t)^\top\boldsymbol{V} -\frac{\eta}{2}\boldsymbol{U}(\boldsymbol{\Delta}_{\boldsymbol{X}}^t)^\top \boldsymbol{U}- \frac{\eta}{2}\boldsymbol{U}\boldsymbol{U}^\top\boldsymbol{\Delta}_{\boldsymbol{X}}^t +\frac{\eta}{2} \boldsymbol{U} (\boldsymbol{\Delta}_{\boldsymbol{Y}}^t)^\top \boldsymbol{V} +\frac{\eta}{2}\boldsymbol{U}\boldsymbol{V}^\top\boldsymbol{\Delta}_{\boldsymbol{Y}}^t +\eta \boldsymbol{\mathcal{E}}_1   \\
				\boldsymbol{\Delta}_{\boldsymbol{Y}}^t - \eta \boldsymbol{V}(\boldsymbol{\Delta}_{\boldsymbol{X}}^t)^\top \boldsymbol{U} - \eta \boldsymbol{\Delta}_{\boldsymbol{Y}}^t \boldsymbol{U}^\top\boldsymbol{U} - \frac{\eta}{2} \boldsymbol{V}(\boldsymbol{\Delta}_{\boldsymbol{Y}}^t)^\top\boldsymbol{V}   - \frac{\eta}{2}\boldsymbol{V}\boldsymbol{V}^\top\boldsymbol{\Delta}_{\boldsymbol{Y}}^t + \frac{\eta}{2} \boldsymbol{V}(\boldsymbol{\Delta}_{\boldsymbol{X}}^t)^\top\boldsymbol{U} + \frac{\eta}{2}\boldsymbol{V}\boldsymbol{U}^\top\boldsymbol{\Delta}_{\boldsymbol{X}}^t +\eta \boldsymbol{\mathcal{E}}_2
				 \end{array} \right]\\
				 =& \left[ \begin{array}{c} \boldsymbol{\Delta}_{\boldsymbol{X}}^t - \eta \boldsymbol{\Delta}_{\boldsymbol{X}}^t \boldsymbol{V}^\top \boldsymbol{V} -\eta \boldsymbol{U}\boldsymbol{U}^\top\boldsymbol{\Delta}_{\boldsymbol{X}}^t+ \frac{\eta}{2}\boldsymbol{U}\boldsymbol{U}^\top \boldsymbol{\Delta}_{\boldsymbol{X}}^t  +\frac{\eta}{2}\boldsymbol{U}\boldsymbol{V}^\top\boldsymbol{\Delta}_{\boldsymbol{Y}}^t - \frac{\eta}{2}\boldsymbol{U}(\boldsymbol{\Delta}_{\boldsymbol{Y}}^t)^\top\boldsymbol{V}-\frac{\eta}{2}\boldsymbol{U}(\boldsymbol{\Delta}_{\boldsymbol{X}}^t)^\top \boldsymbol{U}   +\eta \boldsymbol{\mathcal{E}}_1 \\
					\boldsymbol{\Delta}_{\boldsymbol{Y}}^t  - \eta \boldsymbol{\Delta}_{\boldsymbol{Y}}^t \boldsymbol{U}^\top\boldsymbol{U} -\eta \boldsymbol{V}\boldsymbol{V}^\top\boldsymbol{\Delta}_{\boldsymbol{Y}}^t  +\frac{\eta}{2}\boldsymbol{V}\boldsymbol{V}^\top\boldsymbol{\Delta}_{\boldsymbol{Y}}^t  +\frac{\eta}{2}\boldsymbol{V}\boldsymbol{U}^\top\boldsymbol{\Delta}_{\boldsymbol{X}}^t - \frac{\eta}{2}\boldsymbol{V}(\boldsymbol{\Delta}_{\boldsymbol{X}}^t)^\top\boldsymbol{U}-\frac{\eta}{2}\boldsymbol{V}(\boldsymbol{\Delta}_{\boldsymbol{Y}}^t)^\top\boldsymbol{V}+\eta \boldsymbol{\mathcal{E}}_2 
					 \end{array} \right].\\
			\end{split}
		\end{equation} 
		Here 
		\begin{equation}\label{eq_e001}
		\begin{split}
			\boldsymbol{\mathcal{E}}_1 \coloneqq& -\boldsymbol{\Delta}_{\boldsymbol{X}}^t (\boldsymbol{\Delta}_{\boldsymbol{Y}}^t)^\top\boldsymbol{V} - \frac{1}{2}\boldsymbol{U}(\boldsymbol{\Delta}_{\boldsymbol{X}}^t)^\top\boldsymbol{\Delta}_{\boldsymbol{X}}^t + \frac{1}{2}\boldsymbol{U} (\boldsymbol{\Delta}_{\boldsymbol{Y}}^t)^\top\boldsymbol{\Delta}_{\boldsymbol{Y}}^t,
		\end{split}
		\end{equation}
		\begin{equation}\label{eq_e002}
		\begin{split}
			\boldsymbol{\mathcal{E}}_2 \coloneqq& -\boldsymbol{\Delta}_{\boldsymbol{Y}}^t (\boldsymbol{\Delta}_{\boldsymbol{X}}^t)^\top\boldsymbol{U} -\frac{1}{2}\boldsymbol{V} (\boldsymbol{\Delta}_{\boldsymbol{Y}}^t)^\top\boldsymbol{\Delta}_{\boldsymbol{Y}}^t  + \frac{1}{2}\boldsymbol{V}(\boldsymbol{\Delta}_{\boldsymbol{X}}^t)^\top\boldsymbol{\Delta}_{\boldsymbol{X}}^t 
		\end{split}
		\end{equation}
		denote terms with at least two $\boldsymbol{\Delta}_{\boldsymbol{X}}^t$'s and $\boldsymbol{\Delta}_{\boldsymbol{Y}}^t$'s. By the way we define $\boldsymbol{R}^t$ in \eqref{eq:Rt}, $\left[\begin{array}{c}  \boldsymbol{X}^t\boldsymbol{R}^t  \\\boldsymbol{Y}^t\boldsymbol{R}^t \end{array}\right]^\top\left[\begin{array}{c}   \boldsymbol{U}\\ \boldsymbol{V}\end{array}\right]$ is positive semidefinite. Therefore,
		\[
		\left[\begin{array}{c}  \boldsymbol{X}^t\boldsymbol{R}^t - \boldsymbol{U}\\\boldsymbol{Y}^t\boldsymbol{R}^t-\boldsymbol{V}\end{array}\right]^\top\left[\begin{array}{c}   \boldsymbol{U}\\ \boldsymbol{V}\end{array}\right]  = (\boldsymbol{\Delta}_{\boldsymbol{X}}^t)^\top\boldsymbol{U}+(\boldsymbol{\Delta}_{\boldsymbol{Y}}^t)^\top\boldsymbol{V}
		\]
		 is symmetric. Plugging this fact back to \eqref{eq_076} we have
		\[
			\begin{split}
				   \left[ \begin{array}{c} \widetilde{\mathbb{E}}\widetilde{\boldsymbol{X}}^{t+1}\\ \widetilde{\mathbb{E}}\widetilde{\boldsymbol{Y}}^{t+1}\end{array} \right]  -\left[ \begin{array}{c} \boldsymbol{U}\\\boldsymbol{V} \end{array}\right]  =& \left[ \begin{array}{c} \boldsymbol{\Delta}_{\boldsymbol{X}}^t - \eta \boldsymbol{\Delta}_{\boldsymbol{X}}^t \boldsymbol{V}^\top \boldsymbol{V} -\eta \boldsymbol{U}\boldsymbol{U}^\top\boldsymbol{\Delta}_{\boldsymbol{X}}^t  +\eta \boldsymbol{\mathcal{E}}_1  \\
				\boldsymbol{\Delta}_{\boldsymbol{Y}}^t  - \eta \boldsymbol{\Delta}_{\boldsymbol{Y}}^t \boldsymbol{U}^\top\boldsymbol{U} -\eta \boldsymbol{V}\boldsymbol{V}^\top\boldsymbol{\Delta}_{\boldsymbol{Y}}^t   +\eta \boldsymbol{\mathcal{E}}_2
				 \end{array} \right]\\
				 =& \frac{1}{2} \left[ \begin{array}{c} \boldsymbol{\Delta}_{\boldsymbol{X}}^t \\ \boldsymbol{\Delta}_{\boldsymbol{Y}}^t \end{array} \right] (\boldsymbol{I} - 2\eta\boldsymbol{U}^\top\boldsymbol{U})  + \frac{1}{2}\left( \boldsymbol{I} - 2\eta\left[ \begin{array}{cc} \boldsymbol{U}\boldsymbol{U}^\top & \boldsymbol{0}\\\boldsymbol{0}&\boldsymbol{V}\boldsymbol{V}^\top \end{array}\right] \right)\left[ \begin{array}{c} \boldsymbol{\Delta}_{\boldsymbol{X}}^t \\ \boldsymbol{\Delta}_{\boldsymbol{Y}}^t \end{array} \right] + \eta \boldsymbol{\mathcal{E}},
			\end{split}
		\]
		where $\boldsymbol{\mathcal{E}}\coloneqq \left[ \begin{array}{c}
			\boldsymbol{\mathcal{E}}_1\\
			\boldsymbol{\mathcal{E}}_2
		\end{array} \right]$. Here the last equality uses the fact that $\boldsymbol{U}^\top\boldsymbol{U} = \boldsymbol{V}^\top\boldsymbol{V}$. Recall that we define $\boldsymbol{U}$ by $\widetilde{\boldsymbol{U}}\boldsymbol{\Sigma}^{1/2}$ and $\boldsymbol{V}$ by $\widetilde{\boldsymbol{V}}\boldsymbol{\Sigma}^{1/2}$, $\boldsymbol{U}\boldsymbol{U}^\top$ and $\boldsymbol{V}\boldsymbol{V}^\top$ share the same eigenvalues. And $\|\boldsymbol{U}\boldsymbol{U}^\top\| = \|\boldsymbol{V}\boldsymbol{V}^\top\| = \sigma_1(\boldsymbol{M})$. Therefore, we have
		\[
			\begin{split}
				\alpha_2 =& \left\| \left[ \begin{array}{c} \widetilde{\mathbb{E}}\widetilde{\boldsymbol{X}}^{t+1}\\ \widetilde{\mathbb{E}}\widetilde{\boldsymbol{Y}}^{t+1}\end{array} \right]  -\left[ \begin{array}{c} \boldsymbol{U}\\\boldsymbol{V} \end{array}\right]  \right\| \\
				\leqslant & \frac{1}{2}\|\boldsymbol{I} - 2\eta\boldsymbol{U}^\top\boldsymbol{U}\|\|\boldsymbol{\Delta}^t\|  + \frac{1}{2}\|\boldsymbol{\Delta}^t\|\left\| \boldsymbol{I} - 2\eta\left[ \begin{array}{cc} \boldsymbol{U}\boldsymbol{U}^\top & \boldsymbol{0}\\\boldsymbol{0}&\boldsymbol{V}\boldsymbol{V}^\top \end{array}\right] \right\| +\eta \|\boldsymbol{\mathcal{E}} \|\\
				\leqslant & (1-\eta\sigma_r(\boldsymbol{M}))\|\boldsymbol{\Delta}^t\|  +\eta\|\boldsymbol{\mathcal{E}}\|.
			\end{split}
		\]
		The last inequality uses the fact that $\eta \leqslant \frac{\sigma_r(\boldsymbol{M})}{200\sigma_1^2(\boldsymbol{M})}$. By the definition of $\boldsymbol{\mathcal{E}}$,
		\[
			\|\boldsymbol{\mathcal{E}}\|\leqslant 4\|\boldsymbol{\Delta}^t\|^2\|\boldsymbol{U}\|
		\]
		holds. From \eqref{eq_ind2} and since
		\[
			p\geqslant 1600 C_I^2 \frac{\mu r \kappa^8 \log (n_1\vee n_2)}{n_1\wedge n_2},
		\]
		on the event $E_{gd}^t$,
		\[
		\begin{split}
			\|\boldsymbol{\mathcal{E}}\|\leqslant &4\times C_I \rho^t \sqrt{\frac{\mu r \kappa^6 \log (n_1\vee n_2)}{(n_1\wedge n_2)p}}\sigma_1(\boldsymbol{M})\|\boldsymbol{\Delta}^t\|\\
			\leqslant& 0.1\sigma_r(\boldsymbol{M})\|\boldsymbol{\Delta}^t\|
		\end{split}
		\]
		holds. Therefore, we have
		\begin{equation}\label{eq_009}
			\alpha_2\leqslant (1-0.9\eta \sigma_r(\boldsymbol{M}))\|\boldsymbol{\Delta}^t\|.
		\end{equation}
		
		\subsubsection{Analysis of $\alpha_3$}
		Now we can start to control $\alpha_3$. Rewrite $\alpha_3$ as
		\[
			\begin{split}
				\alpha_3 = & \left\| \left[ \begin{array}{c} \boldsymbol{X}^{t+1}\\\boldsymbol{Y}^{t+1}\end{array} \right] \boldsymbol{R}^{t+1} - \left[\begin{array}{c} \widetilde{\boldsymbol{X}}^{t+1}\\ \widetilde{\boldsymbol{Y}}^{t+1}\end{array}\right]\right\| \\
				=&\left\|  \left[\begin{array}{c} \boldsymbol{X}^{t+1}\\\boldsymbol{Y}^{t+1} \end{array}\right]\boldsymbol{R}^t (\boldsymbol{R}^t)^\top\boldsymbol{R}^{t+1} - \left[ \begin{array}{c} \widetilde{\boldsymbol{X}}^{t+1}\\ \widetilde{\boldsymbol{Y}}^{t+1}\end{array} \right]\right\|.
			\end{split}
		\]
		We want to apply Lemma \ref{ma_lemma37} with 
		\begin{equation}\label{eq_065}
                  \boldsymbol{X}_0 = \left[ \begin{array}{c} \boldsymbol{U}\\\boldsymbol{V} \end{array} \right], \; \boldsymbol{X}_1 =  \left[\begin{array}{c} \boldsymbol{X}^{t+1}\\\boldsymbol{Y}^{t+1} \end{array}\right]\boldsymbol{R}^t, \; \text{and}~\boldsymbol{X}_2 = \left[ \begin{array}{c} \widetilde{\boldsymbol{X}}^{t+1}\\ \widetilde{\boldsymbol{Y}}^{t+1}\end{array} \right].
		\end{equation}
		By the way we define $\boldsymbol{U}$ and $\boldsymbol{V}$, we have $\sigma_1(\boldsymbol{X}_0) = \sqrt{2\sigma_1(\boldsymbol{M})}$, $\sigma_2(\boldsymbol{X}_0) = \sqrt{2\sigma_2(\boldsymbol{M})}$, $\cdots$, $\sigma_r(\boldsymbol{X}_0) = \sqrt{2\sigma_r(\boldsymbol{M})}$, and $\sigma_1(\boldsymbol{X}_0)/\sigma_r(\boldsymbol{X}_0) = \sqrt{\kappa}$. In order to proceed, we first assume we can apply Lemma \ref{ma_lemma37} here:
			\begin{claim}\label{claim_02}
				Under the setup of Lemma \ref{lemma_induction}, on the event $E_{gd}^t\subset E_H \subset E_S$, the requirement of Lemma \ref{ma_lemma37} to apply here is satisfied with $\boldsymbol{X}_0$, $\boldsymbol{X}_1$ and $\boldsymbol{X}_2$ defined as in \eqref{eq_065}. Moreover, by applying Lemma \ref{ma_lemma37}, we have
				\begin{equation}\label{eq_078}
				\begin{split}
				\alpha_3 =& \left\| \left[ \begin{array}{c}\widetilde{\boldsymbol{X}}^{t+1}\\\widetilde{\boldsymbol{Y}}^{t+1} \end{array}\right] - \left[ \begin{array}{c} \boldsymbol{X}^{t+1}\\\boldsymbol{Y}^{t+1} \end{array} \right] \boldsymbol{R}^{t+1}  \right\| \\
				\leqslant& 0.5\eta\sigma_r(\boldsymbol{M}) \|\boldsymbol{\Delta}^t\|.
				\end{split}
				\end{equation}

			\end{claim}


			Now by putting the estimations of $\alpha_1,\alpha_2,\alpha_3$, \eqref{eq_010}, \eqref{eq_009}, \eqref{eq_078} together,
			\begin{equation}\label{eq_045}
				\begin{split}
				&\left\|  \left[ \begin{array}{c} \boldsymbol{X}^{t+1}\\\boldsymbol{Y}^{t+1}\end{array} \right] \boldsymbol{R}^{t+1} - \left[ \begin{array}{c} \boldsymbol{U}\\\boldsymbol{V} \end{array}\right] \right\|\\
				 \leqslant &\alpha_1+\alpha_2+\alpha_3\\
				\leqslant& (1-0.9\eta\sigma_r(\boldsymbol{M}))\|\boldsymbol{\Delta}^t\|  + 0.25\eta  \sigma_r(\boldsymbol{M}) C_I\rho^t \sqrt{\frac{\mu r \kappa^6 \log (n_1\vee n_2)}{(n_1\wedge n_2)p}}\sqrt{\sigma_1(\boldsymbol{M})}  +  0.5\eta\sigma_r(\boldsymbol{M}) \|\boldsymbol{\Delta}^t\|\\
				\leqslant & C_I\rho^{t+1} \sqrt{\frac{\mu r \kappa^6 \log (n_1\vee n_2)}{(n_1\wedge n_2)p}}\sqrt{\sigma_1(\boldsymbol{M})}
				\end{split}
			\end{equation}
			holds on the event $E_{gd}^t\subset E_H \subset E_S$, where the last inequality uses \eqref{eq_ind2} and $\rho = 1-0.05\eta\sigma_r(\boldsymbol{M})$. 

		\begin{proof}[Proof of Claim \ref{claim_02}]  
		
		By the definition of $\boldsymbol{R}^{t+1}$ in \eqref{eq:Rt}, we can verify that $\boldsymbol{R}_1 = (\boldsymbol{R}^t)^\top\boldsymbol{R}^{t+1}$. Recall $\boldsymbol{R}_1$ is defined in Lemma \ref{ma_lemma37}. Now we want to show that $\boldsymbol{R}_2 = \boldsymbol{I}$. In other words, we want to show 
		\[
		\left[\begin{array}{c}\boldsymbol{U}\\\boldsymbol{V}\end{array}\right]^\top \left[ \begin{array}{c} \widetilde{\boldsymbol{X}}^{t+1}\\ \widetilde{\boldsymbol{Y}}^{t+1}\end{array} \right] \succeq \mtx{0}. 
		\]
		
		First, from \eqref{eq_066},
		\[
			\begin{split}
				& \left[\begin{array}{c}\boldsymbol{U}\\\boldsymbol{V}\end{array}\right]^\top \left[ \begin{array}{c} \widetilde{\boldsymbol{X}}^{t+1}\\ \widetilde{\boldsymbol{Y}}^{t+1}\end{array} \right] \\
				= & \boldsymbol{U}^\top\boldsymbol{X}^t\boldsymbol{R}^t - \frac{\eta}{p}\boldsymbol{U}^\top\mathcal{P}_{\Omega}\left( \boldsymbol{X}^t(\boldsymbol{Y}^t)^\top - \boldsymbol{U}\boldsymbol{V}^\top \right)\boldsymbol{V} - \frac{\eta}{2} \boldsymbol{U}^\top\boldsymbol{U}(\boldsymbol{R}^t)^\top \left( (\boldsymbol{X}^t)^\top\boldsymbol{X}^t -(\boldsymbol{Y}^t)^\top\boldsymbol{Y}^t   \right)\boldsymbol{R}^t\\
				& + \boldsymbol{V}^\top\boldsymbol{Y}^t\boldsymbol{R}^t - \frac{\eta}{p}\boldsymbol{V}^\top \left[ \mathcal{P}_{\Omega}\left( \boldsymbol{X}^t(\boldsymbol{Y}^t)^\top - \boldsymbol{U}\boldsymbol{V}^\top \right) \right]^\top\boldsymbol{U}  - \frac{\eta}{2}\boldsymbol{V}^\top\boldsymbol{V}(\boldsymbol{R}^t)^\top \left( (\boldsymbol{Y}^t)^\top\boldsymbol{Y}^t - (\boldsymbol{X}^t)^\top\boldsymbol{X}^t \right)\boldsymbol{R}^t\\
				=& \boldsymbol{U}^\top\boldsymbol{X}^t\boldsymbol{R}^t  + \boldsymbol{V}^\top\boldsymbol{Y}^t\boldsymbol{R}^t - \frac{\eta}{p}\boldsymbol{U}^\top\mathcal{P}_{\Omega}\left( \boldsymbol{X}^t(\boldsymbol{Y}^t)^\top  - \boldsymbol{U}\boldsymbol{V}^\top \right)\boldsymbol{V} - \frac{\eta}{p}\boldsymbol{V}^\top \left[ \mathcal{P}_{\Omega}\left( \boldsymbol{X}^t(\boldsymbol{Y}^t)^\top - \boldsymbol{U}\boldsymbol{V}^\top \right) \right]^\top\boldsymbol{U},
			\end{split}
		\]
		where the last equation holds since $\boldsymbol{U}^\top\boldsymbol{U} = \boldsymbol{V}^\top\boldsymbol{V}$. By the definition of $\boldsymbol{R}^t$, $\boldsymbol{U}^\top\boldsymbol{X}^t\boldsymbol{R}^t+\boldsymbol{V}^\top\boldsymbol{Y}^t\boldsymbol{R}^t$ is positive semidefinite, therefore symmetric. Therefore, $\left[\begin{array}{c}\boldsymbol{U}\\\boldsymbol{V}\end{array}\right]^\top \left[ \begin{array}{c} \widetilde{\boldsymbol{X}}^{t+1}\\ \widetilde{\boldsymbol{Y}}^{t+1}\end{array} \right]$ is symmetric. Moreover, we have 
		\[
			\begin{split}
				 \left\| \left[ \begin{array}{c} \boldsymbol{U}\\\boldsymbol{V} \end{array}\right]^\top\left[ \begin{array}{c} \widetilde{\boldsymbol{X}}^{t+1}\\ \widetilde{\boldsymbol{Y}}^{t+1}\end{array} \right] - \left[ \begin{array}{c} \boldsymbol{U}\\\boldsymbol{V} \end{array}\right]^\top\left[ \begin{array}{c} \boldsymbol{U}\\\boldsymbol{V} \end{array}\right] \right\| \leqslant & \left\| \left[ \begin{array}{c} \boldsymbol{U}\\\boldsymbol{V} \end{array}\right] \right\| \left\| \left[ \begin{array}{c} \widetilde{\boldsymbol{X}}^{t+1}\\ \widetilde{\boldsymbol{Y}}^{t+1}\end{array} \right] - \left[ \begin{array}{c} \boldsymbol{U}\\\boldsymbol{V} \end{array}\right] \right\|\\
				\leqslant &2\sqrt{\sigma_1(\boldsymbol{M})}(\alpha_1+\alpha_2),
			\end{split}
		\]
		where the last inequality holds by triangle inequality and the definition of $\alpha_1$ and $\alpha_2$ in \eqref{eq_046}. From \eqref{eq_010} and \eqref{eq_009},
		\[
		\begin{split}
			& \alpha_1+\alpha_2\\
			 \leqslant & (1-0.9\eta\sigma_r(\boldsymbol{M}))\|\boldsymbol{\Delta}^t\|  + 0.25\eta  \sigma_r(\boldsymbol{M}) C_I\rho^t \sqrt{\frac{\mu r \kappa^6 \log (n_1\vee n_2)}{(n_1\wedge n_2)p}}\sqrt{\sigma_1(\boldsymbol{M})}
		\end{split}
		\]
		holds on the event $E_{gd}^t$. Therefore, from \eqref{eq_ind2}, and the fact that 
		\[
			p\geqslant 16 C_I^2 \frac{\mu r \kappa^8 \log (n_1\vee n_2)}{n_1\wedge n_2}
		\]
		and
		\[
			\eta \leqslant  \frac{\sigma_r(\boldsymbol{M})}{200\sigma_1^2(\boldsymbol{M})},
		\]
		we have 
		\[
			\begin{split}
			& \left\| \left[ \begin{array}{c} \boldsymbol{U}\\\boldsymbol{V} \end{array}\right]^\top\left[ \begin{array}{c} \widetilde{\boldsymbol{X}}^{t+1}\\ \widetilde{\boldsymbol{Y}}^{t+1}\end{array} \right] - \left[ \begin{array}{c} \boldsymbol{U}\\\boldsymbol{V} \end{array}\right]^\top\left[ \begin{array}{c} \boldsymbol{U}\\\boldsymbol{V} \end{array}\right] \right\|\\
			\leqslant & \left\| \left[ \begin{array}{c} \boldsymbol{U}\\\boldsymbol{V} \end{array}\right] \right\| \left\| \left[ \begin{array}{c} \widetilde{\boldsymbol{X}}^{t+1}\\ \widetilde{\boldsymbol{Y}}^{t+1}\end{array} \right] - \left[ \begin{array}{c} \boldsymbol{U}\\\boldsymbol{V} \end{array}\right] \right\|\\
				\leqslant &  2\sqrt{\sigma_1(\boldsymbol{M})} (1-0.65\eta\sigma_r(\boldsymbol{M})) C_I\rho^t \times  \sqrt{\frac{\mu r \kappa^6 \log (n_1\vee n_2)}{(n_1\wedge n_2)p}}\sqrt{\sigma_1(\boldsymbol{M})} \\
				 \leqslant& 0.5\sigma_r(\boldsymbol{M})\leqslant 0.5\sigma_r^2(\boldsymbol{X}_0)
			\end{split} 
		\]
		on the event $E_{gd}^t$. By the fact that $\left[ \begin{array}{c} \boldsymbol{U}\\\boldsymbol{V} \end{array}\right]^\top\left[ \begin{array}{c} \boldsymbol{U}\\\boldsymbol{V} \end{array}\right] = \boldsymbol{U}^\top\boldsymbol{U} + \boldsymbol{V}^\top\boldsymbol{V} = 2\boldsymbol{U}^\top\boldsymbol{U}$, we have 
		\[
			\lambda_r\left( \left[ \begin{array}{c} \boldsymbol{U}\\\boldsymbol{V} \end{array}\right]^\top\left[ \begin{array}{c} \boldsymbol{U}\\\boldsymbol{V} \end{array}\right] \right) = 2\sigma_r(\boldsymbol{M}).
		\]
		By the construction of $\left[ \begin{array}{c} \boldsymbol{U}\\\boldsymbol{V} \end{array}\right]^\top\left[ \begin{array}{c} \widetilde{\boldsymbol{X}}^{t+1}\\ \widetilde{\boldsymbol{Y}}^{t+1}\end{array} \right] $, it is an $r\times r$ symmetric matrix. By the Weyl's inequality,  for all $i = 1,\cdots, r$, any two symmetric matrices $\boldsymbol{A},\boldsymbol{B}\in\mathbb{R}^{r\times r}$ satisfies
			\[
				|\lambda_i(\boldsymbol{A}) - \lambda_i(\boldsymbol{B})|\leqslant \|\boldsymbol{A}-\boldsymbol{B}\|.
			\]
			Therefore, we have
		\[
			\lambda_r\left( \left[ \begin{array}{c} \boldsymbol{U}\\\boldsymbol{V} \end{array}\right]^\top\left[ \begin{array}{c} \widetilde{\boldsymbol{X}}^{t+1}\\ \widetilde{\boldsymbol{Y}}^{t+1}\end{array} \right] \right) \geqslant 1.5 \sigma_r(\boldsymbol{M}),
		\]
		and $\left[ \begin{array}{c} \boldsymbol{U}\\\boldsymbol{V} \end{array}\right]^\top\left[ \begin{array}{c} \widetilde{\boldsymbol{X}}^{t+1}\\ \widetilde{\boldsymbol{Y}}^{t+1}\end{array} \right] \succeq \boldsymbol{0}$. Therefore, we have
		\[
			\boldsymbol{I} = \boldsymbol{R}_2 = \argmin_{\boldsymbol{R}\in\mathsf{O}(r)}\left\| \left[\begin{array}{c} \widetilde{\boldsymbol{X}}^{t+1}\\ \widetilde{\boldsymbol{Y}}^{t+1} \end{array}\right]\boldsymbol{R} - \left[ \begin{array}{c}\boldsymbol{U}\\\boldsymbol{V}\end{array}\right]\right\|_F.
		\]
		Now we want to verify condition \eqref{eq_011} of Lemma \ref{ma_lemma37} is valid here. Since we have already shown 
		\[
		\left\| \left[ \begin{array}{c} \boldsymbol{U}\\\boldsymbol{V} \end{array}\right] \right\| \left\| \left[ \begin{array}{c} \widetilde{\boldsymbol{X}}^{t+1}\\ \widetilde{\boldsymbol{Y}}^{t+1}\end{array} \right] - \left[ \begin{array}{c} \boldsymbol{U}\\\boldsymbol{V} \end{array}\right] \right\| \leqslant 0.5\sigma_r^2(\boldsymbol{X}_0),
		\]
		 the first inequality is verified. Moreover, by the definition of $\boldsymbol{X}^{t+1}$ and $\boldsymbol{Y}^{t+1}$,
		 \[ 
		 \begin{split}
				\boldsymbol{X}^{t+1}\boldsymbol{R}^t =& \boldsymbol{X}^t \boldsymbol{R}^t -\frac{\eta}{p} \mathcal{P}_{\Omega}(\boldsymbol{X}^t (\boldsymbol{Y}^t)^\top - \boldsymbol{U}\boldsymbol{V}^\top)\boldsymbol{Y}^t \boldsymbol{R}^t  -\frac{\eta}{2} (\boldsymbol{X}^t\boldsymbol{R}^t)(\boldsymbol{R}^t)^\top ((\boldsymbol{X}^t)^\top\boldsymbol{X}^t - (\boldsymbol{Y}^t)^\top \boldsymbol{Y}^t)\boldsymbol{R}^t,
				\end{split}
				\]
		\[
		\begin{split}
				\boldsymbol{Y}^{t+1}\boldsymbol{R}^t =& \boldsymbol{Y}^t \boldsymbol{R}^t -\frac{\eta}{p} \left[\mathcal{P}_{\Omega}(\boldsymbol{X}^t (\boldsymbol{Y}^t)^\top - \boldsymbol{U}\boldsymbol{V}^\top)\right]^\top \boldsymbol{X}^t \boldsymbol{R}^t  -\frac{\eta}{2} (\boldsymbol{Y}^t\boldsymbol{R}^t)(\boldsymbol{R}^t)^\top ((\boldsymbol{Y}^t)^\top\boldsymbol{Y}^t - (\boldsymbol{X}^t)^\top \boldsymbol{X}^t)\boldsymbol{R}^t.  
		\end{split}
		 \]
		Hence,
		\begin{equation}\label{eq_079}
			\begin{split}
				& \left\| \left[ \begin{array}{c}\widetilde{\boldsymbol{X}}^{t+1}\\\widetilde{\boldsymbol{Y}}^{t+1} \end{array}\right] - \left[ \begin{array}{c} \boldsymbol{X}^{t+1}\\\boldsymbol{Y}^{t+1} \end{array} \right] \boldsymbol{R}^t \right\|\\
				=& \eta\left\| \left[ \begin{array}{c} \frac{1}{p}\mathcal{P}_{\Omega}\left( \boldsymbol{X}^t(\boldsymbol{Y}^t)^\top - \boldsymbol{U}\boldsymbol{V}^\top \right)\boldsymbol{\Delta}_{\boldsymbol{Y}}^t   +\frac{1}{2}\boldsymbol{\Delta}_{\boldsymbol{X}}^t(\boldsymbol{R}^t)^\top \left( (\boldsymbol{X}^t)^\top\boldsymbol{X}^t - (\boldsymbol{Y}^t)^\top\boldsymbol{Y}^t \right)\boldsymbol{R}^t\\
				\frac{1}{p}\left[\mathcal{P}_{\Omega}\left( \boldsymbol{X}^t(\boldsymbol{Y}^t)^\top - \boldsymbol{U}\boldsymbol{V}^\top \right)\right]^\top \boldsymbol{\Delta}_{\boldsymbol{X}}^t   + \frac{1}{2}\boldsymbol{\Delta}_{\boldsymbol{Y}}^t (\boldsymbol{R}^t)^\top \left(  (\boldsymbol{Y}^t)^\top\boldsymbol{Y}^t-(\boldsymbol{X}^t)^\top\boldsymbol{X}^t  \right)\boldsymbol{R}^t  
				\end{array}\right]  \right\|\\
				\leqslant & \eta \left\| \left[ \begin{array}{c} \boldsymbol{0} \\ \frac{1}{p} \left[\mathcal{P}_{\Omega}\left( \boldsymbol{X}^t(\boldsymbol{Y}^t)^\top - \boldsymbol{U}\boldsymbol{V}^\top \right)\right]^\top   \end{array}  \vphantom{\begin{array}{c} \boldsymbol{0} \\ \frac{1}{p} \left[\mathcal{P}_{\Omega}\left( \boldsymbol{X}^t(\boldsymbol{Y}^t)^\top - \boldsymbol{U}\boldsymbol{V}^\top \right)\right]^\top   \end{array}}
				\begin{array}{c }    \frac{1}{p} \mathcal{P}_{\Omega}\left( \boldsymbol{X}^t(\boldsymbol{Y}^t)^\top - \boldsymbol{U}\boldsymbol{V}^\top \right)\\   \boldsymbol{0} \end{array} \right] \left[ \begin{array}{c}\boldsymbol{\Delta}_{\boldsymbol{X}}^t \\ \boldsymbol{\Delta}_{\boldsymbol{Y}}^t\end{array} \right]\right\| \\
				&+ \frac{\eta}{2}(\|\boldsymbol{\Delta}_{\boldsymbol{X}}^t\|+\|\boldsymbol{\Delta}_{\boldsymbol{Y}}^t\|)   \left\|(\boldsymbol{R}^t)^\top \left( (\boldsymbol{X}^t)^\top\boldsymbol{X}^t - (\boldsymbol{Y}^t)^\top\boldsymbol{Y}^t \right)\boldsymbol{R}^t  \right\| \\ 
				\leqslant & \eta \left( \left\| \frac{1}{p} \mathcal{P}_{\Omega}\left( \boldsymbol{X}^t(\boldsymbol{Y}^t)^\top - \boldsymbol{U}\boldsymbol{V}^\top \right) \right\|  + \left\|(\boldsymbol{R}^t)^\top \left( (\boldsymbol{X}^t)^\top\boldsymbol{X}^t - (\boldsymbol{Y}^t)^\top\boldsymbol{Y}^t \right)\boldsymbol{R}^t  \right\| \right)\|\boldsymbol{\Delta}^t\|.
			\end{split}
		\end{equation}
		In order to bound $  \left\| \frac{1}{p} \mathcal{P}_{\Omega}\left( \boldsymbol{X}^t(\boldsymbol{Y}^t)^\top - \boldsymbol{U}\boldsymbol{V}^\top \right) \right\|$. Recalling \eqref{eq_a001} and combining with Lemma \ref{lemma_spectral_gap} we have
		\begin{equation}\label{eq_080}
			\begin{split}
				&   \left\| \frac{1}{p} \mathcal{P}_{\Omega}\left( \boldsymbol{X}^t(\boldsymbol{Y}^t)^\top - \boldsymbol{U}\boldsymbol{V}^\top \right) \right\|\\
				\leqslant & \left\| \frac{1}{p}\mathcal{P}_{\Omega}(\boldsymbol{\Delta}_{\boldsymbol{X}}^t \boldsymbol{V}^\top) -  \boldsymbol{\Delta}_{\boldsymbol{X}}^t \boldsymbol{V}^\top\right\| + \|\boldsymbol{\Delta}_{\boldsymbol{X}}^t \boldsymbol{V}^\top\| + \left\| \frac{1}{p}\mathcal{P}_{\Omega}\left( \boldsymbol{U}(\boldsymbol{\Delta}_{\boldsymbol{Y}}^t)^\top \right) -\boldsymbol{U}(\boldsymbol{\Delta}_{\boldsymbol{Y}}^t)^\top \right\|+\left\| \boldsymbol{U}(\boldsymbol{\Delta}_{\boldsymbol{Y}}^t)^\top  \right\|\\
				& +  \left\| \frac{1}{p}\mathcal{P}_{\Omega}\left(\boldsymbol{\Delta}_{\boldsymbol{X}}^t (\boldsymbol{\Delta}_{\boldsymbol{Y}}^t)^\top \right) - \boldsymbol{\Delta}_{\boldsymbol{X}}^t (\boldsymbol{\Delta}_{\boldsymbol{Y}}^t)^\top\right\| +\left\| \boldsymbol{\Delta}_{\boldsymbol{X}}^t (\boldsymbol{\Delta}_{\boldsymbol{Y}}^t)^\top \right\|\\
				\leqslant & \frac{\|\boldsymbol{\Omega}-p\boldsymbol{J}\|}{p}(\|\boldsymbol{\Delta}_{\boldsymbol{X}}^t\|_{2,\infty}\|\boldsymbol{V}\|_{2,\infty} + \|\boldsymbol{\Delta}_{\boldsymbol{Y}}^t\|_{2,\infty}\|\boldsymbol{U}\|_{2,\infty}  +\|\boldsymbol{\Delta}_{\boldsymbol{X}}^t\|_{2,\infty} \|\boldsymbol{\Delta}_{\boldsymbol{Y}}^t\|_{2,\infty})\\
				& + \|\boldsymbol{\Delta}_{\boldsymbol{X}}^t\| \|\boldsymbol{V}\|  + \|\boldsymbol{\Delta}_{\boldsymbol{Y}}^t\| \|\boldsymbol{U}\|  +\|\boldsymbol{\Delta}_{\boldsymbol{X}}^t\| \|\boldsymbol{\Delta}_{\boldsymbol{Y}}^t\| .
			\end{split}
		\end{equation}
		And in addition , from \eqref{eq_a002}, 
		\begin{equation}\label{eq_081}	
			\begin{split}
			 &\left\|(\boldsymbol{R}^t)^\top \left( (\boldsymbol{X}^t)^\top\boldsymbol{X}^t - (\boldsymbol{Y}^t)^\top\boldsymbol{Y}^t \right)\boldsymbol{R}^t  \right\|\\
			  =& \left\| \boldsymbol{U}^\top \boldsymbol{\Delta}_{\boldsymbol{X}}^t + (\boldsymbol{\Delta}_{\boldsymbol{X}}^t)^\top \boldsymbol{U} + (\boldsymbol{\Delta}_{\boldsymbol{X}}^t)^\top \boldsymbol{\Delta}_{\boldsymbol{X}}^t - \boldsymbol{V}^\top \boldsymbol{\Delta}_{\boldsymbol{Y}}^t  - (\boldsymbol{\Delta}_{\boldsymbol{Y}}^t)^\top \boldsymbol{V} - (\boldsymbol{\Delta}_{\boldsymbol{Y}}^t)^\top \boldsymbol{\Delta}_{\boldsymbol{Y}}^t\right\|\\
			  \leqslant & 2\|\boldsymbol{U}\|\|\boldsymbol{\Delta}_{\boldsymbol{X}}^t\| + 2\|\boldsymbol{V}\|\|\boldsymbol{\Delta}_{\boldsymbol{Y}}^t\| +\|\boldsymbol{\Delta}_{\boldsymbol{X}}^t\|^2+\|\boldsymbol{\Delta}_{\boldsymbol{Y}}^t\|^2.
			 \end{split}
		\end{equation}
		Combining the estimations \eqref{eq_080} and \eqref{eq_081} together and plugging back into \eqref{eq_079} we have
		\begin{equation}\label{eq_082}
			\begin{split}
				& \left\| \left[ \begin{array}{c}\widetilde{\boldsymbol{X}}^{t+1}\\\widetilde{\boldsymbol{Y}}^{t+1} \end{array}\right] - \left[ \begin{array}{c} \boldsymbol{X}^{t+1}\\\boldsymbol{Y}^{t+1} \end{array} \right] \boldsymbol{R}^t \right\|\\
				\leqslant & \eta  \frac{\|\boldsymbol{\Omega}-p\boldsymbol{J}\|}{p}(\|\boldsymbol{\Delta}_{\boldsymbol{X}}^t\|_{2,\infty}\|\boldsymbol{V}\|_{2,\infty} + \|\boldsymbol{\Delta}_{\boldsymbol{Y}}^t\|_{2,\infty}\|\boldsymbol{U}\|_{2,\infty} +\|\boldsymbol{\Delta}_{\boldsymbol{X}}^t\|_{2,\infty} \|\boldsymbol{\Delta}_{\boldsymbol{Y}}^t\|_{2,\infty})\|\boldsymbol{\Delta}^t\|\\
				&+ \eta \left(\|\boldsymbol{\Delta}_{\boldsymbol{X}}^t\| \|\boldsymbol{V}\|  + \|\boldsymbol{\Delta}_{\boldsymbol{Y}}^t\| \|\boldsymbol{U}\|  +\|\boldsymbol{\Delta}_{\boldsymbol{X}}^t\| \|\boldsymbol{\Delta}_{\boldsymbol{Y}}^t\|  + 2\|\boldsymbol{U}\|\|\boldsymbol{\Delta}_{\boldsymbol{X}}^t\| + 2\|\boldsymbol{V}\|\|\boldsymbol{\Delta}_{\boldsymbol{Y}}^t\| +\|\boldsymbol{\Delta}_{\boldsymbol{X}}^t\|^2 +\|\boldsymbol{\Delta}_{\boldsymbol{Y}}^t\|^2\right)\|\boldsymbol{\Delta}^t\|.
			\end{split}
		\end{equation}
		From \eqref{eq_ind2}, \eqref{eq_ind5} and
		\[
			p\geqslant 110^2 C_I^2 \frac{\mu r \kappa^{11} \log (n_1\vee n_2)}{n_1\wedge n_2},
		\]
		we have 
		\[
			\|\boldsymbol{\Delta}^t\|\leqslant C_I\rho^t\sqrt{\frac{\mu r \kappa^6\log (n_1\vee n_2)}{(n_1\wedge n_2)p}}\sqrt{\sigma_1(\boldsymbol{M})}\leqslant \sqrt{\sigma_1(\boldsymbol{M})}
		\]
		and 
		\[
		\begin{split}
			\|\boldsymbol{\Delta}^t\|_{2,\infty}\leqslant & 110 C_I\rho^t\sqrt{\frac{\mu^2 r^2 \kappa^{12}\log (n_1\vee n_2)}{(n_1\wedge n_2)^2p}}\sqrt{\sigma_1(\boldsymbol{M})}\\
			\leqslant & \sqrt{\frac{\mu r \kappa}{n_1\wedge n_2}} \sqrt{\sigma_1(\boldsymbol{M})}.
		\end{split}
		\]

		Therefore, by applying Lemma \ref{vu_lemma2.2} and given 
		\[
			p\geqslant (6600C_I+32400 C_I^2) \frac{\mu^{1.5}r^{1.5}\kappa^{10}\log (n_1\vee n_2)}{n_1\wedge n_2},
		\] 
		we have
		
		\[
			\begin{split}
				&  \frac{\|\boldsymbol{\Omega}-p\boldsymbol{J}\|}{p}(\|\boldsymbol{\Delta}_{\boldsymbol{X}}^t\|_{2,\infty}\|\boldsymbol{V}\|_{2,\infty} + \|\boldsymbol{\Delta}_{\boldsymbol{Y}}^t\|_{2,\infty}\|\boldsymbol{U}\|_{2,\infty}  +\|\boldsymbol{\Delta}_{\boldsymbol{X}}^t\|_{2,\infty} \|\boldsymbol{\Delta}_{\boldsymbol{Y}}^t\|_{2,\infty}) \\
				&+  \|\boldsymbol{\Delta}_{\boldsymbol{X}}^t\| \|\boldsymbol{V}\|  + \|\boldsymbol{\Delta}_{\boldsymbol{Y}}^t\| \|\boldsymbol{U}\|  +\|\boldsymbol{\Delta}_{\boldsymbol{X}}^t\| \|\boldsymbol{\Delta}_{\boldsymbol{Y}}^t\|    + 2\|\boldsymbol{U}\|\|\boldsymbol{\Delta}_{\boldsymbol{X}}^t\| + 2\|\boldsymbol{V}\|\|\boldsymbol{\Delta}_{\boldsymbol{Y}}^t\| +\|\boldsymbol{\Delta}_{\boldsymbol{X}}^t\|^2+\|\boldsymbol{\Delta}_{\boldsymbol{Y}}^t\|^2  \\
				\leqslant & 3\sqrt{\frac{n_1\wedge n_2}{p}}\sqrt{\frac{\mu r \kappa}{n_1\wedge n_2}}110C_I\rho^t  \sqrt{\frac{\mu^2 r^2 \kappa^{12}\log (n_1\vee n_2)}{(n_1\wedge n_2)^2p}}\sigma_1(\boldsymbol{M})  + 9 C_I\rho^t \sqrt{\frac{\mu r \kappa^6 \log (n_1\vee n_2)}{(n_1\wedge n_2)p}}\sigma_1(\boldsymbol{M})\\
				\leqslant & 330 C_I \rho^t \sqrt{\frac{\mu^3 r^3 \kappa^{13}\log (n_1\vee n_2)}{(n_1\wedge n_2)^2p^2}}\sigma_1(\boldsymbol{M}) + 9 C_I\rho^t \sqrt{\frac{\mu r \kappa^6 \log (n_1\vee n_2)}{(n_1\wedge n_2)p}}\sigma_1(\boldsymbol{M})\\
				\leqslant & \frac{1}{10\kappa}\sigma_r(\boldsymbol{M})\leqslant \frac{1}{2}\sigma_1(\boldsymbol{M}).
			\end{split}
		\]
		
		Therefore, by plugging back to \eqref{eq_082}, 
		\begin{equation}\label{eq_077}
		\left\| \left[ \begin{array}{c}\widetilde{\boldsymbol{X}}^{t+1}\\\widetilde{\boldsymbol{Y}}^{t+1} \end{array}\right] - \left[ \begin{array}{c} \boldsymbol{X}^{t+1}\\\boldsymbol{Y}^{t+1} \end{array} \right] \boldsymbol{R}^t \right\| \leqslant \frac{1}{10\kappa}\eta \sigma_r(\boldsymbol{M}) \|\boldsymbol{\Delta}^t\|,
		\end{equation}
		and
		\[
			\begin{split}
				 &\left\| \left[ \begin{array}{c}\widetilde{\boldsymbol{X}}^{t+1}\\\widetilde{\boldsymbol{Y}}^{t+1} \end{array}\right] - \left[ \begin{array}{c} \boldsymbol{X}^{t+1}\\\boldsymbol{Y}^{t+1} \end{array} \right] \boldsymbol{R}^t \right\| \left\| \left[ \begin{array}{c} \boldsymbol{U}\\\boldsymbol{V}\end{array} \right] \right\|\\
				 \leqslant& \eta \frac{1}{2}\sigma_1(\boldsymbol{M}) 2\sqrt{\sigma_1(\boldsymbol{M})}\|\boldsymbol{\Delta}^t\| \\
				 \leqslant & \eta  \sqrt{\sigma_1(\boldsymbol{M})}^3 C_I \rho^t \sqrt{\frac{\mu r \kappa^6 \log(n_1\vee n_2)}{(n_1\wedge n_2)p}}\sqrt{\sigma_1(\boldsymbol{M})} \\
				 \leqslant &  \frac{1}{4}\sigma_r(\boldsymbol{M})\leqslant \frac{1}{4}\sigma_r^2(\boldsymbol{X}_0)
			 \end{split}
		\]
		holds on the event $E_{gd}^t$. Here the second inequality holds due to \eqref{eq_ind2}, and the third inequality follows $p \geqslant C_I^2 \frac{\mu r\kappa^6 \log(n_1\vee n_2) }{n_1\wedge n_2}$, and $\eta \leqslant \frac{\sigma_r(\boldsymbol{M})}{200\sigma_1^2(\boldsymbol{M})}$.
		
		Therefore, all the requirements in \eqref{eq_011} of Lemma \ref{ma_lemma37} is valid, and Lemma \ref{ma_lemma37} can be applied with $\boldsymbol{X}_0$, $\boldsymbol{X}_1$ and $\boldsymbol{X}_2$ defined as in \eqref{eq_065}. By applying Lemma \ref{ma_lemma37},
		\[
		\begin{split}
					\alpha_3 =& \left\| \left[ \begin{array}{c}\widetilde{\boldsymbol{X}}^{t+1}\\\widetilde{\boldsymbol{Y}}^{t+1} \end{array}\right] - \left[ \begin{array}{c} \boldsymbol{X}^{t+1}\\\boldsymbol{Y}^{t+1} \end{array} \right] \boldsymbol{R}^{t+1}  \right\|\\
					 \leqslant& 5\kappa \left\| \left[ \begin{array}{c}\widetilde{\boldsymbol{X}}^{t+1}\\\widetilde{\boldsymbol{Y}}^{t+1} \end{array}\right] - \left[ \begin{array}{c} \boldsymbol{X}^{t+1}\\\boldsymbol{Y}^{t+1} \end{array} \right] \boldsymbol{R}^t  \right\|.
		\end{split}
		\]
		Along with \eqref{eq_077}, there holds $\alpha_3\leqslant 0.5\eta\sigma_r(\boldsymbol{M}) \|\boldsymbol{\Delta}^t\|$.
\end{proof}


\subsection{Proof of \eqref{eq_ind4}}
	For the induction hypothesis \eqref{eq_ind4}, without loss of generality, we assume $1\leqslant l\leqslant n_1$. From \eqref{eq:leaveoneout1}, we have the following decomposition:
			\[ 
				\begin{split}
					& \left( \left[ \begin{array}{c}\boldsymbol{X}^{t+1,(l)} \\\boldsymbol{Y}^{t+1,(l)} \end{array}\right]\boldsymbol{R}^{t+1,(l)} - \left[ \begin{array}{c} \boldsymbol{U}\\\boldsymbol{V} \end{array} \right] \right)_{l,\cdot}\\
					 =& (\boldsymbol{X}^{t+1,(l)}_{l,\cdot})^\top \boldsymbol{R}^{t+1,(l)} - \boldsymbol{U}_{l,\cdot}^\top \\
					=& (\boldsymbol{X}^{t,(l)}_{l,\cdot})^\top  \boldsymbol{R}^{t+1,(l)} -\boldsymbol{U}_{l,\cdot}^\top   - \eta \left((\boldsymbol{X}^{t,(l)}_{l,\cdot})^\top (\boldsymbol{Y}^{t,(l)})^\top - \boldsymbol{U}_{l,\cdot}^\top \boldsymbol{V}^\top \right)\boldsymbol{Y}^{t,(l)} \boldsymbol{R}^{t+1,(l)}\\
					& -\frac{\eta}{2}(\boldsymbol{X}^{t,(l)}_{l,\cdot})^\top  \left( (\boldsymbol{X}^{t,(l)})^\top \boldsymbol{X}^{t,(l)} - (\boldsymbol{Y}^{t,(l)})^\top\boldsymbol{Y}^{t,(l)} \right)\boldsymbol{R}^{t+1,(l)}\\
					=& \boldsymbol{a}_1 + \boldsymbol{a}_2-\boldsymbol{a}_3,\\
				\end{split} 
			\]
			where
			\[
			\begin{split}
				 \boldsymbol{a}_1  \coloneqq& (\boldsymbol{X}^{t,(l)}_{l,\cdot})^\top  \boldsymbol{R}^{t,(l)} -\boldsymbol{U}_{l,\cdot}^\top - \eta \left((\boldsymbol{X}^{t,(l)}_{l,\cdot})^\top (\boldsymbol{Y}^{t,(l)})^\top - \boldsymbol{U}_{l,\cdot}^\top \boldsymbol{V}^\top \right)\boldsymbol{Y}^{t,(l)} \boldsymbol{R}^{t,(l)},
			\end{split}
			\]
			\[ 
			\begin{split}
				  \boldsymbol{a}_2  \coloneqq & \left[ (\boldsymbol{X}^{t,(l)}_{l,\cdot})^\top  \boldsymbol{R}^{t,(l)} - \eta \left((\boldsymbol{X}^{t,(l)}_{l,\cdot})^\top (\boldsymbol{Y}^{t,(l)})^\top - \boldsymbol{U}_{l,\cdot}^\top \boldsymbol{V}^\top \right)\boldsymbol{Y}^{t,(l)} \boldsymbol{R}^{t,(l)} \right] \left[ (\boldsymbol{R}^{t,(l)})^{-1}\boldsymbol{R}^{t+1,(l)} - \boldsymbol{I} \right]
			\end{split}
			\]
			and
			\[ 
				\boldsymbol{a}_3 \coloneqq \frac{\eta}{2}(\boldsymbol{X}^{t,(l)}_{l,\cdot})^\top  \left( (\boldsymbol{X}^{t,(l)})^\top \boldsymbol{X}^{t,(l)} - (\boldsymbol{Y}^{t,(l)})^\top\boldsymbol{Y}^{t,(l)} \right)\boldsymbol{R}^{t+1,(l)}.
			\]
			 
			First for $\boldsymbol{a}_1$, denote $\boldsymbol{\Delta}_{\boldsymbol{X}}^{t,(l)} \coloneqq \boldsymbol{X}^{t,(l)}\boldsymbol{R}^{t,(l)} - \boldsymbol{U}, \boldsymbol{\Delta}_{\boldsymbol{Y}}^{t,(l)} \coloneqq \boldsymbol{Y}^{t,(l)}\boldsymbol{R}^{t,(l)} - \boldsymbol{V}$, then by a decomposition similar to \eqref{eq_a001},
			\[
				\begin{split}
					&\|\boldsymbol{a}_1\|_2\\
					 =& \left\| (\boldsymbol{\Delta}_{\boldsymbol{X}}^{t,(l)})_{l,\cdot}^\top  - \eta\left[ (\boldsymbol{\Delta}_{\boldsymbol{X}}^{t,(l)})_{l,\cdot}^\top  (\boldsymbol{\Delta}_{\boldsymbol{Y}}^{t,(l)})^\top + (\boldsymbol{\Delta}_{\boldsymbol{X}}^{t,(l)})_{l,\cdot}^\top \boldsymbol{V}^\top   + \boldsymbol{U}_{l,\cdot}^\top (\boldsymbol{\Delta}_{\boldsymbol{Y}}^{t,(l)})^\top  \right] (\boldsymbol{\Delta}_{\boldsymbol{Y}}^{t,(l)} + \boldsymbol{V}) \right\|_2\\
					=& \left\|(\boldsymbol{\Delta}_{\boldsymbol{X}}^{t,(l)})_{l,\cdot}^\top  - \eta (\boldsymbol{\Delta}_{\boldsymbol{X}}^{t,(l)})_{l,\cdot}^\top \boldsymbol{V}^\top\boldsymbol{V}  -\eta \left[(\boldsymbol{\Delta}_{\boldsymbol{X}}^{t,(l)})_{l,\cdot}^\top (\boldsymbol{\Delta}_{\boldsymbol{Y}}^{t,(l)})^\top +\boldsymbol{U}_{l,\cdot}^\top (\boldsymbol{\Delta}_{\boldsymbol{Y}}^{t,(l)})^\top \right] \boldsymbol{Y}^{t,(l)}\boldsymbol{R}^{t,(l)}  - \eta (\boldsymbol{\Delta}_{\boldsymbol{X}}^{t,(l)})_{l,\cdot}^\top \boldsymbol{V}^\top \boldsymbol{\Delta}_{\boldsymbol{Y}}^{t,(l)} \right\|_2\\
					\leqslant & \|\boldsymbol{I} - \eta \boldsymbol{V}^\top\boldsymbol{V}\| \| (\boldsymbol{\Delta}_{\boldsymbol{X}}^{t,(l)})_{l,\cdot} \|_2  + \eta (\|(\boldsymbol{\Delta}_{\boldsymbol{X}}^{t,(l)})_{l,\cdot}\|_2+ \|\boldsymbol{U}_{l,\cdot}\|_2) \|\boldsymbol{\Delta}_{\boldsymbol{Y}}^{t,(l)}\|  \|\boldsymbol{Y}^{t,(l)}\|  +  \eta \|(\boldsymbol{\Delta}_{\boldsymbol{X}}^{t,(l)})_{l,\cdot}\|_2 \|\boldsymbol{V}\|\|\boldsymbol{\Delta}_{\boldsymbol{Y}}^{t,(l)}\| .
				\end{split}
			\]

			From \eqref{eq_ind2},
			\[
			\begin{split}
				   \left\| \left[ \begin{array}{c} \boldsymbol{X}^{t}\\\boldsymbol{Y}^{t} \end{array}\right]\boldsymbol{R}^{t} - \left[ \begin{array}{c} \boldsymbol{U}\\\boldsymbol{V} \end{array}\right]  \right\| \left\| \left[ \begin{array}{c} \boldsymbol{U}\\\boldsymbol{V} \end{array}\right] \right\| \leqslant & 2C_I \rho^t \sqrt{\frac{\mu r \kappa^6 \log (n_1\vee n_2)}{(n_1\wedge n_2)p}} \sigma_1(\boldsymbol{M})\\
				  \leqslant& \frac{ \sigma_r(\boldsymbol{M})}{2}
				\end{split}
			\]  
			holds since 
			\[
				p\geqslant 16 C_I^2 \frac{\mu r \kappa^8 \log(n_1\vee n_2)}{ n_1\wedge n_2 }.
			\]
			Also from \eqref{eq_ind3}, 
			\[
				\begin{split}
				 \left\| \left[ \begin{array}{c} \boldsymbol{X}^{t,(l)}\\\boldsymbol{Y}^{t,(l)} \end{array}\right]\boldsymbol{T}^{t,(l)} - \left[ \begin{array}{c} \boldsymbol{X}^{t}\\\boldsymbol{Y}^{t} \end{array}\right]\boldsymbol{R}^{t} \right\| \left\| \left[ \begin{array}{c} \boldsymbol{U}\\\boldsymbol{V} \end{array}\right] \right\| \leqslant & \left\| \left[ \begin{array}{c} \boldsymbol{X}^{t,(l)}\\\boldsymbol{Y}^{t,(l)} \end{array}\right]\boldsymbol{T}^{t,(l)} - \left[ \begin{array}{c} \boldsymbol{X}^{t}\\\boldsymbol{Y}^{t} \end{array}\right]\boldsymbol{R}^{t} \right\|_F\left\| \left[ \begin{array}{c} \boldsymbol{U}\\\boldsymbol{V} \end{array}\right] \right\|\\
				 \leqslant & 2C_I \rho^t \sqrt{\frac{\mu^2 r^2 \kappa^{10} \log (n_1\vee n_2)}{(n_1\wedge n_2)^2p}} \sigma_1(\boldsymbol{M}) \\
				  \leqslant &\frac{ \sigma_r(\boldsymbol{M})}{4},
				\end{split}
			\] 
			where the last inequality holds since 
			\[
				p\geqslant 64C_I^2 \frac{\mu^2 r^2 \kappa^{12} \log (n_1\vee n_2)}{n_1\wedge n_2}.
			\]
			Applying Lemma \ref{ma_lemma37} with 
			\[
				\boldsymbol{X}_0 \coloneqq   \left[ \begin{array}{c} \boldsymbol{U}\\\boldsymbol{V} \end{array}\right],\;  \boldsymbol{X}_1 \coloneqq \left[ \begin{array}{c} \boldsymbol{X}^{t}\\\boldsymbol{Y}^{t} \end{array}\right]\boldsymbol{R}^{t}, \;\boldsymbol{X}_2 \coloneqq \left[ \begin{array}{c} \boldsymbol{X}^{t,(l)}\\\boldsymbol{Y}^{t,(l)} \end{array}\right]\boldsymbol{T}^{t,(l)},
			\]
			since we define $\boldsymbol{U}$ by $\widetilde{\boldsymbol{U}}\boldsymbol{\Sigma}^{1/2}$ and $\boldsymbol{V}$ by $\widetilde{\boldsymbol{V}}\boldsymbol{\Sigma}^{1/2}$, we have $\sigma_1(\boldsymbol{X}_0) = \sqrt{2\sigma_1(\boldsymbol{M})}$, $\sigma_2(\boldsymbol{X}_0) = \sqrt{2\sigma_2(\boldsymbol{M})}$, $\cdots$, $\sigma_r(\boldsymbol{X}_0) = \sqrt{2\sigma_r(\boldsymbol{M})}$, and $\sigma_1(\boldsymbol{X}_0)/\sigma_r(\boldsymbol{X}_0) = \sqrt{\kappa}$. We have
			\[
			\begin{split}
			  \left\| \left[ \begin{array}{c} \boldsymbol{X}^{t,(l)}\\\boldsymbol{Y}^{t,(l)} \end{array}\right]\boldsymbol{R}^{t,(l)} - \left[ \begin{array}{c} \boldsymbol{X}^{t}\\\boldsymbol{Y}^{t} \end{array}\right]\boldsymbol{R}^{t} \right\|_F  \leqslant & 5\kappa \left\| \left[ \begin{array}{c} \boldsymbol{X}^{t,(l)}\\\boldsymbol{Y}^{t,(l)} \end{array}\right]\boldsymbol{T}^{t,(l)} - \left[ \begin{array}{c} \boldsymbol{X}^{t}\\\boldsymbol{Y}^{t} \end{array}\right]\boldsymbol{R}^{t} \right\|_F. 
			\end{split}
			\] 
			Therefore, by triangle inequality we have
			\begin{equation}\label{eq_019}
				\begin{split}
				 \|\boldsymbol{\Delta}_{\boldsymbol{Y}}^{t,(l)}\| \leqslant & \left\| \left[ \begin{array}{c} \boldsymbol{X}^{t,(l)}\\\boldsymbol{Y}^{t,(l)} \end{array}\right]\boldsymbol{R}^{t,(l)} - \left[ \begin{array}{c} \boldsymbol{U}\\\boldsymbol{V} \end{array}\right] \right\|\\
					\leqslant & \left\| \left[ \begin{array}{c} \boldsymbol{X}^{t,(l)}\\\boldsymbol{Y}^{t,(l)} \end{array}\right]\boldsymbol{R}^{t,(l)} - \left[ \begin{array}{c} \boldsymbol{X}^{t}\\\boldsymbol{Y}^{t} \end{array}\right]\boldsymbol{R}^{t} \right\|_F  + \left\| \left[ \begin{array}{c} \boldsymbol{X}^{t}\\\boldsymbol{Y}^{t} \end{array}\right]\boldsymbol{R}^{t} - \left[ \begin{array}{c} \boldsymbol{U}\\\boldsymbol{V} \end{array}\right]  \right\|\\
					\leqslant & 5\kappa \left\| \left[ \begin{array}{c} \boldsymbol{X}^{t,(l)}\\\boldsymbol{Y}^{t,(l)} \end{array}\right]\boldsymbol{T}^{t,(l)} - \left[ \begin{array}{c} \boldsymbol{X}^{t}\\\boldsymbol{Y}^{t} \end{array}\right]\boldsymbol{R}^{t} \right\|_F + \left\| \left[ \begin{array}{c} \boldsymbol{X}^{t}\\\boldsymbol{Y}^{t} \end{array}\right]\boldsymbol{R}^{t} - \left[ \begin{array}{c} \boldsymbol{U}\\\boldsymbol{V} \end{array}\right]  \right\|\\
					\leqslant & 5\kappa C_I \rho^t \sqrt{\frac{\mu^2 r^2\kappa^{10}\log (n_1\vee n_2)}{(n_1\wedge n_2)^2 p}}\sqrt{\sigma_1(\boldsymbol{M})}  + C_I\rho^t \sqrt{\frac{\mu r \kappa^6 \log(n_1\vee n_2)}{(n_1\wedge n_2)p}}\sqrt{\sigma_1(\boldsymbol{M})}\\
					\leqslant & 2C_I\rho^t \sqrt{\frac{\mu r \kappa^6 \log(n_1\vee n_2)}{(n_1\wedge n_2)p}}\sqrt{\sigma_1(\boldsymbol{M})}.
				\end{split}
			\end{equation}
			For the last inequality, we use the fact that  
			\[
				\frac{25 \mu r \kappa^6}{ n_1\wedge n_2}\leqslant p \leqslant 1.
			\]


			Equipped with \eqref{eq_019}, and combining with the fact that $\|\boldsymbol{Y}^{t,(l)}\| \leqslant \|\boldsymbol{V}\|+\|\boldsymbol{\Delta}_{\boldsymbol{Y}}^{t,(l)}\|$, we have
			\[
				\begin{split}
					&\|\boldsymbol{a}_1\|_2\\
					 \leqslant & (1-\eta \sigma_r(\boldsymbol{M})) \| (\boldsymbol{\Delta}_{\boldsymbol{X}}^{t,(l)})_{l,\cdot} \|_2 \\
					&+ \eta \| (\boldsymbol{\Delta}_{\boldsymbol{X}}^{t,(l)})_{l,\cdot} \|_2 2C_I\rho^t \sqrt{\frac{\mu r \kappa^6 \log(n_1\vee n_2)}{(n_1\wedge n_2)p}}\sqrt{\sigma_1(\boldsymbol{M})} \left( 2 \sqrt{\sigma_1(\boldsymbol{M})} +  2C_I\rho^t \sqrt{\frac{\mu r \kappa^6 \log(n_1\vee n_2)}{(n_1\wedge n_2)p}}\sqrt{\sigma_1(\boldsymbol{M})} \right)\\
					& + \eta\sqrt{\frac{\mu r \kappa}{n_1\wedge n_2}} \sqrt{\sigma_1(\boldsymbol{M})} 2C_I\rho^t   \sqrt{\frac{\mu r \kappa^6 \log(n_1\vee n_2)}{(n_1\wedge n_2)p}}\sqrt{\sigma_1(\boldsymbol{M})}  \left( \sqrt{\sigma_1(\boldsymbol{M})} +  2C_I\rho^t \sqrt{\frac{\mu r \kappa^6 \log(n_1\vee n_2)}{(n_1\wedge n_2)p}}\sqrt{\sigma_1(\boldsymbol{M})} \right).
				\end{split}
			\]
			Given 
			\[
				p \geqslant 4C_I^2 \frac{\mu r \kappa^6 \log (n_1\vee n_2)}{n_1\wedge n_2},	
			\]
			we have 
			\[
				2C_I\rho^t \sqrt{\frac{\mu r \kappa^6 \log(n_1\vee n_2)}{(n_1\wedge n_2)p}}\sqrt{\sigma_1(\boldsymbol{M})} \leqslant \sqrt{\sigma_1(\boldsymbol{M})}.
			\]
			Therefore, 
			\[
				\begin{split}
					& \|\boldsymbol{a}_1\|_2\\
					 \leqslant & (1-\eta \sigma_r(\boldsymbol{M})) \| (\boldsymbol{\Delta}_{\boldsymbol{X}}^{t,(l)})_{l,\cdot} \|_2  + \eta \| (\boldsymbol{\Delta}_{\boldsymbol{X}}^{t,(l)})_{l,\cdot} \|_2 6 C_I\rho^t \sqrt{\frac{\mu r \kappa^6 \log(n_1\vee n_2)}{(n_1\wedge n_2)p}} \sigma_1(\boldsymbol{M})\\ 
					& + \eta\sqrt{\frac{\mu r \kappa}{n_1\wedge n_2}} \sqrt{\sigma_1(\boldsymbol{M})} 4 C_I\rho^t  \sqrt{\frac{\mu r \kappa^6 \log(n_1\vee n_2)}{(n_1\wedge n_2)p}} \sigma_1(\boldsymbol{M}). 
				\end{split}
			\]

			Given 
			\[
				p \geqslant 576C_I^2  \frac{\mu r \kappa^8 \log (n_1\vee n_2)}{n_1\wedge n_2},
			\]
			on the event $E_{gd}^t$, 
			\begin{equation}\label{eq_025}
				\begin{split}
					&\|\boldsymbol{a}_1\|_2\\
					\leqslant & (1-\eta \sigma_r(\boldsymbol{M})) \| (\boldsymbol{\Delta}_{\boldsymbol{X}}^{t,(l)})_{l,\cdot} \|_2   + 0.25 \eta \sigma_r(\boldsymbol{M}) \| (\boldsymbol{\Delta}_{\boldsymbol{X}}^{t,(l)})_{l,\cdot} \|_2 +   \eta \sigma_r(\boldsymbol{M}) 4 C_I\rho^t \sqrt{\frac{\mu^2 r^2 \kappa^9 \log (n_1\vee n_2)}{(n_1\wedge n_2)^2 p}} \sqrt{\sigma_1(\boldsymbol{M})} \\
					 = & (1-0.75 \eta \sigma_r(\boldsymbol{M})) \| (\boldsymbol{\Delta}_{\boldsymbol{X}}^{t,(l)})_{l,\cdot} \|_2    +   4C_I\eta \sigma_r(\boldsymbol{M})  \rho^t \sqrt{\frac{\mu^2 r^2 \kappa^9 \log (n_1\vee n_2)}{(n_1\wedge n_2)^2 p}} \sqrt{\sigma_1(\boldsymbol{M})} 
				\end{split}
			\end{equation}
			At the same time from \eqref{eq_ind4} we have 
			\[
				\begin{split}
					 \|\boldsymbol{a}_1\|_2  \leqslant & (1-0.75\eta\sigma_r(\boldsymbol{M})) \times 100C_I\rho^t\sqrt{\frac{\mu^2 r^2 \kappa^{10}\log(n_1\vee n_2)}{(n_1\wedge n_2)^2p}}\sqrt{\sigma_1(\boldsymbol{M})}\\
					&+4\eta \sigma_r(\boldsymbol{M})  C_I\rho^t \sqrt{\frac{\mu^2 r^2 \kappa^9 \log (n_1\vee n_2)}{(n_1\wedge n_2)^2 p}} \sqrt{\sigma_1(\boldsymbol{M})}\\
					\leqslant& \sqrt{\frac{\mu r \kappa}{n_1\wedge n_2}} \sqrt{\sigma_1(\boldsymbol{M})}\\
				\end{split}	
			\]
			since
			\[
				p \geqslant 10^4 C_I^2 \frac{\mu r \kappa^9 \log(n_1\vee n_2)}{n_1\wedge n_2}	
			\]
			and
			\[
				\eta \leqslant \frac{\sigma_r(\boldsymbol{M})}{200\sigma_r^2(\boldsymbol{M})}.	
			\]
			
			For $\boldsymbol{a}_2$, note
			\begin{equation}\label{eq_061}
				\begin{split}
					&\|\boldsymbol{a}_2\|_2\\
					 =& \left\|\left[ (\boldsymbol{X}^{t,(l)}_{l,\cdot})^\top  \boldsymbol{R}^{t,(l)} - \eta \left((\boldsymbol{X}^{t,(l)}_{l,\cdot})^\top (\boldsymbol{Y}^{t,(l)})^\top - \boldsymbol{U}_{l,\cdot}^\top \boldsymbol{V}^\top \right)\boldsymbol{Y}^{t,(l)} \boldsymbol{R}^{t,(l)} \right]  \left[ (\boldsymbol{R}^{t,(l)})^{-1}\boldsymbol{R}^{t+1,(l)} - \boldsymbol{I} \right]\right\|\\
					\leqslant & \|\boldsymbol{a}_1+\boldsymbol{U}_{l,\cdot}\|_2\|(\boldsymbol{R}^{t,(l)})^{-1}\boldsymbol{R}^{t+1,(l)} - \boldsymbol{I}\|\\
					\leqslant & 2\sqrt{\frac{\mu r \kappa}{n_1\wedge n_2}} \sqrt{\sigma_1(\boldsymbol{M})} \|(\boldsymbol{R}^{t,(l)})^{-1}\boldsymbol{R}^{t+1,(l)} - \boldsymbol{I}\|.
				\end{split}
			\end{equation}
			Here we want to use Lemma \ref{ma_lemma36} to control $\|(\boldsymbol{R}^{t,(l)})^{-1}\boldsymbol{R}^{t+1,(l)} - \boldsymbol{I}\|$. In order to proceed, we first assume the following claim is valid:
			\begin{claim}\label{claim_03}
				Under the setup of Lemma \ref{lemma_induction}, assume $1\leqslant l\leqslant n_1$. Lemma \ref{ma_lemma36} can be applied and on the event $E_{gd}^t$,
				\begin{equation}\label{eq_063}
				\begin{split}
					  \|(\boldsymbol{R}^{t,(l)})^{-1}\boldsymbol{R}^{t+1,(l)} - \boldsymbol{I}\| \leqslant & 76C_I^2\frac{\sigma_1^2(\boldsymbol{M})}{\sigma_r(\boldsymbol{M})}\eta \rho^t \sqrt{\frac{\mu^2 r^2\kappa^{12}\log^2(n_1\vee n_2)}{(n_1\wedge n_2)^2p^2}}
				\end{split}
				\end{equation}
				holds.
			\end{claim}
			The proof of this claim mainly relies on Lemma \ref{ma_lemma36}, and the verification of conditions required by Lemma \ref{ma_lemma36} is very similar to the way we handle $\alpha_1, \alpha_2, \alpha_3$ defined in \eqref{eq_046}. For the purpose of self-containedness, we include the proof of the claim in Appendix \ref{sec_proof_claim_03}.

			 Plugging \eqref{eq_063} back to \eqref{eq_061} we have
			 \begin{equation}
			 \label{eq_026}
			 \begin{split}
				&\|\boldsymbol{a}_2\|_2\\
				 \leqslant &2\sqrt{\frac{\mu r \kappa}{n_1\wedge n_2}}\sqrt{\sigma_1(\boldsymbol{M})} \times 76C_I^2 \frac{\sigma_1^2(\boldsymbol{M})}{\sigma_r(\boldsymbol{M})}\eta\rho^t \sqrt{\frac{\mu^2 r^2 \kappa^{12}\log^2 (n_1\vee n_2)}{(n_1\wedge n_2)^2p^2}}
				\\
				\leqslant & 152C_I^2\eta \rho^t \frac{\sigma_1^2(\boldsymbol{M})}{\sigma_r(\boldsymbol{M})}  \sqrt{\frac{\mu^3 r^3\kappa^{13}\log^2(n_1\vee n_2)}{(n_1\wedge n_2)^3 p^2}}\sqrt{\sigma_1(\boldsymbol{M})}\\
				\leqslant & 25C_I \eta \sigma_r(\boldsymbol{M}) \rho^t \sqrt{\frac{\mu^2 r^2 \kappa^{10} \log (n_1\vee n_2)}{(n_1\wedge n_2)^2 p}}\sqrt{\sigma_1(\boldsymbol{M})},
			 \end{split}
			 \end{equation}
			 where the last inequality uses the fact that
			 \[
				 p \geqslant 37C_I^2 \frac{\mu r \kappa^7 \log (n_1\vee n_2)}{n_1\wedge n_2}.
			 \]

			 Finally, for $\boldsymbol{a}_3$, note the fact that $\boldsymbol{R}^{t+1,(l)}$ and $\boldsymbol{R}^{t,(l)}$ are all orthogonal matrices. And replacing $\boldsymbol{X}$ and $\boldsymbol{Y}$ with $\boldsymbol{X}^{t,(l)}$ and $\boldsymbol{Y}^{t,(l)}$ in \eqref{eq_a002},
			 \begin{equation}\label{eq_083} 
				 \begin{split}
					 \|\boldsymbol{a}_3\|_2 =&  \frac{\eta}{2}\left\|(\boldsymbol{X}^{t,(l)}_{l,\cdot})^\top \left( (\boldsymbol{X}^{t,(l)})^\top \boldsymbol{X}^{t,(l)} - (\boldsymbol{Y}^{t,(l)})^\top\boldsymbol{Y}^{t,(l)} \right)\boldsymbol{R}^{t+1,(l)}\right\|_2 \\
					=&  \frac{\eta}{2}\left\|(\boldsymbol{X}^{t,(l)}_{l,\cdot})^\top \left( (\boldsymbol{X}^{t,(l)})^\top \boldsymbol{X}^{t,(l)} - (\boldsymbol{Y}^{t,(l)})^\top\boldsymbol{Y}^{t,(l)} \right)\boldsymbol{R}^{t,(l)}\right\|_2\\
					\leqslant& \frac{\eta}{2} \|\boldsymbol{X}^{t,(l)}_{l,\cdot}\|_2 \left\| (\boldsymbol{R}^{t,(l)})^\top \left( (\boldsymbol{X}^{t,(l)})^\top \boldsymbol{X}^{t,(l)} - (\boldsymbol{Y}^{t,(l)})^\top\boldsymbol{Y}^{t,(l)} \right)\boldsymbol{R}^{t,(l)}\right\|\\
					\leqslant&  \frac{\eta}{2} \|\boldsymbol{X}^{t,(l)}_{l,\cdot}\|_2 \left( 2\|\boldsymbol{\Delta}_{\boldsymbol{X}}^{t,(l)}\|\|\boldsymbol{U}\| +\|\boldsymbol{\Delta}_{\boldsymbol{X}}^{t,(l)}\|^2 +2\|\boldsymbol{\Delta}_{\boldsymbol{Y}}^{t,(l)}\|\|\boldsymbol{V}\| +\|\boldsymbol{\Delta}_{\boldsymbol{Y}}^{t,(l)}\|^2 \right).
				\end{split}
			\end{equation}
			 From \eqref{eq_ind4}, we have
			 \begin{equation}\label{eq_084} 
				 \begin{split}
					  \|\boldsymbol{X}_{l,\cdot}^{t,(l)}\|_2 \leqslant& \|\boldsymbol{U}_{l,\cdot}\|_2 + \|(\boldsymbol{X}_{l,\cdot}^{t,(l)})^\top  \boldsymbol{R}^{t,(l)} - \boldsymbol{U}_{l,\cdot}^\top \|_2\\
					\leqslant & \sqrt{\frac{\mu r \kappa}{n_1\wedge n_2}}\sqrt{\sigma_1(\boldsymbol{M})}  + 100C_I \rho^t \sqrt{\frac{\mu^2 r^2\kappa^{10}\log(n_1\vee n_2)}{(n_1\wedge n_2)^2p}}\sqrt{\sigma_1(\boldsymbol{M})}\\
					\leqslant & 2\sqrt{\frac{\mu r \kappa}{n_1\wedge n_2}}\sqrt{\sigma_1(\boldsymbol{M})}.
				 \end{split}
				\end{equation}
			 The last line holds since
			 \[
				p\geqslant 10^4 C_I^2 \frac{\mu r \kappa^9\log(n_1\vee n_2)}{n_1\wedge n_2}. 
			\]
			From \eqref{eq_019} and given 
			\[
				p\geqslant 4 C_I^2 \frac{\mu r \kappa^6 \log (n_1\vee n_2)}{n_1\wedge n_2},
			\]
			we have $\|\boldsymbol{\Delta}^{t,(l)}\|\leqslant \sqrt{\sigma_1(\boldsymbol{M})}$. Combining with \eqref{eq_019}, \eqref{eq_083} and \eqref{eq_084}, we have
			 \begin{equation}\label{eq_027}
				\begin{split}
					 \|\boldsymbol{a}_3\|_2 \leqslant& \eta \sqrt{\frac{\mu r \kappa}{n_1\wedge n_2}}\sqrt{\sigma_1(\boldsymbol{M})}   \times 12 C_I \rho^t \sqrt{\frac{\mu r \kappa^6 \log (n_1\vee n_2)}{(n_1\wedge n_2)p}}\sigma_1(\boldsymbol{M})\\
					= &  12 C_I\eta\sigma_r(\boldsymbol{M}) \rho^t \sqrt{\frac{\mu^2 r^2 \kappa^9 \log (n_1\vee n_2)}{(n_1\wedge n_2)^2 p}}\sqrt{\sigma_1(\boldsymbol{M})}.
				\end{split}
			 \end{equation}
			 Putting the estimations on $\boldsymbol{a}_1$, $\boldsymbol{a}_2$ and $\boldsymbol{a}_3$ together, i.e., \eqref{eq_025}, \eqref{eq_026} and \eqref{eq_027}, we have
			 \[
				 \begin{split}
					&\left\|\left( \left[ \begin{array}{c}\boldsymbol{X}^{t+1,(l)} \\\boldsymbol{Y}^{t+1,(l)} \end{array}\right]\boldsymbol{R}^{t+1,(l)} - \left[ \begin{array}{c} \boldsymbol{U}\\\boldsymbol{V} \end{array} \right] \right)_{l,\cdot}\right\|_2\\
					 \leqslant & \|\boldsymbol{a}_1\|_2 +\|\boldsymbol{a}_2\|_2 +\|\boldsymbol{a}_3\|_2 \\
					\leqslant & (1-0.75 \eta \sigma_r(\boldsymbol{M})) \| (\boldsymbol{\Delta}_{\boldsymbol{X}}^{t,(l)})_{l,\cdot} \|_2    +   4 C_I\eta \sigma_r(\boldsymbol{M}) \rho^t \sqrt{\frac{\mu^2 r^2 \kappa^9 \log (n_1\vee n_2)}{(n_1\wedge n_2)^2 p}} \sqrt{\sigma_1(\boldsymbol{M})}\\
					& +25C_I\eta \sigma_r(\boldsymbol{M}) \rho^t \sqrt{\frac{\mu^2 r^2 \kappa^{10} \log (n_1\vee n_2)}{(n_1\wedge n_2)^2 p}}\sqrt{\sigma_1(\boldsymbol{M})} +  12 C_I\eta\sigma_r(\boldsymbol{M})  \rho^t \sqrt{\frac{\mu^2 r^2 \kappa^9 \log (n_1\vee n_2)}{(n_1\wedge n_2)^2 p}}\sqrt{\sigma_1(\boldsymbol{M})}\\
					\leqslant & 100C_I \rho^{t+1} \sqrt{\frac{\mu^2 r^2 \kappa^{10} \log (n_1\vee n_2)}{(n_1\wedge n_2)^2 p}}\sqrt{\sigma_1(\boldsymbol{M})} ,
				\end{split}
			 \]
			 with $\rho = 1-0.05\eta \sigma_r(\boldsymbol{M})$ on the event $E_{gd}^{t+1}$, the last inequality uses \eqref{eq_ind4}. Notice this is the proof for the case of $l$ satisfying $1\leqslant l\leqslant n_1$, the proof for $l$ satisfying $n_1+1\leqslant l\leqslant n_1+n_2$ is almost the same.

\subsection{Proof of \eqref{eq_ind3}}
For \eqref{eq_ind3}, by the choice of $\boldsymbol{T}^{t+1,(l)}$ in \eqref{eq:Rt}, we have
		\[
		\begin{split}
			  \left\| \left[ \begin{array}{c} \boldsymbol{X}^{t+1} \\\boldsymbol{Y}^{t+1} \end{array}\right]\boldsymbol{R}^{t+1} - \left[\begin{array}{c} \boldsymbol{X}^{t+1,(l)} \\ \boldsymbol{Y}^{t+1,(l)} \end{array}\right] \boldsymbol{T}^{t+1,(l)} \right\|_F^2 \leqslant &\left\| \left[ \begin{array}{c} \boldsymbol{X}^{t+1} \\\boldsymbol{Y}^{t+1} \end{array}\right]\boldsymbol{R}^{t} - \left[\begin{array}{c} \boldsymbol{X}^{t+1,(l)} \\ \boldsymbol{Y}^{t+1,(l)} \end{array}\right] \boldsymbol{T}^{t,(l)} \right\|_F^2.
			\end{split}
		\]
		Without loss of generality, we first consider the case that $l$ satisfying $1\leqslant l\leqslant n_1$. First, by plugging in the definition of $\left[ \begin{array}{c} \boldsymbol{X}^{t+1} \\\boldsymbol{Y}^{t+1} \end{array}\right]$ and $\left[\begin{array}{c} \boldsymbol{X}^{t+1,(l)} \\ \boldsymbol{Y}^{t+1,(l)} \end{array}\right]$, we have
		\begin{equation}\label{eq_067} 
			\begin{split}
				 \left[ \begin{array}{c} \boldsymbol{X}^{t+1} \\\boldsymbol{Y}^{t+1} \end{array}\right]\boldsymbol{R}^{t} - \left[\begin{array}{c} \boldsymbol{X}^{t+1,(l)} \\ \boldsymbol{Y}^{t+1,(l)} \end{array}\right] \boldsymbol{T}^{t,(l)} =& \boldsymbol{A}_1 + \eta \left[\begin{array}{c}\boldsymbol{A}_2\\\boldsymbol{A}_3\end{array}\right],\\
			\end{split}
		\end{equation}
		where 
		\[
		\begin{split}
			\boldsymbol{A}_1\coloneqq &\left( \left[\begin{array}{c} \boldsymbol{X}^t\\\boldsymbol{Y}^t \end{array}\right] - \eta\nabla f(\boldsymbol{X}^t,\boldsymbol{Y}^t) \right) \boldsymbol{R}^t   - \left(  \left[\begin{array}{c} \boldsymbol{X}^{t,(l)}\\\boldsymbol{Y}^{t,(l)} \end{array}\right] - \eta\nabla f(\boldsymbol{X}^{t,(l)},\boldsymbol{Y}^{t,(l)}) \right) \boldsymbol{T}^{t,(l)}
		\end{split}
		\]
		\[
		\begin{split}
			\boldsymbol{A}_2\coloneqq&  \mathcal{P}_{l,\cdot}\left( \boldsymbol{X}^{t,(l)}(\boldsymbol{Y}^{t,(l)})^\top - \boldsymbol{U}\boldsymbol{V}^\top \right)\boldsymbol{Y}^{t,(l)}\boldsymbol{T}^{t,(l)}  - \frac{1}{p}\mathcal{P}_{\Omega_{l,\cdot}}\left( \boldsymbol{X}^{t,(l)}(\boldsymbol{Y}^{t,(l)})^\top - \boldsymbol{U}\boldsymbol{V}^\top \right)  \boldsymbol{Y}^{t,(l)}\boldsymbol{T}^{t,(l)}
		\end{split}
		\]
		and
		\[
		\begin{split}
			\boldsymbol{A}_3\coloneqq& \left[ \mathcal{P}_{l,\cdot}\left( \boldsymbol{X}^{t,(l)}(\boldsymbol{Y}^{t,(l)})^\top - \boldsymbol{U}\boldsymbol{V}^\top \right) \right]^\top\boldsymbol{X}^{t,(l)}\boldsymbol{T}^{t,(l)} - \left[\frac{1}{p}\mathcal{P}_{\Omega_{l,\cdot}}\left( \boldsymbol{X}^{t,(l)}(\boldsymbol{Y}^{t,(l)})^\top - \boldsymbol{U}\boldsymbol{V}^\top \right) \right]^\top \boldsymbol{X}^{t,(l)}\boldsymbol{T}^{t,(l)}.
		\end{split}
		\]

		For $\boldsymbol{A}_1$, we have 
		\begin{equation}\label{eq_085} 
			\begin{split}
			&\|\boldsymbol{A}_1\|_F^2\\
			 =& \left\| \left( \left[ \begin{array}{c} \boldsymbol{X}^{t,(l)}\\\boldsymbol{Y}^{t,(l)} \end{array}\right] \boldsymbol{T}^{t,(l)}  - \left[ \begin{array}{c} \boldsymbol{X}^t \\ \boldsymbol{Y}^t \end{array} \right] \boldsymbol{R}^t \right)    - \eta\left(\nabla f(\boldsymbol{X}^{t,(l)}\boldsymbol{T}^{t,(l)},\boldsymbol{Y}^{t,(l)}\boldsymbol{T}^{t,(l)}) - \nabla f(\boldsymbol{X}^t\boldsymbol{R}^t,\boldsymbol{Y}^t\boldsymbol{R}^t)  \right) \right\|_F^2 \\
			=& \left\| \left( \boldsymbol{I} - \eta\int_0^1 \nabla^2 f(*) d\tau \right) \operatorname{vec}\left(\left[ \begin{array}{c} \boldsymbol{X}^{t,(l)}\boldsymbol{T}^{t,(l)} - \boldsymbol{X}^t\boldsymbol{R}^t \\\boldsymbol{Y}^{t,(l)}\boldsymbol{T}^{t,(l)} - \boldsymbol{Y}^t\boldsymbol{R}^t \end{array}\right]\right) \right\|_2^2\\
			\leqslant & \left\| \left[ \begin{array}{c} \boldsymbol{X}^{t,(l)}\boldsymbol{T}^{t,(l)} - \boldsymbol{X}^t\boldsymbol{R}^t \\\boldsymbol{Y}^{t,(l)}\boldsymbol{T}^{t,(l)} - \boldsymbol{Y}^t\boldsymbol{R}^t \end{array}\right] \right\|_F^2  + \eta^2\left\| \left[ \begin{array}{c} \boldsymbol{X}^{t,(l)}\boldsymbol{T}^{t,(l)} - \boldsymbol{X}^t\boldsymbol{R}^t \\\boldsymbol{Y}^{t,(l)}\boldsymbol{T}^{t,(l)} - \boldsymbol{Y}^t\boldsymbol{R}^t \end{array}\right] \right\|_F^2 \max_{0\leqslant \tau \leqslant 1}\left\| \nabla^2 f(*) \right\|^2\\
			& -2\eta \min_{0\leqslant \tau\leqslant 1} \operatorname{vec}\left(  \left[ \begin{array}{c} \boldsymbol{X}^{t,(l)}\boldsymbol{T}^{t,(l)} - \boldsymbol{X}^t\boldsymbol{R}^t \\\boldsymbol{Y}^{t,(l)}\boldsymbol{T}^{t,(l)} - \boldsymbol{Y}^t\boldsymbol{R}^t \end{array}\right]  \right)^\top  \nabla^2 f(*)  \operatorname{vec}\left(  \left[ \begin{array}{c} \boldsymbol{X}^{t,(l)}\boldsymbol{T}^{t,(l)} - \boldsymbol{X}^t\boldsymbol{R}^t \\\boldsymbol{Y}^{t,(l)}\boldsymbol{T}^{t,(l)} - \boldsymbol{Y}^t\boldsymbol{R}^t \end{array}\right]  \right),
			\end{split}
		\end{equation}
		where the first equality uses the fact that $\nabla f(\boldsymbol{X},\boldsymbol{Y}) = \nabla f(\boldsymbol{X}\boldsymbol{R},\boldsymbol{Y}\boldsymbol{R})$ for any $\boldsymbol{R}\in\mathsf{O}(r)$, and here
		\[
		\begin{split}
		\nabla^2 f(*) \coloneqq \nabla^2 f( &\tau (\boldsymbol{X}^{t,(l)}\boldsymbol{T}^{t,(l)} - \boldsymbol{X}^t\boldsymbol{R}^t )  + \boldsymbol{X}^t\boldsymbol{R}^t   , \tau(\boldsymbol{Y}^{t,(l)}\boldsymbol{T}^{t,(l)} - \boldsymbol{Y}^t\boldsymbol{R}^t)+\boldsymbol{Y}^t\boldsymbol{R}^t).
		\end{split}
		\]

		From \eqref{eq_ind3} and \eqref{eq_ind5}, if
		\[
			p\geqslant 	2.42 \times 10^{10}C_0C_I^2 \frac{\mu^2 r^2\kappa^{14}\log (n_1\vee n_2)}{n_1\wedge n_2},
		\]
		we have
		\[
		\begin{split}
		  \left\| \left[ \begin{array}{c} \boldsymbol{X}^{t,(l)}\\\boldsymbol{Y}^{t,(l)} \end{array} \right]\boldsymbol{T}^{t,(l)} - \left[ \begin{array}{c} \boldsymbol{X}^t\\\boldsymbol{Y}^t  \end{array}\right]\boldsymbol{R}^t \right\|_{2,\infty} \leqslant& \frac{1}{1000\kappa\sqrt{n_1+n_2}} \sqrt{\sigma_1(\boldsymbol{M})}
		\end{split}
		\]
		and
		\[
			\left\|  \left[ \begin{array}{c} \boldsymbol{X}^t\\\boldsymbol{Y}^t  \end{array}\right]\boldsymbol{R}^t  - \left[ \begin{array}{c} \boldsymbol{U}\\\boldsymbol{V} \end{array} \right]\right\|_{2,\infty}\leqslant \frac{1}{1000\kappa\sqrt{n_1+n_2}} \sqrt{\sigma_1(\boldsymbol{M})}.
		\]
		Therefore, 
		\[
		\begin{split}
			 \|\tau (\boldsymbol{X}^{t,(l)}\boldsymbol{T}^{t,(l)} - \boldsymbol{X}^t\boldsymbol{R}^t )  + \boldsymbol{X}^t\boldsymbol{R}^t - \boldsymbol{U}\|_{2,\infty} \leqslant &\frac{1}{500\kappa\sqrt{n_1+n_2}} \sqrt{\sigma_1(\boldsymbol{M})},
		\end{split}
		\]
	
		\[
		\begin{split}
		 \| \tau(\boldsymbol{Y}^{t,(l)}\boldsymbol{T}^{t,(l)} - \boldsymbol{Y}^t\boldsymbol{R}^t)+\boldsymbol{Y}^t\boldsymbol{R}^t -\boldsymbol{V}\|_{2,\infty} \leqslant &\frac{1}{500\kappa\sqrt{n_1+n_2}} \sqrt{\sigma_1(\boldsymbol{M})}
		\end{split}
		\]
	 	for any $\tau$ satisfying $0\leqslant \tau\leqslant 1$. And we also have
		\[
			\left\| \left[ \begin{array}{c} \boldsymbol{X}^t\\\boldsymbol{Y}^t \end{array} \right] \boldsymbol{R}^t - \left[ \begin{array}{c} \boldsymbol{U}\\\boldsymbol{V} \end{array}\right]\right\| \leqslant \frac{1}{500\kappa}\sqrt{\sigma_1(\boldsymbol{M})}.
		\]
		Therefore, Lemma \ref{lemma_hessian} can be applied here. Noting $E_{gd}^t\subset E_H$ and 
		\[
				p \geqslant C_{S1}\frac{\mu r \kappa\log(n_1\vee n_2)}{n_1\wedge n_2},
		\]
		we have \eqref{eq_hessian1} and \eqref{eq_hessian2} satisfied. Plugging \eqref{eq_hessian1} and \eqref{eq_hessian2} back to the estimation \eqref{eq_085}, we have 
		\[
			\begin{split}
				 \|\boldsymbol{A}_1\|_F^2  \leqslant & (1-\frac{2}{5}\eta\sigma_r(\boldsymbol{M}) + 25\eta^2\sigma_1^2(\boldsymbol{M}))   \left\| \left[ \begin{array}{c} \boldsymbol{X}^{t,(l)}\\\boldsymbol{Y}^{t,(l)} \end{array} \right]\boldsymbol{T}^{t,(l)} - \left[ \begin{array}{c} \boldsymbol{X}^t\\\boldsymbol{Y}^t  \end{array}\right]\boldsymbol{R}^t \right\|_F^2\\
				\leqslant &(1-0.2\eta\sigma_r(\boldsymbol{M}))\left\| \left[ \begin{array}{c} \boldsymbol{X}^{t,(l)}\\\boldsymbol{Y}^{t,(l)} \end{array} \right]\boldsymbol{T}^{t,(l)} - \left[ \begin{array}{c} \boldsymbol{X}^t\\\boldsymbol{Y}^t  \end{array}\right]\boldsymbol{R}^t \right\|_F^2,
			\end{split}
		\]
		where the last inequality holds since 
		\[
			\eta \leqslant  \frac{\sigma_r(\boldsymbol{M})}{200\sigma_1^2(\boldsymbol{M})}.
		\]
		Therefore,
		\begin{equation}\label{eq_068}
		\begin{split}
			 \|\boldsymbol{A}_1\|_F \leqslant& (1-0.1\eta\sigma_r(\boldsymbol{M}))\left\| \left[ \begin{array}{c} \boldsymbol{X}^{t,(l)}\\\boldsymbol{Y}^{t,(l)} \end{array} \right]\boldsymbol{T}^{t,(l)} - \left[ \begin{array}{c} \boldsymbol{X}^t\\\boldsymbol{Y}^t  \end{array}\right]\boldsymbol{R}^t \right\|_F
			\end{split}
		\end{equation}
		holds on the event $E_{gd}^t$.

		For the second term $\left[\begin{array}{c}
			\boldsymbol{A}_2\\
			\boldsymbol{A}_3
		\end{array}\right]$ in \eqref{eq_067}, by the definition of $\mathcal{P}_{l,\cdot}$ and $\mathcal{P}_{\Omega_{l,\cdot}}$, we can see that entries of 
		\[
		\begin{split}
			&\mathcal{P}_{l,\cdot}\left( \boldsymbol{X}^{t,(l)}(\boldsymbol{Y}^{t,(l)})^\top - \boldsymbol{U}\boldsymbol{V}^\top \right)  - \frac{1}{p}\mathcal{P}_{\Omega_{l,\cdot}}\left( \boldsymbol{X}^{t,(l)}(\boldsymbol{Y}^{t,(l)})^\top - \boldsymbol{U}\boldsymbol{V}^\top \right)
		\end{split}
		\]
		are all zero except on the $l$-th row. Using this fact, we have
		\[ 
		\begin{split}
			  \boldsymbol{A}_2 =& -\left[ \begin{array}{c} \boldsymbol{0} \\ \vdots \\ \boldsymbol{0}\\ \sum_{j} (\frac{1}{p}\delta_{l,j}-1)\left( \boldsymbol{X}^{t,(l)}(\boldsymbol{Y}^{t,(l)})^\top - \boldsymbol{U}\boldsymbol{V}^\top \right)_{l,j} (\boldsymbol{Y}_{j,\cdot}^{t,(l)})^\top \\ \boldsymbol{0} \\\vdots \\ \boldsymbol{0}\end{array}\right] \boldsymbol{T}^{t,(l)}
		\end{split}
		\]
		and
		\[ 
		\begin{split}
			 \boldsymbol{A}_3 =& -\left[ \begin{array}{c} (\frac{1}{p}\delta_{l,1}-1)\left( \boldsymbol{X}^{t,(l)}(\boldsymbol{Y}^{t,(l)})^\top - \boldsymbol{U}\boldsymbol{V}^\top \right)_{l,1} (\boldsymbol{X}_{l,\cdot}^{t,(l)})^\top\\ 
			\vdots  \\ 
			(\frac{1}{p}\delta_{l,j}-1)\left( \boldsymbol{X}^{t,(l)}(\boldsymbol{Y}^{t,(l)})^\top - \boldsymbol{U}\boldsymbol{V}^\top \right)_{l,j} (\boldsymbol{X}_{l,\cdot}^{t,(l)})^\top \\\vdots\\
			(\frac{1}{p}\delta_{l,n_2}-1)\left( \boldsymbol{X}^{t,(l)}(\boldsymbol{Y}^{t,(l)})^\top - \boldsymbol{U}\boldsymbol{V}^\top \right)_{l,n_2} (\boldsymbol{X}_{l,\cdot}^{t,(l)})^\top\\
		\end{array}\right]  \boldsymbol{T}^{t,(l)}.
		\end{split}
		\]

		Therefore, by triangle inequality, 
		\begin{equation}\label{eq_020}
			\begin{split}
				& \left\| \left[ \begin{array}{c}
					\boldsymbol{A}_2\\
					\boldsymbol{A}_3
				\end{array} \right] \right\|_F \\
				\leqslant & \|\boldsymbol{A}_2\|_F + \|\boldsymbol{A}_3\|_F\\
				\leqslant & \left\| \underbrace{\sum_{j} (\frac{1}{p}\delta_{l,j}-1)\left( \boldsymbol{X}^{t,(l)}(\boldsymbol{Y}^{t,(l)})^\top - \boldsymbol{U}\boldsymbol{V}^\top \right)_{l,j} \boldsymbol{Y}_{j,\cdot}^{t,(l)}}_{\boldsymbol{b}_1} \right\|_2\\
		&+ \left\| \underbrace{\left[ \begin{array}{c} (\frac{1}{p}\delta_{l,1}-1)\left( \boldsymbol{X}^{t,(l)}(\boldsymbol{Y}^{t,(l)})^\top - \boldsymbol{U}\boldsymbol{V}^\top \right)_{l,1}  \\ 
			\vdots  \\ 
			(\frac{1}{p}\delta_{l,j}-1)\left( \boldsymbol{X}^{t,(l)}(\boldsymbol{Y}^{t,(l)})^\top - \boldsymbol{U}\boldsymbol{V}^\top \right)_{l,j}   \\\vdots\\
			(\frac{1}{p}\delta_{l,n_2}-1)\left( \boldsymbol{X}^{t,(l)}(\boldsymbol{Y}^{t,(l)})^\top - \boldsymbol{U}\boldsymbol{V}^\top \right)_{l,n_2}  \\
		\end{array}\right]}_{\boldsymbol{b}_2} \right\|_2 \| \boldsymbol{X}_{l,\cdot}^{t,(l)}\|_2,\\		
			\end{split}
		\end{equation}
		where the last inequality uses the fact that $\boldsymbol{T}^{t,(l)}\in\mathsf{O}(r)$.

		For $\boldsymbol{b}_1$, we can write $\boldsymbol{b}_1$ in the following form:
		\[
		\begin{split}
			\boldsymbol{b}_1 =& \sum_{j} (\frac{1}{p}\delta_{l,j}-1)\left( \boldsymbol{X}^{t,(l)}(\boldsymbol{Y}^{t,(l)})^\top - \boldsymbol{U}\boldsymbol{V}^\top \right)_{l,j} \boldsymbol{Y}_{j,\cdot}^{t,(l)}\\
			\coloneqq & \sum_j \boldsymbol{s}_{1,j}.
		\end{split}
		\]
		By the way we define $\boldsymbol{X}^{t,(l)}$ and $\boldsymbol{Y}^{t,(l)}$ in \eqref{eq:leaveoneout1}, \eqref{eq:leaveoneout2}, \eqref{eq:leaveoneout3} and \eqref{eq:leaveoneout4}, we can see that $\boldsymbol{X}^{t,(l)}$ and $\boldsymbol{Y}^{t,(l)}$ are independent of $\delta_{l,1},\cdots ,\delta_{l,n_2}$. Therefore, conditioned on $\boldsymbol{X}^{t,(l)}$ and $\boldsymbol{Y}^{t,(l)}$, $\boldsymbol{s}_{1,j}$'s are independent and $\mathbb{E}_{\delta_{l,\cdot}} \boldsymbol{s}_{1,j} = \boldsymbol{0}$. Moreover, since 
		\begin{equation}\label{eq_a003}
			\begin{split}
				&\boldsymbol{X}^{t,(l)}(\boldsymbol{Y}^{t,(l)})^\top - \boldsymbol{U}\boldsymbol{V}^\top\\
				 =&  \boldsymbol{X}^{t,(l)}\boldsymbol{T}^{t,(l)} (\boldsymbol{T}^{t,(l)} )^\top \boldsymbol{Y}^{t,(l)} - \boldsymbol{U}\boldsymbol{V}^\top\\
				=& (\boldsymbol{X}^{t,(l)}\boldsymbol{T}^{t,(l)} - \boldsymbol{U} )\boldsymbol{V}^\top + \boldsymbol{U}(\boldsymbol{Y}^{t,(l)}\boldsymbol{T}^{t,(l)} - \boldsymbol{V})^\top  + (\boldsymbol{X}^{t,(l)}\boldsymbol{T}^{t,(l)} - \boldsymbol{U} )(\boldsymbol{Y}^{t,(l)}\boldsymbol{T}^{t,(l)} - \boldsymbol{V})^\top.
			\end{split} 
		\end{equation} 
		Therefore, for all $j$,
		\[
			\begin{split}
				& \|\boldsymbol{s}_{1,j}\|_2 \\
				\leqslant & \frac{1}{p} \left\|  \boldsymbol{X}^{t,(l)}(\boldsymbol{Y}^{t,(l)})^\top - \boldsymbol{U}\boldsymbol{V}^\top \right\|_{\ell_{\infty}} \|\boldsymbol{Y}^{t,(l)}\|_{2,\infty}\\
				\leqslant & \frac{1}{p}\left( \| \boldsymbol{X}^{t,(l)}\boldsymbol{T}^{t,(l)} - \boldsymbol{U} \|_{2,\infty} \|\boldsymbol{V}\|_{2,\infty}  + \|\boldsymbol{Y}^{t,(l)}\boldsymbol{T}^{t,(l)} - \boldsymbol{V} \|_{2,\infty} \|\boldsymbol{U}\|_{2,\infty} +\| \boldsymbol{X}^{t,(l)}\boldsymbol{T}^{t,(l)} - \boldsymbol{U} \|_{2,\infty} \|\boldsymbol{Y}^{t,(l)}\boldsymbol{T}^{t,(l)} - \boldsymbol{V} \|_{2,\infty} \right)\\
				&  \times\|\boldsymbol{Y}^{t,(l)}\|_{2,\infty}\\
				 \coloneqq &L_1^{(l)}(\boldsymbol{X}^{t,(l)},\boldsymbol{Y}^{t,(l)})
			\end{split}
		\]
		holds. 
		By matrix Bernstein inequality \cite[Theorem 6.1.1]{tropp2015introduction}, we have 
		\[
		\begin{split}
			&\mathbb{P}\left[  \|\boldsymbol{b}_1\|_2 \geqslant 100\left(\sqrt{\mathbb{E}_{\delta_{l,\cdot}}\sum_j \|\boldsymbol{s}_{1,j}\|_2^2 \log (n_1\vee n_2)}  + L_1^{(l)}(\boldsymbol{X}^{t,(l)},\boldsymbol{Y}^{t,(l)}) \log (n_1\vee n_2)\right) \mid \boldsymbol{X}^{t,(l)},\boldsymbol{Y}^{t,(l)}\right] \\
			\leqslant& (n_1+n_2)^{-15}.
		\end{split}
		\] 
		Therefore, we have
		\[
			\begin{split}
				&\mathbb{P}\left[  \|\boldsymbol{b}_1\|_2 \geqslant 100\left(\sqrt{\mathbb{E}_{\delta_{l,\cdot}}\sum_j \|\boldsymbol{s}_{1,j}\|_2^2 \log (n_1\vee n_2)} + L_1^{(l)}(\boldsymbol{X}^{t,(l)},\boldsymbol{Y}^{t,(l)}) \log (n_1\vee n_2)\right) \right]\\
				= & \mathbb{E}  \left[\mathbb{E}\left[   \mathds{1}_{\|\boldsymbol{b}_1\|_2 \geqslant 100\left(\sqrt{\mathbb{E}_{\delta_{l,\cdot}}\sum_j \|\boldsymbol{s}_{1,j}\|_2^2 \log (n_1\vee n_2)}  + L_1^{(l)}(\boldsymbol{X}^{t,(l)},\boldsymbol{Y}^{t,(l)}) \log (n_1\vee n_2)\right) } \mid \boldsymbol{X}^{t,(l)},\boldsymbol{Y}^{t,(l)}\right]\right]\\
				\leqslant & (n_1+n_2)^{-15}.
			\end{split}
		\]
		In other words, on an event $E_{B}^{t,(l),1}$ with probability $\mathbb{P}[E_{B}^{t,(l),1}]\geqslant 1-(n_1+n_2)^{-15}$, 
		\begin{equation}\label{eq_071}
		\begin{split}
			\|\boldsymbol{b}_1\|_2 \leqslant& 100\left(\sqrt{\mathbb{E}_{\delta_{l,\cdot}}\sum_j \|\boldsymbol{s}_{1,j}\|_2^2 \log (n_1\vee n_2)}   L_1^{(l)}(\boldsymbol{X}^{t,(l)},\boldsymbol{Y}^{t,(l)}) \log (n_1\vee n_2)\right)
			\end{split}
		\end{equation}
		holds.

		On the event $E_{gd}^t$, if
		\[
			p\geqslant 111^2 C_I^2 \frac{\mu r \kappa^{11} \log (n_1\vee n_2)}{n_1\wedge n_2},
		\]
		from \eqref{eq_015}, we have
		\begin{equation}\label{eq_017}
		\begin{split}
			\left\| \left[ \begin{array}{c} \boldsymbol{X}^{t,(l)} \\ \boldsymbol{Y}^{t,(l)} \end{array} \right]\boldsymbol{T}^{t,(l)} - \left[ \begin{array}{c} \boldsymbol{U}\\\boldsymbol{V} \end{array}\right] \right\|_{2,\infty} \leqslant & \sqrt{\frac{\mu r \kappa}{n_1\wedge n_2}}\sqrt{\sigma_1(\boldsymbol{M})},\\
			 \left\| \left[ \begin{array}{c} \boldsymbol{X}^{t,(l)} \\ \boldsymbol{Y}^{t,(l)} \end{array} \right]  \right\|_{2,\infty} \leqslant & 2 \sqrt{\frac{\mu r \kappa}{n_1\wedge n_2}}\sqrt{\sigma_1(\boldsymbol{M})}.
		\end{split}
		\end{equation}
		Therefore, from \eqref{eq_015},
		\begin{equation}\label{eq_072}
			\begin{split}
		 		& L_1^{(l)}(\boldsymbol{X}^{t,(l)},\boldsymbol{Y}^{t,(l)})\\
				\leqslant& \frac{3}{p} \left\| \left[ \begin{array}{c} \boldsymbol{X}^{t,(l)} \\ \boldsymbol{Y}^{t,(l)} \end{array} \right]\boldsymbol{T}^{t,(l)} - \left[ \begin{array}{c} \boldsymbol{U}\\\boldsymbol{V} \end{array}\right] \right\|_{2,\infty} \sqrt{\frac{\mu r \kappa}{n_1\wedge n_2}}  \times \sqrt{\sigma_1(\boldsymbol{M})} \left\| \left[ \begin{array}{c} \boldsymbol{X}^{t,(l)} \\ \boldsymbol{Y}^{t,(l)} \end{array} \right]  \right\|_{2,\infty} \\
				\leqslant& \frac{333C_I}{p}\rho^t \sqrt{\frac{\mu^2 r^2 \kappa^{12}\log (n_1\vee n_2)}{(n_1\wedge n_2)^2p}}  \sqrt{\frac{\mu r \kappa}{n_1\wedge n_2}} \sigma_1(\boldsymbol{M})  \times 2\sqrt{\frac{\mu r \kappa}{n_1\wedge n_2}}\sqrt{\sigma_1(\boldsymbol{M})}\\
				\leqslant & 666C_I\rho^t \sqrt{\frac{\mu^4 r^4\kappa^{14} \log (n_1\vee n_2)}{(n_1\wedge n_2)^4 p^3}}\sqrt{\sigma_1(\boldsymbol{M})}^3.
			\end{split}
		\end{equation}
		Moreover, for $\mathbb{E}_{\delta_{l,\cdot}} \sum_j \|\boldsymbol{s}_{1,j}\|_2^2 $, we have
		\begin{equation}\label{eq_016}
			\begin{split}
				 \mathbb{E}_{\delta_{l,\cdot}} \sum_j \|\boldsymbol{s}_{1,j}\|_2^2  =& \mathbb{E}_{\delta_{l,\cdot}}  \sum_j (\frac{1}{p}\delta_{l,j} - 1)^2 \left( \boldsymbol{X}^{t,(l)}(\boldsymbol{Y}^{t,(l)})^\top - \boldsymbol{U}\boldsymbol{V}^\top \right)_{l,j}^2  \|\boldsymbol{Y}^{t,(l)}_{j,\cdot}\|_2^2\\
				\leqslant & \frac{1}{p}\|\boldsymbol{Y}^{t,(l)} \|_{2,\infty}^2 \left\|\left( \boldsymbol{X}^{t,(l)}(\boldsymbol{Y}^{t,(l)})^\top - \boldsymbol{U}\boldsymbol{V}^\top \right) _{l,\cdot}\right\|_2^2.
			\end{split}
		\end{equation}

		From \eqref{eq_ind2} and \eqref{eq_ind3},
		\begin{equation}\label{eq_053}
			\begin{split}
				& \left\| \left[ \begin{array}{c} \boldsymbol{X}^{t,(l)} \\ \boldsymbol{Y}^{t,(l)} \end{array} \right]\boldsymbol{T}^{t,(l)} - \left[ \begin{array}{c} \boldsymbol{U}\\\boldsymbol{V} \end{array}\right] \right\| \\
				 \leqslant & \left\|\left[ \begin{array}{c} \boldsymbol{X}^t \\\boldsymbol{Y}^t \end{array}\right]\boldsymbol{R}^t -  \left[ \begin{array}{c} \boldsymbol{X}^{t,(l)} \\ \boldsymbol{Y}^{t,(l)} \end{array} \right]\boldsymbol{T}^{t,(l)}  \right\|_F + \left\| \left[ \begin{array}{c} \boldsymbol{X}^t \\\boldsymbol{Y}^t \end{array}\right]\boldsymbol{R}^t  - \left[ \begin{array}{c} \boldsymbol{U}\\\boldsymbol{V} \end{array}\right] \right\| \\
				\leqslant &  C_I\rho^t \sqrt{\frac{\mu^2 r^2 \kappa^{10}\log (n_1\vee n_2)}{(n_1\wedge n_2)^2p}}\sqrt{\sigma_1(\boldsymbol{M})} +C_I\rho^t \sqrt{\frac{\mu r \kappa^6 \log (n_1\vee n_2)}{(n_1\wedge n_2)p}}\sqrt{\sigma_1(\boldsymbol{M})}	\\
			\leqslant & 2C_I\rho^t \sqrt{\frac{\mu r \kappa^{10}\log (n_1\vee n_2)}{(n_1\wedge n_2)p}}\sqrt{\sigma_1(\boldsymbol{M})},
			\end{split}
		\end{equation}
		where the last inequality holds since
		\[
			\frac{\mu r }{n_1\wedge n_2} \leqslant p\leqslant 1.	
		\]

		By triangle inequality, and recall the decomposition \eqref{eq_a003},
		\[
			\begin{split}
				&\left\|\left( \boldsymbol{X}^{t,(l)}(\boldsymbol{Y}^{t,(l)})^\top - \boldsymbol{U}\boldsymbol{V}^\top \right) _{l,\cdot}\right\|_2\\
				 \leqslant& \left\| \boldsymbol{U}_{l,\cdot}^\top (\boldsymbol{Y}^{t,(l)} \boldsymbol{T}^{t,(l)}-\boldsymbol{V} )^\top \right\|_2  + \left\| (\boldsymbol{X}^{t,(l)}\boldsymbol{T}^{t,(l)} - \boldsymbol{U})_{l,\cdot}^\top \boldsymbol{V}^\top \right\|_2  + \left\|  (\boldsymbol{X}^{t,(l)}\boldsymbol{T}^{t,(l)} - \boldsymbol{U})_{l,\cdot}^\top (\boldsymbol{Y}^{t,(l)} \boldsymbol{T}^{t,(l)}-\boldsymbol{V} )^\top \right\|_2\\
				\leqslant & \|\boldsymbol{U}\|_{2,\infty} \|\boldsymbol{Y}^{t,(l)} \boldsymbol{T}^{t,(l)}-\boldsymbol{V} \|  + \|\boldsymbol{X}^{t,(l)}\boldsymbol{T}^{t,(l)} - \boldsymbol{U}\|_{2,\infty} \|\boldsymbol{V}\| + \|\boldsymbol{X}^{t,(l)}\boldsymbol{T}^{t,(l)} - \boldsymbol{U}\|_{2,\infty} \|\boldsymbol{Y}^{t,(l)} \boldsymbol{T}^{t,(l)}-\boldsymbol{V}\|.
			\end{split}
		\] 
		Combining with \eqref{eq_015} and \eqref{eq_053} we have
		\begin{equation}\label{eq_018}
			\begin{split}
				&\left\|\left( \boldsymbol{X}^{t,(l)}(\boldsymbol{Y}^{t,(l)})^\top - \boldsymbol{U}\boldsymbol{V}^\top \right) _{l,\cdot}\right\|_2\\  \leqslant& 2\sqrt{\frac{\mu r \kappa }{n_1\wedge n_2}}\sqrt{\sigma_1(\boldsymbol{M})}   C_I \rho^t  \times\sqrt{\frac{\mu r \kappa^{10}\log (n_1\vee n_2)}{(n_1\wedge n_2)p}}\sqrt{\sigma_1(\boldsymbol{M})} + 111C_I \rho^t \sqrt{\frac{\mu^2 r^2 \kappa^{12}\log(n_1\vee n_2)}{(n_1\wedge n_2)^2p}}\sigma_1(\boldsymbol{M})\\
				& + 111C_I \rho^t \sqrt{\frac{\mu^2 r^2 \kappa^{12}\log(n_1\vee n_2)}{(n_1\wedge n_2)^2p}}   \times 2 C_I \rho^t \sqrt{\frac{\mu r \kappa^{10}\log (n_1\vee n_2)}{(n_1\wedge n_2)p}}\sigma_1(\boldsymbol{M}) \\
				\leqslant & 115 C_I \rho^t \sqrt{\frac{\mu^2 r^2 \kappa^{12}\log (n_1\vee n_2)}{(n_1\wedge n_2)^2 p}}\sigma_1(\boldsymbol{M}),
			\end{split}
		\end{equation} 
		where the last inequality use the fact that 
		\[
			p \geqslant 111^2 C_I^2 \frac{\mu r \kappa^{10} \log (n_1\vee n_2)}{n_1\wedge n_2}
		\]
		and $\rho<1$.

		Putting \eqref{eq_017}, \eqref{eq_016} and \eqref{eq_018} together we have
		\begin{equation}\label{eq_073}
		\begin{split}
			 \mathbb{E}_{\delta_{l,\cdot}}\sum_j \|\boldsymbol{s}_{1,j}\|_2^2  \leqslant & 230^2 C_I^2 \rho^{2t}\frac{\mu^3 r^3\kappa^{13}\log (n_1\vee n_2)}{(n_1\wedge n_2)^3p^2}\sigma_1^3(\boldsymbol{M}).
		\end{split}
		\end{equation}
		So by \eqref{eq_071}, \eqref{eq_072} and \eqref{eq_073}, on the event $E_{B}^{t,(l),1}\bigcap E_{gd}^t$, we have 
		\begin{equation}\label{eq_074} 
			\begin{split}
				&\|\boldsymbol{b}_1\|_2\\
				 \leqslant & 100 \rho^t \left( 230C_I \sqrt{\frac{\mu^3 r^3 \kappa^{13} \log^2(n_1\vee n_2)}{(n_1\wedge n_2)^3 p^2}}  + 666C_I\sqrt{\frac{\mu^4 r^4 \kappa^{14}\log (n_1\vee n_2)}{(n_1\wedge n_2)^4 p^3}}\log (n_1\vee n_2) \right)\sqrt{\sigma_1(\boldsymbol{M})}^3\\
				= & 100C_I\rho^t \sqrt{\frac{\mu^2 r^2 \kappa^{10} \log (n_1\vee n_2)}{(n_1\wedge n_2)^2 p}} \sqrt{\sigma_1(\boldsymbol{M})}\sigma_r(\boldsymbol{M}) \kappa\\
				&\times \left( 230 \sqrt{\frac{\mu r \kappa^3 \log (n_1\vee n_2)}{(n_1\wedge n_2)p}} + 666 \sqrt{\frac{\mu^2 r^2 \kappa^4 \log^2 (n_1\vee n_2)}{(n_1\wedge n_2)^2p^2}} \right)\\
				\leqslant & 0.025\sigma_r(\boldsymbol{M}) C_I\rho^t \sqrt{\frac{\mu^2 r^2 \kappa^{10} \log (n_1\vee n_2)}{(n_1\wedge n_2)^2 p}} \sqrt{\sigma_1(\boldsymbol{M})}, 
			\end{split}
		\end{equation}
		where the last inequality holds since 
		\[
			p \geqslant 3.3856\times 10^{12} \frac{\mu r \kappa^5 \log (n_1\vee n_2)}{n_1\wedge n_2}.
		\]
		
		For $\boldsymbol{b}_2$ defined in \eqref{eq_020}, we can use almost the same argument. We can write $\boldsymbol{b}_2$ as
		\[
		\begin{split}
			\boldsymbol{b}_2 =& \sum_j \boldsymbol{e}_j (\frac{1}{p}\delta_{l,j}-1)\left( \boldsymbol{X}^{t,(l)}(\boldsymbol{Y}^{t,(l)})^\top - \boldsymbol{U}\boldsymbol{V}^\top \right)_{l,j} \\
			 \coloneqq & \sum_j \boldsymbol{s}_{2,j}.
		\end{split}
		\]
		By the definition of $\boldsymbol{X}^{t,(l)}$ and $\boldsymbol{Y}^{t,(l)}$, we can see that $\boldsymbol{X}^{t,(l)}$ and $\boldsymbol{Y}^{t,(l)}$ are independent of $\delta_{l,1},\cdots ,\delta_{l,n_2}$. Therefore, conditioned on $\boldsymbol{X}^{t,(l)}$ and $\boldsymbol{Y}^{t,(l)}$, $\boldsymbol{s}_{2,j}$'s are independent and $\mathbb{E}_{\delta_{l,\cdot}} \boldsymbol{s}_{2,j} = \boldsymbol{0}$. Note for all $j$, 
		\begin{equation}\label{eq_b001}
			\begin{split}
				& \|\boldsymbol{s}_{2,j}\|_2\\
				 \leqslant & \frac{1}{p} \left\|  \boldsymbol{X}^{t,(l)}(\boldsymbol{Y}^{t,(l)})^\top - \boldsymbol{U}\boldsymbol{V}^\top \right\|_{\ell_{\infty}} \\
				\leqslant & \frac{1}{p}\left( \| \boldsymbol{X}^{t,(l)}\boldsymbol{T}^{t,(l)} - \boldsymbol{U} \|_{2,\infty} \|\boldsymbol{V}\|_{2,\infty}  + \|\boldsymbol{Y}^{t,(l)}\boldsymbol{T}^{t,(l)} - \boldsymbol{V} \|_{2,\infty} \|\boldsymbol{U}\|_{2,\infty} +\| \boldsymbol{X}^{t,(l)}\boldsymbol{T}^{t,(l)} - \boldsymbol{U} \|_{2,\infty} \|\boldsymbol{Y}^{t,(l)}\boldsymbol{T}^{t,(l)} - \boldsymbol{V} \|_{2,\infty} \right)\\
				\coloneqq & L_2^{(l)}(\boldsymbol{X}^{t,(l)},\boldsymbol{Y}^{t,(l)}).
			\end{split}
		\end{equation}
%
		By matrix Bernstein inequality \cite[Theorem 6.1.1]{tropp2015introduction}, we have
		\[
		\begin{split}
			& \mathbb{P}\left[  \|\boldsymbol{b}_2\|_2 \geqslant 100\left(\sqrt{\mathbb{E}_{\delta_{l,\cdot}}\sum_j \|\boldsymbol{s}_{2,j}\|_2^2 \log (n_1\vee n_2)}  + L_2^{(l)}(\boldsymbol{X}^{t,(l)},\boldsymbol{Y}^{t,(l)}) \log (n_1\vee n_2)\right) \mid \boldsymbol{X}^{t,(l)},\boldsymbol{Y}^{t,(l)}\right] \\
			\leqslant& (n_1+n_2)^{-15}.
		\end{split}
		\]

		Using the same argument in $\boldsymbol{b}_1$, we have that on an event $E_{B}^{t,(l),2}$ with probability $\mathbb{P}[E_{B}^{t,(l),2}]\geqslant 1-(n_1+n_2)^{-15}$, 
		\begin{equation}\label{eq_055}
		\begin{split}
			\|\boldsymbol{b}_2\|_2 \leqslant& 100\left(\sqrt{\mathbb{E}_{\delta_{l,\cdot}}\sum_j \|\boldsymbol{s}_{2,j}\|_2^2 \log (n_1\vee n_2)}  + L_2^{(l)}(\boldsymbol{X}^{t,(l)},\boldsymbol{Y}^{t,(l)}) \log (n_1\vee n_2)\right)
			\end{split}
		\end{equation}
		holds. Note on the event $E_{B}^{t,(l),2}\bigcap E_{gd}^t$, the estimation of $\|\boldsymbol{s}_{2,j}\|$ and $\mathbb{E}_{\delta_{l,\cdot}}\sum_j \|\boldsymbol{s}_{2,j}\|_2^2$ are in the same fashion with the one we did on $\boldsymbol{s}_{1,j}$: 
		On the event $E_{gd}^t$, from \eqref{eq_015}, \eqref{eq_017} and \eqref{eq_b001},
		\[
			\begin{split}
				  L_2^{(l)}(\boldsymbol{X}^{t,(l)},\boldsymbol{Y}^{t,(l)}) \leqslant &  \frac{3}{p}\left\| \left[ \begin{array}{c} \boldsymbol{X}^{t,(l)} \\ \boldsymbol{Y}^{t,(l)} \end{array} \right]\boldsymbol{T}^{t,(l)} - \left[ \begin{array}{c} \boldsymbol{U}\\\boldsymbol{V} \end{array}\right] \right\|_{2,\infty} \sqrt{\frac{\mu r \kappa}{n_1\wedge n_2}}\sqrt{\sigma_1(\boldsymbol{M})}\\
				\leqslant & \frac{1}{p}333 C_I \rho^t \sqrt{\frac{\mu^2 r^2 \kappa^{12}\log(n_1\vee n_2)}{(n_1\wedge n_2)^2p}}\sqrt{\sigma_1(\boldsymbol{M})}   \times\sqrt{\frac{\mu r \kappa}{n_1\wedge n_2}}\sqrt{\sigma_1(\boldsymbol{M})}\\
				= & 333 C_I \rho^t \sqrt{\frac{\mu^3 r^3 \kappa^{13}\log(n_1\vee n_2)}{(n_1\wedge n_2)^3p^3}}\sigma_1(\boldsymbol{M}).
 			\end{split}	
		\]
		At the same time,
		\[
			 \begin{split}
				  \mathbb{E}_{\delta_{l,\cdot}} \sum_{j}\|\boldsymbol{s}_{2,j}\|_2^2 =& \mathbb{E}_{\delta_{l,\cdot}} \sum_{j} (\frac{1}{p}\delta_{l,j}-1)^2\left( \boldsymbol{X}^{t,(l)}(\boldsymbol{Y}^{t,(l)})^\top - \boldsymbol{U}\boldsymbol{V}^\top \right)_{l,j}^2\\
				 \leqslant & \frac{1}{p} \left\|\left( \boldsymbol{X}^{t,(l)}(\boldsymbol{Y}^{t,(l)})^\top - \boldsymbol{U}\boldsymbol{V}^\top \right) _{l,\cdot}\right\|_2^2\\
				 \leqslant & 115^2 C_I^2 \rho^{2t}\frac{\mu^2 r^2\kappa^{12}\log(n_1\vee n_2)}{(n_1\wedge n_2)^2p^2}\sigma_1^2(\boldsymbol{M}),
			 \end{split}
		\]
		where the last inequality follows from \eqref{eq_018}. Therefore, on the event $E_{gd}^t \bigcap E_B^{t,(l),2}$,
		\begin{equation}\label{eq_064}
		 \begin{split}
				 & \|\boldsymbol{A}_3\|_F \\
				 =& \|\boldsymbol{b}_2\|_2 \|\boldsymbol{X}_{l,\cdot}^{t,(l)}\|_2\\
				 \leqslant & 100\left( 115C_I\rho^t \sqrt{\frac{\mu^2 r^2 \kappa^{12}\log^2(n_1\vee n_2)}{(n_1\wedge n_2)^2 p^2}} \sigma_1(\boldsymbol{M})   + 333 C_I\rho^t \sqrt{\frac{\mu^3 r^3 \kappa^{13}\log(n_1\vee n_2)}{(n_1\wedge n_2)^3p^3}}\sigma_1(\boldsymbol{M})\log(n_1\vee n_2) \right)\\
				 &\times 2\sqrt{\frac{\mu r\kappa}{n_1\wedge n_2}}\sqrt{\sigma_1(\boldsymbol{M})}\\
				 = & 100C_I\rho^t \sqrt{\frac{\mu^2 r^2 \kappa^{10} \log (n_1\vee n_2)}{(n_1\wedge n_2)^2 p}} \sqrt{\sigma_1(\boldsymbol{M})}\sigma_r(\boldsymbol{M}) \kappa \\
				 &\times \left( 230 \sqrt{\frac{\mu r \kappa^3 \log (n_1\vee n_2)}{(n_1\wedge n_2)p}} + 666 \sqrt{\frac{\mu^2 r^2 \kappa^4 \log^2 (n_1\vee n_2)}{(n_1\wedge n_2)^2p^2}} \right)\\
				 \leqslant & 0.025\sigma_r(\boldsymbol{M}) C_I\rho^t \sqrt{\frac{\mu^2 r^2 \kappa^{10} \log (n_1\vee n_2)}{(n_1\wedge n_2)^2 p}} \sqrt{\sigma_1(\boldsymbol{M})}, 
			 \end{split}
			\end{equation}
		where the second inequality uses \eqref{eq_017} and the last inequality holds since 
		\[
			p \geqslant 3.3856\times 10^{12} \frac{\mu r \kappa^5 \log (n_1\vee n_2)}{n_1\wedge n_2}.
		\]

		So in summary by \eqref{eq_020}, \eqref{eq_074} and \eqref{eq_064}, on the event $E_{B}^{t,(l),1}\bigcap E_{B}^{t,(l),2}\bigcap E_{gd}^t$ we have 
		\begin{equation}\label{eq_075}
		\begin{split}
			 \left\| \left[ \begin{array}{c}
				\boldsymbol{A}_2\\
				\boldsymbol{A}_3
			\end{array} \right] \right\|_F  \leqslant &0.05 \sigma_r(\boldsymbol{M})C_I \rho^t \sqrt{\frac{\mu^2 r^2 \kappa^{10} \log (n_1\vee n_2)}{(n_1\wedge n_2)^2 p}} \sqrt{\sigma_1(\boldsymbol{M})} .
		\end{split}
		\end{equation}
		Combining the estimations \eqref{eq_068} and \eqref{eq_075} for $\boldsymbol{A}_1$, $\boldsymbol{A}_2$ and $\boldsymbol{A}_3$ together, and using \eqref{eq_067}, we can see that on the event $E_{B}^{t,(l),1}\bigcap E_{B}^{t,(l),2}\bigcap E_{gd}^t$,
		\[
			\begin{split}
			& \left\| \left[ \begin{array}{c} \boldsymbol{X}^{t+1} \\\boldsymbol{Y}^{t+1} \end{array}\right]\boldsymbol{R}^{t} - \left[\begin{array}{c} \boldsymbol{X}^{t+1,(l)} \\ \boldsymbol{Y}^{t+1,(l)} \end{array}\right] \boldsymbol{T}^{t,(l)} \right\|_F \\\leqslant & \|\boldsymbol{A}_1\|_F +\eta \left\| \left[ \begin{array}{c}
				\boldsymbol{A}_2\\
				\boldsymbol{A}_3
			\end{array} \right] \right\|_F \\
			 \leqslant &  (1-0.1\eta\sigma_r(\boldsymbol{M}))\left\| \left[ \begin{array}{c} \boldsymbol{X}^t\\\boldsymbol{Y}^t  \end{array}\right]\boldsymbol{R}^t -\left[ \begin{array}{c} \boldsymbol{X}^{t,(l)}\\\boldsymbol{Y}^{t,(l)} \end{array} \right]\boldsymbol{T}^{t,(l)}\right\|_F +0.05\eta\sigma_r(\boldsymbol{M}) C_I \rho^t \sqrt{\frac{\mu^2 r^2 \kappa^{10} \log (n_1\vee n_2)}{(n_1\wedge n_2)^2 p}} \sqrt{\sigma_1(\boldsymbol{M})} \\
			 \leqslant & C_I \rho^{t+1} \sqrt{\frac{\mu^2 r^2 \kappa^{10} \log (n_1\vee n_2)}{(n_1\wedge n_2)^2 p}} \sqrt{\sigma_1(\boldsymbol{M})}
			\end{split}
		\]
		holds for $\rho = 1-0.05\eta \sigma_r(\boldsymbol{M})$ and fixed $l$ satisfying $1\leqslant l \leqslant n_1$, and the last inequality uses \eqref{eq_ind3}. The proof is all the same for $l$ satisfying $n_1+1\leqslant l\leqslant n_1+n_2$. Let $E_{gd}^{t+1} = E_{gd}^t \bigcap \left( \bigcap_{l = 1}^{n_1+n_2} E_B^{t,(l),1} \right)\bigcap \left( \bigcap_{l = 1}^{n_1+n_2} E_B^{t,(l),2} \right)$, so $E_{gd}^{t+1}\subset E_{gd}^t$, and from union bound, we have $\mathbb{P}[E_{gd}^t \backslash E_{gd}^{t+1}]\leqslant (n_1+n_2)^{-10}$.

\subsection{Proof of \eqref{eq_ind5}}\label{sec_proof_eq_ind5}

Finally, we want to show that \eqref{eq_ind5} can be directly implied by \eqref{eq_ind2}, \eqref{eq_ind4} and \eqref{eq_ind3}. First, for any $l$ satisfies $1\leqslant l\leqslant n_1+n_2$,
\begin{equation}\label{eq_008}
	\begin{split}
		&\left\| \left( \left[ \begin{array}{c} \boldsymbol{X}^t \\\boldsymbol{Y}^t\end{array}\right]\boldsymbol{R}^t -\left[ \begin{array}{c} \boldsymbol{U}\\\boldsymbol{V} \end{array} \right] \right)_{l,\cdot} \right\|_2 
		\\
		\leqslant & \left\| \left( \left[\begin{array}{c} \boldsymbol{X}^t \\ \boldsymbol{Y}^t\end{array}\right]\boldsymbol{R}^t -\left[\begin{array}{c} \boldsymbol{X}^{t,(l)} \\ \boldsymbol{Y}^{t,(l)} \end{array}\right]\boldsymbol{R}^{t,(l)} \right)_{l,\cdot} \right\|_2  + \left\| \left( \left[\begin{array}{c} \boldsymbol{X}^{t,(l)} \\ \boldsymbol{Y}^{t,(l)} \end{array}\right]\boldsymbol{R}^{t,(l)} - \left[ \begin{array}{c} \boldsymbol{U}\\\boldsymbol{V}\end{array} \right] \right)_{l,\cdot} \right\|_2\\
		\leqslant & \left\|  \left[\begin{array}{c} \boldsymbol{X}^t \\ \boldsymbol{Y}^t\end{array}\right]\boldsymbol{R}^t -\left[\begin{array}{c} \boldsymbol{X}^{t,(l)} \\ \boldsymbol{Y}^{t,(l)} \end{array}\right]\boldsymbol{R}^{t,(l)}  \right\|_F + \left\| \left( \left[\begin{array}{c} \boldsymbol{X}^{t,(l)} \\ \boldsymbol{Y}^{t,(l)} \end{array}\right]\boldsymbol{R}^{t,(l)} - \left[ \begin{array}{c} \boldsymbol{U}\\\boldsymbol{V}\end{array} \right] \right)_{l,\cdot} \right\|_2.
	\end{split}
\end{equation}
The second term of the last line is already controlled by \eqref{eq_ind4}, so our main goal is to control the first term. In order to do so, we want to apply Lemma \ref{ma_lemma37} with 
\[
	\boldsymbol{X}_0 \coloneqq \left[ \begin{array}{c}
		\boldsymbol{U}\\
		\boldsymbol{V}
	\end{array} \right],\;	\boldsymbol{X}_1 \coloneqq \left[ \begin{array}{c} \boldsymbol{X}^t \\\boldsymbol{Y}^t\end{array}\right]\boldsymbol{R}^t,\; \boldsymbol{X}_2 \coloneqq \left[ \begin{array}{c} \boldsymbol{X}^{t,(l)}\\ \boldsymbol{Y}^{t,(l)}\end{array} \right]\boldsymbol{T}^{t,(l)}.
\]
Note by the definition of $\boldsymbol{U}$ and $\boldsymbol{V}$, we have $\sigma_1(\boldsymbol{X}_0) = \sqrt{2\sigma_1(\boldsymbol{M})}$, $\sigma_2(\boldsymbol{X}_0) = \sqrt{2\sigma_2(\boldsymbol{M})}$, $\cdots$, $\sigma_r(\boldsymbol{X}_0) = \sqrt{2\sigma_r(\boldsymbol{M})}$, and $\sigma_1(\boldsymbol{X}_0)/\sigma_r(\boldsymbol{X}_0) = \sqrt{\kappa}$. In order to apply the lemma, note from \eqref{eq_ind2} we have
	\[
	\begin{split}
		 \left\| \left[ \begin{array}{c} \boldsymbol{X}^t\\ \boldsymbol{Y}^t \end{array}\right] \boldsymbol{R}^t - \left[ \begin{array}{c} \boldsymbol{U}\\\boldsymbol{V}\end{array} \right]\right\| \left\| \left[ \begin{array}{c} \boldsymbol{U}\\\boldsymbol{V} \end{array} \right] \right\| \leqslant& 2 C_I \rho^t\sqrt{\frac{\mu r \kappa^6\log (n_1\vee n_2)}{(n_1\wedge n_2)p}}\sigma_1(\boldsymbol{M}).
	\end{split}
	\]
	And as long as 
	\[
		p\geqslant 16C_I^2 \frac{\mu r \kappa^8 \log (n_1\vee n_2)}{n_1\wedge n_2},
	\]
	we have 
	\[
		\left\| \left[ \begin{array}{c} \boldsymbol{X}^t\\ \boldsymbol{Y}^t \end{array}\right] \boldsymbol{R}^t - \left[ \begin{array}{c} \boldsymbol{U}\\\boldsymbol{V}\end{array} \right]\right\| \left\| \left[ \begin{array}{c} \boldsymbol{U}\\\boldsymbol{V} \end{array} \right] \right\| \leqslant \frac{1}{2}\sigma_r(\boldsymbol{M})\leqslant \frac{1}{2}\sigma_r^2(\boldsymbol{X}_0). 
	\]
		And also we have
	\[
		\begin{split}
			 \left\| \left[ \begin{array}{c} \boldsymbol{X}^t \\ \boldsymbol{Y}^t\end{array} \right]\boldsymbol{R}^t - \left[ \begin{array}{c} \boldsymbol{X}^{t,(l)}\\ \boldsymbol{Y}^{t,(l)}\end{array} \right]\boldsymbol{T}^{t,(l)} \right\|\left\| \left[ \begin{array}{c} \boldsymbol{U}\\\boldsymbol{V} \end{array}\right] \right\| \leqslant & \left\| \left[ \begin{array}{c} \boldsymbol{X}^t \\ \boldsymbol{Y}^t\end{array} \right]\boldsymbol{R}^t - \left[ \begin{array}{c} \boldsymbol{X}^{t,(l)}\\ \boldsymbol{Y}^{t,(l)}\end{array} \right]\boldsymbol{T}^{t,(l)} \right\|_F\left\| \left[ \begin{array}{c} \boldsymbol{U}\\\boldsymbol{V} \end{array}\right] \right\|\\
			\leqslant & 2C_I\rho^t\sqrt{\frac{\mu^2 r^2\kappa^{10}\log (n_1\vee n_2)}{(n_1\wedge n_2)^2p}}\sigma_1(\boldsymbol{M})\\
			\leqslant & \frac{1}{4}\sigma_r(\boldsymbol{M})\\
			\leqslant & \frac{1}{4}\sigma_r^2(\boldsymbol{X}_0).
		\end{split}
	\]
	Here second inequality we use \eqref{eq_ind3} and third inequality holds because we have
	\[
		p\geqslant 64C_I^2\frac{\mu^2 r^2\kappa^{12}\log (n_1\vee n_2)}{n_1\wedge n_2}.
	\]
	Now by applying Lemma \ref{ma_lemma37} we have
	\begin{equation}\label{eq_007}
		\begin{split}
			 \left\| \left[ \begin{array}{c} \boldsymbol{X}^t \\ \boldsymbol{Y}^t\end{array} \right]\boldsymbol{R}^t -\left[ \begin{array}{c} \boldsymbol{X}^{t,(l)}\\ \boldsymbol{Y}^{t,(l)}\end{array} \right]\boldsymbol{R}^{t,(l)} \right\|_F  \leqslant & 5\kappa \left\| \left[ \begin{array}{c} \boldsymbol{X}^t \\ \boldsymbol{Y}^t\end{array} \right]\boldsymbol{R}^t -\left[ \begin{array}{c} \boldsymbol{X}^{t,(l)}\\ \boldsymbol{Y}^{t,(l)}\end{array} \right]\boldsymbol{T}^{t,(l)} \right\|_F \\
			\leqslant & 10 C_I\rho^t \kappa \sqrt{\frac{\mu^2r^2\kappa^{10}\log (n_1\vee n_2)}{(n_1\wedge n_2)^2 p}}\sqrt{\sigma_1(\boldsymbol{M})}.
		\end{split}
	\end{equation}
	Plugging \eqref{eq_ind4} and \eqref{eq_007} into \eqref{eq_008} we have \eqref{eq_ind5}.

	Finally letting 
	\[ 
	\begin{split}
		C_{S3} =& 3.3856\times 10^{12}+ 6600C_I +32400C_I^2  +C_3+ 333^2C_3^2  + 2.42 \times 10^{10} C_0C_I^2 + C_{S1}
	\end{split}	
	\]		 
	finishes the whole proof of Lemma \ref{lemma_induction}.
 
\section{Discussion}
In this paper we study the convergence of vanilla gradient descent for the purpose of nonconvex matrix completion with spectral initialization. Our result can be viewed as the theoretical justification for the numerical phenomenon identified in \cite{zheng2016convergence} that vanilla gradient descent for the nonconvex optimization \eqref{eq:mc_noncvx} without $l_{2, \infty}$-regularization is effective and efficient to yield the exact low-rank recovery based upon a few observations. On the other hand, our work extends the result in \cite{ma2017implicit} from the completion of positive semidefinite matrices to that of rectangular matrices. Furthermore, we improve the required sampling rates in \cite{ma2017implicit} by combining the leave-one-out technique therein and a series of powerful results in the past literature of matrix completion, such as some key lemmas in \cite{chen2015incoherence, bhojanapalli2014universal, li2016recovery, chen2017memory, zheng2016convergence}.

Our technical contributions can be potentially used in other problems where the leave-one-out techniques are useful or have been applied. For example, we have mentioned that the leave-one-out analysis has been employed in \cite{abbe2017entrywise} in exact spectral clustering in community detection without cleaning or regularization, while the technical contributions in our paper is potentially useful in improving their theoretical results particularly in the case that the number of clusters is allowed to grow with the number of nodes. Moreover, our technique is also potentially useful to sharpen the leave-one-out analysis in \cite{ding2018leave} for matrix completion by Singular Value Projection and therefore improve their sampling rates results. 

\section*{Acknowledgements}
D. Liu would gratefully acknowledge the financial support from China Scholarship Council.

\bibliographystyle{plainnat}
\bibliography{cite}

\appendix

\section{Proof of Lemma \ref{lemma_hessian}}
\label{sec:hess_proof}

For the proof, we mainly follow the technical framework introduced by \cite{ma2017implicit} and extend their result to the rectangular case. Within the proof, we employ Lemma 4.4 from \cite{chen2017memory} (Lemma \ref{chen_lemma4.4} in this paper) as well as Lemma 9 from \cite{zheng2016convergence} (Lemma \ref{zheng_lemma9} in this paper) to simplify the proof, and get a weaker assumption \eqref{lemma_hessian1} (in this paper) comparing to equation (63a) in \citet[Lemma 7]{ma2017implicit} by a factor of $\log (n_1\vee n_2)$.

\begin{proof}
	For the Hessian, we can compute as \cite{ge2016matrix,ge2017no,zhu2017global} did and have
	\[
		\begin{split}
			& \operatorname{vec}\left( \left[ \begin{array}{c}\boldsymbol{D}_{\boldsymbol{X}}\\ \boldsymbol{D}_{\boldsymbol{Y}} \end{array}\right] \right)^\top \nabla^2f(\boldsymbol{X},\boldsymbol{Y})\operatorname{vec}\left( \left[ \begin{array}{c}\boldsymbol{D}_{\boldsymbol{X}}\\ \boldsymbol{D}_{\boldsymbol{Y}} \end{array}\right] \right)  \\
		= & \frac{2}{p}\langle \mathcal{P}_{\Omega}( \boldsymbol{X}\boldsymbol{Y}^\top - \boldsymbol{U}\boldsymbol{V}^\top  ),\mathcal{P}_{\Omega} ( \boldsymbol{D}_{\boldsymbol{X}}\boldsymbol{D}_{\boldsymbol{Y}}^\top  ) \rangle + \frac{1}{p}\left\| \mathcal{P}_{\Omega}(\boldsymbol{D}_{\boldsymbol{X}}\boldsymbol{Y}^\top + \boldsymbol{X}\boldsymbol{D}_{\boldsymbol{Y}}^\top) \right\|_F^2\\
		 &+ \frac{1}{2}\langle \boldsymbol{X}^\top\boldsymbol{X}-\boldsymbol{Y}^\top\boldsymbol{Y},\boldsymbol{D}_{\boldsymbol{X}}^\top \boldsymbol{D}_{\boldsymbol{X}} - \boldsymbol{D}_{\boldsymbol{Y}}^\top\boldsymbol{D}_{\boldsymbol{Y}}\rangle  + \frac{1}{4} \left\| \boldsymbol{D}_{\boldsymbol{X}}^\top\boldsymbol{X}+\boldsymbol{X}^\top\boldsymbol{D}_{\boldsymbol{X}} - \boldsymbol{Y}^\top\boldsymbol{D}_{\boldsymbol{Y}} - \boldsymbol{D}_{\boldsymbol{Y}}^\top\boldsymbol{Y} \right\|_F^2.
		\end{split}
	\]
	
First we consider the population level, i.e., $\mathbb{E}\left[ \operatorname{vec}\left( \left[ \begin{array}{c}\boldsymbol{D}_{\boldsymbol{X}}\\ \boldsymbol{D}_{\boldsymbol{Y}} \end{array}\right] \right)^\top \nabla^2f(\boldsymbol{X},\boldsymbol{Y})\operatorname{vec}\left( \left[ \begin{array}{c}\boldsymbol{D}_{\boldsymbol{X}}\\ \boldsymbol{D}_{\boldsymbol{Y}} \end{array}\right] \right) \right] $. Denoting $\boldsymbol{\Delta}_{\boldsymbol{X}} \coloneqq \boldsymbol{X}-\boldsymbol{U}, \boldsymbol{\Delta}_{\boldsymbol{Y}} \coloneqq \boldsymbol{Y} - \boldsymbol{V}$, and using similar decomposition as in \eqref{eq_a001} and \eqref{eq_a002}, we have
\begin{equation}\label{eq_b005}
	\begin{split}
		& \mathbb{E}\left[ \operatorname{vec}\left( \left[ \begin{array}{c}\boldsymbol{D}_{\boldsymbol{X}}\\ \boldsymbol{D}_{\boldsymbol{Y}} \end{array}\right] \right)^\top \nabla^2f(\boldsymbol{X},\boldsymbol{Y})\operatorname{vec}\left( \left[ \begin{array}{c}\boldsymbol{D}_{\boldsymbol{X}}\\ \boldsymbol{D}_{\boldsymbol{Y}} \end{array}\right] \right) \right] \\
		= & 2\langle \boldsymbol{\Delta}_{\boldsymbol{X}}\boldsymbol{V}^\top + \boldsymbol{U}\boldsymbol{\Delta}_{\boldsymbol{Y}}^\top+ \boldsymbol{\Delta}_{\boldsymbol{X}}\boldsymbol{\Delta}_{\boldsymbol{Y}}^\top, \boldsymbol{D}_{\boldsymbol{X}}\boldsymbol{D}_{\boldsymbol{Y}}^\top \rangle + \left\| \boldsymbol{D}_{\boldsymbol{X}}\boldsymbol{V}^\top + \boldsymbol{D}_{\boldsymbol{X}}\boldsymbol{\Delta}_{\boldsymbol{Y}}^\top + \boldsymbol{U}\boldsymbol{D}_{\boldsymbol{Y}}^\top+\boldsymbol{\Delta}_{\boldsymbol{X}}\boldsymbol{D}_{\boldsymbol{Y}}^\top \right\|_F^2\\
		& + \frac{1}{2} \langle \boldsymbol{U}^\top\boldsymbol{\Delta}_{\boldsymbol{X}} + \boldsymbol{\Delta}_{\boldsymbol{X}}^\top\boldsymbol{U} + \boldsymbol{\Delta}_{\boldsymbol{X}}^\top\boldsymbol{\Delta}_{\boldsymbol{X}} - \boldsymbol{\Delta}_{\boldsymbol{Y}}^\top\boldsymbol{V}   - \boldsymbol{V}^\top\boldsymbol{\Delta}_{\boldsymbol{Y}}-  \boldsymbol{\Delta}_{\boldsymbol{Y}}^\top\boldsymbol{\Delta}_{\boldsymbol{Y}}, \boldsymbol{D}_{\boldsymbol{X}}^\top\boldsymbol{D}_{\boldsymbol{X}}-\boldsymbol{D}_{\boldsymbol{Y}}^\top\boldsymbol{D}_{\boldsymbol{Y}} \rangle\\
		& + \frac{1}{4}\left\| \boldsymbol{D}_{\boldsymbol{X}}^\top\boldsymbol{U}+\boldsymbol{D}_{\boldsymbol{X}}^\top\boldsymbol{\Delta}_{\boldsymbol{X}} + \boldsymbol{U}^\top\boldsymbol{D}_{\boldsymbol{X}}+\boldsymbol{\Delta}_{\boldsymbol{X}}^\top\boldsymbol{D}_{\boldsymbol{X}} - \boldsymbol{V}^\top\boldsymbol{D}_{\boldsymbol{Y}}-\boldsymbol{\Delta}_{\boldsymbol{Y}}^\top\boldsymbol{D}_{\boldsymbol{Y}}-\boldsymbol{D}_{\boldsymbol{Y}}^\top\boldsymbol{V}-\boldsymbol{D}_{\boldsymbol{Y}}^\top\boldsymbol{\Delta}_{\boldsymbol{Y}} \right\|_F^2\\
		=& \left\| \boldsymbol{D}_{\boldsymbol{X}}\boldsymbol{V}^\top+\boldsymbol{U}\boldsymbol{D}_{\boldsymbol{Y}}^\top \right\|_F^2 + \frac{1}{4}\left\| \boldsymbol{D}_{\boldsymbol{X}}^\top\boldsymbol{U}+ \boldsymbol{U}^\top\boldsymbol{D}_{\boldsymbol{X}}- \boldsymbol{V}^\top\boldsymbol{D}_{\boldsymbol{Y}}-\boldsymbol{D}_{\boldsymbol{Y}}^\top\boldsymbol{V}\right\|_F^2 + \mathcal{E}_1.
	\end{split}
\end{equation}
Here we use the fact that $\boldsymbol{U}^\top\boldsymbol{U} = \boldsymbol{V}^\top\boldsymbol{V}$, and $\mathcal{E}_1$ contains terms with $\boldsymbol{\Delta}_{\boldsymbol{X}}$'s and $\boldsymbol{\Delta}_{\boldsymbol{Y}}$'s, i.e., 
\[
	\begin{split}
		&\mathcal{E}_1\\
		 =& 2\langle \boldsymbol{\Delta}_{\boldsymbol{X}}\boldsymbol{V}^\top + \boldsymbol{U}\boldsymbol{\Delta}_{\boldsymbol{Y}}^\top+ \boldsymbol{\Delta}_{\boldsymbol{X}}\boldsymbol{\Delta}_{\boldsymbol{Y}}^\top, \boldsymbol{D}_{\boldsymbol{X}}\boldsymbol{D}_{\boldsymbol{Y}}^\top \rangle  +2 \langle  \boldsymbol{\Delta}_{\boldsymbol{X}}\boldsymbol{D}_{\boldsymbol{Y}}^\top + \boldsymbol{D}_{\boldsymbol{X}}\boldsymbol{\Delta}_{\boldsymbol{Y}}^\top ,\boldsymbol{D}_{\boldsymbol{X}}\boldsymbol{V}^\top + \boldsymbol{U}\boldsymbol{D}_{\boldsymbol{Y}}^\top \rangle  + \left\|   \boldsymbol{D}_{\boldsymbol{X}}\boldsymbol{\Delta}_{\boldsymbol{Y}}^\top  +\boldsymbol{\Delta}_{\boldsymbol{X}}\boldsymbol{D}_{\boldsymbol{Y}}^\top \right\|_F^2\\
		& + \frac{1}{2} \langle \boldsymbol{U}^\top\boldsymbol{\Delta}_{\boldsymbol{X}} + \boldsymbol{\Delta}_{\boldsymbol{X}}^\top\boldsymbol{U} + \boldsymbol{\Delta}_{\boldsymbol{X}}^\top\boldsymbol{\Delta}_{\boldsymbol{X}} - \boldsymbol{\Delta}_{\boldsymbol{Y}}^\top\boldsymbol{V}  - \boldsymbol{V}^\top\boldsymbol{\Delta}_{\boldsymbol{Y}}-  \boldsymbol{\Delta}_{\boldsymbol{Y}}^\top\boldsymbol{\Delta}_{\boldsymbol{Y}}, \boldsymbol{D}_{\boldsymbol{X}}^\top\boldsymbol{D}_{\boldsymbol{X}}-\boldsymbol{D}_{\boldsymbol{Y}}^\top\boldsymbol{D}_{\boldsymbol{Y}} \rangle\\
		& + \frac{1}{2} \langle \boldsymbol{D}_{\boldsymbol{X}}^\top\boldsymbol{\Delta}_{\boldsymbol{X}} +\boldsymbol{\Delta}_{\boldsymbol{X}}^\top\boldsymbol{D}_{\boldsymbol{X}} -\boldsymbol{\Delta}_{\boldsymbol{Y}}^\top\boldsymbol{D}_{\boldsymbol{Y}}-\boldsymbol{D}_{\boldsymbol{Y}}^\top\boldsymbol{\Delta}_{\boldsymbol{Y}} ,\boldsymbol{D}_{\boldsymbol{X}}^\top\boldsymbol{U}+ \boldsymbol{U}^\top\boldsymbol{D}_{\boldsymbol{X}}- \boldsymbol{V}^\top\boldsymbol{D}_{\boldsymbol{Y}} -\boldsymbol{D}_{\boldsymbol{Y}}^\top\boldsymbol{V}\rangle \\
		&   + \frac{1}{4}\left\|  \boldsymbol{D}_{\boldsymbol{X}}^\top\boldsymbol{\Delta}_{\boldsymbol{X}}  +\boldsymbol{\Delta}_{\boldsymbol{X}}^\top\boldsymbol{D}_{\boldsymbol{X}}  -\boldsymbol{\Delta}_{\boldsymbol{Y}}^\top\boldsymbol{D}_{\boldsymbol{Y}} -\boldsymbol{D}_{\boldsymbol{Y}}^\top\boldsymbol{\Delta}_{\boldsymbol{Y}} \right\|_F^2.
	\end{split}	
\]

Multiplying terms through we have
\[
	\begin{split}
		& \mathbb{E}\left[ \operatorname{vec}\left( \left[ \begin{array}{c}\boldsymbol{D}_{\boldsymbol{X}}\\ \boldsymbol{D}_{\boldsymbol{Y}} \end{array}\right] \right)^\top \nabla^2f(\boldsymbol{X},\boldsymbol{Y})\operatorname{vec}\left( \left[ \begin{array}{c}\boldsymbol{D}_{\boldsymbol{X}}\\ \boldsymbol{D}_{\boldsymbol{Y}} \end{array}\right] \right) \right] \\
		=& \left\|\boldsymbol{D}_{\boldsymbol{X}}\boldsymbol{V}^\top \right\|_F^2 + \left\| \boldsymbol{U}\boldsymbol{D}_{\boldsymbol{Y}}^\top \right\|_F^2 +\frac{1}{2} \left\| \boldsymbol{D}_{\boldsymbol{X}}^\top\boldsymbol{U} \right\|_F^2  +\frac{1}{2}\left\| \boldsymbol{V}^\top\boldsymbol{D}_{\boldsymbol{Y}} \right\|_F^2  - \langle \boldsymbol{D}_{\boldsymbol{X}}^\top\boldsymbol{U},\boldsymbol{D}_{\boldsymbol{Y}}^\top\boldsymbol{V} \rangle  +\frac{1}{2}\langle \boldsymbol{D}_{\boldsymbol{X}}^\top\boldsymbol{U},\boldsymbol{U}^\top\boldsymbol{D}_{\boldsymbol{X}} \rangle \\
		&+ \frac{1}{2}\langle \boldsymbol{D}_{\boldsymbol{Y}}^\top\boldsymbol{V},\boldsymbol{V}^\top\boldsymbol{D}_{\boldsymbol{Y}} \rangle + \langle  \boldsymbol{D}_{\boldsymbol{X}}^\top\boldsymbol{U}, \boldsymbol{V}^\top\boldsymbol{D}_{\boldsymbol{Y}} \rangle + \mathcal{E}_1\\
		=&  \left\|\boldsymbol{D}_{\boldsymbol{X}}\boldsymbol{V}^\top \right\|_F^2 + \left\| \boldsymbol{U}\boldsymbol{D}_{\boldsymbol{Y}}^\top \right\|_F^2 + \frac{1}{2}\left\| \boldsymbol{D}_{\boldsymbol{X}}^\top\boldsymbol{U}  -  \boldsymbol{D}_{\boldsymbol{Y}}^\top\boldsymbol{V}   \right\|_F^2  +\frac{1}{2} \langle \boldsymbol{U}^\top\boldsymbol{D}_{\boldsymbol{X}}+\boldsymbol{V}^\top\boldsymbol{D}_{\boldsymbol{Y}}, \boldsymbol{D}_{\boldsymbol{X}}^\top\boldsymbol{U}+\boldsymbol{D}_{\boldsymbol{Y}}^\top\boldsymbol{V}\rangle+\mathcal{E}_1.
	\end{split}
\]	
Now for the fourth term, we split $\boldsymbol{U}$ as $\boldsymbol{U}-\boldsymbol{X}_2+\boldsymbol{X}_2$, $\boldsymbol{V}$ as $\boldsymbol{V}-\boldsymbol{Y}_2+\boldsymbol{Y}_2$, and plug it back. Then we have
\[
	\begin{split}
		& \mathbb{E}\left[ \operatorname{vec}\left( \left[ \begin{array}{c}\boldsymbol{D}_{\boldsymbol{X}}\\ \boldsymbol{D}_{\boldsymbol{Y}} \end{array}\right] \right)^\top \nabla^2f(\boldsymbol{X},\boldsymbol{Y})\operatorname{vec}\left( \left[ \begin{array}{c}\boldsymbol{D}_{\boldsymbol{X}}\\ \boldsymbol{D}_{\boldsymbol{Y}} \end{array}\right] \right) \right] \\
		=&  \left\|\boldsymbol{D}_{\boldsymbol{X}}\boldsymbol{V}^\top \right\|_F^2 + \left\| \boldsymbol{U}\boldsymbol{D}_{\boldsymbol{Y}}^\top \right\|_F^2 + \frac{1}{2}\left\| \boldsymbol{D}_{\boldsymbol{X}}^\top\boldsymbol{U}  -  \boldsymbol{D}_{\boldsymbol{Y}}^\top\boldsymbol{V}   \right\|_F^2  + \frac{1}{2}\langle \boldsymbol{X}_2^\top\boldsymbol{D}_{\boldsymbol{X}}+\boldsymbol{Y}_2^\top\boldsymbol{D}_{\boldsymbol{Y}}, \boldsymbol{D}_{\boldsymbol{X}}^\top\boldsymbol{X}_2+\boldsymbol{D}_{\boldsymbol{Y}}^\top\boldsymbol{Y}_2 \rangle+\mathcal{E}_1+\mathcal{E}_2,
	\end{split}
\] 	
where $\mathcal{E}_2$ contains terms with $\boldsymbol{U}-\boldsymbol{X}_2$'s and $\boldsymbol{V}-\boldsymbol{Y}_2$'s, i.e., 
\[
	\begin{split}
		 \mathcal{E}_2  =& \frac{1}{2}\langle (\boldsymbol{U}-\boldsymbol{X}_2)^\top\boldsymbol{D}_{\boldsymbol{X}}+(\boldsymbol{V}-\boldsymbol{Y}_2)^\top\boldsymbol{D}_{\boldsymbol{Y}}, \boldsymbol{D}_{\boldsymbol{X}}^\top\boldsymbol{X}_2+\boldsymbol{D}_{\boldsymbol{Y}}^\top\boldsymbol{Y}_2\rangle\\
		 & + \frac{1}{2} \langle \boldsymbol{X}_2^\top\boldsymbol{D}_{\boldsymbol{X}}+\boldsymbol{Y}_2^\top\boldsymbol{D}_{\boldsymbol{Y}}, \boldsymbol{D}_{\boldsymbol{X}}^\top (\boldsymbol{U}-\boldsymbol{X}_2)+ \boldsymbol{D}_{\boldsymbol{Y}}^\top(\boldsymbol{V}-\boldsymbol{Y}_2) \rangle\\
		&+ \frac{1}{2} \langle(\boldsymbol{U}-\boldsymbol{X}_2)^\top\boldsymbol{D}_{\boldsymbol{X}}+(\boldsymbol{V}-\boldsymbol{Y}_2)^\top\boldsymbol{D}_{\boldsymbol{Y}} , \boldsymbol{D}_{\boldsymbol{X}}^\top (\boldsymbol{U}-\boldsymbol{X}_2)+ \boldsymbol{D}_{\boldsymbol{Y}}^\top(\boldsymbol{V}-\boldsymbol{Y}_2) \rangle.
	\end{split}	
\]

By the way we define $\widehat{\boldsymbol{R}}$ in \eqref{lemma_hessian2}, $\left[ \begin{array}{c} \boldsymbol{X}_2\\ \boldsymbol{Y}_2 \end{array} \right]^\top \left[ \begin{array}{c} \boldsymbol{D}_{\boldsymbol{X}}\\ \boldsymbol{D}_{\boldsymbol{Y}} \end{array}\right]$ is symmetric. Using this fact we have
\begin{equation}\label{eq_004}
	\begin{split}
		& \mathbb{E}\left[ \operatorname{vec}\left( \left[ \begin{array}{c}\boldsymbol{D}_{\boldsymbol{X}}\\ \boldsymbol{D}_{\boldsymbol{Y}} \end{array}\right] \right)^\top \nabla^2f(\boldsymbol{X},\boldsymbol{Y})\operatorname{vec}\left( \left[ \begin{array}{c}\boldsymbol{D}_{\boldsymbol{X}}\\ \boldsymbol{D}_{\boldsymbol{Y}} \end{array}\right] \right) \right] \\
		=& \left\|\boldsymbol{D}_{\boldsymbol{X}}\boldsymbol{V}^\top \right\|_F^2 + \left\| \boldsymbol{U}\boldsymbol{D}_{\boldsymbol{Y}}^\top \right\|_F^2 + \frac{1}{2}\left\| \boldsymbol{D}_{\boldsymbol{X}}^\top\boldsymbol{U}  -  \boldsymbol{D}_{\boldsymbol{Y}}^\top\boldsymbol{V}   \right\|_F^2  + \frac{1}{2}\left\| \boldsymbol{X}_2^\top\boldsymbol{D}_{\boldsymbol{X}}+\boldsymbol{Y}_2^\top\boldsymbol{D}_{\boldsymbol{Y}} \right\|_F^2 +\mathcal{E}_1+\mathcal{E}_2.
	\end{split}
\end{equation}
	
For $\mathcal{E}_1+\mathcal{E}_2$, by the way we define them, we have the following bound:
\[
	\begin{split}
		& |\mathcal{E}_1+\mathcal{E}_2|\\
		\leqslant & 9 [ (\|\boldsymbol{U}-\boldsymbol{X}_2\|+\|\boldsymbol{V}-\boldsymbol{Y}_2\|)(\|\boldsymbol{X}_2\|+\|\boldsymbol{Y}_2\|)  + (\|\boldsymbol{U}-\boldsymbol{X}_2\|+\|\boldsymbol{V}-\boldsymbol{Y}_2\|)^2](\|\boldsymbol{D}_{\boldsymbol{X}}\|_F^2 + \|\boldsymbol{D}_{\boldsymbol{Y}}\|_F^2)\\
		& + 9 [(\|\boldsymbol{\Delta}_{\boldsymbol{X}}\| + \|\boldsymbol{\Delta}_{\boldsymbol{Y}}\|)(\|\boldsymbol{U}\|+\|\boldsymbol{V}\|) + (\|\boldsymbol{\Delta}_{\boldsymbol{X}}\| + \|\boldsymbol{\Delta}_{\boldsymbol{Y}}\|)^2  ](\|\boldsymbol{D}_{\boldsymbol{X}}\|_F^2 + \|\boldsymbol{D}_{\boldsymbol{Y}}\|_F^2).
	\end{split}
\]	
From the assumption,
\[
\left\| \left[ \begin{array}{c} \boldsymbol{X}_2-\boldsymbol{U}\\ \boldsymbol{Y}_2-\boldsymbol{V} \end{array}\right] \right\| \leqslant  \frac{1}{500 \kappa}\sqrt{\sigma_1(\boldsymbol{M})},
\]
\[
\left\| \left[ \begin{array}{c}
		\boldsymbol{X}-\boldsymbol{U}\\
		\boldsymbol{Y}-\boldsymbol{V}
		 \end{array}\right]\right\|_{2,\infty} \leqslant  \frac{1}{500\kappa\sqrt{n_1+n_2}}\sqrt{\sigma_1(\boldsymbol{M})},
\]		 
and	 
\[
\begin{split}
\left\| \left[ \begin{array}{c}
		\boldsymbol{X}-\boldsymbol{U}\\
		\boldsymbol{Y}-\boldsymbol{V}
		 \end{array}\right]\right\|\leqslant&  \left\| \left[ \begin{array}{c}
		\boldsymbol{X}-\boldsymbol{U}\\
		\boldsymbol{Y}-\boldsymbol{V}
		 \end{array}\right]\right\|_F \\
		 \leqslant & \sqrt{n_1+n_2} \left\| \left[ \begin{array}{c}
		\boldsymbol{X}-\boldsymbol{U}\\
		\boldsymbol{Y}-\boldsymbol{V}
		 \end{array}\right]\right\|_{2,\infty} \\
		 \leqslant & \frac{1}{500\kappa}\sqrt{\sigma_1(\boldsymbol{M})},
		\end{split}
\]		  
therefore we have
		 \begin{equation}\label{eq_005}
		 	|\mathcal{E}_1+\mathcal{E}_2|\leqslant \frac{1}{5} \sigma_r(\boldsymbol{M})\left\| \left[ \begin{array}{c} \boldsymbol{D}_{\boldsymbol{X}}\\ \boldsymbol{D}_{\boldsymbol{Y}}\end{array}\right] \right\|_F^2.
		 \end{equation}
	
	Now we start to consider the difference between population level and empirical level, comparing with \eqref{eq_b005}:
	\[
		\begin{split}
			  &\operatorname{vec}\left( \left[ \begin{array}{c}\boldsymbol{D}_{\boldsymbol{X}}\\ \boldsymbol{D}_{\boldsymbol{Y}} \end{array}\right] \right)^\top \nabla^2f(\boldsymbol{X},\boldsymbol{Y})\operatorname{vec}\left( \left[ \begin{array}{c}\boldsymbol{D}_{\boldsymbol{X}}\\ \boldsymbol{D}_{\boldsymbol{Y}} \end{array}\right] \right) -\mathbb{E}\left[ \operatorname{vec}\left( \left[ \begin{array}{c}\boldsymbol{D}_{\boldsymbol{X}}\\ \boldsymbol{D}_{\boldsymbol{Y}} \end{array}\right] \right)^\top \nabla^2f(\boldsymbol{X},\boldsymbol{Y})\operatorname{vec}\left( \left[ \begin{array}{c}\boldsymbol{D}_{\boldsymbol{X}}\\ \boldsymbol{D}_{\boldsymbol{Y}} \end{array}\right] \right) \right]\\
		=& \circled{1}+\circled{2}+\circled{3}+\circled{4},\\
		\end{split} 
	\]
	where $D(\cdot,\cdot)$ denotes the difference between population level and empirical level, i.e.,
	\begin{equation}\label{eq_062}
	\begin{split}
		  D(\boldsymbol{A}\boldsymbol{C}^\top,\boldsymbol{B}\boldsymbol{D}^\top) \coloneqq & \frac{1}{p}\langle \mathcal{P}_{\Omega}(\boldsymbol{A}\boldsymbol{C}^\top),\mathcal{P}_{\Omega}(\boldsymbol{B}\boldsymbol{D}^\top) \rangle - \langle \boldsymbol{A}\boldsymbol{C}^\top , \boldsymbol{B}\boldsymbol{D}^\top \rangle.
	\end{split}
	\end{equation}
	And
	\[
	\begin{split}
		\circled{1} \coloneqq& 2D(\boldsymbol{\Delta}_{\boldsymbol{X}}\boldsymbol{V}^\top, \boldsymbol{D}_{\boldsymbol{X}}\boldsymbol{D}_{\boldsymbol{Y}}^\top)+ 2D(\boldsymbol{U}\boldsymbol{\Delta}_{\boldsymbol{Y}}^\top,\boldsymbol{D}_{\boldsymbol{X}}\boldsymbol{D}_{\boldsymbol{Y}}^\top)  + 2D(\boldsymbol{\Delta}_{\boldsymbol{X}}\boldsymbol{\Delta}_{\boldsymbol{Y}}^\top,\boldsymbol{D}_{\boldsymbol{X}}\boldsymbol{D}_{\boldsymbol{Y}}^\top)+2D(\boldsymbol{D}_{\boldsymbol{X}}\boldsymbol{V}^\top,\boldsymbol{\Delta}_{\boldsymbol{X}}\boldsymbol{D}_{\boldsymbol{Y}}^\top)\\
		& +2D(\boldsymbol{D}_{\boldsymbol{X}}\boldsymbol{\Delta}_{\boldsymbol{Y}}^\top, \boldsymbol{U}\boldsymbol{D}_{\boldsymbol{Y}}^\top) + 2D(\boldsymbol{D}_{\boldsymbol{X}}\boldsymbol{\Delta}_{\boldsymbol{Y}}^\top,\boldsymbol{\Delta}_{\boldsymbol{X}}\boldsymbol{D}_{\boldsymbol{Y}}^\top),
	\end{split}
	\]
	\[
	\begin{split}
		\circled{2} \coloneqq& D(\boldsymbol{D}_{\boldsymbol{X}}\boldsymbol{V}^\top,\boldsymbol{D}_{\boldsymbol{X}}\boldsymbol{V}^\top) +D(\boldsymbol{U}\boldsymbol{D}_{\boldsymbol{Y}}^\top, \boldsymbol{U}\boldsymbol{D}_{\boldsymbol{Y}}^\top) + 2D(\boldsymbol{D}_{\boldsymbol{X}}\boldsymbol{V}^\top,\boldsymbol{U}\boldsymbol{D}_{\boldsymbol{Y}}^\top) ,
	\end{split}
	\]
	\[
		\circled{3} \coloneqq D(\boldsymbol{D}_{\boldsymbol{X}}\boldsymbol{\Delta}_{\boldsymbol{Y}}^\top, \boldsymbol{D}_{\boldsymbol{X}}\boldsymbol{\Delta}_{\boldsymbol{Y}}^\top) + D(\boldsymbol{\Delta}_{\boldsymbol{X}}\boldsymbol{D}_{\boldsymbol{Y}}^\top, \boldsymbol{\Delta}_{\boldsymbol{X}}\boldsymbol{D}_{\boldsymbol{Y}}^\top),
	\]
	\[
		\circled{4} \coloneqq 2D(\boldsymbol{D}_{\boldsymbol{X}}\boldsymbol{V}^\top,\boldsymbol{D}_{\boldsymbol{X}}\boldsymbol{\Delta}_{\boldsymbol{Y}}^\top)+2D(\boldsymbol{U}\boldsymbol{D}_{\boldsymbol{Y}}^\top, \boldsymbol{\Delta}_{\boldsymbol{X}}\boldsymbol{D}_{\boldsymbol{Y}}^\top).
	\]
	Now for terms with different circled numbers, we deal with them with different bounds. First, for $\circled{1}$, we apply the following lemma:
	\begin{lemma}[{\citealt[Lemma 4.4]{chen2017memory}}]\label{chen_lemma4.4}
		Let $D(\cdot,\cdot)$ defined as in \eqref{eq_062}, for all $\boldsymbol{A}\in\mathbb{R}^{n_1\times r},\boldsymbol{B}\in\mathbb{R}^{n_1\times r},\boldsymbol{C}\in\mathbb{R}^{n_2\times r},\boldsymbol{D}\in\mathbb{R}^{n_2\times r}$, we have
		\[
			\begin{split}
				& |D(\boldsymbol{A}\boldsymbol{C}^\top,\boldsymbol{B}\boldsymbol{D}^\top)|\\
				\leqslant& p^{-1}\|\boldsymbol{\Omega}-p\boldsymbol{J}\| \sqrt{\sum_{k=1}^{n_1}\|\boldsymbol{A}_{k,\cdot}\|_2^2\|\boldsymbol{B}_{k,\cdot}\|_2^2} \sqrt{\sum_{k=1}^{n_2}\|\boldsymbol{C}_{k,\cdot}\|_2^2\|\boldsymbol{D}_{k,\cdot}\|_2^2}\\
				\leqslant&  p^{-1}\|\boldsymbol{\Omega}-p\boldsymbol{J}\| \min(\|\boldsymbol{A}\|_{2,\infty}\|\boldsymbol{B}\|_F, \|\boldsymbol{A}\|_F\|\boldsymbol{B}\|_{2,\infty}) \times \min(\|\boldsymbol{C}\|_{2,\infty}\|\boldsymbol{D}\|_F, \|\boldsymbol{C}\|_F\|\boldsymbol{D}\|_{2,\infty})
			\end{split} 
		\]
	\end{lemma}
	Therefore, 
	\[
		\begin{split}
			  |\circled{1}| \leqslant & \frac{2\|\boldsymbol{\Omega}-p\boldsymbol{J}\|}{p} [  \|\boldsymbol{\Delta}_{\boldsymbol{X}}\|_{2,\infty}\|\boldsymbol{V}\|_{2,\infty}\|\boldsymbol{D}_{\boldsymbol{X}}\|_F\|\boldsymbol{D}_{\boldsymbol{Y}}\|_F  +  \|\boldsymbol{\Delta}_{\boldsymbol{Y}}\|_{2,\infty}\|\boldsymbol{U}\|_{2,\infty}\|\boldsymbol{D}_{\boldsymbol{X}}\|_F\|\boldsymbol{D}_{\boldsymbol{Y}}\|_F ]\\
			& + \frac{2\|\boldsymbol{\Omega}-p\boldsymbol{J}\|}{p} [\|\boldsymbol{\Delta}_{\boldsymbol{X}}\|_{2,\infty}\|\boldsymbol{\Delta}_{\boldsymbol{Y}}\|_{2,\infty}\|\boldsymbol{D}_{\boldsymbol{X}}\|_F\|\boldsymbol{D}_{\boldsymbol{Y}}\|_F   +\|\boldsymbol{\Delta}_{\boldsymbol{X}}\|_{2,\infty}\|\boldsymbol{V}\|_{2,\infty}\|\boldsymbol{D}_{\boldsymbol{X}}\|_F\|\boldsymbol{D}_{\boldsymbol{Y}}\|_F]  \\
			&  +  \frac{2\|\boldsymbol{\Omega}-p\boldsymbol{J}\|}{p} [\|\boldsymbol{\Delta}_{\boldsymbol{Y}}\|_{2,\infty}\|\boldsymbol{U}\|_{2,\infty}\|\boldsymbol{D}_{\boldsymbol{X}}\|_F\|\boldsymbol{D}_{\boldsymbol{Y}}\|_F + \|\boldsymbol{\Delta}_{\boldsymbol{X}}\|_{2,\infty}\|\boldsymbol{\Delta}_{\boldsymbol{Y}}\|_{2,\infty}\|\boldsymbol{D}_{\boldsymbol{X}}\|_F\|\boldsymbol{D}_{\boldsymbol{Y}}\|_F ].
		\end{split}
	\]
	Using Lemma \ref{vu_lemma2.2} and using the fact that 
	\[
		 \left\| \left[ \begin{array}{c}
		\boldsymbol{X}-\boldsymbol{U}\\
		\boldsymbol{Y}-\boldsymbol{V}
		 \end{array}\right]\right\|_{2,\infty} \leqslant  \frac{1}{500 \kappa\sqrt{n_1+n_2}}\sqrt{\sigma_1(\boldsymbol{M})},
	\]
	if 
	\[
		 p \geqslant C_3 \frac{\mu r \log (n_1\vee n_2)}{n_1\wedge n_2},
	\]
	we have 
	\begin{equation}\label{eq_c001}
		|\circled{1}| \leqslant  12C_3 \sqrt{\frac{\mu r \kappa}{p}} \frac{1}{\kappa\sqrt{n_1+n_2}}\sigma_1(\boldsymbol{M}) (\|\boldsymbol{D}_{\boldsymbol{X}}\|_F^2 + \|\boldsymbol{D}_{\boldsymbol{Y}}\|_F^2).
	\end{equation}

	For $\circled{2}$, we apply the following lemma:
	\begin{lemma}[{\citealt[Theorem 4.1]{candes2009exact}}]\label{candes_thm4.1}
		Define subspace 
		\[
		\begin{split}
			\mathcal{T} \coloneqq \{\boldsymbol{M}\in\mathbb{R}^{n_1\times n_2}\mid &\boldsymbol{M} = \boldsymbol{A}\boldsymbol{V}^\top +\boldsymbol{U}\boldsymbol{B}^\top\;\textrm{for any} \;\boldsymbol{A}\in\mathbb{R}^{n_1\times r},\boldsymbol{B}\in\mathbb{R}^{n_2\times r}\}.
		\end{split}
		\]
		There is an absolute constant $C_1$, such that if $p\geqslant C_1\frac{\mu r\log (n_1\vee n_2)}{ (n_1\wedge n_2)}$, on an event $E_{Ca}$ with probability $\mathbb{P}[E_{Ca}]\geqslant 1-(n_1+n_2)^{-11}$, 
		\[
			p^{-1}\|\mathcal{P}_{\mathcal{T}}\mathcal{P}_{\Omega}\mathcal{P}_{\mathcal{T}} - p\mathcal{P}_{\mathcal{T}}\|\leqslant 0.1
		\]
		holds.
	\end{lemma}
	Therefore, 
	\begin{equation}\label{eq_c002}
		\begin{split}
			|\circled{2}| = & |D(\boldsymbol{D}_{\boldsymbol{X}}\boldsymbol{V}^\top ,\boldsymbol{D}_{\boldsymbol{X}}\boldsymbol{V}^\top) + D(\boldsymbol{U}\boldsymbol{D}_{\boldsymbol{Y}}^\top ,\boldsymbol{U}\boldsymbol{D}_{\boldsymbol{Y}}^\top ) + 2D(\boldsymbol{D}_{\boldsymbol{X}}\boldsymbol{V}^\top,\boldsymbol{U}\boldsymbol{D}_{\boldsymbol{Y}}^\top ) |\\
			=& |D(\boldsymbol{D}_{\boldsymbol{X}}\boldsymbol{V}^\top+\boldsymbol{U}\boldsymbol{D}_{\boldsymbol{Y}}^\top,\boldsymbol{D}_{\boldsymbol{X}}\boldsymbol{V}^\top+\boldsymbol{U}\boldsymbol{D}_{\boldsymbol{Y}}^\top)|\\
			\leqslant &0.1 \|\boldsymbol{D}_{\boldsymbol{X}}\boldsymbol{V}^\top + \boldsymbol{U}\boldsymbol{D}_{\boldsymbol{Y}}^\top\|_F^2
		\end{split}	
	\end{equation}
	given 
	\[
		p \geqslant C_1 \frac{\mu r \log (n_1\vee n_2)}{n_1\wedge n_2}.
		\]
	
	For $\circled{3}$, we need the following lemma:
	\begin{lemma}[{\citealt[Lemma 9]{zheng2016convergence}}]\label{zheng_lemma9}
		If $p\geqslant C_2\frac{\log (n_1\vee n_2)}{n_1\wedge n_2}$ for some absolute constant $C_2$, then on an event $E_Z$ with probability $\mathbb{P}[E_Z] \geqslant 1-(n_1+n_2)^{-11}$, uniformly for all matrices $\boldsymbol{A}\in\mathbb{R}^{n_1\times r},\boldsymbol{B}\in\mathbb{R}^{n_2\times r}$,
		\[
		\begin{split}
			  p^{-1}\left\| \mathcal{P}_{\Omega}(\boldsymbol{A}\boldsymbol{B}^\top) \right\|_F^2 \leqslant & 2(n_1\vee n_2) \min\left\{ \|\boldsymbol{A}\|_F^2\|\boldsymbol{B}\|_{2,\infty}^2,\|\boldsymbol{A}\|_{2,\infty}^2 \|\boldsymbol{B}\|_F^2\right\}
		\end{split}
		\]
		holds.
	\end{lemma}
	In order to apply Lemma \ref{zheng_lemma9} in our case, note
	\[ 
	\begin{split}
			 \|\boldsymbol{A}\boldsymbol{B}^\top\|_F^2 =& \sum_{i,j}\langle \boldsymbol{A}_{i,\cdot},\boldsymbol{B}_{j,\cdot} \rangle^2\\
			 \leqslant &  \sum_{i,j} \|\boldsymbol{A}_{i,\cdot}\|_2^2 \|\boldsymbol{B}_{j,\cdot}\|_2^2\\
			  \leqslant &  (n_1\vee n_2) \min\left\{ \|\boldsymbol{A}\|_F^2\|\boldsymbol{B}\|_{2,\infty}^2, \|\boldsymbol{A}\|_{2,\infty}^2 \|\boldsymbol{B}\|_F^2 \right\}.
	\end{split}
	\]
	Therefore, by triangle inequality,
	\[
	\begin{split}
		 |D(\boldsymbol{A}\boldsymbol{B}^\top,\boldsymbol{A}\boldsymbol{B}^\top)| \leqslant& 3(n_1\vee n_2) \min\left\{ \|\boldsymbol{A}\|_F^2\|\boldsymbol{B}\|_{2,\infty}^2,\|\boldsymbol{A}\|_{2,\infty}^2 \|\boldsymbol{B}\|_F^2\right\}.
	\end{split}
	\]
	So we have 
	\[
		\begin{split}
			|\circled{3}| \leqslant & 3(n_1\vee n_2) \|\boldsymbol{D}_{\boldsymbol{X}}\|_F^2\|\boldsymbol{\Delta}_{\boldsymbol{Y}}\|_{2,\infty}^2 + 3(n_1\vee n_2) \|\boldsymbol{D}_{\boldsymbol{Y}}\|_F^2\|\boldsymbol{\Delta}_{\boldsymbol{X}}\|_{2,\infty}^2 .
		\end{split}	
	\]
	Using the fact that 
	\[
		  \left\| \left[ \begin{array}{c}
		\boldsymbol{X}-\boldsymbol{U}\\
		\boldsymbol{Y}-\boldsymbol{V}
		 \end{array}\right]\right\|_{2,\infty} \leqslant  \frac{1}{500 \kappa\sqrt{n_1+n_2}}\sqrt{\sigma_1(\boldsymbol{M})},
	\]
	we further have 
	\begin{equation}\label{eq_c003}
	\begin{split}
		|\circled{3}| \leqslant & 3(n_1\vee n_2)\frac{1}{250000\kappa^2(n_1+n_2)}  \times \sigma_1(\boldsymbol{M})(\|\boldsymbol{D}_{\boldsymbol{X}}\|_F^2 + \|\boldsymbol{D}_{\boldsymbol{Y}}\|_F^2).
	\end{split}
	\end{equation}

	Finally, for $\circled{4}$, by triangle inequality,
	\[
		\begin{split}
			 |D(\boldsymbol{D}_{\boldsymbol{X}}\boldsymbol{V}^\top,\boldsymbol{D}_{\boldsymbol{X}}\boldsymbol{\Delta}_{\boldsymbol{Y}}^\top)|  =& |p^{-1}\langle \mathcal{P}_{\Omega}(\boldsymbol{D}_{\boldsymbol{X}}\boldsymbol{V}^\top), \mathcal{P}_{\Omega}(\boldsymbol{D}_{\boldsymbol{X}}\boldsymbol{\Delta}_{\boldsymbol{Y}}^\top) \rangle - \langle \boldsymbol{D}_{\boldsymbol{X}}\boldsymbol{V}^\top,\boldsymbol{D}_{\boldsymbol{X}}\boldsymbol{\Delta}_{\boldsymbol{Y}}^\top\rangle|\\
			\leqslant & \sqrt{p^{-1}\|\mathcal{P}_{\Omega}(\boldsymbol{D}_{\boldsymbol{X}}\boldsymbol{V}^\top)\|_F^2}\sqrt{p^{-1}\|\mathcal{P}_{\Omega}(\boldsymbol{D}_{\boldsymbol{X}}\boldsymbol{\Delta}_{\boldsymbol{Y}}^\top)\|_F^2}  + |\langle \boldsymbol{D}_{\boldsymbol{X}}\boldsymbol{V}^\top,\boldsymbol{D}_{\boldsymbol{X}}\boldsymbol{\Delta}_{\boldsymbol{Y}}^\top\rangle|.
		\end{split}
	\]
	Now by applying Lemma \ref{zheng_lemma9} and Lemma \ref{candes_thm4.1} we have
	\[
		\begin{split}
		 |D(\boldsymbol{D}_{\boldsymbol{X}}\boldsymbol{V}^\top,\boldsymbol{D}_{\boldsymbol{X}}\boldsymbol{\Delta}_{\boldsymbol{Y}}^\top)| \leqslant&  \sqrt{2(n_1\vee n_2)\|\boldsymbol{D}_{\boldsymbol{X}}\|_F^2 \|\boldsymbol{\Delta}_{\boldsymbol{Y}}\|_{2,\infty}^2}\sqrt{(1+0.1)\|\boldsymbol{D}_{\boldsymbol{X}}\boldsymbol{V}^\top\|_F^2}  + \|\boldsymbol{V}\|\|\boldsymbol{\Delta}_{\boldsymbol{Y}}\| \|\boldsymbol{D}_{\boldsymbol{X}}\|_F^2\\
		\leqslant & \sqrt{3 (n_1\vee n_2)} \|\boldsymbol{\Delta}_{\boldsymbol{Y}}\|_{2,\infty} \|\boldsymbol{V}\| \|\boldsymbol{D}_{\boldsymbol{X}}\|_F^2 + \|\boldsymbol{V}\|\|\boldsymbol{\Delta}_{\boldsymbol{Y}}\| \|\boldsymbol{D}_{\boldsymbol{X}}\|_F^2.
		\end{split}
	\]
	Similarly, we also have
	\[
	\begin{split}
		 |D(\boldsymbol{U}\boldsymbol{D}_{\boldsymbol{Y}}^\top,\boldsymbol{\Delta}_{\boldsymbol{X}}\boldsymbol{D}_{\boldsymbol{Y}}^\top)|  \leqslant &  \sqrt{3 (n_1\vee n_2)} \|\boldsymbol{\Delta}_{\boldsymbol{X}}\|_{2,\infty} \|\boldsymbol{U}\| \|\boldsymbol{D}_{\boldsymbol{Y}}\|_F^2 + \|\boldsymbol{U}\|\|\boldsymbol{\Delta}_{\boldsymbol{X}}\| \|\boldsymbol{D}_{\boldsymbol{Y}}\|_F^2.
	\end{split}
	\]

	Using the fact that  
	\[
		\left\| \left[ \begin{array}{c}
		\boldsymbol{X}-\boldsymbol{U}\\
		\boldsymbol{Y}-\boldsymbol{V}
		 \end{array}\right]\right\| \leqslant  \frac{1}{500 \kappa}\sqrt{\sigma_1(\boldsymbol{M})}.
	\] 
	Therefore, 
	\begin{equation}\label{eq_c004}
		 \begin{split}
			  |\circled{4}|  \leqslant& 2\sqrt{3 (n_1\vee n_2)} \|\boldsymbol{\Delta}_{\boldsymbol{Y}}\|_{2,\infty} \|\boldsymbol{V}\| \|\boldsymbol{D}_{\boldsymbol{X}}\|_F^2  + 2\|\boldsymbol{V}\|\|\boldsymbol{\Delta}_{\boldsymbol{Y}}\| \|\boldsymbol{D}_{\boldsymbol{X}}\|_F^2  + 2\sqrt{3 (n_1\vee n_2)} \|\boldsymbol{\Delta}_{\boldsymbol{X}}\|_{2,\infty} \|\boldsymbol{U}\| \|\boldsymbol{D}_{\boldsymbol{Y}}\|_F^2 \\
			 & +2 \|\boldsymbol{U}\|\|\boldsymbol{\Delta}_{\boldsymbol{X}}\| \|\boldsymbol{D}_{\boldsymbol{Y}}\|_F^2\\
			 \leqslant & 2\sqrt{3(n_1\vee n_2)} \frac{1}{500 \kappa\sqrt{n_1+n_2}} \times \sigma_1(\boldsymbol{M}) (\|\boldsymbol{D}_{\boldsymbol{X}}\|_F^2 + \|\boldsymbol{D}_{\boldsymbol{Y}}\|_F^2)   + \frac{2}{500\kappa}\sigma_1(\boldsymbol{M}) (\|\boldsymbol{D}_{\boldsymbol{X}}\|_F^2 + \|\boldsymbol{D}_{\boldsymbol{Y}}\|_F^2).
		 \end{split}
		\end{equation}

	Putting the estimation for $\circled{1},\circled{2},\circled{3}$ and $\circled{4}$ together, i.e., \eqref{eq_c001}, \eqref{eq_c002}, \eqref{eq_c003}, \eqref{eq_c004}, if 
	\[
		 p \geqslant (C_1+C_2+C_3) \frac{\mu r \log (n_1\vee n_2)}{n_1\wedge n_2},
	\]
	then
	\[
		\begin{split}
			& \left| \operatorname{vec}\left( \left[ \begin{array}{c}\boldsymbol{D}_{\boldsymbol{X}}\\ \boldsymbol{D}_{\boldsymbol{Y}} \end{array}\right] \right)^\top \nabla^2f(\boldsymbol{X},\boldsymbol{Y})\operatorname{vec}\left( \left[ \begin{array}{c}\boldsymbol{D}_{\boldsymbol{X}}\\ \boldsymbol{D}_{\boldsymbol{Y}} \end{array}\right] \right)   -\mathbb{E}\left[ \operatorname{vec}\left( \left[ \begin{array}{c}\boldsymbol{D}_{\boldsymbol{X}}\\ \boldsymbol{D}_{\boldsymbol{Y}} \end{array}\right] \right)^\top \nabla^2f(\boldsymbol{X},\boldsymbol{Y})\operatorname{vec}\left( \left[ \begin{array}{c}\boldsymbol{D}_{\boldsymbol{X}}\\ \boldsymbol{D}_{\boldsymbol{Y}} \end{array}\right] \right) \right]\right|\\
			\leqslant &  12C_3 \sqrt{\frac{\mu r \kappa}{p}} \frac{1}{\kappa\sqrt{n_1+n_2}}\sigma_1(\boldsymbol{M}) (\|\boldsymbol{D}_{\boldsymbol{X}}\|_F^2 + \|\boldsymbol{D}_{\boldsymbol{Y}}\|_F^2)  + 0.1 \|\boldsymbol{D}_{\boldsymbol{X}}\boldsymbol{V}^\top + \boldsymbol{U}\boldsymbol{D}_{\boldsymbol{Y}}^\top\|_F^2\\
			& + 3(n_1\vee n_2)\frac{1}{250000\kappa^2(n_1+n_2)} \times  \sigma_1(\boldsymbol{M})(\|\boldsymbol{D}_{\boldsymbol{X}}\|_F^2 + \|\boldsymbol{D}_{\boldsymbol{Y}}\|_F^2)\\
			& + 2\sqrt{3(n_1\vee n_2)} \frac{1}{500 \kappa\sqrt{n_1+n_2}} \times \sigma_1(\boldsymbol{M}) (\|\boldsymbol{D}_{\boldsymbol{X}}\|_F^2 + \|\boldsymbol{D}_{\boldsymbol{Y}}\|_F^2)  + \frac{2}{500\kappa}\sigma_1(\boldsymbol{M}) (\|\boldsymbol{D}_{\boldsymbol{X}}\|_F^2 + \|\boldsymbol{D}_{\boldsymbol{Y}}\|_F^2)
		\end{split}
	\]
	holds on an event $E_H = E_S\bigcap E_{Ca} \bigcap E_Z$ with probability $\mathbb{P}[E_H] = \mathbb{P}[E_S\bigcap E_{Ca} \bigcap E_Z] \geqslant 1-3(n_1+n_2)^{-11} $. If in addition 
	\[
		p\geqslant 14400C_3^2\frac{\mu r \kappa}{n_1\wedge n_2},
	\]
	then 
	\begin{equation}\label{eq_006}
		\begin{split}
			& \left| \operatorname{vec}\left( \left[ \begin{array}{c}\boldsymbol{D}_{\boldsymbol{X}}\\ \boldsymbol{D}_{\boldsymbol{Y}} \end{array}\right] \right)^\top \nabla^2f(\boldsymbol{X},\boldsymbol{Y})\operatorname{vec}\left( \left[ \begin{array}{c}\boldsymbol{D}_{\boldsymbol{X}}\\ \boldsymbol{D}_{\boldsymbol{Y}} \end{array}\right] \right)  -\mathbb{E}\left[ \operatorname{vec}\left( \left[ \begin{array}{c}\boldsymbol{D}_{\boldsymbol{X}}\\ \boldsymbol{D}_{\boldsymbol{Y}} \end{array}\right] \right)^\top \nabla^2f(\boldsymbol{X},\boldsymbol{Y})\operatorname{vec}\left( \left[ \begin{array}{c}\boldsymbol{D}_{\boldsymbol{X}}\\ \boldsymbol{D}_{\boldsymbol{Y}} \end{array}\right] \right) \right]\right|\\
			\leqslant &  \frac{1}{5} \sigma_r(\boldsymbol{M}) (\|\boldsymbol{D}_{\boldsymbol{X}}\|_F^2 + \|\boldsymbol{D}_{\boldsymbol{Y}}\|_F^2)  + \frac{1}{5}(\|\boldsymbol{D}_{\boldsymbol{X}}\boldsymbol{V}^\top\|_F^2 + \|\boldsymbol{U}\boldsymbol{D}_{\boldsymbol{Y}}^\top\|_F^2).
		\end{split}
	\end{equation}
	Now by putting \eqref{eq_004}, \eqref{eq_005}, \eqref{eq_006} together, we have
	\[
		\begin{split}
			 & \operatorname{vec}\left( \left[ \begin{array}{c}\boldsymbol{D}_{\boldsymbol{X}}\\ \boldsymbol{D}_{\boldsymbol{Y}} \end{array}\right] \right)^\top \nabla^2f(\boldsymbol{X},\boldsymbol{Y})\operatorname{vec}\left( \left[ \begin{array}{c}\boldsymbol{D}_{\boldsymbol{X}}\\ \boldsymbol{D}_{\boldsymbol{Y}} \end{array}\right] \right) \\\geqslant &\left\|\boldsymbol{D}_{\boldsymbol{X}}\boldsymbol{V}^\top \right\|_F^2 + \left\| \boldsymbol{U}\boldsymbol{D}_{\boldsymbol{Y}}^\top \right\|_F^2  - \frac{1}{5} \sigma_r(\boldsymbol{M}) (\|\boldsymbol{D}_{\boldsymbol{X}}\|_F^2 + \|\boldsymbol{D}_{\boldsymbol{Y}}\|_F^2)  -\frac{1}{5} \sigma_r(\boldsymbol{M}) (\|\boldsymbol{D}_{\boldsymbol{X}}\|_F^2 + \|\boldsymbol{D}_{\boldsymbol{Y}}\|_F^2) \\
			 &- \frac{1}{5}(\|\boldsymbol{D}_{\boldsymbol{X}}\boldsymbol{V}^\top\|_F^2 + \|\boldsymbol{U}\boldsymbol{D}_{\boldsymbol{Y}}^\top\|_F^2)\\
			 \geqslant & \frac{1}{5}\sigma_r(\boldsymbol{M})(\|\boldsymbol{D}_{\boldsymbol{X}}\|_F^2 + \|\boldsymbol{D}_{\boldsymbol{Y}}\|_F^2), 
		\end{split}
	\] 
	where the last inequality we use the fact that $\|\boldsymbol{D}_{\boldsymbol{X}}\boldsymbol{V}^\top\|_F^2 \geqslant \sigma_r^2(\boldsymbol{V})\|\boldsymbol{D}_{\boldsymbol{X}}\|_F^2 = \sigma_r(\boldsymbol{M}) \|\boldsymbol{D}_{\boldsymbol{X}}\|_F^2$ and also $\|\boldsymbol{U}\boldsymbol{D}_{\boldsymbol{Y}}^\top\|_F^2 \geqslant \sigma_r(\boldsymbol{M})\|\boldsymbol{D}_{\boldsymbol{Y}}\|_F^2$. For the upper bound, we also have
	\[
	\begin{split}
	& \operatorname{vec}\left( \left[ \begin{array}{c}\boldsymbol{D}_{\boldsymbol{X}}\\ \boldsymbol{D}_{\boldsymbol{Y}} \end{array}\right] \right)^\top \nabla^2f(\boldsymbol{X},\boldsymbol{Y})\operatorname{vec}\left( \left[ \begin{array}{c}\boldsymbol{D}_{\boldsymbol{X}}\\ \boldsymbol{D}_{\boldsymbol{Y}} \end{array}\right] \right) \\
	\leqslant &\left\|\boldsymbol{D}_{\boldsymbol{X}}\boldsymbol{V}^\top \right\|_F^2 + \left\| \boldsymbol{U}\boldsymbol{D}_{\boldsymbol{Y}}^\top \right\|_F^2 + \frac{1}{2}\left\| \boldsymbol{D}_{\boldsymbol{X}}^\top\boldsymbol{U}  -  \boldsymbol{D}_{\boldsymbol{Y}}^\top\boldsymbol{V}   \right\|_F^2 + \frac{1}{2}\left\| \boldsymbol{X}_2^\top\boldsymbol{D}_{\boldsymbol{X}}+\boldsymbol{Y}_2^\top\boldsymbol{D}_{\boldsymbol{Y}} \right\|_F^2\\
	&+ \frac{1}{5} \sigma_r(\boldsymbol{M}) (\|\boldsymbol{D}_{\boldsymbol{X}}\|_F^2 + \|\boldsymbol{D}_{\boldsymbol{Y}}\|_F^2)   +\frac{1}{5} \sigma_r(\boldsymbol{M}) (\|\boldsymbol{D}_{\boldsymbol{X}}\|_F^2 + \|\boldsymbol{D}_{\boldsymbol{Y}}\|_F^2)\\
	&+\frac{1}{5}(\|\boldsymbol{D}_{\boldsymbol{X}}\boldsymbol{V}^\top\|_F^2 + \|\boldsymbol{U}\boldsymbol{D}_{\boldsymbol{Y}}^\top\|_F^2)\\
	\leqslant & \frac{6}{5}\sigma_1(\boldsymbol{M}) (\|\boldsymbol{D}_{\boldsymbol{X}}\|_F^2 + \|\boldsymbol{D}_{\boldsymbol{Y}}\|_F^2) + \|\boldsymbol{D}_{\boldsymbol{X}}^\top\boldsymbol{U}\|_F^2 + \| \boldsymbol{D}_{\boldsymbol{Y}}^\top\boldsymbol{V}\|_F^2 + \| \boldsymbol{X}_2^\top\boldsymbol{D}_{\boldsymbol{X}}\|_F^2+\|\boldsymbol{Y}_2^\top\boldsymbol{D}_{\boldsymbol{Y}}  \|_F^2 \\
	&+ \frac{2}{5} \sigma_r(\boldsymbol{M}) (\|\boldsymbol{D}_{\boldsymbol{X}}\|_F^2 + \|\boldsymbol{D}_{\boldsymbol{Y}}\|_F^2)  \\
	\leqslant &  \frac{13}{5}\sigma_1(\boldsymbol{M}) (\|\boldsymbol{D}_{\boldsymbol{X}}\|_F^2 + \|\boldsymbol{D}_{\boldsymbol{Y}}\|_F^2) + \|\boldsymbol{X}_2\|^2\|\boldsymbol{D}_{\boldsymbol{X}}\|_F^2 + \|\boldsymbol{Y}_2\|^2\|\boldsymbol{D}_{\boldsymbol{Y}}\|_F^2 \\ 
	\leqslant & 5\sigma_1(\boldsymbol{M}) (\|\boldsymbol{D}_{\boldsymbol{X}}\|_F^2 + \|\boldsymbol{D}_{\boldsymbol{Y}}\|_F^2),
	\end{split}
	\]
	where the last inequality we use the fact that 
	\[
		\left\| \left[ \begin{array}{c} \boldsymbol{X}_2-\boldsymbol{U}\\ \boldsymbol{Y}_2-\boldsymbol{V} \end{array}\right] \right\| \leqslant  \frac{1}{500 \kappa}\sqrt{\sigma_1(\boldsymbol{M})}. 
	\]
	Choosing $C_{S1} = C_1+C_2+C_3+14400C_3^2$ finishes the proof.
\end{proof}

\section{Proof of Lemma \ref{lemma_initialization}}
\label{sec:initialization_proof}
In this section we first summarize some useful lemmas from \cite{ma2017implicit}. We then follow the technical framework in \cite{ma2017implicit} but replace \citet[Lemma 39]{ma2017implicit} with \citet[Lemma 2]{chen2015incoherence} (Lemma \ref{chen_lemma2} in this paper) to get a better initialization guarantee.

\subsection{Useful lemmas}
\label{sec:initial_lemmas}
Here we summarize some useful lemmas in \cite{abbe2017entrywise} and \cite{ma2017implicit}. We relax the PSD assmptions on $\boldsymbol{M}_1$ in Lemma \ref{ma_lemma45}, Lemma \ref{ma_lemma46} and Lemma \ref{ma_lemma47} to symmetric assumptions by following the proof framework introduced in \cite{ma2017implicit}. In fact, lemmas listed in this section can be derived from Davis-Kahan Sin$\Theta$ theorem \citep{davis1970rotation}. We summarize lemmas here since they are intensively used throughout the proof. Moreover, for the simplicity of the expression, we made some additional assumptions on the eignevalues of $\boldsymbol{M}_1$ within the following lemmas (i.e., $\lambda_r(\boldsymbol{M}_1)>0$, $\lambda_r(\boldsymbol{M}_1)>\lambda_{r+1}(\boldsymbol{M}_1)$, $\lambda_{r+1}(\boldsymbol{M}_1) = 0$ and $\lambda_1(\boldsymbol{M}_1) = -\lambda_n(\boldsymbol{M}_1)$), the results still hold (with a more complicated expression) without those extra assumptions. Recall that here $\lambda_1(\boldsymbol{A})\geqslant \lambda_2(\boldsymbol{A})\geqslant \cdots \geqslant\lambda_n(\boldsymbol{A})$ stands for eigenvalues of symmetric matrix $\boldsymbol{A}\in\mathbb{R}^{n\times n}$.


First, we need a specified version of \citet[Lemma 3]{abbe2017entrywise}:

\begin{lemma}[{\citealt[Lemma 3]{abbe2017entrywise}}]\label{abbe_lemma3}
	Let $\boldsymbol{M}_1,\boldsymbol{M}_2\in\mathbb{R}^{n\times n}$ be two symmetric matrices with top-$r$ eigenvalue decomposition $\widetilde{\boldsymbol{U}}_1\boldsymbol{\Lambda}_1\widetilde{\boldsymbol{U}}_1^\top$ and $\widetilde{\boldsymbol{U}}_2\boldsymbol{\Lambda}_2\widetilde{\boldsymbol{U}}_2^\top$ correspondingly. Then if $\lambda_r(\boldsymbol{M}_1)>0$, $\lambda_r(\boldsymbol{M}_1)>\lambda_{r+1}(\boldsymbol{M}_1)$ and 
	\[
		\|\boldsymbol{M}_1-\boldsymbol{M}_2\|\leqslant \frac{1}{4}\min(\lambda_r(\boldsymbol{M}_1),\lambda_r(\boldsymbol{M}_1)-\lambda_{r+1}(\boldsymbol{M}_1)),
	\]
	we have
	\[
	\begin{split}
		 \|\widetilde{\boldsymbol{U}}_1^\top\widetilde{\boldsymbol{U}}_2 - \sgn(\widetilde{\boldsymbol{U}}_1^\top\widetilde{\boldsymbol{U}}_2)\| \leqslant & 4\frac{\|\boldsymbol{M}_1-\boldsymbol{M}_2\|^2}{\min\{\lambda_r(\boldsymbol{M}_1),\lambda_r(\boldsymbol{M}_1)-\lambda_{r+1}(\boldsymbol{M}_1)\}^2}
	\end{split}
	\]
	and 
	\[
		\|(\widetilde{\boldsymbol{U}}_1^\top\widetilde{\boldsymbol{U}}_2)^{-1}\| \leqslant 2.
	\]
\end{lemma}

And we also need some useful lemmas from \cite{ma2017implicit}:

\begin{lemma}[{\citealt[Lemma 45]{ma2017implicit}}]\label{ma_lemma45}
	Let $\boldsymbol{M}_1,\boldsymbol{M}_2\in\mathbb{R}^{n\times n}$ be symmetric matrices with top-$r$ eigenvalue decomposition $\widetilde{\boldsymbol{U}}_1\boldsymbol{\Lambda}_1\widetilde{\boldsymbol{U}}_1^\top$ and $\widetilde{\boldsymbol{U}}_2\boldsymbol{\Lambda}_2\widetilde{\boldsymbol{U}}_2^\top$ correspondingly. Assume $\lambda_{r}(\boldsymbol{M}_1)>0, \lambda_{r+1}(\boldsymbol{M}_1) = 0$ and $\|\boldsymbol{M}_1-\boldsymbol{M}_2\|\leqslant \frac{1}{4}\lambda_r(\boldsymbol{M}_1)$. Denote
	\[
		\widetilde{\boldsymbol{Q}} \coloneqq \argmin_{\boldsymbol{R}\in \mathsf{O}(r)}\|\widetilde{\boldsymbol{U}}_2\boldsymbol{R}-\widetilde{\boldsymbol{U}}_1\|_F.
	\]
	Then 
	\[
		\|\widetilde{\boldsymbol{U}}_2\widetilde{\boldsymbol{Q}} -\widetilde{\boldsymbol{U}}_1\|\leqslant \frac{3}{\lambda_r(\boldsymbol{M}_1)}\|\boldsymbol{M}_1-\boldsymbol{M}_2\|.
	\]
\end{lemma}

\begin{lemma}[{\citealt[Lemma 46]{ma2017implicit}}]\label{ma_lemma46}
	Let $\boldsymbol{M}_1,\boldsymbol{M}_2,\boldsymbol{M}_3\in\mathbb{R}^{n\times n}$ be symmetric matrices with top-$r$ eigenvalue decomposition $\widetilde{\boldsymbol{U}}_1\boldsymbol{\Lambda}_1\widetilde{\boldsymbol{U}}_1^\top$, $\widetilde{\boldsymbol{U}}_2\boldsymbol{\Lambda}_2\widetilde{\boldsymbol{U}}_2^\top$ and $\widetilde{\boldsymbol{U}}_3\boldsymbol{\Lambda}_3\widetilde{\boldsymbol{U}}_3^\top$ correspondingly. Assume $\lambda_1(\boldsymbol{M}_1) = -\lambda_n(\boldsymbol{M}_1)$, $\lambda_{r}(\boldsymbol{M}_1)>0$, $\lambda_{r+1}(\boldsymbol{M}_1) = 0$ and $\|\boldsymbol{M}_1-\boldsymbol{M}_2\|\leqslant \frac{1}{4}\lambda_r(\boldsymbol{M}_1)$, $\|\boldsymbol{M}_1-\boldsymbol{M}_3\|\leqslant \frac{1}{4}\lambda_r(\boldsymbol{M}_1)$. Denote
	\[
		\widetilde{\boldsymbol{Q}} \coloneqq \argmin_{\boldsymbol{R}\in \mathsf{O}(r)}\|\widetilde{\boldsymbol{U}}_2\boldsymbol{R}-\widetilde{\boldsymbol{U}}_3\|_F.
	\]
	Then 
	\[
		\|\boldsymbol{\Lambda}_2^{1/2}\widetilde{\boldsymbol{Q}} - \widetilde{\boldsymbol{Q}}\boldsymbol{\Lambda}_3^{1/2}\| \leqslant 15 \frac{\lambda_1(\boldsymbol{M}_1)}{\lambda_{r}^{3/2}(\boldsymbol{M}_1)}\|\boldsymbol{M}_2-\boldsymbol{M}_3\|
		\]
		and
		\[
		 \|\boldsymbol{\Lambda}_2^{1/2}\widetilde{\boldsymbol{Q}} - \widetilde{\boldsymbol{Q}}\boldsymbol{\Lambda}_3^{1/2}\|_F \leqslant 15 \frac{\lambda_1(\boldsymbol{M}_1)}{\lambda_{r}^{3/2}(\boldsymbol{M}_1)}\|(\boldsymbol{M}_2-\boldsymbol{M}_3)\widetilde{\boldsymbol{U}}_2\|_F.
	\]
\end{lemma}

\begin{lemma}[{\citealt[Lemma 47]{ma2017implicit}}]\label{ma_lemma47}
	Let $\boldsymbol{M}_1,\boldsymbol{M}_2\in\mathbb{R}^{n\times n}$ be symmetric matrices with top-$r$ eigenvalue decomposition $\widetilde{\boldsymbol{U}}_1\boldsymbol{\Lambda}_1\widetilde{\boldsymbol{U}}_1^\top$ and $\widetilde{\boldsymbol{U}}_2\boldsymbol{\Lambda}_2\widetilde{\boldsymbol{U}}_2^\top$ correspondingly. Assume $\lambda_1(\boldsymbol{M}_1) = -\lambda_n(\boldsymbol{M}_1),\lambda_{r}(\boldsymbol{M}_1)>0, \lambda_{r+1}(\boldsymbol{M}_1) = 0$ and 
	\[
		\|\boldsymbol{M}_1-\boldsymbol{M}_2\|\leqslant \frac{1}{40} \frac{\lambda_r^{5/2}(\boldsymbol{M}_1)}{\lambda_1^{3/2}(\boldsymbol{M}_1)}.
	\]
	Denote $\boldsymbol{X}_1 = \widetilde{\boldsymbol{U}}_1\boldsymbol{\Lambda}_1^{1/2}$ and $\boldsymbol{X}_2 = \widetilde{\boldsymbol{U}}_2\boldsymbol{\Lambda}_2^{1/2}$ and define
	\[
		\widetilde{\boldsymbol{Q}} \coloneqq \argmin_{\boldsymbol{R}\in \mathsf{O}(r)}\|\widetilde{\boldsymbol{U}}_2\boldsymbol{R}-\widetilde{\boldsymbol{U}}_1\|_F 
		\]
		and
		\[
		  \boldsymbol{H} \coloneqq \argmin_{\boldsymbol{R}\in \mathsf{O}(r)}\| \boldsymbol{X}_2\boldsymbol{R}-\boldsymbol{X}_1\|_F.
	\]
	Then
	\[
		\|\widetilde{\boldsymbol{Q}} - \boldsymbol{H}\|\leqslant 15\frac{\lambda_1^{3/2}(\boldsymbol{M}_1)}{ \lambda_r^{5/2}(\boldsymbol{M}_1) }\|\boldsymbol{M}_1-\boldsymbol{M}_2\|
	\]
	holds.
\end{lemma}

\subsection{Proof}
In this subsection, we will follow the technical framework in \cite{ma2017implicit}: First we give an upper bound of $\|\frac{1}{p}\mathcal{P}_{\Omega}(\boldsymbol{M})-\boldsymbol{M}\|$, and then prove Lemma \ref{lemma_initialization} by applying the lemmas introduced in Section \ref{sec:initial_lemmas}. As claimed before, here we replace \citet[Lemma 39]{ma2017implicit} with \citet[Lemma 2]{chen2015incoherence} to give an upper bound of $\|\frac{1}{p}\mathcal{P}_{\Omega}(\boldsymbol{M})-\boldsymbol{M}\|$ and obtain a tighter error bound of the initializations.




Define the symmetric matrix 
\begin{equation}\label{eq:Mbar}
	\overline{\boldsymbol{M}} \coloneqq \left[\begin{array}{cc}
		\boldsymbol{0} & \boldsymbol{M}\\\boldsymbol{M}^\top & \boldsymbol{0}
	\end{array}\right].
\end{equation}
The SVD $\mtx{M}=\widetilde{\mtx{U}}\mtx{\Sigma}\widetilde{\mtx{V}}^\top$ implies the following eigenvalue decomposition of $\overline{\mtx{M}}$:
\[
	\overline{\boldsymbol{M}}  = \frac{1}{\sqrt{2}} \left[ \begin{array}{cc} \widetilde{\boldsymbol{U}} & \widetilde{\boldsymbol{U}}\\ \widetilde{\boldsymbol{V}} & -\widetilde{\boldsymbol{V}}\end{array}\right] \left[\begin{array}{cc} \boldsymbol{\Sigma} & \boldsymbol{0}\\\boldsymbol{0}& -\boldsymbol{\Sigma}\end{array}\right]\frac{1}{\sqrt{2}} \left[ \begin{array}{cc} \widetilde{\boldsymbol{U}} & \widetilde{\boldsymbol{U}}\\ \widetilde{\boldsymbol{V}} & -\widetilde{\boldsymbol{V}}\end{array}\right]^\top. 
\]
From the eigenvalue decomposition, we can see $\lambda_1(\overline{\boldsymbol{M}}) = \sigma_1(\boldsymbol{M})$, $\cdots$, $\lambda_r(\overline{\boldsymbol{M}}) = \sigma_r(\boldsymbol{M})$, $\lambda_{r+1}(\overline{\boldsymbol{M}}) = 0$, $\cdots$, $\lambda_{n_1+n_2-r}(\overline{\boldsymbol{M}}) = 0$, $\lambda_{n_1+n_2-r+1}(\overline{\boldsymbol{M}}) = -\sigma_r(\boldsymbol{M})$, $\cdots$, $\lambda_{n_1+n_2}(\overline{\boldsymbol{M}}) = -\sigma_1(\boldsymbol{M})$. At the same time, we define
\[
	\frac{1}{p}\mathcal{P}_{\overline{\Omega}}(\overline{\boldsymbol{M}}) = \left[\begin{array}{cc}
		\boldsymbol{0} & \frac{1}{p}\mathcal{P}_{\Omega}(\boldsymbol{M})\\\frac{1}{p}\mathcal{P}_{\Omega}(\boldsymbol{M})^\top & \boldsymbol{0}
	\end{array}\right]
\]
with
\[
\begin{split}
	\overline{\Omega} \coloneqq \{(i,j)\mid& 1\leqslant i,j\leqslant n_1+n_2,\; (i,j-n_1)\in \Omega\\
	& \;\textrm{or}\; (j,i-n_1)\in\Omega\}.
\end{split}
\]

Applying Lemma \ref{chen_lemma2} on $\overline{\boldsymbol{M}}$ here, then 
\begin{equation}\label{eigen-gap} 
	\begin{split}
		 &\left\| \frac{1}{p}\mathcal{P}_{\overline{\Omega}}(\overline{\boldsymbol{M}}) - \overline{\boldsymbol{M}} \right\| \\
		 \leqslant & C_4 \left( \frac{\log (n_1+n_2)}{p} \|\overline{\boldsymbol{M}}\|_{\ell_{\infty}} +\sqrt{\frac{\log (n_1+n_2)}{p}}\|\overline{\boldsymbol{M}}\|_{2,\infty} \right)\\
		 \leqslant & C_4 \left( \frac{\log (n_1+n_2)}{p}\|\boldsymbol{U}\|_{2,\infty}\|\boldsymbol{V}\|_{2,\infty}  + \sqrt{\frac{\log (n_1+n_2)}{p}} (\|\boldsymbol{U}\|\|\boldsymbol{V}\|_{2,\infty} \vee \|\boldsymbol{V}\|\|\boldsymbol{U}\|_{2,\infty})\right)\\
		 \leqslant & 2C_4 \left( \frac{\mu r \kappa \log (n_1\vee n_2)}{(n_1\wedge n_2)p} + \sqrt{ \frac{\mu r \kappa \log (n_1\vee n_2)}{(n_1\wedge n_2)p}} \right)\sigma_1(\boldsymbol{M})\\
		 \leqslant & 4C_4  \sqrt{ \frac{\mu r \kappa \log (n_1\vee n_2)}{(n_1\wedge n_2)p}} \sigma_1(\boldsymbol{M})
	\end{split}
\end{equation}
	holds on an event $E_{Ch1}$ with probability $\mathbb{P}[E_{Ch1}]\geqslant 1-(n_1+n_2)^{-11}$. The last inequality holds given 
	\[
		p \geqslant  \frac{\mu r \kappa \log(n_1\vee n_2)}{n_1\wedge n_2}.
	\]
In addition if
\[
	p\geqslant 25600C_4^2 \frac{\mu r \kappa^6\log (n_1\vee n_2)}{n_1\wedge n_2},
\]
we have
\begin{equation}\label{eq_086}
	\left\| \frac{1}{p}\mathcal{P}_{\overline{\Omega}}(\overline{\boldsymbol{M}}) - \overline{\boldsymbol{M}} \right\|  \leqslant \frac{1}{40\sqrt{\kappa}^3}\sigma_r(\boldsymbol{M}) \leqslant \frac{1}{4}\sigma_r(\boldsymbol{M})
\end{equation}
holds on an event $E_{Ch1}$.

For the simplicity of notations, we denote $\overline{\boldsymbol{M}}^0$ as 
\begin{equation}\label{eq:M0bar}
	\overline{\boldsymbol{M}}^0 \coloneqq \left[ \begin{array}{cc}
		\boldsymbol{0} & \boldsymbol{M}^0\\
		(\boldsymbol{M}^0)^\top & \boldsymbol{0}
	\end{array}\right],
\end{equation}
and denote $\overline{\boldsymbol{M}}^{0,(l)}$ as 
\begin{equation}\label{eq:M0lbar}
	\overline{\boldsymbol{M}}^{0,(l)}\coloneqq \left[ \begin{array}{cc}
		\boldsymbol{0} & \boldsymbol{M}^{0,(l)} \\ (\boldsymbol{M}^{0,(l)})^\top & \boldsymbol{0}
	\end{array} \right].	 
\end{equation}
$\boldsymbol{M}^0$ and $\boldsymbol{M}^{0,(l)}$ are defined in \eqref{eq:M0} and \eqref{eq:M0l}, correspondingly.



Again by Lemma \ref{chen_lemma2}, we can see on an event $E_{Ch1}$, for all $1\leqslant l\leqslant n_1+n_2$, 
\[
	 \left\| \overline{\boldsymbol{M}}^{0,(l)} - \overline{\boldsymbol{M}} \right\|  \leqslant  4C_4  \sqrt{ \frac{\mu r \kappa \log (n_1\vee n_2)}{(n_1\wedge n_2)p}} \sigma_1(\boldsymbol{M}).
\]
If
\[
	p\geqslant 25600 C_4^2 \frac{\mu r \kappa^6\log (n_1\vee n_2)}{n_1\wedge n_2},
\]
we also have 
\begin{equation}\label{eq_039}
	 \left\| \overline{\boldsymbol{M}}^{0,(l)} - \overline{\boldsymbol{M}} \right\| \leqslant \frac{1}{40\sqrt{\kappa}^3}\sigma_r(\boldsymbol{M}) \leqslant \frac{1}{4}\sigma_r(\boldsymbol{M}).
\end{equation}

Now assume $\boldsymbol{M}^0$ has SVD $\boldsymbol{A}\boldsymbol{D}\boldsymbol{B}^\top$, then by construction, $\overline{\boldsymbol{M}}^0$ have following eigendecomposition:
\[
	\overline{\boldsymbol{M}}^0 = \frac{1}{\sqrt{2}} \left[ \begin{array}{cc} \boldsymbol{A}&\boldsymbol{A}\\\boldsymbol{B} & -\boldsymbol{B}\end{array} \right]\left[ \begin{array}{cc} \boldsymbol{D}&\boldsymbol{0}\\\boldsymbol{0}&-\boldsymbol{D}  \end{array} \right] \frac{1}{\sqrt{2}} \left[ \begin{array}{cc} \boldsymbol{A}&\boldsymbol{A}\\\boldsymbol{B} & -\boldsymbol{B}\end{array} \right]^\top.
\]
So if $\widetilde{\boldsymbol{X}}^0 \boldsymbol{\Sigma}^0(\widetilde{\boldsymbol{Y}}^0)^\top$ is the top-$r$ singular value decomposition of $\boldsymbol{M}^0$, we can also have 
\[
	\frac{1}{\sqrt{2}}\left[ \begin{array}{c}\widetilde{\boldsymbol{X}}^0 \\ \widetilde{\boldsymbol{Y}}^0 \end{array}\right]\boldsymbol{\Sigma}^0 \frac{1}{\sqrt{2}}\left[ \begin{array}{c}\widetilde{\boldsymbol{X}}^0 \\ \widetilde{\boldsymbol{Y}}^0 \end{array}\right]^\top
\]
to be the top-$r$ eigenvalue decomposition of $\overline{\boldsymbol{M}}^0$. So by Weyl's inequality and \eqref{eq_086}, we have
\begin{equation}\label{eq_d005}
	\frac{3}{4}\sigma_r(\boldsymbol{M})\leqslant \sigma_r(\boldsymbol{\Sigma}^0)\leqslant \sigma_1(\boldsymbol{\Sigma}^0)\leqslant 2\sigma_1(\boldsymbol{M}).
\end{equation}
Similarly, the same arguments also applies for $\overline{\boldsymbol{M}}^{0,(l)}$. From Weyl's inequality and \eqref{eq_039}, we have
\begin{equation}\label{eq_d006}
	\frac{3}{4}\sigma_r(\boldsymbol{M})\leqslant \sigma_r(\boldsymbol{\Sigma}^{0,(l)})\leqslant \sigma_1(\boldsymbol{\Sigma}^{0,(l)})\leqslant 2\sigma_1(\boldsymbol{M}).
\end{equation}
Now let $\boldsymbol{X}^0 \coloneqq \widetilde{\boldsymbol{X}}^0 (\boldsymbol{\Sigma}^0)^{1/2},\boldsymbol{Y}^0 \coloneqq \widetilde{\boldsymbol{Y}}^0 (\boldsymbol{\Sigma}^0)^{1/2}$, and $\boldsymbol{X}^{0,(l)} \coloneqq \widetilde{\boldsymbol{X}}^{0,(l)} (\boldsymbol{\Sigma}^{0,(l)})^{1/2},\boldsymbol{Y}^{0,(l)} \coloneqq \widetilde{\boldsymbol{Y}}^{0,(l)} (\boldsymbol{\Sigma}^{0,(l)})^{1/2}$, where $\boldsymbol{M}^{0,(l)}$ has top-$r$ singular value decomposition $\widetilde{\boldsymbol{X}}^{0,(l)}\boldsymbol{\Sigma}^{0,(l)}(\widetilde{\boldsymbol{Y}}^{0,(l)})^\top$. Let 
\[
	\widetilde{\boldsymbol{W}} \coloneqq \frac{1}{\sqrt{2}}\left[\begin{array}{c} \widetilde{\boldsymbol{U}}\\\widetilde{\boldsymbol{V}} \end{array}\right], \quad \boldsymbol{W} \coloneqq \frac{1}{\sqrt{2}}\left[\begin{array}{c} \boldsymbol{U}\\\boldsymbol{V} \end{array}\right],
	\]
	\[
	  \widetilde{\boldsymbol{Z}}^0 \coloneqq \frac{1}{\sqrt{2}}\left[ \begin{array}{c} \widetilde{\boldsymbol{X}}^0\\\widetilde{\boldsymbol{Y}}^0 \end{array} \right],\quad  \boldsymbol{Z}^0 \coloneqq \frac{1}{\sqrt{2}}\left[ \begin{array}{c} \boldsymbol{X}^0\\ \boldsymbol{Y}^0 \end{array} \right],
\]
and also we can denote
\begin{equation}\label{eq:Z0l}
\widetilde{\boldsymbol{Z}}^{0,(l)} \coloneqq \frac{1}{\sqrt{2}}\left[ \begin{array}{c} \widetilde{\boldsymbol{X}}^{0,(l)}\\\widetilde{\boldsymbol{Y}}^{0,(l)} \end{array} \right],\quad  \boldsymbol{Z}^{0,(l)} \coloneqq \frac{1}{\sqrt{2}}\left[ \begin{array}{c} \boldsymbol{X}^{0,(l)}\\ \boldsymbol{Y}^{0,(l)} \end{array} \right].
\end{equation}
Moreover, define
\[
	\boldsymbol{Q}^0 \coloneqq \argmin_{\boldsymbol{R}\in\mathsf{O}(r)}\left\| \widetilde{\boldsymbol{Z}}^0 \boldsymbol{R} -  \widetilde{\boldsymbol{W}}\right\|_F,
	\]
	\[
	 \boldsymbol{Q}^{0,(l)} \coloneqq \argmin_{\boldsymbol{R}\in\mathsf{O}(r)}\left\| \widetilde{\boldsymbol{Z}}^{0,(l)} \boldsymbol{R} -  \widetilde{\boldsymbol{W}}\right\|_F.
\]
\subsubsection{Proof for \eqref{eq:ini1}}
For spectral norm, by triangle inequality, we have
\begin{equation}\label{eq_d004}
	\begin{split}
		 \left\| \boldsymbol{Z}^0\boldsymbol{R}^0 - \boldsymbol{W} \right\|  =& \left\| \widetilde{\boldsymbol{Z}}^0 (\boldsymbol{\Sigma}^0)^{1/2}(\boldsymbol{R}^0-\boldsymbol{Q}^0) +  \widetilde{\boldsymbol{Z}}^0 \left( (\boldsymbol{\Sigma}^0)^{1/2}\boldsymbol{Q}^0 - \boldsymbol{Q}^0\boldsymbol{\Sigma}^{1/2} \right)  + \left( \widetilde{\boldsymbol{Z}}^0 \boldsymbol{Q}^0 -\widetilde{\boldsymbol{W}} \right)\boldsymbol{\Sigma}^{1/2} \right\|\\
		\leqslant &\|(\boldsymbol{\Sigma}^0)^{1/2}\| \left\|\boldsymbol{R}^0-\boldsymbol{Q}^0\right\| + \| (\boldsymbol{\Sigma}^0)^{1/2}\boldsymbol{Q}^0 - \boldsymbol{Q}^0\boldsymbol{\Sigma}^{1/2}\|  + \|\boldsymbol{\Sigma}^{1/2}\|\left\|  \widetilde{\boldsymbol{Z}}^0 \boldsymbol{Q}^0 -\widetilde{\boldsymbol{W}} \right\|.
	\end{split}
\end{equation}
Now applying Lemma \ref{ma_lemma47} with $\boldsymbol{M}_1 = \overline{\boldsymbol{M}},\boldsymbol{M}_2 = \overline{\boldsymbol{M}}^0$, we have
\begin{equation}\label{eq_d001}
	\|\boldsymbol{R}^0-\boldsymbol{Q}^0\| \leqslant 15 \frac{\sqrt{\kappa}^3}{\sigma_r(\boldsymbol{M})} \|\overline{\boldsymbol{M}} - \overline{\boldsymbol{M}}^0\|;
\end{equation}
applying Lemma \ref{ma_lemma46} with $\boldsymbol{M}_1 = \boldsymbol{M}_2 = \overline{\boldsymbol{M}}, \boldsymbol{M}_3 = \overline{\boldsymbol{M}}^0$, we have
\begin{equation}\label{eq_d002}
	\| (\boldsymbol{\Sigma}^0)^{1/2}\boldsymbol{Q}^0 - \boldsymbol{Q}^0\boldsymbol{\Sigma}^{1/2}\| \leqslant 15 \frac{\kappa}{\sqrt{\sigma_r(\boldsymbol{M})}}\|\overline{\boldsymbol{M}} - \overline{\boldsymbol{M}}^0\|;
\end{equation}
finally, applying Lemma \ref{ma_lemma45} with $\boldsymbol{M}_1 = \overline{\boldsymbol{M}},\boldsymbol{M}_2 = \overline{\boldsymbol{M}}^0$, we have
\begin{equation}\label{eq_d003}
	 \left\|  \widetilde{\boldsymbol{Z}}^0 \boldsymbol{Q}^0 -\widetilde{\boldsymbol{W}} \right\| \leqslant \frac{3}{\sigma_r(\boldsymbol{M})} \|\overline{\boldsymbol{M}} - \overline{\boldsymbol{M}}^0\|.
\end{equation}

Plugging the estimations \eqref{eq_d001}, \eqref{eq_d002} and \eqref{eq_d003} back to \eqref{eq_d004}, and using \eqref{eq_d005} and \eqref{eq_d006},
\begin{equation}\label{eq_042}
	\begin{split}
	 \left\| \left[\begin{array}{c}\boldsymbol{X}^0\\ \boldsymbol{Y}^0\end{array}\right]\boldsymbol{R}^0 - \left[\begin{array}{c}\boldsymbol{U}\\ \boldsymbol{V}\end{array}\right] \right\| = &\sqrt{2} \left\| \boldsymbol{Z}^0\boldsymbol{R}^0 - \boldsymbol{W} \right\|\\
	\leqslant & 30 \left(  \frac{\sqrt{\sigma_1(\boldsymbol{M})\kappa^3}}{\sigma_r(\boldsymbol{M})} + \frac{\kappa}{\sqrt{\sigma_r(\boldsymbol{M})}}+\frac{\sqrt{\sigma_1(\boldsymbol{M})}}{\sigma_r(\boldsymbol{M})}\right)\|\overline{\boldsymbol{M}} - \overline{\boldsymbol{M}}^0\| \\
	 \leqslant & 360C_4 \sqrt{\frac{\mu r \kappa^6\log(n_1\vee n_2)}{(n_1\wedge n_2)p}}\sqrt{\sigma_1(\boldsymbol{M})}
	\end{split}
\end{equation}
holds. For the last inequality we use the estimation \eqref{eigen-gap}.


\subsubsection{Proof for \eqref{eq:ini2}}

Now we start to consider the bound of $\left\| \left( \left[\begin{array}{c}\boldsymbol{X}^{0,(l)}\\ \boldsymbol{Y}^{0,(l)} \end{array}\right] \boldsymbol{R}^{0,(l)} - \left[\begin{array}{c}\boldsymbol{U}\\ \boldsymbol{V}\end{array}\right]\right)_{l,\cdot} \right\|_2$. By triangle inequality,
\begin{equation}\label{eq_051}
	\begin{split}
		  \left\| \left( \boldsymbol{Z}^{0,(l)} \boldsymbol{R}^{0,(l)} -  \boldsymbol{W} \right)_{l,\cdot} \right\|_2 =& \left\| \left( \boldsymbol{Z}^{0,(l)}\boldsymbol{R}^{0,(l)} -\boldsymbol{Z}^{0,(l)} \boldsymbol{Q}^{0,(l)}+\boldsymbol{Z}^{0,(l)}\boldsymbol{Q}^{0,(l)}- \boldsymbol{W}\right)_{l,\cdot} \right\|_2\\
		\leqslant & \left\| \left(\boldsymbol{Z}^{0,(l)} \boldsymbol{Q}^{0,(l)}- \boldsymbol{W}\right)_{l,\cdot} \right\|_2  + \left\| (\boldsymbol{Z}^{0,(l)}_{l,\cdot})^\top (\boldsymbol{R}^{0,(l)} - \boldsymbol{Q}^{0,(l)})  \right\|_2.
	\end{split}
\end{equation}
First we give a bound of the first term. Note
\[
	\boldsymbol{W} = \widetilde{\boldsymbol{W}}\boldsymbol{\Sigma}^{1/2} = \widetilde{\boldsymbol{W}}\boldsymbol{\Sigma}\widetilde{\boldsymbol{W}}^\top\widetilde{\boldsymbol{W}} \boldsymbol{\Sigma}^{-1/2} = \overline{\boldsymbol{M}}\widetilde{\boldsymbol{W}}\boldsymbol{\Sigma}^{-1/2},
\]
where the last equality holds since 
\[
	\begin{split}
	 \overline{\boldsymbol{M}} \widetilde{\boldsymbol{W}} =&  \frac{1}{\sqrt{2}} \left[ \begin{array}{cc} \widetilde{\boldsymbol{U}} & \widetilde{\boldsymbol{U}}\\ \widetilde{\boldsymbol{V}} & -\widetilde{\boldsymbol{V}}\end{array}\right] \left[\begin{array}{cc} \boldsymbol{\Sigma} & \boldsymbol{0}\\\boldsymbol{0}& -\boldsymbol{\Sigma}\end{array}\right] \cdot \frac{1}{\sqrt{2}} \left[ \begin{array}{cc} \widetilde{\boldsymbol{U}} & \widetilde{\boldsymbol{U}}\\ \widetilde{\boldsymbol{V}} & -\widetilde{\boldsymbol{V}}\end{array}\right]^\top \frac{1}{\sqrt{2}}\left[\begin{array}{c}
		\widetilde{\boldsymbol{U}} \\ \widetilde{\boldsymbol{V}}
	\end{array}\right]\\
	=&  \frac{1}{\sqrt{2}}\left[\begin{array}{c}
		\widetilde{\boldsymbol{U}} \\ \widetilde{\boldsymbol{V}}
	\end{array}\right] \boldsymbol{\Sigma}\frac{1}{\sqrt{2}}\left[\begin{array}{c}
		\widetilde{\boldsymbol{U}} \\ \widetilde{\boldsymbol{V}}
	\end{array}\right]^\top \frac{1}{\sqrt{2}}\left[\begin{array}{c}
		\widetilde{\boldsymbol{U}} \\ \widetilde{\boldsymbol{V}}
	\end{array}\right]  + \frac{1}{\sqrt{2}}\left[\begin{array}{c}
		\widetilde{\boldsymbol{U}} \\ -\widetilde{\boldsymbol{V}}
	\end{array}\right](-\boldsymbol{\Sigma})\frac{1}{\sqrt{2}}\left[\begin{array}{c}
		\widetilde{\boldsymbol{U}} \\ -\widetilde{\boldsymbol{V}}
	\end{array}\right]^\top \frac{1}{\sqrt{2}}\left[\begin{array}{c}
		\widetilde{\boldsymbol{U}} \\ \widetilde{\boldsymbol{V}}
	\end{array}\right]\\
	 =& \widetilde{\boldsymbol{W}}\boldsymbol{\Sigma}\widetilde{\boldsymbol{W}}^\top\widetilde{\boldsymbol{W}},
	\end{split}
\]
the last equality uses the fact that $\widetilde{\boldsymbol{U}}^\top\widetilde{\boldsymbol{U}} =  \boldsymbol{I} = \widetilde{\boldsymbol{V}}^\top\widetilde{\boldsymbol{V}}$. Similarly, we also have
\[
	\boldsymbol{Z}^{0,(l)} = \widetilde{\boldsymbol{Z}}^{0,(l)} (\boldsymbol{\Sigma}^{0,(l)})^{1/2} = \overline{\boldsymbol{M}}^{0,(l)}\widetilde{\boldsymbol{Z}}^{0,(l)} (\boldsymbol{\Sigma}^{0,(l)})^{-1/2}.
\]
By the way we define $\overline{\boldsymbol{M}}^{0,(l)}$ and $\overline{\boldsymbol{M}}$ in \eqref{eq:M0lbar} and \eqref{eq:Mbar}, $\overline{\boldsymbol{M}}^{0,(l)}_{l,\cdot} = \overline{\boldsymbol{M}}_{l,\cdot}$. By triangle inequality we have
\begin{equation}\label{eq_089}
	\begin{split}
		&\left\| \left( \boldsymbol{Z}^{0,(l)} \boldsymbol{Q}^{0,(l)} -  \boldsymbol{W} \right)_{l,\cdot}\right\|_2
		\\
		=& \left\| \left(\overline{\boldsymbol{M}}^{0,(l)}\widetilde{\boldsymbol{Z}}^{0,(l)} (\boldsymbol{\Sigma}^{0,(l)})^{-1/2}\boldsymbol{Q}^{0,(l)} - \overline{\boldsymbol{M}}\widetilde{\boldsymbol{W}}\boldsymbol{\Sigma}^{-1/2}  \right)_{l,\cdot} \right\|_2\\
		=& \left\|(\overline{\boldsymbol{M}}_{l,\cdot})^\top\left( \widetilde{\boldsymbol{Z}}^{0,(l)}(\boldsymbol{\Sigma}^{0,(l)})^{-1/2}\boldsymbol{Q}^{0,(l)} - \widetilde{\boldsymbol{W}}\boldsymbol{\Sigma}^{-1/2} \right)\right\|_2\\
		 =& \left\|(\overline{\boldsymbol{M}}_{l,\cdot})^\top \left( \widetilde{\boldsymbol{Z}}^{0,(l)}\left[(\boldsymbol{\Sigma}^{0,(l)})^{-1/2}\boldsymbol{Q}^{0,(l)} -\boldsymbol{Q}^{0,(l)}\boldsymbol{\Sigma}^{-1/2} \right]  +\left[\widetilde{\boldsymbol{Z}}^{0,(l)}\boldsymbol{Q}^{0,(l)}- \widetilde{\boldsymbol{W}}\right]\boldsymbol{\Sigma}^{-1/2} \right)\right\|_2\\
		 \leqslant & \|\overline{\boldsymbol{M}}_{l,\cdot}\|_2 \left( \|(\boldsymbol{\Sigma}^{0,(l)})^{-1/2}\boldsymbol{Q}^{0,(l)} -\boldsymbol{Q}^{0,(l)}\boldsymbol{\Sigma}^{-1/2} \| \vphantom{\frac{1}{\sqrt{\sigma_r(\boldsymbol{M})}}} +\|\widetilde{\boldsymbol{Z}}^{0,(l)}\boldsymbol{Q}^{0,(l)}- \widetilde{\boldsymbol{W}}\| \frac{1}{\sqrt{\sigma_r(\boldsymbol{M})}}\right).
	\end{split}
\end{equation}
By Lemma \ref{ma_lemma45} with $\boldsymbol{M}_1 = \overline{\boldsymbol{M}},\boldsymbol{M}_2 = \overline{\boldsymbol{M}}^{0,(l)}$, we have
\begin{equation}\label{eq_087}
	\|\widetilde{\boldsymbol{Z}}^{0,(l)}\boldsymbol{Q}^{0,(l)}- \widetilde{\boldsymbol{W}}\|\leqslant \frac{3}{\sigma_r(\boldsymbol{M})} \|\overline{\boldsymbol{M}} - \overline{\boldsymbol{M}}^{0,(l)}\|.
\end{equation}
By Lemma \ref{ma_lemma46} with $\boldsymbol{M}_1=\boldsymbol{M}_3 = \overline{\boldsymbol{M}},\boldsymbol{M}_2 = \overline{\boldsymbol{M}}^{0,(l)}$, we have
\begin{equation}\label{eq_088}
	\begin{split}
		  \|(\boldsymbol{\Sigma}^{0,(l)})^{-1/2}\boldsymbol{Q}^{0,(l)} -\boldsymbol{Q}^{0,(l)}\boldsymbol{\Sigma}^{-1/2} \|  =& \left\|(\boldsymbol{\Sigma}^{0,(l)})^{-1/2} \left(\boldsymbol{Q}^{0,(l)}\boldsymbol{\Sigma}^{1/2} -(\boldsymbol{\Sigma}^{0,(l)})^{1/2}\boldsymbol{Q}^{0,(l)}\right)\boldsymbol{\Sigma}^{-1/2}\right\|\\
		 \leqslant & \|(\boldsymbol{\Sigma}^{0,(l)})^{-1/2} \| \|\boldsymbol{\Sigma}^{-1/2}\| \|\boldsymbol{Q}^{0,(l)}\boldsymbol{\Sigma}^{1/2} -(\boldsymbol{\Sigma}^{0,(l)})^{1/2}\boldsymbol{Q}^{0,(l)}\|\\
		 \leqslant & \frac{20}{\sigma_r(\boldsymbol{M})}\frac{\kappa}{\sqrt{\sigma_r(\boldsymbol{M})}} \|\overline{\boldsymbol{M}} - \overline{\boldsymbol{M}}^{0,(l)}\|.
	\end{split}
\end{equation}
The last inequality uses the fact that $\sigma_r(\boldsymbol{\Sigma}^{0,(l)}) \geqslant \frac{3}{4}\sigma_r(\boldsymbol{M})$.

Putting estimations \eqref{eq_087} and \eqref{eq_088} together and plugging back to \eqref{eq_089} we have
\begin{equation}\label{eq_001}
	\begin{split}
		 \left\| \left( \boldsymbol{Z}^{0,(l)} \boldsymbol{Q}^{0,(l)} -  \boldsymbol{W} \right)_{l,\cdot}\right\|_2 \leqslant & \|\overline{\boldsymbol{M}}_{l,\cdot}\|_2 \frac{23\kappa}{\sqrt{\sigma_r(\boldsymbol{M})}^3} \|\overline{\boldsymbol{M}} - \overline{\boldsymbol{M}}^{0,(l)}\|\\
		\leqslant & \max(\|\boldsymbol{U}\|\|\boldsymbol{V}\|_{2,\infty} , \|\boldsymbol{V}\|\|\boldsymbol{U}\|_{2,\infty}) \times  \frac{92 C_4\kappa}{\sqrt{\sigma_r(\boldsymbol{M})}^3} \sqrt{ \frac{\mu r \kappa \log (n_1\vee n_2)}{(n_1\wedge n_2)p}} \sigma_1(\boldsymbol{M})\\
		\leqslant& 92 C_4 \sqrt{ \frac{\mu^2 r^2 \kappa^7 \log (n_1\vee n_2)}{(n_1\wedge n_2)^2p}} \sqrt{\sigma_1(\boldsymbol{M})}.
	\end{split}
\end{equation}
In order to control the second term in \eqref{eq_051}, note from \eqref{eq_001},
\[
	\begin{split}
		 \|\boldsymbol{Z}^{0,(l)}_{l,\cdot}\|_2 \leqslant & \|\boldsymbol{W}_{l,\cdot}\|_2 + \left\| \left( \boldsymbol{Z}^{0,(l)} \boldsymbol{Q}^{0,(l)} -  \boldsymbol{W} \right)_{l,\cdot}\right\|_2 \\
		\leqslant&  \|\boldsymbol{W}\|_{2,\infty} + \left\| \left( \boldsymbol{Z}^{0,(l)} \boldsymbol{Q}^{0,(l)} -  \boldsymbol{W} \right)_{l,\cdot}\right\|_2 \\
		\leqslant & \left( \sqrt{\frac{\mu r \kappa}{n_1\wedge n_2}}  + 92C_4\sqrt{ \frac{\mu^2 r^2 \kappa^7 \log (n_1\vee n_2)}{(n_1\wedge n_2)^2p}} \right)\sqrt{\sigma_1(\boldsymbol{M})}.
	\end{split}
\]
Then by Lemma \ref{ma_lemma47} with $\boldsymbol{M}_1 = \overline{\boldsymbol{M}},\boldsymbol{M}_2 = \overline{\boldsymbol{M}}^{0,(l)}$, we have
\[
	\begin{split}
		&\left\| (\boldsymbol{Z}^{0,(l)}_{l,\cdot})^\top (\boldsymbol{R}^{0,(l)} - \boldsymbol{Q}^{0,(l)})  \right\|_2 \\
		\leqslant & \|\boldsymbol{Z}^{0,(l)}_{l,\cdot}\|_2 \|\boldsymbol{R}^{0,(l)} - \boldsymbol{Q}^{0,(l)}\|\\
		\leqslant & \|\boldsymbol{Z}^{0,(l)}_{l,\cdot}\|_2 15 \frac{\sqrt{\kappa}^3}{\sigma_r(\boldsymbol{M})}\|\overline{\boldsymbol{M}} - \overline{\boldsymbol{M}}^{0,(l)}\|\\
		\leqslant & 60C_4 \frac{\sqrt{\kappa}^3}{\sigma_r(\boldsymbol{M})} \sqrt{ \frac{\mu r \kappa \log (n_1\vee n_2)}{(n_1\wedge n_2)p}} \sigma_1(\boldsymbol{M})  \times \left( \sqrt{\frac{\mu r \kappa}{n_1\wedge n_2}}  + 92C_4\sqrt{ \frac{\mu^2 r^2 \kappa^7 \log (n_1\vee n_2)}{(n_1\wedge n_2)^2p}} \right)\sqrt{\sigma_1(\boldsymbol{M})}.
	\end{split}
\]
So as long as we have
\[
	p\geqslant 92^2C_4^2\frac{\mu r \kappa^6\log (n_1\vee n_2)}{n_1 \wedge n_2},	
\]
 then  
	\begin{equation}\label{eq_002}
	\begin{split}
		 \left\| (\boldsymbol{Z}^{0,(l)}_{l,\cdot})^\top (\boldsymbol{R}^{0,(l)} - \boldsymbol{Q}^{0,(l)})  \right\|_2 \leqslant & 120C_4  \sqrt{ \frac{\mu^2 r^2 \kappa^7 \log (n_1\vee n_2)}{(n_1\wedge n_2)^2p}} \sqrt{\sigma_1(\boldsymbol{M})}.
	\end{split}
	\end{equation}
Putting estimation \eqref{eq_001} and \eqref{eq_002} together we have
\begin{equation}
\label{eq_043}
\begin{split}
	 \left\| \left( \left[\begin{array}{c}\boldsymbol{X}^{0,(l)}\\ \boldsymbol{Y}^{0,(l)} \end{array}\right] \boldsymbol{R}^{0,(l)} - \left[\begin{array}{c}\boldsymbol{U}\\ \boldsymbol{V}\end{array}\right]\right)_{l,\cdot} \right\|_2 =& \sqrt{2}\left\| \left( \boldsymbol{Z}^{0,(l)} \boldsymbol{R}^{0,(l)} -  \boldsymbol{W} \right)_{l,\cdot} \right\|_2  \\ 
	\leqslant & 212\sqrt{2}C_4  \sqrt{ \frac{\mu^2 r^2 \kappa^7 \log (n_1\vee n_2)}{(n_1\wedge n_2)^2p}} \sqrt{\sigma_1(\boldsymbol{M})}.
	\end{split}
\end{equation}

\subsubsection{Proof for \eqref{eq:ini3}}

Finally, we want to give a bound for $\left\| \left[ \begin{array}{c}\boldsymbol{X}^0\\\boldsymbol{Y}^0\end{array}\right]\boldsymbol{R}^0 -\left[ \begin{array}{c}\boldsymbol{X}^{0,(l)}\\\boldsymbol{Y}^{0,(l)}\end{array}\right]\boldsymbol{T}^{0,(l)}  \right\|_F$. Without loss of generality, assume that $l$ satisfies $1\leqslant l\leqslant n_1$. First denote 
\[
\boldsymbol{B} \coloneqq \argmin_{\boldsymbol{R}\in\mathsf{O}(r)} \|\widetilde{\boldsymbol{Z}}^{0,(l)}\boldsymbol{R}-\widetilde{\boldsymbol{Z}}^0\|_F.
\]
From the choice of $\boldsymbol{T}^{0,(l)}$ in \eqref{eq:T0l}, we have
\begin{equation}\label{eq_028}
	\left\| \boldsymbol{Z}^0\boldsymbol{R}^0 -\boldsymbol{Z}^{0,(l)}\boldsymbol{T}^{0,(l)}  \right\|_F \leqslant  \|\boldsymbol{Z}^{0,(l)}\boldsymbol{B} - \boldsymbol{Z}^0\|_F.
\end{equation}
By triangle inequality,
\begin{equation}\label{eq_029}
	\begin{split}
		&\|\boldsymbol{Z}^{0,(l)}\boldsymbol{B} - \boldsymbol{Z}^0\|_F\\
		 =& \left\| \widetilde{\boldsymbol{Z}}^{0,(l)}(\boldsymbol{\Sigma}^{0,(l)})^{1/2}\boldsymbol{B} - \widetilde{\boldsymbol{Z}}^0(\boldsymbol{\Sigma}^0)^{1/2} \right\|_F\\
		=& \left\| \widetilde{\boldsymbol{Z}}^{0,(l)}\left[ (\boldsymbol{\Sigma}^{0,(l)})^{1/2}\boldsymbol{B} - \boldsymbol{B}(\boldsymbol{\Sigma}^0)^{1/2} \right]  + (\widetilde{\boldsymbol{Z}}^{0,(l)}\boldsymbol{B} - \widetilde{\boldsymbol{Z}}^0)(\boldsymbol{\Sigma}^0)^{1/2} \right\|_F \\
		\leqslant & \left\| \widetilde{\boldsymbol{Z}}^{0,(l)}\left[ (\boldsymbol{\Sigma}^{0,(l)})^{1/2}\boldsymbol{B} - \boldsymbol{B}(\boldsymbol{\Sigma}^0)^{1/2} \right] \right\|_F + \left\|(\widetilde{\boldsymbol{Z}}^{0,(l)}\boldsymbol{B} - \widetilde{\boldsymbol{Z}}^0)(\boldsymbol{\Sigma}^0)^{1/2} \right\|_F\\
		\leqslant & \left\|  (\boldsymbol{\Sigma}^{0,(l)})^{1/2}\boldsymbol{B} - \boldsymbol{B}(\boldsymbol{\Sigma}^0)^{1/2}  \right\|_F   + \left\| \widetilde{\boldsymbol{Z}}^{0,(l)}\boldsymbol{B} - \widetilde{\boldsymbol{Z}}^0 \right\|_F\|(\boldsymbol{\Sigma}^0)^{1/2}\|.
	\end{split}
\end{equation}
By Lemma \ref{ma_lemma46} with $\boldsymbol{M}_1 = \overline{\boldsymbol{M}},\boldsymbol{M}_2 = \overline{\boldsymbol{M}}^{0,(l)},\boldsymbol{M}_3 = \overline{\boldsymbol{M}}^{0}$, we have
\begin{equation}\label{eq_030}
\begin{split}
	 \left\| (\boldsymbol{\Sigma}^{0,(l)})^{1/2}\boldsymbol{B} - \boldsymbol{B}(\boldsymbol{\Sigma}^0)^{1/2} \right\|_F \leqslant & 15 \frac{\kappa}{\sqrt{\sigma_r(\boldsymbol{M})}}\left\| \left(\overline{\boldsymbol{M}}^0 - \overline{\boldsymbol{M}}^{0,(l)}\right)\widetilde{\boldsymbol{Z}}^{0,(l)} \right\|_F.
\end{split}
\end{equation}

Moreover, by Davis-Kahan Sin$\Theta$ theorem \citep{davis1970rotation}, we have
\begin{equation}\label{eq_031}
	\begin{split}
		 \left\| \widetilde{\boldsymbol{Z}}^{0,(l)}\boldsymbol{B} - \widetilde{\boldsymbol{Z}}^0 \right\|_F \leqslant & \sqrt{2}\left\| \left(\boldsymbol{I} - \widetilde{\boldsymbol{Z}}^0(\widetilde{\boldsymbol{Z}}^0)^\top\right) \widetilde{\boldsymbol{Z}}^{0,(l)} \right\|_F\\
		\leqslant & \frac{2\sqrt{2}}{\sigma_r(\boldsymbol{M})}\left\| \left(\overline{\boldsymbol{M}}^0 - \overline{\boldsymbol{M}}^{0,(l)}\right)\widetilde{\boldsymbol{Z}}^{0,(l)} \right\|_F.
	\end{split}
\end{equation}
So putting the estimations \eqref{eq_029}, \eqref{eq_030} and \eqref{eq_031} together we have
\begin{equation}\label{eq_032}
	\begin{split}
		& \|\boldsymbol{Z}^{0,(l)}\boldsymbol{B} - \boldsymbol{Z}^0\|_F\\
		  \leqslant & 15 \frac{\kappa}{\sqrt{\sigma_r(\boldsymbol{M})}}  \left\| \left(\overline{\boldsymbol{M}}^0 - \overline{\boldsymbol{M}}^{0,(l)}\right)\widetilde{\boldsymbol{Z}}^{0,(l)} \right\|_F  + 4\frac{\sqrt{\kappa}}{\sqrt{\sigma_r(\boldsymbol{M})}}  \left\| \left(\overline{\boldsymbol{M}}^0 - \overline{\boldsymbol{M}}^{0,(l)}\right)\widetilde{\boldsymbol{Z}}^{0,(l)} \right\|_F\\
		\leqslant & 20 \frac{\kappa}{\sqrt{\sigma_r(\boldsymbol{M})}}  \left\| \left(\overline{\boldsymbol{M}}^0 - \overline{\boldsymbol{M}}^{0,(l)}\right)\widetilde{\boldsymbol{Z}}^{0,(l)} \right\|_F.
	\end{split}
\end{equation}

By the way we define $\overline{\boldsymbol{M}}^0$ and $\overline{\boldsymbol{M}}^{0,(l)}$ in \eqref{eq:M0bar} and \eqref{eq:M0lbar}, 
\[
	\begin{split}
		 \left(\overline{\boldsymbol{M}}^0 - \overline{\boldsymbol{M}}^{0,(l)}\right) \widetilde{\boldsymbol{Z}}^{0,(l)}  =&  \left[ \begin{array}{c}
		\boldsymbol{0}\\
		\vdots\\
		\boldsymbol{0}\\
		\sum_j \left( \frac{1}{p}\delta_{l,j}-1 \right)\overline{M}_{l,n_1+j} (\widetilde{\boldsymbol{Z}}_{n_1+j,\cdot}^{0,(l)})^\top\\
		\boldsymbol{0}\\
		\vdots\\
		\boldsymbol{0}\\
		\left( \frac{1}{p}\delta_{l,1} - 1 \right)\overline{M}_{n_1+1,l}(\widetilde{\boldsymbol{Z}}_{l,\cdot}^{0,(l)})^\top\\
		\vdots\\
		\left( \frac{1}{p}\delta_{l,j} - 1 \right)\overline{M}_{n_1+j,l}(\widetilde{\boldsymbol{Z}}_{l,\cdot}^{0,(l)})^\top\\
		\vdots\\
		\left( \frac{1}{p}\delta_{l,n_2} - 1 \right)\overline{M}_{n_1+n_2,l}(\widetilde{\boldsymbol{Z}}_{l,\cdot}^{0,(l)})^\top\\ 
		\end{array} \right].
	\end{split}
\]
Recall that here we assume $1\leqslant l\leqslant n_1$. Therefore by triangle inequality,
\begin{equation}\label{eq_092}
	\begin{split}
		&\left\|\left(\overline{\boldsymbol{M}}^0 - \overline{\boldsymbol{M}}^{0,(l)}\right) \widetilde{\boldsymbol{Z}}^{0,(l)}\right\|_F\\
		 \leqslant & \left\| \sum_j \left( \frac{1}{p}\delta_{l,j}-1 \right)\overline{M}_{l,n_1+j} \widetilde{\boldsymbol{Z}}_{n_1+j,\cdot}^{0,(l)} \right\|_2  + \left\| \left[ \begin{array}{c}
			\left( \frac{1}{p}\delta_{l,1} - 1 \right)\overline{M}_{n_1+1,l}(\widetilde{\boldsymbol{Z}}_{l,\cdot}^{0,(l)})^\top\\
			\vdots\\
			\left( \frac{1}{p}\delta_{l,j} - 1 \right)\overline{M}_{n_1+j,l}(\widetilde{\boldsymbol{Z}}_{l,\cdot}^{0,(l)})^\top\\
			\vdots\\
			\left( \frac{1}{p}\delta_{l,n_2} - 1 \right)\overline{M}_{n_1+n_2,l}(\widetilde{\boldsymbol{Z}}_{l,\cdot}^{0,(l)})^\top\\ 
		\end{array} \right] \right\|_F\\
		=&  \left\| \underbrace{\sum_j \left( \frac{1}{p}\delta_{l,j}-1 \right)\overline{M}_{l,n_1+j} \widetilde{\boldsymbol{Z}}_{n_1+j,\cdot}^{0,(l)}}_{\boldsymbol{a}_1} \right\|_2  + \left\| \underbrace{\left[ \begin{array}{c}
		\left( \frac{1}{p}\delta_{l,1} - 1 \right)\overline{M}_{n_1+1,l}\\
		\vdots\\
		\left( \frac{1}{p}\delta_{l,j} - 1 \right)\overline{M}_{n_1+j,l}\\
		\vdots\\
		\left( \frac{1}{p}\delta_{l,n_2} - 1 \right)\overline{M}_{n_1+n_2,l}\\
		\end{array} \right]}_{\boldsymbol{a}_2} \right\|_2\|\widetilde{\boldsymbol{Z}}_{l,\cdot}^{0,(l)}\|_2.
	\end{split}
\end{equation}

Note by \eqref{eq:Z0l} and the fact that $\boldsymbol{M}^{0,(l)}$ has top-$r$ singular value decomposition $\widetilde{\boldsymbol{X}}^{0,(l)}\boldsymbol{\Sigma}^{0,(l)}(\widetilde{\boldsymbol{Y}}^{0,(l)})^\top$, $\widetilde{\boldsymbol{Z}}_{n_1+j,\cdot}^{0,(l)}$ is independent of $\delta_{l,j}$'s. For $\boldsymbol{a}_1$,
\[
	\boldsymbol{a}_1 = \sum_j\left( \frac{1}{p}\delta_{l,j} - 1 \right)\overline{M}_{l,n_1+j}\widetilde{\boldsymbol{Z}}_{n_1+j,\cdot}^{0,(l)} \coloneqq \sum_j \boldsymbol{s}_{1,j}.
\]
Conditioned on $\widetilde{\boldsymbol{Z}}_{n_1+j,\cdot}^{0,(l)}$, $\boldsymbol{s}_{1,j}$'s are independent, and $\mathbb{E}_{\delta_{l,\cdot}} \boldsymbol{s}_{1,j} = \boldsymbol{0}$. We also have 
\[
\begin{split}
	\|\boldsymbol{s}_{1,j}\|_2 \leqslant& \frac{1}{p}\|\overline{\boldsymbol{M}}\|_{\ell_{\infty}}\|\widetilde{\boldsymbol{Z}}^{0,(l)}\|_{2,\infty}\\
	 \leqslant& \frac{1}{p}\|\boldsymbol{U}\|_{2,\infty}\|\boldsymbol{V}\|_{2,\infty} \|\widetilde{\boldsymbol{Z}}^{0,(l)}\|_{2,\infty},
\end{split}
\]
and
\[
	\begin{split}
		 \left\|\mathbb{E}_{\delta_{l,\cdot}}\sum_j \boldsymbol{s}_{1,j}^\top\boldsymbol{s}_{1,j}\right\|  =& \sum_j \mathbb{E}_{\delta_{l,\cdot}}\left( \frac{1}{p}\delta_{l,j} - 1 \right)^2 \overline{M}_{l,n_1+j}^2 \|\widetilde{\boldsymbol{Z}}_{n_1+j,\cdot}^{0,(l)}\|_2^2\\
		\leqslant &\frac{1}{p} \|\widetilde{\boldsymbol{Z}}^{0,(l)}\|_{2,\infty}^2 \|\overline{\boldsymbol{M}}_{l,\cdot}\|_2^2\\
		\leqslant & \frac{1}{p} \|\widetilde{\boldsymbol{Z}}^{0,(l)}\|_{2,\infty}^2 \max\left(\|\boldsymbol{U}\|\|\boldsymbol{V}\|_{2,\infty}, \|\boldsymbol{V}\|\|\boldsymbol{U}\|_{2,\infty}\right)^2.
	\end{split}
\]
For $\left\| \sum_j \mathbb{E}_{\delta_{l,\cdot}} \boldsymbol{s}_{1,j}\boldsymbol{s}_{1,j}^\top \right\|$ we have the same bound. Then by matrix Bernstein inequality \cite[Theorem 6.1.1]{tropp2015introduction},  
\[
	\begin{split}
		&\mathbb{P}\left[  \|\boldsymbol{a}_1\|_2 \geqslant 100\left(  \sqrt{\frac{\log (n_1\vee n_2)}{p}}   \max\left(\|\boldsymbol{U}\|\|\boldsymbol{V}\|_{2,\infty}, \|\boldsymbol{V}\|\|\boldsymbol{U}\|_{2,\infty}\right)    + \frac{\log (n_1\vee n_2)}{p} \|\boldsymbol{U}\|_{2,\infty}\|\boldsymbol{V}\|_{2,\infty}  \vphantom{\sqrt{\frac{\log (n_1\vee n_2)}{p}}}  \right)  \|\widetilde{\boldsymbol{Z}}^{0,(l)}\|_{2,\infty} \mid \widetilde{\boldsymbol{Z}}^{0,(l)} \right]\\
		 \leqslant & (n_1+n_2)^{-15}.
	\end{split}
\] 
Therefore,
\[
	\begin{split}
		&\mathbb{P}\left[  \|\boldsymbol{a}_1\|_2 \geqslant 100\left(  \sqrt{\frac{\log (n_1\vee n_2)}{p}}  \max\left(\|\boldsymbol{U}\|\|\boldsymbol{V}\|_{2,\infty}, \|\boldsymbol{V}\|\|\boldsymbol{U}\|_{2,\infty}\right)  + \frac{\log (n_1\vee n_2)}{p} \|\boldsymbol{U}\|_{2,\infty}\|\boldsymbol{V}\|_{2,\infty}  \vphantom{\sqrt{\frac{\log (n_1\vee n_2)}{p}}}  \right) \|\widetilde{\boldsymbol{Z}}^{0,(l)}\|_{2,\infty} \right]\\
		=& \mathbb{E}\left[ \mathbb{E}\left[ \mathds{1}_{\|\boldsymbol{a}_1\|_2 \geqslant 100\left(  \sqrt{\frac{\log (n_1\vee n_2)}{p}}  \max\left(\|\boldsymbol{U}\|\|\boldsymbol{V}\|_{2,\infty}, \|\boldsymbol{V}\|\|\boldsymbol{U}\|_{2,\infty}\right)   + \frac{\log (n_1\vee n_2)}{p} \|\boldsymbol{U}\|_{2,\infty}\|\boldsymbol{V}\|_{2,\infty}  \vphantom{\sqrt{\frac{\log (n_1\vee n_2)}{p}}}  \right)  \|\widetilde{\boldsymbol{Z}}^{0,(l)}\|_{2,\infty}} \mid \widetilde{\boldsymbol{Z}}^{0,(l)} \right] \right]\\
		\leqslant & (n_1+n_2)^{-15}.
	\end{split}	
\]
In other words, on an event $E_{B}^{0,(l),1}$ with probability $\mathbb{P}[E_{B}^{0,(l),1}]\geqslant 1-(n_1+n_2)^{-15}$, we have
\begin{equation}\label{eq_090}
	\begin{split}
	 \|\boldsymbol{a}_1\|_2 \leqslant& 100\sqrt{\frac{\log (n_1\vee n_2)}{p}}\|\widetilde{\boldsymbol{Z}}^{0,(l)}\|_{2,\infty} \max\left(\|\boldsymbol{U}\|\|\boldsymbol{V}\|_{2,\infty}, \|\boldsymbol{V}\|\|\boldsymbol{U}\|_{2,\infty}\right) \\
	 &+ 100\frac{\log (n_1\vee n_2)}{p} \|\boldsymbol{U}\|_{2,\infty}\|\boldsymbol{V}\|_{2,\infty} \|\widetilde{\boldsymbol{Z}}^{0,(l)}\|_{2,\infty}\\
	\leqslant & 100 \left( \sqrt{\frac{\mu r \kappa \log (n_1\vee n_2) }{(n_1\wedge n_2)p}} + \frac{\mu r \kappa \log (n_1\vee n_2) }{(n_1\wedge n_2)p} \right) \sigma_1(\boldsymbol{M}) \|\widetilde{\boldsymbol{Z}}^{0,(l)}\|_{2,\infty}.
	\end{split}
\end{equation}

For $\boldsymbol{a}_2$, we can decompose it as

\[
\begin{split}
	\boldsymbol{a}_2 =& \left[ \begin{array}{c}
		\left( \frac{1}{p}\delta_{l,1} - 1 \right)\overline{M}_{n_1+1,l}\\
		\vdots\\
		\left( \frac{1}{p}\delta_{l,j} - 1 \right)\overline{M}_{n_1+j,l}\\
		\vdots\\
		\left( \frac{1}{p}\delta_{l,n_2} - 1 \right)\overline{M}_{n_1+n_2,l}\\
		\end{array} \right] \\
		=& \sum_j \left( \frac{1}{p}\delta_{l,j}-1 \right) \overline{M}_{n_1+j,l} \boldsymbol{e}_j\\
		 =& \sum_j \boldsymbol{s}_{2,j}.
		\end{split}
\]
Then we have $\mathbb{E}\boldsymbol{s}_{2,j} = \boldsymbol{0}$, 
\[
\|\boldsymbol{s}_{2,j}\|_2 \leqslant \frac{1}{p}\|\overline{\boldsymbol{M}}\|_{\ell_{\infty}}\leqslant \frac{1}{p}\|\boldsymbol{U}\|_{2,\infty}\|\boldsymbol{V}\|_{2,\infty} 
\]
and
\[
	\begin{split}
		\|\mathbb{E}\sum_j \boldsymbol{s}_{2,j}\boldsymbol{s}_{2,j}^\top\| \leqslant & \sum_j \mathbb{E}\|\boldsymbol{s}_{2,j}\|_2^2\\
		=& \sum_j \mathbb{E}\left( \frac{1}{p}\delta_{l,j}-1 \right)^2 \overline{M}_{n_1+j,l}^2\\
		\leqslant & \sum_j \frac{1}{p} \overline{M}_{n_1+j,l}^2\\
		\leqslant & \frac{1}{p} \max\left(\|\boldsymbol{U}\|\|\boldsymbol{V}\|_{2,\infty}, \|\boldsymbol{V}\|\|\boldsymbol{U}\|_{2,\infty}\right)^2.
	\end{split}
\]
Therefore by matrix Bernstein inequality \cite[Theorem 6.1.1]{tropp2015introduction} again, on an event $E_B^{0,2}$ with probability $\mathbb{P}[E_B^{0,2}]\geqslant 1-(n_1+n_2)^{-15}$, we have
\begin{equation}\label{eq_091}
\begin{split}
	 \|\boldsymbol{a}_2\|_2 \leqslant & 100\left( \sqrt{\frac{\mu r \kappa \log (n_1\vee n_2) }{(n_1\wedge n_2)p}} + \frac{\mu r \kappa \log (n_1\vee n_2) }{(n_1\wedge n_2)p} \right)\sigma_1(\boldsymbol{M}).
\end{split}
\end{equation}
So putting \eqref{eq_092}, \eqref{eq_090} and \eqref{eq_091} together we have
\begin{equation}\label{eq_033}
	\begin{split}
		 \left\|\left(\overline{\boldsymbol{M}}^0 - \overline{\boldsymbol{M}}^{0,(l)}\right) \widetilde{\boldsymbol{Z}}^{0,(l)}\right\|_F \leqslant & \|\boldsymbol{a}_1\|_2 + \|\boldsymbol{a}_2\|_2\|\widetilde{\boldsymbol{Z}}^{0,(l)}\|_{2,\infty}\\
		\leqslant & 200 \left( \sqrt{\frac{\mu r \kappa \log (n_1\vee n_2) }{(n_1\wedge n_2)p}} + \frac{\mu r \kappa \log (n_1\vee n_2) }{(n_1\wedge n_2)p} \right) \sigma_1(\boldsymbol{M}) \|\widetilde{\boldsymbol{Z}}^{0,(l)}\|_{2,\infty}
	\end{split}
\end{equation}
on an event $E_B^0 = \left(\bigcap_{l=1}^{n_1+n_2}E_{B}^{0,(l),1}\right)\bigcap E_B^{0,2}$. Moreover, by applying union bound we have $\mathbb{P}[E_B^0] \geqslant 1-(n_1+n_2)^{-11}$.

Now we need to bound $\|\widetilde{\boldsymbol{Z}}^{0,(l)}\|_{2,\infty}$. We have the following claim:
\begin{claim}\label{claim_01}
	Under the setup of Lemma \ref{lemma_initialization}, on an event $E_{Claim}$ with probability $\mathbb{P}[E_{Claim}] \geqslant 1- 3(n_1+n_2)^{-11}$, the following inequality
	\begin{equation}\label{eq_034}
	\begin{split}
		\|\widetilde{\boldsymbol{Z}}^{0,(l)}\|_{2,\infty} \leqslant&  (4+4\kappa+9C_5 \kappa^2)  \|\widetilde{\boldsymbol{W}}\|_{2,\infty}\\
		 \leqslant& (8+9C_5) \kappa^2\frac{1}{\sqrt{\sigma_r(\boldsymbol{M})}}\|\boldsymbol{W}\|_{2,\infty}
	\end{split}
	\end{equation}
	holds with the absolute constant $C_5$ defined in Lemma \ref{abbe_lemma4}.
\end{claim}

If the claim is true, from \eqref{eq_028}, \eqref{eq_032}, \eqref{eq_033} and \eqref{eq_034} and if 
\[
	p \geqslant \frac{\mu r \kappa \log(n_1\vee n_2)}{n_1\wedge n_2},	
\]
then
\begin{equation}\label{eq_044}
\begin{split}
	 \left\| \left[ \begin{array}{c}\boldsymbol{X}^0\\\boldsymbol{Y}^0\end{array}\right]\boldsymbol{R}^0 -\left[ \begin{array}{c}\boldsymbol{X}^{0,(l)}\\\boldsymbol{Y}^{0,(l)}\end{array}\right]\boldsymbol{T}^{0,(l)}  \right\|_F =& \sqrt{2}\left\| \boldsymbol{Z}^0\boldsymbol{R}^0 -\boldsymbol{Z}^{0,(l)}\boldsymbol{T}^{0,(l)}  \right\|_F\\
	 \leqslant & 20\sqrt{2}\frac{\kappa}{\sqrt{\sigma_r(\boldsymbol{M})}}  \left\| \left(\overline{\boldsymbol{M}}^0 - \overline{\boldsymbol{M}}^{0,(l)}\right)\widetilde{\boldsymbol{Z}}^{0,(l)} \right\|_F \\
	 \leqslant & (64000\sqrt{2}+72000\sqrt{2}C_5)  \sqrt{\frac{\mu^2 r^2 \kappa^{10} \log (n_1\vee n_2)}{(n_1\wedge n_2)^2p}}\sqrt{\sigma_1(\boldsymbol{M})}
\end{split}
\end{equation}
holds for any $l$ satisfying $1\leqslant l\leqslant n_1$. For the case $n_1+1\leqslant l\leqslant n_1+n_2$, we can use the same argument.

Note on an event 
\begin{align*}
E_{init} &= E_{Ch1}\bigcap E_{Claim}\bigcap E_{H}\bigcap E_B^0 
\\
&= E_S\bigcap E_{Ca}\bigcap E_Z\bigcap E_{Ch1}\bigcap E_{Ch2}\bigcap E_A\bigcap E_B^0,
\end{align*}
 \eqref{eq_042}, \eqref{eq_043} and \eqref{eq_044} hold. Choosing $C_I$ to be 
\[
	C_I =  64000\sqrt{2}+212\sqrt{2}C_4 +72000\sqrt{2}C_5
\]
and $C_{S2}$ to be
\[
	C_{S2} = 256+ 25600 C_4^2 + C_5,
\]
using union bound $\mathbb{P}[E_{init}]\geqslant 1-7(n_1+n_2)^{-11} \geqslant 1-(n_1+n_2)^{-10}$, which finishes the proof.

\begin{proof}[Proof of Claim \ref{claim_01}]
Follow the way people did in \cite{ma2017implicit}, let $\overline{\boldsymbol{M}}^{0,(l),\textrm{zero}}$ be the matrix derived by zeroing out the $l$-th row and column of $\overline{\boldsymbol{M}}^{0,(l)}$, and $\widetilde{\boldsymbol{Z}}^{0,(l),\textrm{zero}}\in\mathbb{R}^{(n_1+n_2)\times r}$ containing the leading $r$ eigenvectors of $\overline{\boldsymbol{M}}^{0,(l),\textrm{zero}}$. 
Notice  
\begin{equation}\label{eq_037}
	\begin{split}
	&\left\|\widetilde{\boldsymbol{Z}}^{0,(l),\textrm{zero}} \sgn\left((\widetilde{\boldsymbol{Z}}^{0,(l),\textrm{zero}})^\top \widetilde{\boldsymbol{W}} \right)  - \widetilde{\boldsymbol{Z}}^{0,(l),\textrm{zero}}(\widetilde{\boldsymbol{Z}}^{0,(l),\textrm{zero}})^\top\widetilde{\boldsymbol{W}} \right\|_{2,\infty}\\ 
	= & \left\| \widetilde{\boldsymbol{Z}}^{0,(l),\textrm{zero}}(\widetilde{\boldsymbol{Z}}^{0,(l),\textrm{zero}})^\top \widetilde{\boldsymbol{W}}  \left( (\widetilde{\boldsymbol{Z}}^{0,(l),\textrm{zero}})^\top\widetilde{\boldsymbol{W}} \right)^{-1}  \left( \sgn\left( (\widetilde{\boldsymbol{Z}}^{0,(l),\textrm{zero}})^\top\widetilde{\boldsymbol{W}} \right) -(\widetilde{\boldsymbol{Z}}^{0,(l),\textrm{zero}})^\top\widetilde{\boldsymbol{W}}\right) \right\|_{2,\infty}\\
	\leqslant & \left\| \widetilde{\boldsymbol{Z}}^{0,(l),\textrm{zero}}(\widetilde{\boldsymbol{Z}}^{0,(l),\textrm{zero}})^\top \widetilde{\boldsymbol{W}} \right\|_{2,\infty}  \left\| \left( (\widetilde{\boldsymbol{Z}}^{0,(l),\textrm{zero}})^\top\widetilde{\boldsymbol{W}} \right)^{-1} \right\|  \left\| \sgn\left( (\widetilde{\boldsymbol{Z}}^{0,(l),\textrm{zero}})^\top\widetilde{\boldsymbol{W}} \right) -(\widetilde{\boldsymbol{Z}}^{0,(l),\textrm{zero}})^\top\widetilde{\boldsymbol{W}} \right\|.
	\end{split}
\end{equation}

By triangle inequality,
\begin{equation}\label{eq_052}
\begin{split}
	& \left\|\overline{\boldsymbol{M}}^{0,(l),\textrm{zero}} - \overline{\boldsymbol{M}}\right\| \\
	\leqslant & \left\| \overline{\boldsymbol{M}}^{0,(l),\textrm{zero}} - \overline{\boldsymbol{M}}^{(l),\textrm{zero}} \right\|  + \left\| \left[ \begin{array}{ccccc} \boldsymbol{0}&& \overline{M}_{1,l} &&\boldsymbol{0} \\ &&\vdots&&\\ \overline{M}_{l,1} &\cdots& \overline{M}_{l,l}&\cdots& \overline{M}_{l,n_1+n_2}\\  &&\vdots& &\\\boldsymbol{0}&& \overline{M}_{n_1+n_2,l} &&\boldsymbol{0} \\ \end{array} \right] \right\|,
\end{split}
\end{equation}
where here we define $\overline{\boldsymbol{M}}^{(l),\textrm{zero}}$ as $\overline{\boldsymbol{M}}$ zeroing out the $l$-th row and column of $\overline{\boldsymbol{M}}$. The first part we can again apply Lemma \ref{chen_lemma2} on $\overline{\boldsymbol{M}}^{(l),\textrm{zero}}$ to see
\[
	\left\| \overline{\boldsymbol{M}}^{0,(l),\textrm{zero}} - \overline{\boldsymbol{M}}^{(l),\textrm{zero}} \right\| \leqslant 4C_4\sqrt{\frac{\mu r \kappa \log (n_1\vee n_2)}{(n_1\wedge n_2)p}} \sigma_1(\boldsymbol{M})
\]
holds on an event $E_{Ch2}$ with probability $\mathbb{P}[E_{Ch2}]\geqslant 1-(n_1+n_2)^{-11}$. Therefore since
\[
	p \geqslant 1024C_4^2 \frac{\mu r \kappa^3 \log (n_1\vee n_2)}{n_1\wedge n_2},
\]
we have  
\begin{equation}\label{eq_035}
	\left\| \overline{\boldsymbol{M}}^{0,(l),\textrm{zero}} - \overline{\boldsymbol{M}}^{(l),\textrm{zero}} \right\|\leqslant \frac{1}{8} \sigma_r(\boldsymbol{M}).
\end{equation}
Moreover, for the second part of the right hand side of \eqref{eq_052}, we have

\begin{equation}\label{eq_003}
\begin{split}
& \left\| \left[ \begin{array}{ccccc} \boldsymbol{0}&& \overline{M}_{1,l} &&\boldsymbol{0} \\ &&\vdots&&\\ \overline{M}_{l,1} &\cdots& \overline{M}_{l,l}&\cdots& \overline{M}_{l,n_1+n_2}\\  &&\vdots& &\\\boldsymbol{0}&& \overline{M}_{n_1+n_2,l} &&\boldsymbol{0} \\ \end{array} \right] \right\|\\
 \leqslant& \left\| \left[ \begin{array}{ccccc}  &&& & \\  \overline{M}_{l,1} &\cdots& \overline{M}_{l,l}&\cdots& \overline{M}_{l,n_1+n_2}\\ &&   & &   \end{array} \right] \right\| +\left\| \left[ \begin{array}{ccc}  &\overline{M}_{1,l}&  \\&\vdots&  \\  & \overline{M}_{l-1,l}& \\&0& \\& \overline{M}_{l+1,l}& \\   &\vdots&\\&\overline{M}_{n_1+n_2,l}&  \end{array} \right] \right\|\\
	 \leqslant & \|\overline{\boldsymbol{M}}_{l,\cdot}\|_2+\|\overline{\boldsymbol{M}}_{\cdot,l}\|_2\\
	 \leqslant & 2\max\{\|\boldsymbol{U}\|\|\boldsymbol{V}\|_{2,\infty},\|\boldsymbol{V}\|\|\boldsymbol{U}\|_{2,\infty}\}\\
	 \leqslant & 2\sqrt{\frac{\mu r \kappa}{n_1\wedge n_2}}\sigma_1(\boldsymbol{M}).
\end{split}
\end{equation}

As long as 
\[
	 256 \frac{\mu r \kappa^3}{n_1\wedge n_2}\leqslant p \leqslant 1,
\]
plugging back to \eqref{eq_003} we have
\begin{equation}\label{eq_036}
\begin{split}
  \left\| \left[ \begin{array}{ccccc} \boldsymbol{0}&& \overline{M}_{1,l} &&\boldsymbol{0} \\ &&\vdots&&\\ \overline{M}_{l,1} &\cdots& \overline{M}_{l,l}&\cdots& \overline{M}_{l,n_1+n_2}\\  &&\vdots& &\\\boldsymbol{0}&& \overline{M}_{n_1+n_2,l} &&\boldsymbol{0} \\ \end{array} \right] \right\| \leqslant & \frac{1}{8}\sigma_r(\boldsymbol{M}).
\end{split}
\end{equation}
Combining the estimation \eqref{eq_035} and \eqref{eq_036} together we have
\begin{equation}\label{eq_038}
	\left\|\overline{\boldsymbol{M}}^{0,(l),\textrm{zero}} - \overline{\boldsymbol{M}}\right\|\leqslant \frac{1}{4}\sigma_r(\boldsymbol{M}).
\end{equation}
Applying Lemma \ref{abbe_lemma3} here, we have
\[
	\left\| \left( (\widetilde{\boldsymbol{Z}}^{0,(l),\textrm{zero}})^\top\widetilde{\boldsymbol{W}} \right)^{-1} \right\|\leqslant 2 \]
	and
	\[ \left\| \sgn\left( (\widetilde{\boldsymbol{Z}}^{0,(l),\textrm{zero}})^\top\widetilde{\boldsymbol{W}} \right) -(\widetilde{\boldsymbol{Z}}^{0,(l),\textrm{zero}})^\top\widetilde{\boldsymbol{W}} \right\|\leqslant \frac{1}{4}.
\]
Therefore from \eqref{eq_037} we have
\[
\begin{split}
	 \left\|\widetilde{\boldsymbol{Z}}^{0,(l),\textrm{zero}} \sgn\left((\widetilde{\boldsymbol{Z}}^{0,(l),\textrm{zero}})^\top \widetilde{\boldsymbol{W}} \right) - \widetilde{\boldsymbol{Z}}^{0,(l),\textrm{zero}}(\widetilde{\boldsymbol{Z}}^{0,(l),\textrm{zero}})^\top\widetilde{\boldsymbol{W}} \right\|_{2,\infty} \leqslant& \frac{1}{2} \left\| \widetilde{\boldsymbol{Z}}^{0,(l),\textrm{zero}}(\widetilde{\boldsymbol{Z}}^{0,(l),\textrm{zero}})^\top \widetilde{\boldsymbol{W}} \right\|_{2,\infty} 
\end{split}
\]
and
\[
	\begin{split}
		&\|\widetilde{\boldsymbol{Z}}^{0,(l),\textrm{zero}} \|_{2,\infty}\\
		 =& \left\| \widetilde{\boldsymbol{Z}}^{0,(l),\textrm{zero}} \sgn\left((\widetilde{\boldsymbol{Z}}^{0,(l),\textrm{zero}})^\top \widetilde{\boldsymbol{W}} \right) \right\|_{2,\infty}\\
		\leqslant& \left\| \widetilde{\boldsymbol{Z}}^{0,(l),\textrm{zero}}(\widetilde{\boldsymbol{Z}}^{0,(l),\textrm{zero}})^\top\widetilde{\boldsymbol{W}} \right\|_{2,\infty}   + \left\|\widetilde{\boldsymbol{Z}}^{0,(l),\textrm{zero}} \sgn\left((\widetilde{\boldsymbol{Z}}^{0,(l),\textrm{zero}})^\top \widetilde{\boldsymbol{W}} \right)  - \widetilde{\boldsymbol{Z}}^{0,(l),\textrm{zero}}(\widetilde{\boldsymbol{Z}}^{0,(l),\textrm{zero}})^\top\widetilde{\boldsymbol{W}} \right\|_{2,\infty} \\
		 \leqslant& 2 \left\| \widetilde{\boldsymbol{Z}}^{0,(l),\textrm{zero}}(\widetilde{\boldsymbol{Z}}^{0,(l),\textrm{zero}})^\top \widetilde{\boldsymbol{W}} \right\|_{2,\infty}.
	\end{split} 
\]

In order to give a control of 
\[
\left\| \widetilde{\boldsymbol{Z}}^{0,(l),\textrm{zero}}(\widetilde{\boldsymbol{Z}}^{0,(l),\textrm{zero}})^\top \widetilde{\boldsymbol{W}} \right\|_{2,\infty}, 
\]
we need Lemma 4 and Lemma 14 in \cite{abbe2017entrywise}. For the purpose of simplicity we combine those two lemmas together and only include those useful bounds in our case:
\begin{lemma}[{\citealt[Lemma 4 and Lemma 14 rewrited]{abbe2017entrywise}}]\label{abbe_lemma4}
	Under our setup, there is some absolute constant $C_5$, if $p\geqslant C_5\frac{\mu^2 r^2\kappa^6 \log(n_1\vee n_2)}{(n_1\wedge n_2)}$, then on an event $E_A$ with probability $\mathbb{P}[E_A]\geqslant 1-(n_1+n_2)^{-11}$,
	\[
	\begin{split}
		  \max_{l} \| \widetilde{\boldsymbol{Z}}^{0,(l),\textrm{zero}}( \widetilde{\boldsymbol{Z}}^{0,(l),\textrm{zero}})^\top \widetilde{\boldsymbol{W}} - \widetilde{\boldsymbol{W}}\|_{2,\infty}  \leqslant& 4\kappa \|\widetilde{\boldsymbol{Z}}^0(\widetilde{\boldsymbol{Z}}^0)^\top\widetilde{\boldsymbol{W}}\|_{2,\infty} + \|\widetilde{\boldsymbol{W}}\|_{2,\infty} 
	\end{split}
	\] 
	and
	\[
		\|\widetilde{\boldsymbol{Z}}^0\|_{2,\infty} \leqslant C_5\left(\kappa \|\widetilde{\boldsymbol{W}}\|_{2,\infty}+ \sqrt{\frac{n_1\wedge n_2}{p}}\frac{\|\overline{\boldsymbol{M}}\|_{\ell_{\infty}}\|\overline{\boldsymbol{M}}\|_{2,\infty}}{\sigma_r^2(\boldsymbol{M})}\right)
	\]
	holds.
\end{lemma}

By the lemma we have
\[
	\begin{split}
		 \|\widetilde{\boldsymbol{Z}}^{0,(l),\textrm{zero}} \|_{2,\infty} \leqslant& 2 \left\| \widetilde{\boldsymbol{Z}}^{0,(l),\textrm{zero}}(\widetilde{\boldsymbol{Z}}^{0,(l),\textrm{zero}})^\top \widetilde{\boldsymbol{W}} \right\|_{2,\infty}\\
		\leqslant & 4\|\widetilde{\boldsymbol{W}}\|_{2,\infty} + 8\kappa \|\widetilde{\boldsymbol{Z}}^0(\widetilde{\boldsymbol{Z}}^0)^\top\widetilde{\boldsymbol{W}}\|_{2,\infty} \\
		\leqslant & 4\|\widetilde{\boldsymbol{W}}\|_{2,\infty} + 8\kappa \|\widetilde{\boldsymbol{Z}}^0\|_{2,\infty}\|(\widetilde{\boldsymbol{Z}}^0)^\top\widetilde{\boldsymbol{W}}\| \\
		\leqslant & 4\|\widetilde{\boldsymbol{W}}\|_{2,\infty} + 8\kappa \|\widetilde{\boldsymbol{Z}}^0\|_{2,\infty} \\
		\leqslant & \left( 4+8C_5 \kappa^2 + 8\sqrt{2}C_5\sqrt{\frac{\mu^2 r^2\kappa^6}{(n_1\wedge n_2)p}} \right) \|\widetilde{\boldsymbol{W}}\|_{2,\infty}.
	\end{split}
\]
The fourth inequality uses the fact that $\|(\widetilde{\boldsymbol{Z}}^0)^\top\widetilde{\boldsymbol{W}}\|\leqslant 1$ since $\widetilde{\boldsymbol{Z}}^0$ and $\widetilde{\boldsymbol{W}}$ both have orthonormal columns, and the last inequality uses the fact that 
\[
\begin{split}
	\|\overline{\boldsymbol{M}}\|_{2,\infty} \leqslant & \max(\|\boldsymbol{U}\|\|\boldsymbol{V}\|_{2,\infty}, \|\boldsymbol{V}\|\|\boldsymbol{U}\|_{2,\infty})\\
	 \leqslant & \sqrt{\sigma_1(\boldsymbol{M})}\sqrt{2}\|\boldsymbol{W}\|_{2,\infty} \\
	 \leqslant & \sqrt{2}\sigma_1(\boldsymbol{M})\|\widetilde{\boldsymbol{W}}\|_{2,\infty}.
\end{split}
\]
So as long as 
\[
	p \geqslant 128\frac{\mu^2 r^2 \kappa^2}{n_1\wedge n_2},
\]
we have 
\begin{equation}\label{eq_040}
	\|\widetilde{\boldsymbol{Z}}^{0,(l),\textrm{zero}} \|_{2,\infty} \leqslant (4+9 C_5\kappa^2) \|\widetilde{\boldsymbol{W}}\|_{2,\infty}.
\end{equation}
Recall that in \eqref{eq_038} and \eqref{eq_039}, we have already shown 
\[
\left\|\overline{\boldsymbol{M}}^{0,(l),\textrm{zero}} - \overline{\boldsymbol{M}}\right\|\leqslant \frac{1}{4}\sigma_r(\boldsymbol{M})
\] 
and
\[
\left\|\overline{\boldsymbol{M}}^{0,(l)} - \overline{\boldsymbol{M}}\right\|\leqslant \frac{1}{4}\sigma_r(\boldsymbol{M}) 
\]
hold on the events $E_{Ch2}$ and $E_{Ch1}$, respectively. Therefore, by the Davis-Kahan Sin$\Theta$ theorem \citep{davis1970rotation}, we have
\[
\begin{split}
	  \left\| \widetilde{\boldsymbol{Z}}^{0,(l)}\sgn\left( (\widetilde{\boldsymbol{Z}}^{0,(l)})^\top \widetilde{\boldsymbol{Z}}^{0,(l),\textrm{zero}} \right) - \widetilde{\boldsymbol{Z}}^{0,(l),\textrm{zero}} \right\|_F \leqslant& \frac{2\sqrt{2}}{\sigma_r(\boldsymbol{M})} \left\| \left( \overline{\boldsymbol{M}}^{0,(l),\textrm{zero}} - \overline{\boldsymbol{M}}^{0,(l)} \right)\widetilde{\boldsymbol{Z}}^{0,(l),\textrm{zero}}    \right\|_F. 
\end{split}
\]
For $i\neq l$, we have
\[
\begin{split}
	 \left( \overline{\boldsymbol{M}}^{0,(l)} - \overline{\boldsymbol{M}}^{0,(l),\textrm{zero}} \right)_{i,\cdot}^\top \widetilde{\boldsymbol{Z}}^{0,(l),\textrm{zero}}  =& \left( \overline{\boldsymbol{M}}^{0,(l)} - \overline{\boldsymbol{M}}^{0,(l),\textrm{zero}} \right)_{i,l}(\widetilde{\boldsymbol{Z}}^{0,(l),\textrm{zero}}_{l,\cdot})^\top \\
	=& \boldsymbol{0}.
\end{split}
\]
The last equation holds since by construction we have $\widetilde{\boldsymbol{Z}}^{0,(l),\textrm{zero}}_{l,\cdot} = \boldsymbol{0}$. In order to see this, note the fact that by definition, entries on $l$-th row of $\overline{\boldsymbol{M}}^{0,(l),\textrm{zero}} $ are identical zeros, so if there is an eigenvector $\boldsymbol{v}$ with $v_l \neq 0$, the corresponding eigenvalue must be zero. Since $\widetilde{\boldsymbol{Z}}^{0,(l),\textrm{zero}}$ is the collection of top-$r$ eigenvectors. By Weyl's inequality and $\left\|\overline{\boldsymbol{M}}^{0,(l),\textrm{zero}} - \overline{\boldsymbol{M}}\right\|\leqslant \frac{1}{4}\sigma_r(\boldsymbol{M})$ we have the corresponding eigenvalues are all positive. Therefore we have $\widetilde{\boldsymbol{Z}}^{0,(l),\textrm{zero}}_{l,\cdot} = \boldsymbol{0}$.

So we have
\[
	\begin{split}
	 \left\| \left( \overline{\boldsymbol{M}}^{0,(l),\textrm{zero}} - \overline{\boldsymbol{M}}^{0,(l)} \right)\widetilde{\boldsymbol{Z}}^{0,(l),\textrm{zero}}    \right\|_F =& \left\|\left( \overline{\boldsymbol{M}}^{0,(l),\textrm{zero}} - \overline{\boldsymbol{M}}^{0,(l)} \right)_{l,\cdot}^\top \widetilde{\boldsymbol{Z}}^{0,(l),\textrm{zero}} \right\|_2\\
	 =& \left\|\overline{\boldsymbol{M}}_{l,\cdot}^\top \widetilde{\boldsymbol{Z}}^{0,(l),\textrm{zero}}\right\|_2\\
	 \leqslant & \|\overline{\boldsymbol{M}}\|_{2,\infty}\\
	 \leqslant & \sigma_1(\boldsymbol{M}) \max\{\|\widetilde{\boldsymbol{U}}\|_{2,\infty}, \|\widetilde{\boldsymbol{V}}\|_{2,\infty}\}\\
	 \leqslant & \sqrt{2}\sigma_1(\boldsymbol{M}) \|\widetilde{\boldsymbol{W}}\|_{2,\infty}.
	\end{split}
\]

Therefore,
\begin{equation}\label{eq_041}
\begin{split}
	 \left\| \widetilde{\boldsymbol{Z}}^{0,(l)}\sgn\left( (\widetilde{\boldsymbol{Z}}^{0,(l)})^\top \widetilde{\boldsymbol{Z}}^{0,(l),\textrm{zero}} \right) - \widetilde{\boldsymbol{Z}}^{0,(l),\textrm{zero}} \right\|_F  \leqslant & \frac{4}{\sigma_r(\boldsymbol{M})}\sigma_1(\boldsymbol{M}) \|\widetilde{\boldsymbol{W}}\|_{2,\infty} \\
	=& 4 \kappa\|\widetilde{\boldsymbol{W}}\|_{2,\infty}.
\end{split}
\end{equation}
Putting \eqref{eq_040} and \eqref{eq_041} together we have
\[
	\begin{split}
	  \left\| \widetilde{\boldsymbol{Z}}^{0,(l)}\right\|_{2,\infty} =&  \|\widetilde{\boldsymbol{Z}}^{0,(l)}\sgn\left( (\widetilde{\boldsymbol{Z}}^{0,(l)})^\top \widetilde{\boldsymbol{Z}}^{0,(l),\textrm{zero}} \right)\|_{2,\infty} \\
	\leqslant& \|\widetilde{\boldsymbol{Z}}^{0,(l),\textrm{zero}} \|_{2,\infty}  + \left\| \widetilde{\boldsymbol{Z}}^{0,(l)}\sgn\left( (\widetilde{\boldsymbol{Z}}^{0,(l)})^\top \widetilde{\boldsymbol{Z}}^{0,(l),\textrm{zero}} \right) - \widetilde{\boldsymbol{Z}}^{0,(l),\textrm{zero}} \right\|_F\\
	\leqslant&   (4+4\kappa+9C_5 \kappa^2) \|\widetilde{\boldsymbol{W}}\|_{2,\infty},
	\end{split}
\]
holds on an event $E_{Claim} =  E_{Ch1} \bigcap E_{Ch2} \bigcap E_A$, using union bound we have $\mathbb{P}[E_{Claim}]\geqslant 1-3(n_1+n_2)^{-11}$, which proves the claim.
\end{proof}

\section{Proof of Claim \ref{claim_03}}\label{sec_proof_claim_03}
\begin{proof}

Similar to what we did in the control of spectral norm, define the auxiliary iteration as
			\[ 
				\begin{split}
					 \widetilde{\boldsymbol{X}}^{t+1,(l)}  \coloneqq & \boldsymbol{X}^{t,(l)}\boldsymbol{R}^{t,(l)} - \frac{\eta}{p}\mathcal{P}_{\Omega_{-l,\cdot}}\left( \boldsymbol{X}^{t,(l)}\left( \boldsymbol{Y}^{t,(l)} \right)^\top - \boldsymbol{U}\boldsymbol{V}^\top \right)\boldsymbol{V} - \eta\mathcal{P}_{  l,\cdot}\left( \boldsymbol{X}^{t,(l)}\left( \boldsymbol{Y}^{t,(l)} \right)^\top - \boldsymbol{U}\boldsymbol{V}^\top \right)\boldsymbol{V}\\
					& -\frac{\eta}{2}\boldsymbol{U}(\boldsymbol{R}^{t,(l)})^\top\left( \left( \boldsymbol{X}^{t,(l)} \right)^\top\boldsymbol{X}^{t,(l)} - \left( \boldsymbol{Y}^{t,(l)}\right)^\top\boldsymbol{Y}^{t,(l)} \right) \boldsymbol{R}^{t,(l)},\\
				\end{split}
		\]
		\[ 
		\begin{split}
			  \widetilde{\boldsymbol{Y}}^{t+1,(l)} \coloneqq & \boldsymbol{Y}^{t,(l)}\boldsymbol{R}^{t,(l)}  - \frac{\eta}{p}\left[\mathcal{P}_{\Omega_{-l,\cdot}}\left( \boldsymbol{X}^{t,(l)} \left( \boldsymbol{Y}^{t,(l)} \right)^\top - \boldsymbol{U}\boldsymbol{V}^\top \right)\right]^\top \boldsymbol{U} -\eta \left[\mathcal{P}_{l,\cdot}\left( \boldsymbol{X}^{t,(l)} \left( \boldsymbol{Y}^{t,(l)} \right)^\top - \boldsymbol{U}\boldsymbol{V}^\top \right)\right]^\top \boldsymbol{U}\\
			&- \frac{\eta}{2}\boldsymbol{V}(\boldsymbol{R}^{t,(l)})^\top\left( \left( \boldsymbol{Y}^{t,(l)} \right)^\top\boldsymbol{Y}^{t,(l)}- \left(\boldsymbol{X}^{t,(l)}\right)^\top\boldsymbol{X}^{t,(l)} \right) \boldsymbol{R}^{t,(l)}.
				\end{split}
			\]

			 Here we want apply Lemma \ref{ma_lemma36} with
			 \[
					\boldsymbol{C} = \left[ \begin{array}{c} \widetilde{\boldsymbol{X}}^{t+1,(l)}\\ \widetilde{\boldsymbol{Y}}^{t+1,(l)}\end{array}\right]^\top\left[\begin{array}{c}
						\boldsymbol{U}\\
						\boldsymbol{V}
					\end{array}\right]   ,
					\]
					\[\boldsymbol{E} = \left[ \begin{array}{c} \boldsymbol{X}^{t+1,(l)}\boldsymbol{R}^{t,(l)}  - \widetilde{\boldsymbol{X}}^{t+1,(l)} \\ \boldsymbol{Y}^{t+1,(l)}\boldsymbol{R}^{t,(l)}  - \widetilde{\boldsymbol{Y}}^{t+1,(l)}\end{array} \right]^\top\left[\begin{array}{c}
						\boldsymbol{U}\\
						\boldsymbol{V}
					\end{array}\right].
			 \] 
			 By definition of $\boldsymbol{R}^{t+1,(l)}$ we have
			 \[
			 \begin{split}
				  (\boldsymbol{R}^{t,(l)})^{-1}\boldsymbol{R}^{t+1,(l)} =& \argmin_{\boldsymbol{R}}\left\| \left[ \begin{array}{c} \boldsymbol{X}^{t+1,(l)}\boldsymbol{R}^{t,(l)}   \\ \boldsymbol{Y}^{t+1,(l)}\boldsymbol{R}^{t,(l)}   \end{array} \right] \boldsymbol{R} - \left[ \begin{array}{c} \boldsymbol{U}\\\boldsymbol{V} \end{array}\right] \right\|_F\\
				 = & \operatorname{sgn}(\boldsymbol{C}+\boldsymbol{E}).
			\end{split}
			 \]
			 If $\boldsymbol{C}$ is a positive definite matrix, then $ \operatorname{sgn}(\boldsymbol{C})= \boldsymbol{I}$, and we have
			 \[
				\begin{split}
					  \|(\boldsymbol{R}^{t,(l)})^{-1}\boldsymbol{R}^{t+1,(l)} - \boldsymbol{I}\| =&\|\operatorname{sgn}(\boldsymbol{C}+\boldsymbol{E}) - \operatorname{sgn}(\boldsymbol{C})\|\\
					 \leqslant& \frac{1}{\sigma_r(\boldsymbol{P})} \left\|\left[\begin{array}{c}
						\boldsymbol{U}\\
						\boldsymbol{V}
					\end{array}\right]^\top\left[ \begin{array}{c} \boldsymbol{X}^{t+1,(l)}\boldsymbol{R}^{t,(l)}  - \widetilde{\boldsymbol{X}}^{t+1,(l)} \\ \boldsymbol{Y}^{t+1,(l)}\boldsymbol{R}^{t,(l)}  - \widetilde{\boldsymbol{Y}}^{t+1,(l)}\end{array} \right]\right\|.\\
				\end{split}
			 \]

The remaining part are devoted to verifying the required conditions of Lemma \ref{ma_lemma36}, $\boldsymbol{C}$ is a positive definite matrix and upper bounding
\[
	\left\|\left[\begin{array}{c}
						\boldsymbol{U}\\
						\boldsymbol{V}
					\end{array}\right]^\top\left[ \begin{array}{c} \boldsymbol{X}^{t+1,(l)}\boldsymbol{R}^{t,(l)}  - \widetilde{\boldsymbol{X}}^{t+1,(l)} \\ \boldsymbol{Y}^{t+1,(l)}\boldsymbol{R}^{t,(l)}  - \widetilde{\boldsymbol{Y}}^{t+1,(l)}\end{array} \right]\right\|.
\]

			Let $\boldsymbol{P} \coloneqq \left[ \begin{array}{c} \boldsymbol{U}\\\boldsymbol{V} \end{array}\right]^\top \left[ \begin{array}{c} \widetilde{\boldsymbol{X}}^{t+1,(l)}\\ \widetilde{\boldsymbol{Y}}^{t+1,(l)}\end{array}\right] $, we have
			\[ 
				\begin{split}
					 \boldsymbol{P}  = & \boldsymbol{U}^\top \boldsymbol{X}^{t,(l)}\boldsymbol{R}^{t,(l)}  - \frac{\eta}{p}\boldsymbol{U}^\top\mathcal{P}_{\Omega_{-l,\cdot}}\left( \boldsymbol{X}^{t,(l)}\left( \boldsymbol{Y}^{t,(l)} \right)^\top - \boldsymbol{U}\boldsymbol{V}^\top \right)\boldsymbol{V} - \eta\boldsymbol{U}^\top\mathcal{P}_{ l,\cdot}\left( \boldsymbol{X}^{t,(l)}\left( \boldsymbol{Y}^{t,(l)} \right)^\top - \boldsymbol{U}\boldsymbol{V}^\top \right)\boldsymbol{V}\\
					& -\frac{\eta}{2}\boldsymbol{U}^\top\boldsymbol{U}(\boldsymbol{R}^{t,(l)})^\top  \left( \left( \boldsymbol{X}^{t,(l)} \right)^\top\boldsymbol{X}^{t,(l)} - \left( \boldsymbol{Y}^{t,(l)}\right)^\top\boldsymbol{Y}^{t,(l)} \right)\boldsymbol{R}^{t,(l)}  + \boldsymbol{V}^\top \boldsymbol{Y}^{t,(l)}\boldsymbol{R}^{t,(l)} \\
					& - \frac{\eta}{p}\boldsymbol{V}^\top\left[\mathcal{P}_{\Omega_{-l,\cdot}}\left( \boldsymbol{X}^{t,(l)} \left( \boldsymbol{Y}^{t,(l)} \right)^\top - \boldsymbol{U}\boldsymbol{V}^\top \right)\right]^\top \boldsymbol{U} -\eta \boldsymbol{V}^\top\left[\mathcal{P}_{l,\cdot}\left( \boldsymbol{X}^{t,(l)} \left( \boldsymbol{Y}^{t,(l)} \right)^\top - \boldsymbol{U}\boldsymbol{V}^\top \right)\right]^\top \boldsymbol{U}\\
					&- \frac{\eta}{2}\boldsymbol{V}^\top\boldsymbol{V}(\boldsymbol{R}^{t,(l)})^\top \left( \left( \boldsymbol{Y}^{t,(l)} \right)^\top\boldsymbol{Y}^{t,(l)}- \left(\boldsymbol{X}^{t,(l)}\right)^\top\boldsymbol{X}^{t,(l)} \right)\boldsymbol{R}^{t,(l)}\\
					=& \boldsymbol{U}^\top \boldsymbol{X}^{t,(l)}\boldsymbol{R}^{t,(l)}  - \frac{\eta}{p}\boldsymbol{U}^\top\mathcal{P}_{\Omega_{-l,\cdot}}\left( \boldsymbol{X}^{t,(l)}\left( \boldsymbol{Y}^{t,(l)} \right)^\top - \boldsymbol{U}\boldsymbol{V}^\top \right)\boldsymbol{V}  - \eta\boldsymbol{U}^\top\mathcal{P}_{  l,\cdot}\left( \boldsymbol{X}^{t,(l)}\left( \boldsymbol{Y}^{t,(l)} \right)^\top - \boldsymbol{U}\boldsymbol{V}^\top \right)\boldsymbol{V}\\
					&  + \boldsymbol{V}^\top \boldsymbol{Y}^{t,(l)}\boldsymbol{R}^{t,(l)}   - \frac{\eta}{p}\boldsymbol{V}^\top\left[\mathcal{P}_{\Omega_{-l,\cdot}}\left( \boldsymbol{X}^{t,(l)} \left( \boldsymbol{Y}^{t,(l)} \right)^\top - \boldsymbol{U}\boldsymbol{V}^\top \right)\right]^\top \boldsymbol{U} \\
					& -\eta \boldsymbol{V}^\top\left[\mathcal{P}_{l,\cdot}\left( \boldsymbol{X}^{t,(l)} \left( \boldsymbol{Y}^{t,(l)} \right)^\top - \boldsymbol{U}\boldsymbol{V}^\top \right)\right]^\top \boldsymbol{U} ,
				\end{split}
			\]
                         here the last equality use the fact that $\boldsymbol{U}^\top\boldsymbol{U} = \boldsymbol{V}^\top \boldsymbol{V}$. By the choice of $\boldsymbol{R}^{t,(l)}$, we also have 
 $\boldsymbol{U}^\top \boldsymbol{X}^{t,(l)}\boldsymbol{R}^{t,(l)}+ \boldsymbol{V}^\top \boldsymbol{Y}^{t,(l)}\boldsymbol{R}^{t,(l)}$
is symmetric, therefore $\boldsymbol{P}$ is symmetric.

			Denote
			\[ 
				\begin{split}
					&\widetilde{\mathbb{E}} \widetilde{\boldsymbol{X}}^{t+1,(l)} \\
					\coloneqq & \boldsymbol{X}^{t,(l)}\boldsymbol{R}^{t,(l)} - \eta \left( \boldsymbol{X}^{t,(l)}\left( \boldsymbol{Y}^{t,(l)} \right)^\top - \boldsymbol{U}\boldsymbol{V}^\top \right)\boldsymbol{V} -\frac{\eta}{2}\boldsymbol{U}(\boldsymbol{R}^{t,(l)})^\top\left( \left( \boldsymbol{X}^{t,(l)} \right)^\top\boldsymbol{X}^{t,(l)} - \left( \boldsymbol{Y}^{t,(l)}\right)^\top\boldsymbol{Y}^{t,(l)} \right) \boldsymbol{R}^{t,(l)},\\
				\end{split}
			\]
			and
			\[  
		\begin{split}
			& \widetilde{\mathbb{E}} \widetilde{\boldsymbol{Y}}^{t+1,(l)}\\
			 \coloneqq & \boldsymbol{Y}^{t,(l)}\boldsymbol{R}^{t,(l)}  - \eta  \left( \boldsymbol{X}^{t,(l)} \left( \boldsymbol{Y}^{t,(l)} \right)^\top - \boldsymbol{U}\boldsymbol{V}^\top \right)^\top \boldsymbol{U} - \frac{\eta}{2}\boldsymbol{V}(\boldsymbol{R}^{t,(l)})^\top\left( \left( \boldsymbol{Y}^{t,(l)} \right)^\top\boldsymbol{Y}^{t,(l)}- \left(\boldsymbol{X}^{t,(l)}\right)^\top\boldsymbol{X}^{t,(l)} \right) \boldsymbol{R}^{t,(l)}.
				\end{split}
			\] 

			In order to see all the eigenvalues of $\boldsymbol{P}$ are positive, first by triangle inequality,
			\begin{equation}\label{eq_056}
			\begin{split}
				  \left\| \left[ \begin{array}{c}\widetilde{\boldsymbol{X}}^{t+1,(l)} -\boldsymbol{U}\\  \widetilde{\boldsymbol{Y}}^{t+1,(l)}-\boldsymbol{V}\end{array} \right] \right\|  \leqslant& \left\|  \left[ \begin{array}{c}\widetilde{\mathbb{E}} \widetilde{\boldsymbol{X}}^{t+1,(l)}\\  \widetilde{\mathbb{E}}\widetilde{\boldsymbol{Y}}^{t+1,(l)} \end{array} \right] -  \left[ \begin{array}{c}\widetilde{\boldsymbol{X}}^{t+1,(l)}\\  \widetilde{\boldsymbol{Y}}^{t+1,(l)} \end{array} \right]\right\|  + \left\| \left[ \begin{array}{c}\widetilde{\mathbb{E}}\widetilde{\boldsymbol{X}}^{t+1,(l)}\\  \widetilde{\mathbb{E}}\widetilde{\boldsymbol{Y}}^{t+1,(l)} \end{array} \right] - \left[ \begin{array}{c}
					\boldsymbol{U}\\\boldsymbol{V}
				\end{array}\right]  \right\|.	
			\end{split}
			\end{equation}

			For the first term of the right hand side of \eqref{eq_056}, note 
			\begin{equation}\label{eq_057}
				\begin{split}
				&\left\| \left[ \begin{array}{c}\widetilde{\mathbb{E}}\widetilde{\boldsymbol{X}}^{t+1,(l)}\\  \widetilde{\mathbb{E}}\widetilde{\boldsymbol{Y}}^{t+1,(l)} \end{array} \right] -  \left[ \begin{array}{c}\widetilde{\boldsymbol{X}}^{t+1,(l)}\\  \widetilde{\boldsymbol{Y}}^{t+1,(l)} \end{array} \right]\right\| \\
					 =& \eta \left\|\left[ \begin{array}{c}
					-\mathcal{P}_{-l,\cdot}\left( \boldsymbol{X}^{t,(l)}\left( \boldsymbol{Y}^{t,(l)} \right)^\top - \boldsymbol{U}\boldsymbol{V}^\top \right)\boldsymbol{V} + \frac{1}{p}\mathcal{P}_{\Omega_{-l,\cdot}}\left( \boldsymbol{X}^{t,(l)}\left( \boldsymbol{Y}^{t,(l)} \right)^\top - \boldsymbol{U}\boldsymbol{V}^\top \right)\boldsymbol{V}   \\
					-\left[\mathcal{P}_{-l,\cdot}\left( \boldsymbol{X}^{t,(l)}\left( \boldsymbol{Y}^{t,(l)} \right)^\top - \boldsymbol{U}\boldsymbol{V}^\top \right)\right]^\top\boldsymbol{U}  + \frac{1}{p}\left[\mathcal{P}_{\Omega_{-l,\cdot}}\left( \boldsymbol{X}^{t,(l)}\left( \boldsymbol{Y}^{t,(l)} \right)^\top - \boldsymbol{U}\boldsymbol{V}^\top \right)\right]^\top\boldsymbol{U}
				\end{array}   \right]\right\|\\
				\leqslant & 2\eta \|\boldsymbol{U}\| \left\| \frac{1}{p}\mathcal{P}_{\Omega_{-l,\cdot}}\left( \boldsymbol{X}^{t,(l)}\left( \boldsymbol{Y}^{t,(l)} \right)^\top - \boldsymbol{U}\boldsymbol{V}^\top \right) - \mathcal{P}_{-l,\cdot}\left( \boldsymbol{X}^{t,(l)}\left( \boldsymbol{Y}^{t,(l)} \right)^\top - \boldsymbol{U}\boldsymbol{V}^\top \right) \right\|\\
				\leqslant & 2\eta \|\boldsymbol{U}\| \left\| \frac{1}{p}\mathcal{P}_{\Omega}\left( \boldsymbol{X}^{t,(l)}\left( \boldsymbol{Y}^{t,(l)} \right)^\top- \boldsymbol{U}\boldsymbol{V}^\top \right) - \left( \boldsymbol{X}^{t,(l)}\left( \boldsymbol{Y}^{t,(l)} \right)^\top - \boldsymbol{U}\boldsymbol{V}^\top \right) \right\|.
				\end{split}
			\end{equation}
			The last line uses the fact that 
			\[
			\begin{split}
			&\frac{1}{p}\mathcal{P}_{\Omega_{-l,\cdot}}\left( \boldsymbol{X}^{t,(l)}\left( \boldsymbol{Y}^{t,(l)} \right)^\top - \boldsymbol{U}\boldsymbol{V}^\top \right) - \mathcal{P}_{-l,\cdot}\left( \boldsymbol{X}^{t,(l)}\left( \boldsymbol{Y}^{t,(l)} \right)^\top - \boldsymbol{U}\boldsymbol{V}^\top \right)
			\end{split}
			\]
			is a matrix with $l$-th row all zero and 
			\[
				\left\|\left[ \begin{array}{c}
					\boldsymbol{A}\\
					\boldsymbol{0}
				\end{array} \right]\right\|	\leqslant \left\|\left[ \begin{array}{c}
					\boldsymbol{A}\\
					\boldsymbol{b}^\top
				\end{array} \right]\right\|
			\]
			for any matrix $\boldsymbol{A}$ and vector $\boldsymbol{b}$ with suitable shape. Using Lemma \ref{lemma_spectral_gap}, we have
			\[
				\begin{split}
					& \left\| \frac{1}{p}\mathcal{P}_{\Omega}\left( \boldsymbol{X}^{t,(l)}\left( \boldsymbol{Y}^{t,(l)} \right)^\top - \boldsymbol{U}\boldsymbol{V}^\top \right)  - \left( \boldsymbol{X}^{t,(l)}\left( \boldsymbol{Y}^{t,(l)} \right)^\top - \boldsymbol{U}\boldsymbol{V}^\top \right) \right\|\\
					\leqslant & \left\| \frac{1}{p}\mathcal{P}_{\Omega}\left( \left(\boldsymbol{X}^{t,(l)}\boldsymbol{T}^{t,(l)}-\boldsymbol{U}\right)\boldsymbol{V}^\top  \right)  - \left(\boldsymbol{X}^{t,(l)}\boldsymbol{T}^{t,(l)}-\boldsymbol{U}\right)\boldsymbol{V}^\top   \right\|\\
					& + \left\| \frac{1}{p}\mathcal{P}_{\Omega}\left( \boldsymbol{U}\left(\boldsymbol{Y}^{t,(l)}\boldsymbol{T}^{t,(l)}-\boldsymbol{V}\right)^\top  \right)  - \boldsymbol{U}\left(\boldsymbol{Y}^{t,(l)}\boldsymbol{T}^{t,(l)}-\boldsymbol{V}\right)^\top  \right\|\\
					&+\left\| \frac{1}{p}\mathcal{P}_{\Omega}\left( \left(\boldsymbol{X}^{t,(l)}\boldsymbol{T}^{t,(l)}-\boldsymbol{U}\right)\left(\boldsymbol{Y}^{t,(l)}\boldsymbol{T}^{t,(l)}-\boldsymbol{V}\right)^\top  \right) - \left(\boldsymbol{X}^{t,(l)}\boldsymbol{T}^{t,(l)}-\boldsymbol{U}\right)\left(\boldsymbol{Y}^{t,(l)}\boldsymbol{T}^{t,(l)}-\boldsymbol{V}\right)^\top   \right\|\\
					\leqslant& \frac{\|\boldsymbol{\Omega}-p\boldsymbol{J}\|}{p} \left( \left\|\boldsymbol{X}^{t,(l)}\boldsymbol{T}^{t,(l)}-\boldsymbol{U}\right\|_{2,\infty}\|\boldsymbol{V}\|_{2,\infty}  + \|\boldsymbol{U}\|_{2,\infty}\left\|\boldsymbol{Y}^{t,(l)}\boldsymbol{T}^{t,(l)}-\boldsymbol{V}\right\|_{2,\infty} \right)\\
					&+ \frac{\|\boldsymbol{\Omega}-p\boldsymbol{J}\|}{p} \left\|\boldsymbol{X}^{t,(l)}\boldsymbol{T}^{t,(l)}-\boldsymbol{U}\right\|_{2,\infty}\left\|\boldsymbol{Y}^{t,(l)}\boldsymbol{T}^{t,(l)}-\boldsymbol{V}\right\|_{2,\infty}.
				\end{split}
			\]
			Here we use the fact that 
			\[
				\boldsymbol{X}^{t,(l)}\left( \boldsymbol{Y}^{t,(l)} \right)^\top  = \left(\boldsymbol{X}^{t,(l)}\boldsymbol{T}^{t,(l)}\right)\left( \boldsymbol{Y}^{t,(l)}\boldsymbol{T}^{t,(l)} \right)^\top. 
			\]
			On the event $E_{gd}^t$, from \eqref{eq_ind3} and \eqref{eq_ind5}, we have
			\begin{equation}\label{eq_015}
				\begin{split}
				 \left\| \left[ \begin{array}{c} \boldsymbol{X}^{t,(l)} \\ \boldsymbol{Y}^{t,(l)} \end{array} \right]\boldsymbol{T}^{t,(l)} - \left[ \begin{array}{c} \boldsymbol{U}\\\boldsymbol{V} \end{array}\right] \right\|_{2,\infty} \leqslant & \left\|\left[ \begin{array}{c} \boldsymbol{X}^t \\\boldsymbol{Y}^t \end{array}\right]\boldsymbol{R}^t -  \left[ \begin{array}{c} \boldsymbol{X}^{t,(l)} \\ \boldsymbol{Y}^{t,(l)} \end{array} \right]\boldsymbol{T}^{t,(l)}  \right\|_F  + \left\| \left[ \begin{array}{c} \boldsymbol{X}^t \\\boldsymbol{Y}^t \end{array}\right]\boldsymbol{R}^t  - \left[ \begin{array}{c} \boldsymbol{U}\\\boldsymbol{V} \end{array}\right] \right\|_{2,\infty}\\
				\leqslant & 111C_I \rho^t \sqrt{\frac{\mu^2 r^2 \kappa^{12}\log (n_1\vee n_2)}{(n_1\wedge n_2)^2p}}\sqrt{\sigma_1(\boldsymbol{M})}.
				\end{split}
			\end{equation} 	

			From Lemma \ref{vu_lemma2.2} and \eqref{eq_015},
			\begin{equation}\label{eq_058}
				\begin{split}
					& \left\| \frac{1}{p}\mathcal{P}_{\Omega}\left( \boldsymbol{X}^{t,(l)}\left( \boldsymbol{Y}^{t,(l)} \right)^\top - \boldsymbol{U}\boldsymbol{V}^\top \right) - \left( \boldsymbol{X}^{t,(l)}\left( \boldsymbol{Y}^{t,(l)} \right)^\top - \boldsymbol{U}\boldsymbol{V}^\top \right) \right\|\\
					\leqslant& \sqrt{\frac{n_1\wedge n_2}{p}}\times 111C_I\rho^t\sqrt{\frac{\mu^2 r^2\kappa^{12}\log(n_1\vee n_2)}{(n_1\wedge n_2)^2p}}\sqrt{\sigma_1(\boldsymbol{M})}\\
					&\times \left(2\sqrt{\frac{\mu r \kappa}{n_1\wedge n_2}}\sqrt{\sigma_1(\boldsymbol{M})} + 111C_I\rho^t\sqrt{\frac{\mu^2 r^2\kappa^{12}\log(n_1\vee n_2)}{(n_1\wedge n_2)^2p}}\sqrt{\sigma_1(\boldsymbol{M})}  \right).
				\end{split}	
			\end{equation}
			For the second term of the right hand side of \eqref{eq_056}, we deal with it very similar to the way we deal with $\alpha_2$ defined in \eqref{eq_046}: Note
			\[ 
				\begin{split}
				&  \left[ \begin{array}{c}\widetilde{\mathbb{E}}\widetilde{\boldsymbol{X}}^{t+1,(l)}\\ \widetilde{\mathbb{E}} \widetilde{\boldsymbol{Y}}^{t+1,(l)} \end{array} \right] - \left[ \begin{array}{c}
					\boldsymbol{U}\\\boldsymbol{V}
				\end{array}\right]\\
				=& \left[ \begin{array}{c}
					\boldsymbol{X}^{t,(l)}\boldsymbol{R}^{t,(l)} - \eta\left(\boldsymbol{X}^{t,(l)}(\boldsymbol{Y}^{t,(l)})^\top - \boldsymbol{U}\boldsymbol{V}^\top \right)\boldsymbol{V}  \\
				   \boldsymbol{Y}^{t,(l)}\boldsymbol{R}^{t,(l)} - \eta\left(\boldsymbol{X}^{t,(l)}(\boldsymbol{Y}^{t,(l)})^\top - \boldsymbol{U}\boldsymbol{V}^\top \right)^\top \boldsymbol{U}  
				   \end{array}\right]\\
				   &  -\left[ \begin{array}{c}
					 \frac{\eta}{2}\boldsymbol{U}(\boldsymbol{R}^{t,(l)})^\top\left( \left( \boldsymbol{X}^{t,(l)} \right)^\top\boldsymbol{X}^{t,(l)} - \left( \boldsymbol{Y}^{t,(l)} \right)^\top\boldsymbol{Y}^{t,(l)} \right)\boldsymbol{R}^{t,(l)}  \\
				 \frac{\eta}{2}\boldsymbol{V}(\boldsymbol{R}^{t,(l)})^\top\left(\left( \boldsymbol{Y}^{t,(l)} \right)^\top\boldsymbol{Y}^{t,(l)} -  \left( \boldsymbol{X}^{t,(l)} \right)^\top\boldsymbol{X}^{t,(l)}   \right)\boldsymbol{R}^{t,(l)}  
				   \end{array} \right]\\
				    & -\left[ \begin{array}{c}
					\boldsymbol{U}\\
				 \boldsymbol{V}
				   \end{array}\right]\\
				=& \left[ \begin{array}{c} \boldsymbol{\Delta}_{\boldsymbol{X}}^{t,(l)} - \eta \boldsymbol{\Delta}_{\boldsymbol{X}}^{t,(l)} \boldsymbol{V}^\top \boldsymbol{V} - \eta\boldsymbol{U}(\boldsymbol{\Delta}_{\boldsymbol{Y}}^{t,(l)})^\top\boldsymbol{V}   \\
					\boldsymbol{\Delta}_{\boldsymbol{Y}}^{t,(l)} - \eta \boldsymbol{V}(\boldsymbol{\Delta}_{\boldsymbol{X}}^{t,(l)})^\top \boldsymbol{U} - \eta \boldsymbol{\Delta}_{\boldsymbol{Y}}^{t,(l)} \boldsymbol{U}^\top\boldsymbol{U}  
					 \end{array} \right]\\
					 &+ \left[ \begin{array}{c}  -\frac{\eta}{2}\boldsymbol{U}(\boldsymbol{\Delta}_{\boldsymbol{X}}^{t,(l)})^\top \boldsymbol{U} - \frac{\eta}{2}\boldsymbol{U}\boldsymbol{U}^\top\boldsymbol{\Delta}_{\boldsymbol{X}}^{t,(l)} +\frac{\eta}{2} \boldsymbol{U} (\boldsymbol{\Delta}_{\boldsymbol{Y}}^{t,(l)})^\top \boldsymbol{V}   \\
					  - \frac{\eta}{2} \boldsymbol{V}(\boldsymbol{\Delta}_{\boldsymbol{Y}}^{t,(l)})^\top\boldsymbol{V} - \frac{\eta}{2}\boldsymbol{V}\boldsymbol{V}^\top\boldsymbol{\Delta}_{\boldsymbol{Y}}^{t,(l)} + \frac{\eta}{2} \boldsymbol{V}(\boldsymbol{\Delta}_{\boldsymbol{X}}^{t,(l)})^\top\boldsymbol{U}  
					 \end{array} \right]\\
					 &+ \left[ \begin{array}{c}    \frac{\eta}{2}\boldsymbol{U}\boldsymbol{V}^\top\boldsymbol{\Delta}_{\boldsymbol{Y}}^{t,(l)} +\eta \boldsymbol{\mathcal{E}}_1 \\
					      \frac{\eta}{2}\boldsymbol{V}\boldsymbol{U}^\top\boldsymbol{\Delta}_{\boldsymbol{X}}^{t,(l)} +\eta \boldsymbol{\mathcal{E}}_2
					 \end{array} \right],
				\end{split}	 
			\]
			where $\boldsymbol{\mathcal{E}}_1,\boldsymbol{\mathcal{E}}_2$ denote those terms with at least two $\boldsymbol{\Delta}_{\boldsymbol{X}}^{t,(l)}$'s and $\boldsymbol{\Delta}_{\boldsymbol{Y}}^{t,(l)}$'s (the expression of $\boldsymbol{\mathcal{E}}_1$ and $\boldsymbol{\mathcal{E}}_2$ one can refer to \eqref{eq_e001} and \eqref{eq_e002}, replacing $\boldsymbol{\Delta}_{\boldsymbol{X}}^{t }$ and $\boldsymbol{\Delta}_{\boldsymbol{Y}}^{t }$ by $\boldsymbol{\Delta}_{\boldsymbol{X}}^{t,(l)}$ and $\boldsymbol{\Delta}_{\boldsymbol{Y}}^{t,(l)}$). Again by the way we define $\boldsymbol{R}^{t,(l)}$,
			\[
				\left[ \begin{array}{c}
					\boldsymbol{\Delta}_{\boldsymbol{X}}^{t,(l)} \\ \boldsymbol{\Delta}_{\boldsymbol{Y}}^{t,(l)}
				\end{array} \right]^\top\left[ \begin{array}{c}
					\boldsymbol{U}\\
					\boldsymbol{V}
				\end{array} \right]	= (\boldsymbol{\Delta}_{\boldsymbol{X}}^{t,(l)})^\top\boldsymbol{U}+(\boldsymbol{\Delta}_{\boldsymbol{Y}}^{t,(l)})^\top\boldsymbol{V}
			\]
			is symmetric. Plugging back we have
			\[
				\begin{split}
				&  \left[ \begin{array}{c} \widetilde{\mathbb{E}}\widetilde{\boldsymbol{X}}^{t+1,(l)}\\  \widetilde{\mathbb{E}}\widetilde{\boldsymbol{Y}}^{t+1,(l)} \end{array} \right] - \left[ \begin{array}{c}
					\boldsymbol{U}\\\boldsymbol{V}
				\end{array}\right]\\
				 = & \left[ \begin{array}{c} \boldsymbol{\Delta}_{\boldsymbol{X}}^{t,(l)} - \eta \boldsymbol{\Delta}_{\boldsymbol{X}}^{t,(l)} \boldsymbol{V}^\top \boldsymbol{V}  -\eta \boldsymbol{U}\boldsymbol{U}^\top\boldsymbol{\Delta}_{\boldsymbol{X}}^{t,(l)}  +\eta \boldsymbol{\mathcal{E}}_1 \\
					\boldsymbol{\Delta}_{\boldsymbol{Y}}^{t,(l)}  - \eta \boldsymbol{\Delta}_{\boldsymbol{Y}}^{t,(l)} \boldsymbol{U}^\top\boldsymbol{U}  - \eta \boldsymbol{V}\boldsymbol{V}^\top\boldsymbol{\Delta}_{\boldsymbol{Y}}^{t,(l)} +\eta \boldsymbol{\mathcal{E}}_2
					 \end{array} \right]\\
				=& \frac{1}{2}\left[\begin{array}{c}
					\boldsymbol{\Delta}_{\boldsymbol{X}}^{t,(l)} \\
					\boldsymbol{\Delta}_{\boldsymbol{Y}}^{t,(l)} 
				\end{array}\right](\boldsymbol{I}-2\eta\boldsymbol{U}^\top\boldsymbol{U})  + \frac{1}{2}\left( \boldsymbol{I} - 2\eta \left[ \begin{array}{cc}
					\boldsymbol{U}\boldsymbol{U}^\top & \boldsymbol{0}\\
					\boldsymbol{0} & \boldsymbol{V}\boldsymbol{V}^\top
				\end{array} \right]\right)\left[\begin{array}{c}
					\boldsymbol{\Delta}_{\boldsymbol{X}}^{t,(l)} \\
					\boldsymbol{\Delta}_{\boldsymbol{Y}}^{t,(l)} 
				\end{array}\right] +\eta\boldsymbol{\mathcal{E}},
				\end{split}	
			\]
			where the last line we use the fact that $\boldsymbol{U}^\top\boldsymbol{U} = \boldsymbol{V}^\top\boldsymbol{V}$, and $\boldsymbol{\mathcal{E}}\coloneqq \left[ \begin{array}{c}
				\boldsymbol{\mathcal{E}}_1\\
				\boldsymbol{\mathcal{E}}_2
			\end{array} \right]$. Since $\boldsymbol{U}\boldsymbol{U}^\top$ and $\boldsymbol{V}\boldsymbol{V}^\top$ sharing the same eigenvalues, we have
			\[
				\begin{split}
				&\left\|  \left[ \begin{array}{c}\widetilde{\mathbb{E}}\widetilde{\boldsymbol{X}}^{t+1,(l)}\\  \widetilde{\mathbb{E}}\widetilde{\boldsymbol{Y}}^{t+1,(l)} \end{array} \right] - \left[ \begin{array}{c}
					\boldsymbol{U}\\\boldsymbol{V}
				\end{array}\right] \right\|\\
				\leqslant & \frac{1}{2}\|\boldsymbol{I} - 2\eta \boldsymbol{U}^\top\boldsymbol{U}\| \|\boldsymbol{\Delta}^{t,(l)}\|  + \frac{1}{2}\|\boldsymbol{\Delta}^{t,(l)}\| \left\| \boldsymbol{I} - 2\eta \left[ \begin{array}{cc}
					\boldsymbol{U}\boldsymbol{U}^\top & \boldsymbol{0}\\
					\boldsymbol{0} & \boldsymbol{V}\boldsymbol{V}^\top
				\end{array} \right] \right\| + \eta \|\boldsymbol{\mathcal{E}}\|\\
				\leqslant & (1-\eta \sigma_r(\boldsymbol{M})) \|\boldsymbol{\Delta}^{t,(l)}\| + \eta \|\boldsymbol{\mathcal{E}}\|.
				\end{split} 
			\]
			By the definition of $\boldsymbol{\mathcal{E}}$, we have
			\[
				\|\boldsymbol{\mathcal{E}}\| \leqslant 4\|\boldsymbol{\Delta}^{t,(l)}\|^2\|\boldsymbol{U}\|.	
			\]
			From \eqref{eq_019}, 
			\begin{equation}\label{eq_059}
				\begin{split}
				& \left\|  \left[ \begin{array}{c} \widetilde{\mathbb{E}}\widetilde{\boldsymbol{X}}^{t+1,(l)}\\  \widetilde{\mathbb{E}}\widetilde{\boldsymbol{Y}}^{t+1,(l)} \end{array} \right] - \left[ \begin{array}{c}
					\boldsymbol{U}\\\boldsymbol{V}
					\end{array}\right] \right\|\\
					\leqslant & (1-\eta\sigma_r(\boldsymbol{M}))\times 2C_I\rho^t \sqrt{\frac{\mu r \kappa^6\log(n_1\vee n_2)}{(n_1\wedge n_2)p}}\sqrt{\sigma_1(\boldsymbol{M})}  + 4\eta\left( 2C_I\rho^t \sqrt{\frac{\mu r \kappa^6\log(n_1\vee n_2)}{(n_1\wedge n_2)p}}\sqrt{\sigma_1(\boldsymbol{M})} \right)^2 \sqrt{\sigma_1(\boldsymbol{M})}
				\end{split}
			\end{equation}
			holds. Combining \eqref{eq_056}, \eqref{eq_057}, \eqref{eq_058} and \eqref{eq_059} together, we have
			\[
				\begin{split}
				&\left\| \left[ \begin{array}{c}\widetilde{\boldsymbol{X}}^{t+1,(l)} -\boldsymbol{U}\\  \widetilde{\boldsymbol{Y}}^{t+1,(l)}-\boldsymbol{V}\end{array} \right] \right\| \\
				\leqslant & 2\eta \sqrt{\sigma_1(\boldsymbol{M})} \sqrt{\frac{n_1\wedge n_2}{p}} \times 111C_I\rho^t\sqrt{\frac{\mu^2 r^2\kappa^{12}\log(n_1\vee n_2)}{(n_1\wedge n_2)^2p}}\sqrt{\sigma_1(\boldsymbol{M})} \\
				&~~\times \left(2\sqrt{\frac{\mu r \kappa}{n_1\wedge n_2}}\sqrt{\sigma_1(\boldsymbol{M})} + 111C_I\rho^t\sqrt{\frac{\mu^2 r^2\kappa^{12}\log(n_1\vee n_2)}{(n_1\wedge n_2)^2p}}\sqrt{\sigma_1(\boldsymbol{M})}  \right)\\
					&+(1-\eta\sigma_r(\boldsymbol{M}))  2C_I\rho^t \sqrt{\frac{\mu r \kappa^6\log(n_1\vee n_2)}{(n_1\wedge n_2)p}}\sqrt{\sigma_1(\boldsymbol{M})} + 4\eta \left( 2C_I\rho^t \sqrt{\frac{\mu r \kappa^6\log(n_1\vee n_2)}{(n_1\wedge n_2)p}}\sqrt{\sigma_1(\boldsymbol{M})} \right)^2  \sqrt{\sigma_1(\boldsymbol{M})}\\
					\leqslant & \eta \sigma_r(\boldsymbol{M})C_I\rho^t \sqrt{\frac{\mu r \kappa^6 \log (n_1\vee n_2)}{(n_1\wedge n_2)p}}\sqrt{\sigma_1(\boldsymbol{M})} +(1-\eta\sigma_r(\boldsymbol{M}))  2C_I\rho^t \sqrt{\frac{\mu r \kappa^6\log(n_1\vee n_2)}{(n_1\wedge n_2)p}}\sqrt{\sigma_1(\boldsymbol{M})}\\
					& +\eta \sigma_r(\boldsymbol{M})C_I\rho^t \sqrt{\frac{\mu r \kappa^6 \log (n_1\vee n_2)}{(n_1\wedge n_2)p}}\sqrt{\sigma_1(\boldsymbol{M})}\\
					= & 2C_I\rho^t \sqrt{\frac{\mu r \kappa^6\log(n_1\vee n_2)}{(n_1\wedge n_2)p}}\sqrt{\sigma_1(\boldsymbol{M})}\\
					\leqslant & \frac{1}{4\kappa}\sqrt{\sigma_1(\boldsymbol{M})},
				\end{split}
			\]
			where the second inequality holds since
			\[
				p \geqslant (666^2+111^2C_I^2)\frac{\mu^2 r^2\kappa^{11}\log(n_1\vee n_2)}{n_1\wedge n_2}	
			\]
			and the last line holds since
			\[
				p \geqslant 64C_I^2 \frac{\mu r \kappa^8\log(n_1\vee n_2)}{n_1\wedge n_2}.	
			\]
			Therefore, 
			\begin{equation}\label{eq_060}
			\begin{split}
		  \left\| \left[ \begin{array}{c} \boldsymbol{U}\\\boldsymbol{V} \end{array}\right]^\top\left[ \begin{array}{c} \boldsymbol{U}\\\boldsymbol{V} \end{array}\right]-\left[ \begin{array}{c} \boldsymbol{U}\\\boldsymbol{V} \end{array}\right]^\top \left[ \begin{array}{c} \widetilde{\boldsymbol{X}}^{t+1,(l)}\\ \widetilde{\boldsymbol{Y}}^{t+1,(l)}\end{array}\right] \right\| \leqslant &\left\| \left[ \begin{array}{c}
				\boldsymbol{U}\\
				\boldsymbol{V}
			\end{array} \right] \right\|\left\| \left[ \begin{array}{c}\widetilde{\boldsymbol{X}}^{t+1,(l)} -\boldsymbol{U}\\  \widetilde{\boldsymbol{Y}}^{t+1,(l)}-\boldsymbol{V}\end{array} \right] \right\| \leqslant 0.5\sigma_r(\boldsymbol{M}).
			\end{split}
			\end{equation}
			By Weyl's inequality, we see eigenvalues of $\boldsymbol{P}$ are all nonnegative. Combining with the fact that $\boldsymbol{P}$ is symmetric, we can see $\boldsymbol{P}$ is positive definite. And also from Weyl's inequality, $\sigma_r(\boldsymbol{P})\geqslant 1.5\sigma_r(\boldsymbol{M})$. 

			Moreover, by the definition of $\boldsymbol{X}^{t,(l)}$ and $\boldsymbol{Y}^{t,(l)}$, as well as the assumption that $1\leqslant l\leqslant n_1$,
			\[
				\begin{split}
					&\boldsymbol{X}^{t+1,(l)}\boldsymbol{R}^{t,(l)} \\
					 = & \boldsymbol{X}^{t,(l)}\boldsymbol{R}^{t,(l)} - \frac{\eta}{p}\mathcal{P}_{\Omega_{-l,\cdot}}\left( \boldsymbol{X}^{t,(l)} (\boldsymbol{Y}^{t,(l)})^\top- \boldsymbol{U}\boldsymbol{V}^\top \right) \boldsymbol{Y}^{t,(l)}\boldsymbol{R}^{t,(l)}  - \eta \mathcal{P}_{l,\cdot}\left( \boldsymbol{X}^{t,(l)} (\boldsymbol{Y}^{t,(l)})^\top- \boldsymbol{U}\boldsymbol{V}^\top \right)\boldsymbol{Y}^{t,(l)}\boldsymbol{R}^{t,(l)} \\
					& -\frac{\eta}{2}\boldsymbol{X}^{t,(l)}\boldsymbol{R}^{t,(l)}(\boldsymbol{R}^{t,(l)})^\top  \left( (\boldsymbol{X}^{t,(l)})^\top\boldsymbol{X}^{t,(l)} - (\boldsymbol{Y}^{t,(l)})^\top\boldsymbol{Y}^{t,(l)} \right)\boldsymbol{R}^{t,(l)},\\
			\end{split}
			\]
			\[
			\begin{split}
					& \boldsymbol{Y}^{t+1,(l)}\boldsymbol{R}^{t,(l)} \\
					=&  \boldsymbol{Y}^{t,(l)}\boldsymbol{R}^{t,(l)}  -\frac{\eta}{p} \left[ \mathcal{P}_{\Omega_{-l,\cdot}}\left( \boldsymbol{X}^{t,(l)} (\boldsymbol{Y}^{t,(l)})^\top- \boldsymbol{U}\boldsymbol{V}^\top \right) \right]^\top \boldsymbol{X}^{t,(l)}\boldsymbol{R}^{t,(l)} - \eta \left[ \mathcal{P}_{l,\cdot}\left( \boldsymbol{X}^{t,(l)} (\boldsymbol{Y}^{t,(l)})^\top- \boldsymbol{U}\boldsymbol{V}^\top \right) \right]^\top \boldsymbol{X}^{t,(l)}\boldsymbol{R}^{t,(l)}\\
		& -\frac{\eta}{2}\boldsymbol{Y}^{t,(l)}\boldsymbol{R}^{t,(l)}(\boldsymbol{R}^{t,(l)})^\top   \left( (\boldsymbol{Y}^{t,(l)})^\top\boldsymbol{Y}^{t,(l)} - (\boldsymbol{X}^{t,(l)})^\top\boldsymbol{X}^{t,(l)} \right)\boldsymbol{R}^{t,(l)}.
				\end{split}	
			\]
		Therefore,
			\begin{equation}\label{eq_021}
				\begin{split}
					 &\left[ \begin{array}{c} \boldsymbol{X}^{t+1,(l)}\boldsymbol{R}^{t,(l)}  - \widetilde{\boldsymbol{X}}^{t+1,(l)} \\ \boldsymbol{Y}^{t+1,(l)}\boldsymbol{R}^{t,(l)}  - \widetilde{\boldsymbol{Y}}^{t+1,(l)}\end{array} \right] \\
					 =& \left[ \begin{array}{cc} \boldsymbol{0} & \eta \boldsymbol{A} \\ \eta \boldsymbol{A}^\top & \boldsymbol{0}  \end{array} \right]\left[ \begin{array}{c} \boldsymbol{\Delta}_{\boldsymbol{X}}^{t,(l)} \\ \boldsymbol{\Delta}_{\boldsymbol{Y}}^{t,(l)}\end{array}\right]  + \left[ \begin{array}{c} -\frac{\eta}{2}\boldsymbol{\Delta}_{\boldsymbol{X}}^{t,(l)}(\boldsymbol{R}^{t,(l)})^\top\left( \left( \boldsymbol{X}^{t,(l)} \right)^\top\boldsymbol{X}^{t,(l)} - \left( \boldsymbol{Y}^{t,(l)}\right)^\top\boldsymbol{Y}^{t,(l)} \right)\boldsymbol{R}^{t,(l)}  \\ 
					 -\frac{\eta}{2}\boldsymbol{\Delta}_{\boldsymbol{Y}}^{t,(l)}(\boldsymbol{R}^{t,(l)})^\top  \left( \left( \boldsymbol{Y}^{t,(l)} \right)^\top\boldsymbol{Y}^{t,(l)}- \left(\boldsymbol{X}^{t,(l)}\right)^\top\boldsymbol{X}^{t,(l)} \right)\boldsymbol{R}^{t,(l)}  \end{array}\right]  
				\end{split}
			\end{equation}
			with 
			\[
			\begin{split}
			\boldsymbol{A} \coloneqq& -\frac{1}{p}\mathcal{P}_{\Omega_{-l,\cdot}}\left( \boldsymbol{X}^{t,(l)} \left( \boldsymbol{Y}^{t,(l)} \right)^\top - \boldsymbol{U}\boldsymbol{V}^\top \right)  -  \mathcal{P}_{l,\cdot}\left( \boldsymbol{X}^{t,(l)} \left( \boldsymbol{Y}^{t,(l)} \right)^\top - \boldsymbol{U}\boldsymbol{V}^\top \right).
			\end{split}
			\] 
	First in order to give a bound of $\|\boldsymbol{A}\|$,	 we can first decompose $\boldsymbol{A}$ as
			\begin{equation}\label{eq_095}
				\begin{split}
					\boldsymbol{A} =& \underbrace{-\frac{1}{p}\mathcal{P}_{\Omega}\left( \boldsymbol{X}^{t,(l)} \left( \boldsymbol{Y}^{t,(l)} \right)^\top - \boldsymbol{U}\boldsymbol{V}^\top \right)}_{\boldsymbol{A}_1}\\
					& + \underbrace{\frac{1}{p}\mathcal{P}_{\Omega_{l,\cdot}}\left( \boldsymbol{X}^{t,(l)} \left( \boldsymbol{Y}^{t,(l)} \right)^\top - \boldsymbol{U}\boldsymbol{V}^\top \right) -  \mathcal{P}_{l,\cdot}\left( \boldsymbol{X}^{t,(l)} \left( \boldsymbol{Y}^{t,(l)} \right)^\top - \boldsymbol{U}\boldsymbol{V}^\top \right)}_{\boldsymbol{A}_2} .
				\end{split}
			\end{equation}
			
			From Lemma \ref{lemma_spectral_gap} and Lemma \ref{vu_lemma2.2}, 
			\[
				\begin{split}
					\|\boldsymbol{A}_1\| \leqslant & \left\| \frac{1}{p}\mathcal{P}_{\Omega}\left( \boldsymbol{X}^{t,(l)} \left( \boldsymbol{Y}^{t,(l)} \right)^\top - \boldsymbol{U}\boldsymbol{V}^\top \right)  - \left( \boldsymbol{X}^{t,(l)} \left( \boldsymbol{Y}^{t,(l)} \right)^\top - \boldsymbol{U}\boldsymbol{V}^\top\right) \right\| + \left\|  \boldsymbol{X}^{t,(l)} \left( \boldsymbol{Y}^{t,(l)} \right)^\top - \boldsymbol{U}\boldsymbol{V}^\top  \right\|   \\
					\leqslant& C_3 \sqrt{\frac{n_1\wedge n_2}{p}} \left(\|\boldsymbol{X}^{t,(l)} \boldsymbol{T}^{t,(l)} - \boldsymbol{U}\|_{2,\infty}\|\boldsymbol{V}\|_{2,\infty}  + \|\boldsymbol{U}\|_{2,\infty}\|\boldsymbol{Y}^{t,(l)}\boldsymbol{T}^{t,(l)}-\boldsymbol{V}\|_{2,\infty} \right)\\
					& + C_3 \sqrt{\frac{n_1\wedge n_2}{p}} \|\boldsymbol{X}^{t,(l)} \boldsymbol{T}^{t,(l)} - \boldsymbol{U}\|_{2,\infty}\|\boldsymbol{Y}^{t,(l)}\boldsymbol{T}^{t,(l)}-\boldsymbol{V}\|_{2,\infty} \\
					&+ \|\boldsymbol{X}^{t,(l)}\boldsymbol{R}^{t,(l)} - \boldsymbol{U}\|\|\boldsymbol{V}\| + \|\boldsymbol{U}\|\|\boldsymbol{Y}^{t,(l)}\boldsymbol{R}^{t,(l)} - \boldsymbol{V}\| + \|\boldsymbol{X}^{t,(l)}\boldsymbol{R}^{t,(l)} - \boldsymbol{U}\|\|\boldsymbol{Y}^{t,(l)}\boldsymbol{R}^{t,(l)} - \boldsymbol{V}\|
				\end{split}
			\]
			holds on the event $E_{gd}^t\subset E_S$. 

			From \eqref{eq_015} and \eqref{eq_019},
			\begin{equation}\label{eq_093}
				 \begin{split}
					& \|\boldsymbol{A}_1\|\\
					\leqslant & C_3 \sqrt{\frac{n_1\wedge n_2}{p}}   \left[ 222C_I\rho^t\sqrt{\frac{\mu^2 r^2\kappa^{12}\log (n_1\vee n_2)}{(n_1\wedge n_2)^2p}}\sqrt{\sigma_1(\boldsymbol{M})}  \sqrt{\frac{\mu r \kappa}{n_1\wedge n_2}}\sqrt{\sigma_1(\boldsymbol{M})} \right.\\
					&~~~~~~~~~~~~~~~~~~~~ \left.+ \left( 111C_I\rho^t\sqrt{\frac{\mu^2 r^2\kappa^{12}\log (n_1\vee n_2)}{(n_1\wedge n_2)^2p}}\sqrt{\sigma_1(\boldsymbol{M})}  \right)^2 \right] \\
					&+ 4C_I \rho^t\sqrt{\frac{\mu r \kappa^6\log (n_1\vee n_2)}{(n_1\wedge n_2)p}}\sigma_1(\boldsymbol{M}) +\left( 2C_I \rho^t \sqrt{\frac{\mu r \kappa^6 \log(n_1\vee n_2)}{(n_1\wedge n_2)p}}\sqrt{\sigma_1(\boldsymbol{M})} \right)^2\\
					\leqslant & C_3 \sqrt{\frac{n_1\wedge n_2}{p}}  \times 333C_I \rho^t \sqrt{\frac{\mu^2 r^2 \kappa^{12}\log(n_1\vee n_2)}{(n_1\wedge n_2)^2p} }\sqrt{\sigma_1(\boldsymbol{M})}  \times \sqrt{\frac{\mu r \kappa}{n_1\wedge n_2}}\sqrt{\sigma_1(\boldsymbol{M})} \\
					&+ 5C_I \rho^t\sqrt{\frac{\mu r \kappa^6\log (n_1\vee n_2)}{(n_1\wedge n_2)p}}\sigma_1(\boldsymbol{M}) \\
					\leqslant& 6C_I \rho^t\sqrt{\frac{\mu r \kappa^6\log (n_1\vee n_2)}{(n_1\wedge n_2)p}}\sigma_1(\boldsymbol{M})
				\end{split}
			\end{equation}
			where the second inequality holds since 
			\[
				p \geqslant 111^2 C_I^2 \frac{\mu r \kappa^{11}\log (n_1\vee n_2)}{n_1\wedge n_2}
			\]
			and the last inequality holds since
			\[
				p \geqslant 333^2 C_3^2\frac{\mu^2 r^2 \kappa^7 }{n_1\wedge n_2}.
			\]

			Note by the definition of $\boldsymbol{A}_2$, we have $\|\boldsymbol{A}_2\| = \|(\boldsymbol{A}_2)_{l,\cdot}\|_2$, and note $(\boldsymbol{A}_2)_{l,\cdot}$ here is exactly $\boldsymbol{b}_2$ we define in \eqref{eq_020}, therefore we directly use the result \eqref{eq_055} and \eqref{eq_064}:
			\begin{equation}\label{eq_094}
			\begin{split}
				&\|\boldsymbol{A}_2\|\\
				 =& \|\boldsymbol{b}_2\|_2\\
				 \leqslant& 100\rho^t \left( 115C_I\sqrt{\frac{\mu^2 r^2 \kappa^{12} \log^2(n_1\vee n_2)}{(n_1\wedge n_2)^2 p^2}}  + 333C_I \sqrt{\frac{\mu^3 r^3 \kappa^{13}\log^3 (n_1\vee n_2)}{(n_1\wedge n_2)^3 p^3}}  \right) \sigma_1(\boldsymbol{M}) \\
				\leqslant & C_I \rho^t\sqrt{\frac{\mu r \kappa^6 \log (n_1\vee n_2)}{(n_1\wedge n_2)p}}\sigma_1(\boldsymbol{M})
			\end{split}
			\end{equation}
			holds on the event $E_{gd}^{t+1}$, where the last inequality holds since 
			\[
				p \geqslant 5.29\times 10^8 \frac{\mu r \kappa^6 \log (n_1\vee n_2)}{n_1\wedge n_2}.
			\]
			 
			 By putting \eqref{eq_093}, \eqref{eq_094} and \eqref{eq_095} together we have
			 \begin{equation}\label{eq_022}
				 \|\boldsymbol{A}\|\leqslant \|\boldsymbol{A}_1\|+\|\boldsymbol{A}_2\| \leqslant 7C_I \rho^t\sqrt{\frac{\mu r \kappa^6 \log (n_1\vee n_2)}{(n_1\wedge n_2)p}}\sigma_1(\boldsymbol{M})
			 \end{equation}
			 holds on the event $E_{gd}^{t+1}$. Moreover,
			 \begin{equation}\label{eq_023}
				 \begin{split}
					& \left\| \left[  \begin{array}{c}  -\frac{\eta}{2}\boldsymbol{\Delta}_{\boldsymbol{X}}^{t,(l)}(\boldsymbol{R}^{t,(l)})^\top \left( \left( \boldsymbol{X}^{t,(l)} \right)^\top\boldsymbol{X}^{t,(l)} - \left( \boldsymbol{Y}^{t,(l)}\right)^\top\boldsymbol{Y}^{t,(l)} \right)\boldsymbol{R}^{t,(l)} \\ 
					 -\frac{\eta}{2}\boldsymbol{\Delta}_{\boldsymbol{Y}}^{t,(l)}(\boldsymbol{R}^{t,(l)})^\top  \left( \left( \boldsymbol{Y}^{t,(l)} \right)^\top\boldsymbol{Y}^{t,(l)}- \left(\boldsymbol{X}^{t,(l)}\right)^\top\boldsymbol{X}^{t,(l)} \right)\boldsymbol{R}^{t,(l)}  \end{array}   \right]   \right\|\\
					\leqslant& \frac{\eta}{2} \left\|\boldsymbol{\Delta}_{\boldsymbol{X}}^{t,(l)}(\boldsymbol{R}^{t,(l)})^\top \vphantom{ \left( \left( \boldsymbol{X}^{t,(l)} \right)^\top\boldsymbol{X}^{t,(l)} - \left( \boldsymbol{Y}^{t,(l)}\right)^\top\boldsymbol{Y}^{t,(l)} \right)\boldsymbol{R}^{t,(l)} } \left( \left( \boldsymbol{X}^{t,(l)} \right)^\top\boldsymbol{X}^{t,(l)} - \left( \boldsymbol{Y}^{t,(l)}\right)^\top\boldsymbol{Y}^{t,(l)} \right)\boldsymbol{R}^{t,(l)} \right\|  \\
					&+ \frac{\eta}{2}\left\|\boldsymbol{\Delta}_{\boldsymbol{Y}}^{t,(l)}(\boldsymbol{R}^{t,(l)})^\top  \left( \left( \boldsymbol{Y}^{t,(l)} \right)^\top\boldsymbol{Y}^{t,(l)}- \left(\boldsymbol{X}^{t,(l)}\right)^\top\boldsymbol{X}^{t,(l)} \right)\boldsymbol{R}^{t,(l)}\right\|  \\
					\leqslant & \frac{\eta}{2} \left(\|\boldsymbol{\Delta}_{\boldsymbol{X}}^{t,(l)}\| + \|\boldsymbol{\Delta}_{\boldsymbol{Y}}^{t,(l)}\|\right)   \left(2\|\boldsymbol{\Delta}_{\boldsymbol{X}}^{t,(l)}\|\|\boldsymbol{U}\| + \|\boldsymbol{\Delta}_{\boldsymbol{X}}^{t,(l)}\|^2  + 2\|\boldsymbol{\Delta}_{\boldsymbol{Y}}^{t,(l)}\|\|\boldsymbol{V}\| + \|\boldsymbol{\Delta}_{\boldsymbol{Y}}^{t,(l)}\|^2\right)\\
					\leqslant & \frac{\eta}{2}\times 4 C_I \rho^t \sqrt{\frac{\mu r \kappa^6\log (n_1\vee n_2)}{(n_1\wedge n_2)p}}\sqrt{\sigma_1(\boldsymbol{M})}  \times 12  C_I \rho^t \sqrt{\frac{\mu r \kappa^6\log (n_1\vee n_2)}{(n_1\wedge n_2)p}}\sigma_1(\boldsymbol{M})\\
					\leqslant & 24 C_I^2 \eta \rho^{2t}\sqrt{\frac{\mu^2 r^2 \kappa^{12} \log^2 (n_1\vee n_2)}{(n_1\wedge n_2)^2p^2}}\sqrt{\sigma_1(\boldsymbol{M})}^3, 
				\end{split}
			 \end{equation}
			 where the third inequality uses \eqref{eq_019} and 
			 \[
				 p\geqslant 4 C_I^2 \frac{\mu r \kappa^6 \log (n_1\vee n_2)}{n_1\wedge n_2}.
			 \] 
			 
			 Now from \eqref{eq_021}, \eqref{eq_022}, \eqref{eq_023} and also \eqref{eq_019} we can see that
			 \begin{equation}\label{eq_024}
					\begin{split}
					& \left\|\left[ \begin{array}{c} \boldsymbol{X}^{t+1,(l)}\boldsymbol{R}^{t,(l)}  - \widetilde{\boldsymbol{X}}^{t+1,(l)} \\ \boldsymbol{Y}^{t+1,(l)}\boldsymbol{R}^{t,(l)}  - \widetilde{\boldsymbol{Y}}^{t+1,(l)}\end{array} \right]\right\|\\
						  \leqslant & \eta \times 7C_I \rho^t\sqrt{\frac{\mu r \kappa^6 \log (n_1\vee n_2)}{(n_1\wedge n_2)p}}\sigma_1(\boldsymbol{M}) \times  2C_I \rho^t\sqrt{\frac{\mu r \kappa^6 \log (n_1\vee n_2)}{(n_1\wedge n_2)p}}\sqrt{\sigma_1(\boldsymbol{M})}\\
						& + 24 C_I^2 \eta \rho^{2t}\sqrt{\frac{\mu^2 r^2 \kappa^{12} \log^2 (n_1\vee n_2)}{(n_1\wedge n_2)^2p^2}}\sqrt{\sigma_1(\boldsymbol{M})}^3\\
						\leqslant & 38C_I^2 \eta\rho^t \sqrt{\frac{\mu^2 r^2 \kappa^{12}\log^2 (n_1\vee n_2)}{(n_1\wedge n_2)^2p^2}}\sqrt{\sigma_1(\boldsymbol{M})}^3.
					\end{split}
			 \end{equation}

			 Therefore, 
			 \[
			 \begin{split}
				  \left\|\left[\begin{array}{c}
					\boldsymbol{U}\\
					\boldsymbol{V}
				\end{array}\right]^\top\left[ \begin{array}{c} \boldsymbol{X}^{t+1,(l)}\boldsymbol{R}^{t,(l)}  - \widetilde{\boldsymbol{X}}^{t+1,(l)} \\ \boldsymbol{Y}^{t+1,(l)}\boldsymbol{R}^{t,(l)}  - \widetilde{\boldsymbol{Y}}^{t+1,(l)}\end{array} \right]\right\| \leqslant &\left\|\left[\begin{array}{c}
					\boldsymbol{U}\\
					\boldsymbol{V}
				\end{array}\right]\right\| \left\|\left[ \begin{array}{c} \boldsymbol{X}^{t+1,(l)}\boldsymbol{R}^{t,(l)}  - \widetilde{\boldsymbol{X}}^{t+1,(l)} \\ \boldsymbol{Y}^{t+1,(l)}\boldsymbol{R}^{t,(l)}  - \widetilde{\boldsymbol{Y}}^{t+1,(l)}\end{array} \right]\right\|\\
				\leqslant&  \sigma_r(\boldsymbol{M})\\
				\leqslant& \sigma_r(\boldsymbol{P}),
			\end{split}
			 \]
			 where the second last inequality uses the fact that
			 \[
					p \geqslant 76C_I^2\frac{\mu r \kappa^6 \log (n_1\vee n_2)}{n_1\wedge n_2}
			 \]
			 and
			 \[
					\eta \leqslant \frac{\sigma_r(\boldsymbol{M})}{200\sigma_1^2(\boldsymbol{M})}.
			 \]

			 Therefore, we have
			 \[
				\begin{split}
					  \|(\boldsymbol{R}^{t,(l)})^{-1}\boldsymbol{R}^{t+1,(l)} - \boldsymbol{I}\| =&\|\operatorname{sgn}(\boldsymbol{C}+\boldsymbol{E}) - \operatorname{sgn}(\boldsymbol{C})\|\\
					 \leqslant& \frac{1}{\sigma_r(\boldsymbol{P})} \left\|\left[\begin{array}{c}
						\boldsymbol{U}\\
						\boldsymbol{V}
					\end{array}\right]^\top\left[ \begin{array}{c} \boldsymbol{X}^{t+1,(l)}\boldsymbol{R}^{t,(l)}  - \widetilde{\boldsymbol{X}}^{t+1,(l)} \\ \boldsymbol{Y}^{t+1,(l)}\boldsymbol{R}^{t,(l)}  - \widetilde{\boldsymbol{Y}}^{t+1,(l)}\end{array} \right]\right\|\\
						\leqslant & 2\frac{\sqrt{\sigma_1(\boldsymbol{M})}}{\sigma_r(\boldsymbol{M})}\times 38C_I^2 \eta\rho^t \sqrt{\frac{\mu^2 r^2 \kappa^{12}\log^2 (n_1\vee n_2)}{(n_1\wedge n_2)^2p^2}}\sqrt{\sigma_1(\boldsymbol{M})}^3\\
						\leqslant & 76C_I^2 \frac{\sigma_1^2(\boldsymbol{M})}{\sigma_r(\boldsymbol{M})}\eta\rho^t \sqrt{\frac{\mu^2 r^2 \kappa^{12}\log^2 (n_1\vee n_2)}{(n_1\wedge n_2)^2p^2}},
				\end{split}
			 \]
			 where the second inequality uses the fact that $\sigma_r(\boldsymbol{P})\geqslant 1.5\sigma_r(\boldsymbol{M})$ and the third one uses \eqref{eq_024}.
	
			\end{proof}

\end{document}